\documentclass[11pt]{article}

\usepackage[colorlinks=true,linkcolor=blue,urlcolor=black,citecolor=blue,breaklinks]{hyperref}
\usepackage{color}
\usepackage{enumerate}
\usepackage{graphicx}
\usepackage[nice]{nicefrac}
\usepackage{algorithm, algcompatible}
\usepackage{varwidth}%
\usepackage{mathrsfs}
\usepackage{mathtools}
\usepackage{tikz}
\usepackage{mathdots}
\usepackage[labelfont=bf,font={small}]{caption}
\usepackage[labelfont=bf]{subcaption}
\usepackage{wrapfig}
\usepackage{overpic}
\usepackage{multirow}
\usepackage{tabularx}
\usepackage[inline]{enumitem}

\usepackage[numbers]{natbib}

\usepackage{fullpage}
\usepackage{authblk}

\usepackage{amsmath,amsthm,amssymb,colonequals,etoolbox}
\usepackage{thmtools}
\usepackage{thm-restate}
\usepackage{setspace}

\usepackage{cleveref}

\newcommand{\vocab}{\emph}

\newtheorem{thm}{Theorem}[section]
\newtheorem{myprop}[thm]{Proposition}
\newtheorem{mylemma}[thm]{Lemma}
\newtheorem{mycorollary}[thm]{Corollary}
\newtheorem{myfact}[thm]{Fact}

\newtheorem{mythm}[thm]{Theorem}
\newtheorem{myconjecture}[thm]{Conjecture}

\newtheorem{myremark}[thm]{Remark}
\newtheorem{myprob}[thm]{Problem}

\newtheorem{mydef}{Definition}[section]
\newtheorem{myassump}{Assumption}[section]

\DeclareMathOperator*{\argmin}{arg\!\min}

\DeclareMathOperator*{\rank}{rank}
\DeclareMathOperator*{\diag}{diag}
\DeclareMathOperator*{\Tr}{\mathrm{tr}}

\DeclareMathOperator*{\Span}{span}

\newcommand{\R}{\ensuremath{\mathbb{R}}}
\newcommand{\C}{\ensuremath{\mathbb{C}}}

\newcommand{\N}{\ensuremath{\mathbb{N}}}

\newcommand{\norm}[1]{\lVert #1 \rVert}
\newcommand{\bignorm}[1]{\left\lVert #1 \right\rVert}

\newcommand{\ip}[2]{\ensuremath{\langle #1, #2 \rangle}}

\newcommand{\E}{\mathbb{E}}
\newcommand{\abs}[1]{\ensuremath{| #1 |}}

\newcommand{\floor}[1]{\lfloor #1 \rfloor}
\newcommand{\bigfloor}[1]{\left\lfloor #1 \right\rfloor}
\newcommand{\ceil}[1]{\lceil #1 \rceil}

\newcommand{\ind}{\mathbf{1}}
\renewcommand{\vec}{\mathrm{vec}}

\renewcommand{\Pr}{\mathbb{P}}
\newcommand{\T}{\mathsf{T}}

\newcommand{\calA}{\mathcal{A}}
\newcommand{\calC}{\mathcal{C}}

\newcommand{\calE}{\mathcal{E}}

\newcommand{\calF}{\mathcal{F}}

\newcommand{\calI}{\mathcal{I}}
\newcommand{\calP}{\mathcal{P}}

\renewcommand{\vec}{\mathrm{vec}}
\newcommand{\opnorm}[1]{\norm{#1}_{\mathrm{op}}}

\newcommand{\bmattwo}[4]{\begin{bmatrix} #1 & #2 \\ #3 & #4 \end{bmatrix}}

\newcommand{\cvectwo}[2]{\begin{bmatrix} #1 \\ #2 \end{bmatrix}}
\newcommand{\rvectwo}[2]{\begin{bmatrix} #1 & #2 \end{bmatrix}}

\numberwithin{equation}{section}

\renewcommand{\ge}{\geqslant}
\renewcommand{\geq}{\geqslant}
\renewcommand{\le}{\leqslant}
\renewcommand{\leq}{\leqslant}

\newcommand{\ulam}{\underline{\lambda}}

\newcommand{\informal}[1]{informal; #1}

\newcommand{\marginal}[2]{\mu_{#2}[#1]}

\newcommand{\IID}{iid}
\newcommand{\condNum}{\gamma}
\newcommand{\Tnew}{{T'}}

\newcommand{\Px}{\mathsf{P}_x}
\newcommand{\Pxt}[1]{\mathsf{P}_{x,#1}}
\newcommand{\PxA}[1]{\mathsf{P}_x^{#1}}
\newcommand{\PxAt}[2]{\mathsf{P}_{x,#2}^{#1}}
\newcommand{\Pxi}{\mathsf{P}_{\xi}}
\newcommand{\Pxy}{\mathsf{P}_{x,y}^{W_\star}}

\newcommand{\calPx}{\calP_x}

\newcommand{\smallDash}{\textnormal{\textsf{-}}}
\newcommand{\problemName}[1]{$\mathsf{#1}$}
\newcommand{\PSR}{\problemName{Ind\smallDash{}Seq\smallDash{}LS}}
\newcommand{\PLR}{\problemName{Ind\smallDash{}LDS\smallDash{}LS}}
\newcommand{\DSR}{\problemName{Seq\smallDash{}LS}}
\newcommand{\DLR}{\problemName{LDS\smallDash{}LS}}
\newcommand{\LDSLS}{\DLR}
\newcommand{\TSB}{TrajSB}

\newcommand{\wTSB}{wTrajSB}

\newcommand{\Sysid}{\problemName{LDS\smallDash{}SysID}}

\newcommand{\sfSym}{\mathsf{Sym}}
\newcommand{\sfR}{\mathsf{R}}

\newcommand{\uGam}{\underline{\Gamma}}
\newcommand{\csb}{c_{\mathsf{sb}}}

\newcommand{\kcont}{k_{\mathsf{c}}}

\newcommand{\uMu}{\underline{\mu}}

\newcommand{\tmix}{\tau_{\mathsf{mix}}}
\newcommand{\tvnorm}[1]{\norm{#1}_{\mathrm{tv}}}

\newcommand{\Tri}[3]{\mathsf{Tri}(#1, #2; #3)}

\newcommand{\Tnd}{T_{\mathsf{nd}}}

\newcommand{\rmd}{\mathrm{d}}

\newcommand{\e}{\varepsilon}

\newcommand{\kmix}{k_{\mathrm{mix}}}

\author[1]{Stephen Tu}
\author[1]{Roy Frostig}
\author[2]{Mahdi Soltanolkotabi}
\affil[1]{Google Research, Brain team}
\affil[2]{University of Southern California}

\date{April 1, 2022. Revised:\ \today}

\title{Learning from many trajectories}

\begin{document}

\maketitle

\begin{abstract}

We initiate a study of %
supervised learning from many independent sequences (``trajectories'') of
non-independent covariates, reflecting tasks in sequence modeling,
control,
and reinforcement learning.
Conceptually, our multi-trajectory setup sits between two traditional settings
in statistical learning theory:
learning from independent examples %
and learning from a single auto-correlated sequence.
Our conditions for efficient learning generalize the former setting---trajectories
must be non-degenerate in ways that extend standard requirements
for independent examples.
Notably, we
do not require that trajectories be ergodic, long, nor strictly stable.

For linear least-squares regression,
given $n$-dimensional examples produced by
$m$ trajectories, each of length $T$,
we observe a notable change in statistical efficiency as the number of trajectories increases from
a few (namely $m \lesssim n$) to many (namely $m \gtrsim n$).
Specifically, we establish
that the worst-case error rate of this
problem is $\Theta(n / m T)$ whenever $m \gtrsim n$.
Meanwhile, when $m \lesssim n$, %
we establish a (sharp) lower bound of $\Omega(n^2 / m^2 T)$
on the worst-case error rate,
realized by a simple, marginally unstable linear dynamical system.
A key upshot is that, in domains where trajectories regularly reset,
the error rate eventually
behaves as if
\emph{all} of the examples %
were independent, drawn from their marginals.
As a corollary of our analysis,
we also improve guarantees for the linear system identification problem.

\end{abstract}

\newpage 

\setcounter{tocdepth}{2}
\begin{spacing}{0.945}
\tableofcontents
\end{spacing}
\newpage

\section{Introduction}
\label{sec:intro}

Statistical learning theory aims to characterize the worst-case efficiency
of learning from example data. Its most common setup assumes
that examples are independently and identically distributed (\vocab{\IID{}})
draws from an underlying data distribution, but
various branches of theory---not to mention deployed applications
of machine learning---consume non-independent data as well.
An especially fruitful setting, and the focus of this paper, is
in learning from sequential data,
where examples are generated by some ordered stochastic process
that renders them possibly correlated.
Naturally, sequential processes describe application domains
spanning engineering and the sciences, such as
robotics~\citep{nguyentuong2011modellearning},
data center cooling~(e.g.\ \cite{lazic2018cooling}),
language (e.g.\ \cite{sutskever2014seq2seq,belanger2015ldstext}),
neuroscience (e.g.\ \cite{linderman2017basketball,glaser2020recurrent}),
and economic forecasting~\citep{mcdonald2017timeseries}.
Learning over sequential data can also capture some
formulations of imitation learning~\citep{osa2018imitation} and reinforcement learning~\citep{chen2021decisiontransformer,janner2021rlseq}.

In supervised learning, one learns to predict output \vocab{labels} from input \vocab{covariates},
given example pairings of the two.
Formal treatments of learning from sequential data
typically concern a \emph{single} inter-dependent chain of covariates.
Where these treatments vary is in their assumptions
about the underlying process that generates the covariate chain.
For instance, some assume that the process is auto-regressive (e.g.\ 
\cite{lai1983autoregressive,goldenshluger2001autoregressive,gonzalez2020autoregressive})
or ergodic (e.g.\ \cite{yu1994mixing,duchi2012ergodicmd}).
Others assume that it is a linear dynamical system
(e.g.\ \cite{simchowitz18learning,faradonbeh2018unstable,sarkar2019sysid}).

In this paper, we examine what happens when we learn from \emph{many}
independent chains rather than from one, as one does anyway in many applications
(e.g.\ \cite{pomerleau1989alvinn,khansari2011lfd,brants2007language,jozefowicz2016exploring}).
\Cref{fig:dependence-schematic} depicts the data dependence structure of our setup in
comparison with its two natural counterparts.
Learning from a dataset of many short (constant length) chains
ought to be similar to independent learning, even if each chain is highly intra-dependent.
On the other hand, for any non-trivial chain length,
intuition suggests that the error can degrade relative to the total sample size in the worst case,
since a greater proportion of the data may contain correlations.
Lower bounds even show that, when one sees only a single chain,
this degradation is outright necessary in the worst case \citep{bresler2020leastsquaresmarkov}.
Do we see any such effect with many chains?

We study this question by sharply characterizing worst-case error
rates of a fundamental task---linear regression---imposed over
a general sequential data model.
Our findings reveal a remarkable phenomenon: %
after seeing sufficiently many chains ($m$) relative to the example dimension $n$,
no matter the chain length $T$,
\emph{the error rate matches that of learning from the same total number $mT$ of independent examples},
drawn from their respective marginal distributions.

In our data model, each chain, called a \vocab{trajectory},
comprises a sequence of covariates $\{x_t\}$ generated from a stochastic process.
Each covariate is accompanied by a noisy linear response $y_t$ as its label.
A training set $\{(x_t^{(i)}, y_t^{(i)})\}_{i=1,t=1}^{m,T}$ comprises $m$ independent chains,
each of length $T$.
From such a training set, an estimator produces a hypothesis that predicts the label of any covariate.
The resulting hypothesis is evaluated according to its mean-squared prediction error
over a fresh chain of length $\Tnew$, possibly unequal to $T$---a notion of
risk defined naturally over a trajectory.
All of our risk upper bounds are guarantees for the ordinary least-squares estimator in particular.

A concrete, recurring example in this paper takes the covariate-generating process
to be a linear dynamical system (LDS).
Specifically, fixing matrices $A \in \R^{n \times n}$,
$B \in \R^{n \times d}$, and $W_\star \in \R^{p \times n}$,
a single trajectory $\{(x_t, y_t)\}_{t \ge 1}$ is generated as follows.
Let $x_0 = 0$, and for $t \ge 1$:
\begin{align*}
  x_t &= A x_{t-1} + Bw_t, &&\text{\textcolor{gray}{(linear dynamics)}}\\
  y_t &= W_\star x_t + \xi_t, &&\text{\textcolor{gray}{(linear regression)}}
\end{align*}
where the $\{w_t\}_{t \ge 1}$ are \IID{} centered isotropic Gaussian draws
and $\{\xi_t\}_{t \ge 1}$ is a sub-Gaussian martingale difference sequence
(with respect to past covariates $\{x_k\}_{k=1}^t$ and noise variables $\{\xi_k\}_{k=1}^{t-1}$).
Incidentally, combining linear dynamical systems with linear regression
captures the basic problem of linear system identification (as in \cite{simchowitz18learning})
as a special case.

In other instantiations of learning from trajectories, the covariates $\{x_t\}$
may be generated by a different process;
what remains common is the superimposed regression task
set up by the ground truth $W_\star$ and the noise $\{\xi_t\}$.
The key condition that we will introduce, which renders a
covariate process amenable to regression,
is that it satisfies a \vocab{trajectory small-ball} criterion (\Cref{def:trajectory_small_ball}).
\Cref{sec:results:upper:small_ball_examples} shows that LDS-generated data conforms to
the trajectory small-ball condition in particular, as do many other distributions.

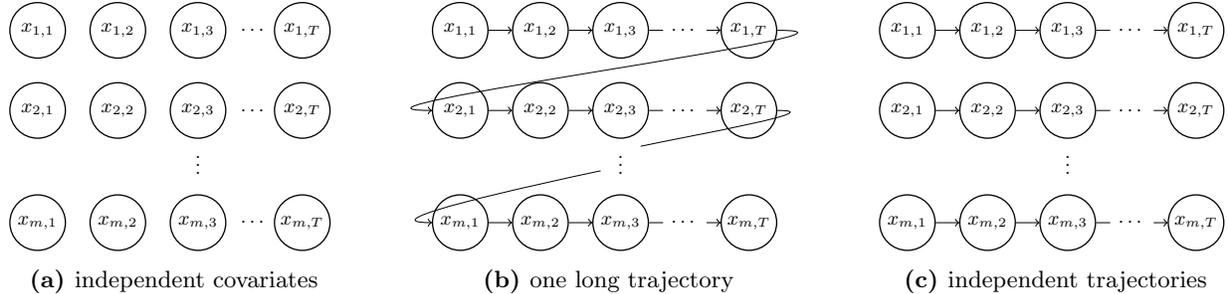
\begin{figure}
\centering
\centering
\begin{minipage}{.28\columnwidth}%
\centering
\scalebox{.82}{%
\begin{tikzpicture}[
  y=-1cm,
  scale=1.3,
  every node/.style={font=\footnotesize},
  xx/.style={circle, draw=black, inner sep=0pt, minimum size=9mm, semithick},
  dd/.style={}
]

\node[xx] (x11) at (0, 0) {$x_{1,1}$};
\node[xx] (x12) at (1, 0) {$x_{1,2}$};
\node[xx] (x13) at (2, 0) {$x_{1,3}$};
\node[dd] (cd1) at (2.7, 0) {$\cdots$};
\node[xx] (x1T) at (3.3, 0) {$x_{1,T}$};

\node[xx] (x21) at (0, 1) {$x_{2,1}$};
\node[xx] (x22) at (1, 1) {$x_{2,2}$};
\node[xx] (x23) at (2, 1) {$x_{2,3}$};
\node[dd] (cd2) at (2.7, 1) {$\cdots$};
\node[xx] (x2T) at (3.3, 1) {$x_{2,T}$};

\node[dd] at (2, 1.6) {$\vdots$};

\node[xx] (x31) at (0, 2.4) {$x_{m,1}$};
\node[xx] (x32) at (1, 2.4) {$x_{m,2}$};
\node[xx] (x33) at (2, 2.4) {$x_{m,3}$};
\node[dd] (cd3) at (2.7, 2.4) {$\cdots$};
\node[xx] (x3T) at (3.3, 2.4) {$x_{m,T}$};
\end{tikzpicture}
}%
\subcaption{independent covariates}
\label{fig:dependence-schematic:iid}
\end{minipage}
\hfill
\begin{minipage}{0.34\columnwidth}
\centering
\scalebox{.82}{%
\begin{tikzpicture}[
  y=-1cm,
  scale=1.3,
  every node/.style={font=\footnotesize},
  xx/.style={circle, draw=black, inner sep=0pt, minimum size=9mm, semithick},
  dd/.style={}
]
\coordinate (left) at (-.7, -20);
\coordinate (right) at (4.3, 20);
\begin{pgfinterruptboundingbox}
\clip (left) rectangle (right);
\end{pgfinterruptboundingbox}

\node[xx] (x11) at (0, 0) {$x_{1,1}$};
\node[xx] (x12) at (1, 0) {$x_{1,2}$};
\node[xx] (x13) at (2, 0) {$x_{1,3}$};
\node[dd] (cd1) at (2.8, 0) {$\cdots$};
\node[xx] (x1T) at (3.6, 0) {$x_{1,T}$};

\node[xx] (x21) at (0, 1) {$x_{2,1}$};
\node[xx] (x22) at (1, 1) {$x_{2,2}$};
\node[xx] (x23) at (2, 1) {$x_{2,3}$};
\node[dd] (cd2) at (2.8, 1) {$\cdots$};
\node[xx] (x2T) at (3.6, 1) {$x_{2,T}$};

\node[dd] (mid) at (2, 1.6) {$\vdots$};
\node[dd] (md1) at (2, 1.5) {$\phantom{\cdots}$};
\node[dd] (md2) at (2, 1.7) {$\phantom{\cdots}$};

\node[xx] (x31) at (0, 2.4) {$x_{m,1}$};
\node[xx] (x32) at (1, 2.4) {$x_{m,2}$};
\node[xx] (x33) at (2, 2.4) {$x_{m,3}$};
\node[dd] (cd3) at (2.8, 2.4) {$\cdots$};
\node[xx] (x3T) at (3.6, 2.4) {$x_{m,T}$};

\draw [->] (x11) to (x12);
\draw [->] (x12) to (x13);
\draw [-] (x13) to (cd1);
\draw [->] (cd1) to (x1T);

\draw [->] (x21) to (x22);
\draw [->] (x22) to (x23);
\draw [-] (x23) to (cd2);
\draw [->] (cd2) to (x2T);

\draw [->] (x31) to (x32);
\draw [->] (x32) to (x33);
\draw [-] (x33) to (cd3);
\draw [->] (cd3) to (x3T);
      
\draw [->] (x1T.east) to[out=0, in=180] (x21.west);
\draw [-] (x2T.east) to[out=0, in=12] (md1);
\draw [->] (md2) to[out=192, in=180] (x31.west);

\coordinate (top) at (current bounding box.north);
\coordinate (bottom) at (current bounding box.south);
\pgfresetboundingbox
\path[use as bounding box] (left|-bottom) rectangle (right|-top);
\end{tikzpicture}
}%
\subcaption{one long trajectory}
\label{fig:dependence-schematic:single}
\end{minipage}
\hfill
\begin{minipage}{.30\columnwidth}
\centering
\scalebox{0.82}{%
\begin{tikzpicture}[
  y=-1cm,
  scale=1.3,
  every node/.style={font=\footnotesize},
  xx/.style={circle, draw=black, inner sep=0pt, minimum size=9mm, semithick},
  dd/.style={}
]

\node[xx] (x11) at (0, 0) {$x_{1,1}$};
\node[xx] (x12) at (1, 0) {$x_{1,2}$};
\node[xx] (x13) at (2, 0) {$x_{1,3}$};
\node[dd] (cd1) at (2.8, 0) {$\cdots$};
\node[xx] (x1T) at (3.6, 0) {$x_{1,T}$};

\node[xx] (x21) at (0, 1) {$x_{2,1}$};
\node[xx] (x22) at (1, 1) {$x_{2,2}$};
\node[xx] (x23) at (2, 1) {$x_{2,3}$};
\node[dd] (cd2) at (2.8, 1) {$\cdots$};
\node[xx] (x2T) at (3.6, 1) {$x_{2,T}$};

\node[dd] at (2, 1.6) {$\vdots$};

\node[xx] (x31) at (0, 2.4) {$x_{m,1}$};
\node[xx] (x32) at (1, 2.4) {$x_{m,2}$};
\node[xx] (x33) at (2, 2.4) {$x_{m,3}$};
\node[dd] (cd3) at (2.8, 2.4) {$\cdots$};
\node[xx] (x3T) at (3.6, 2.4) {$x_{m,T}$};

\draw [->] (x11) to (x12);
\draw [->] (x12) to (x13);
\draw [-] (x13) to (cd1);
\draw [->] (cd1) to (x1T);

\draw [->] (x21) to (x22);
\draw [->] (x22) to (x23);
\draw [-] (x23) to (cd2);
\draw [->] (cd2) to (x2T);

\draw [->] (x31) to (x32);
\draw [->] (x32) to (x33);
\draw [-] (x33) to (cd3);
\draw [->] (cd3) to (x3T);
\end{tikzpicture}
}%
\subcaption{independent trajectories}
\label{fig:dependence-schematic:many}
\end{minipage}%
\caption{%
The covariate dependence structure
induced by three data models on $mT$ many training examples.
In (\subref{fig:dependence-schematic:iid}):
independent examples, typical of basic statistical learning.
In~(\subref{fig:dependence-schematic:single}):
the data models often considered in the sequential learning literature,
comprising a long auto-correlated chain of examples.
Learning in this setting can be infeasible in general,
so oftentimes ergodicity is assumed %
in order to rule out strong long-range dependencies,
essentially inducing an ``independent resetting'' effect across time.
The effective reset frequency then factors uniformly into error bounds,
in a way suggesting that one learns only one independent example's worth
within each effective reset window (cf.\ \Cref{sec:related}).
In~(\subref{fig:dependence-schematic:many}):
our multi-trajectory data model.
Our accompanying assumptions allow
for non-ergodic chains,
and for arbitrary chain lengths $T$,
while introducing \emph{explicit} independent resets.
Decoupling the $m$ resets from the sequential data model
lets us vary the training set dimensions $(m, T)$ freely,
without affecting other data assumptions,
as we study their effect on error rates.
We find that with enough trajectories $m$,
the worst-case error rate behaves the same as in the
independent setting depicted in~(\subref{fig:dependence-schematic:iid});
i.e., one learns as though every example were independently drawn from
its marginal distribution.
Some
recent work in system identification
assumes a data model related to ours
(specifically linear dynamical data) and
likewise avoids ergodicity;
our bounds improve these guarantees $T$-fold where applicable,
and upgrade the regimes in which they apply (cf.\ \Cref{sec:related}).
}
\label{fig:dependence-schematic}
\end{figure}

Our main results (\Cref{sec:results:upper,sec:results:lower_bounds})
sharply characterize worst-case rates
of learning from trajectory data as a function of
the training trajectory count $m$,
the training trajectory length $T$,
the evaluation length $\Tnew$,
the covariate and response dimensions $n$ and $p$,
and scale parameters of noise in the data model
(such as the variance of the noise $\{\xi_t\}$).
Restricting only to terms of covariate dimension $n$, 
training set size $m$ and $T$,
and evaluation length $\Tnew$, our bounds imply the following summary statement:
\begin{thm}[\informal{error rate with many small-ball trajectories, $\Tnew \le T$}]
\label{thm:simple-rate-many-traj-in-window}
If $m \gtrsim n$, $\Tnew \le T$, %
and covariate trajectories are drawn from a \vocab{trajectory small-ball} distribution,
then the worst-case excess prediction risk (over evaluation horizon $\Tnew$)
for linear regression from $m$ many
trajectories of $n$-dimensional covariates, each of length $T$, is $\Theta(n / (m T))$.
\end{thm}

In drawing comparisons to learning from independent examples,
it makes sense to consider training and evaluations lengths $T$ and $\Tnew$
equal (cf.~\Cref{sec:problem}),
rendering \Cref{thm:simple-rate-many-traj-in-window} applicable.
The theorem thus echoes our main point above:
the same rate of $\Theta(n/(mT))$ describes regression on $mT$ independent examples
(details on this point are expanded in \Cref{sec:prob:separations}).

Further structural assumptions are needed (cf.~\Cref{sec:prob:separations}) in order
to cover the remaining range of problem dimensions,
namely few trajectories ($m \lesssim n$) or extended evaluations ($\Tnew > T$),
and to that end we return to linear dynamical systems as a focus.
Our remaining risk upper bounds, targeting learning under linear dynamics,
require that
the dynamics matrix $A$ be \vocab{marginally unstable}
(meaning that its spectral radius $\rho(A)$ is at most one) and diagonalizable.
When trajectories are longer at test time than during training (i.e., $\Tnew > T$),
marginal instability is practically necessary, otherwise the risk
can scale exponentially in $\Tnew - T$.
The assumption otherwise still allows for unstable---and therefore non-ergodic---systems at $\rho(A) = 1$.
For simplicity, we also require that the control matrix $B$ have full row rank.
Our bounds then imply the following summary statement
about regression when the number of trajectories is limited:
\begin{thm}[\informal{error rate with few LDS trajectories}]
\label{thm:simple-rate-few-traj}
If $m \lesssim n$, $mT \gtrsim n$,
and covariate trajectories are drawn from a linear dynamical system whose
dynamics $A$ are marginally unstable and diagonalizable,
then the worst-case excess prediction risk (over evaluation horizon $\Tnew$)
for linear regression from $m$ many
trajectories of $n$-dimensional covariates, each of length $T$,
is $\tilde\Theta(n / (m T) \cdot \max\{n\Tnew/(mT), 1\})$.
\end{thm}

If the evaluation horizon $\Tnew$ is a constant, the rate in~\Cref{thm:simple-rate-few-traj}
recovers that of~\Cref{thm:simple-rate-many-traj-in-window}, up to log factors and extra assumptions.
To draw further comparison, suppose that the training and evaluation horizons are equal,
i.e.,\ that $\Tnew = T$.
On the face of it, the rate in~\Cref{thm:simple-rate-few-traj}
is evidently weaker than that of~\Cref{thm:simple-rate-many-traj-in-window},
by up to a factor of the covariate dimension $n$. 
But the varying premises---of many vs.\ few trajectories---necessarily constrain
the risk definitions to differ. %
Under a fixed data budget $N := mT = m\Tnew$, fewer trajectories $m$ imply a
longer horizon $\Tnew$ over which the risk is evaluated.
Intuitively, a longer evaluation horizon makes for a different problem,
and renders the rate comparison invalid.

A more sound comparison
across regimes is possible by first normalizing the notion of performance within
a problem instance.
To this end,
we can consider the worst-case risk of learning from trajectories \emph{relative} to that of learning
from independent examples \emph{in the same regime}.
Constructing the latter baseline is somewhat subtle (cf.~\Cref{sec:prob:general-problems}).
To decorrelate the problem of learning from trajectories while maintaining
its temporal structure otherwise, we can imagine drawing from its marginal distributions
independently at each time step.
The resulting dataset is independent, but not identically distributed.
Although the rates for the sequential and decorrelated regression problems are---as
already highlighted---remarkably the same under many trajectories, %
the few-trajectory rate in~\Cref{thm:simple-rate-few-traj} is indeed weaker than
the $\Theta(n/(mT))$
rate that we prove for its decorrelated baseline (cf.~\Cref{stmt:upper_bound_ind_lds_ls}).

Since the more general~\Cref{thm:simple-rate-many-traj-in-window}
already describes what happens under many trajectories ($m \gtrsim n$)
and a strict evaluation horizon ($\Tnew \le T$),
what remains is a somewhat niche regime: many trajectories and an extended evaluation horizon
$\Tnew > T$.
For completeness, our bounds supply the following summary statement:
\begin{thm}[\informal{error rate with many LDS trajectories}]
\label{thm:simple-rate-many-traj}
If $m \gtrsim n$
and covariate trajectories are drawn from a linear dynamical system whose
dynamics $A$ are marginally unstable and diagonalizable,
then the worst-case excess prediction risk (over evaluation horizon $\Tnew$)
for linear regression from $m$ many
trajectories of $n$-dimensional covariates, each of length $T$,
is $\Theta(n / (m T) \cdot \max\{\Tnew/T, 1\})$.
\end{thm}

Using the tools of our analysis,
we also develop upper bounds for parameter error instead of prediction risk,
which inform recovery of the ground truth $W_\star$
and (by reduction) of the dynamics matrix $A$ in LDS.
The latter captures the linear system identification problem.
Our upper bounds improve on its worst-case guarantees by a factor of $1/T$ where applicable,
and extend the parameter ranges in which guarantees hold at all.

\section{Related work}
\label{sec:related}

Linear regression is a basic and well-studied problem.
The two treatments most closely related to our work are \cite{hsu14randomdesign} and \cite{mourtada19exactminimax},
who develop sharp finite-sample characterizations of the risk of random design linear regression
(i.e., from \IID{} examples).
Discussion and references therein cover the broader problem over its long history.

A common approach to studying dependent covariates
is to assume that the data-generating process is 
ergodic
(see e.g.\ \cite{yu1994mixing,meir2000timeseries,mohri2008rademachermixing,steinwart2009fastlearningmixing,mohri2010stability,duchi2012ergodicmd,kuznetsov2017mixing,mcdonald2017timeseries,shalizi2021book} and references therein).
The key phenomenon at play is that $N$ correlated examples
are statistically similar to $N/\tmix$ independent examples,
where $\tmix$ is the process \emph{mixing-time}.
Relying on this idea,
generalization bounds informing independent data can typically be
ported to the ergodic setting, where the effective sample size is simply
``deflated'' by a factor of $\tmix$. 
Since mixing-based bounds become vacuous
as $\tmix \rightarrow \infty$, 
they do not present an effective strategy for studying
dynamics that do not mix.
A critical instance of this arises in linear dynamical systems:
in LDS, the ergodicity condition amounts
to \vocab{stability} of the dynamics matrix $A$ (i.e.,~$\rho(A) < 1$),
where $\tmix \to \infty$ as $\rho(A) \to 1$~\citep[e.g.][Thm.~17.6.2]{meynandtweedie1993}.
Marginally unstable systems, in which $\rho(A) = 1$,
are thus not captured.

A recent line of work uncovers ways to sharpen generalization bounds based on the
specific structure of \emph{realizable} least-squares
regression problems over an ergodic trajectory.
For realizable linear regression with stationary covariates,
results from \cite{bresler2020leastsquaresmarkov}
imply that,
after the trajectory length exceeds an initial burn-in time
scaling as $\tmix n$, the minimax (excess) risk coincides 
with the classic iid rates.
Additionally, \cite{ziemann2022littlemixing}
show that the empirical risk minimizer exhibits similar behavior
in realizable nonparametric regression problems,
provided certain small-ball assumptions of the underlying process hold.
While these results sharpen our understanding of how the mixing time $\tmix$
affects regression risk bounds, they ultimately rely on ergodicity.
Since learning from a single trajectory is generally impossible without ergodicity,
we are led to study other sequential learning configurations.
The two, however, are not mutually exclusive:
our results actually apply when mixing, and in fact show that
the empirical risk minimizer is minimax optimal (after a burn-in time scaling with the mixing time).
This eschews the need for algorithmic modifications to learning
from mixing trajectory data~\cite{bresler2020leastsquaresmarkov}.
We give details on this in \Cref{sec:app:high_prob_upper_bounds}.

\paragraph{Non-temporal dependency structures.}
Covariates and responses can be inter-dependent in many ways, not only via temporal structure.
A recent resurgence of work investigates learning under an Ising model structure over
covariates~\cite{bresler2015ising,dagan2019weaklydependent,ghosal2020ising,dagan2021multipleising},
as well as over responses ~\cite{daskalakis2019regression,dagan2021dependent} (conditioned on the covariates).
At a conceptual level, the extension from a single temporally dependent trajectory
to multiple trajectories is analogous to the extension from single observations
to Ising models with multiple independent observations.
Incidentally, in this area, investigations \emph{began} by studying learning under \emph{multiple} independent observations,
and progressed towards guarantees on learning from a single one.
Relating these two data models---trajectories and Ising grids---under intercompatible assumptions
may reveal interesting connections between these results.

\paragraph{System identification.}
A special case of our LDS-specific data model captures
\vocab{linear system identification} 
with full state observation:
the task of recovering the dynamical system parameters $A$
from observations of trajectories. 
While classic results 
are asymptotic in nature~(see e.g.~\cite{lai1982leastsquares,lai1983autoregressive,ljung1998sysidbook}),
recent work gives finite-sample guarantees for recovery of linear systems with fully observed states
\citep{simchowitz18learning,dean2020lqr,yassir2020sysid,faradonbeh2018unstable,sarkar2019sysid,jedra2019lowerbounds,tsiamis2021lowerbounds},
and also partially observed states
\citep{oymak2019lti,simchowitz2019semiparametric,tsiamis2019ssi,sarkar2021sysid,zheng2021ltimultiple}.
The proof of our upper bounds builds on
the ``small-ball'' arguments from \cite{simchowitz18learning}
(that, in turn, extend \cite{mendelson2015learningwithoutconc,koltchinskii2015smin}),
which do not require ergodicity.

To the best of our knowledge, our results are the first
to quantify the trade-offs between
few long trajectories and many short trajectories.
Nearly all finite-sample guarantees for linear system identification
consider a \emph{single} trajectory,
with a few notable exceptions. First, \cite{dean2020lqr}
allow for $m \ge 1$ trajectories with fully observed states
and make no assumptions on the dynamics matrix $A$.
However, their analysis discards all but the last state transition
within a trajectory, reducing to \IID{} learning over only $m$ examples.
Second, \cite{zheng2021ltimultiple,xin2022multitraj} study the recovery of Markov parameters from
partially observed states over many trajectories.
However, their error bounds
do not decrease with longer training horizons $T$, since the
number of Markov parameters one must recover
scales with the trajectory length.
Third, \cite{xing2021multiplicative} consider multiple trajectories
where the noise enters \emph{multiplicatively} instead of additively.
Their main finite-sample parameter recovery result (Theorem 2) states that
the operator norm of the parameter error scales as $\sqrt{T/m}$, 
with the additional restriction that $T \gtrsim n^2$. To achieve consistency, this result fixes
the trajectory length $T$ and takes the trajectory count $m \to \infty$.
By contrast, our analysis varies the two quantities $T$ and $m$ independently.
Finally, a line of work concurrent to ours investigates learning from multiple
sources of linear dynamical systems~\citep{chen2022learning,modi2022learning}.
This is a latent variable model, where the underlying index
of the LDS must be disambiguated from data. 
This model is more general than the one studied
in this paper, and specializing the corresponding results to our
setup yields sub-optimal bounds and unnecessary requirements.
We discuss this in \Cref{sec:results:upper:comparison}, after presenting upper bounds in detail.

Furthermore, our LDS setup (\Cref{sec:problem:lds-trajectories})
decouples the covariate dynamics model $A$
from the observation model $W_\star$,
and our risk definition additionally allows for an arbitrary evaluation horizon $\Tnew$.
The risk over an arbitrary evaluation horizon is harder to control than parameter
error, which corresponds to an evaluation length of one. This is because
the larger signal-to-noise ratio accrued by a less stable system
magnifies the prediction error over the entire evaluation horizon.
Although the observation model that we consider is mentioned in \cite{simchowitz18learning},
the general setup with matching upper and lower bounds are all,
to the best of our knowledge, new contributions.

A complementary line of work studies the
problem of online sequence prediction in a no-regret framework, where the
baseline expert class comprises of trajectories generated by a linear dynamical system~\citep{hazan2017spectralfiltering,hazan2018spectralfiltering,ghai2020noregret}.
These results also allow for marginally unstable dynamics but are otherwise not directly comparable.
Other efforts look beyond linear systems to identifying various non-linear classes, such as
exponentially stable non-linear systems~\citep{sattar2020nonlinear,foster2020nonlinear}
and marginally unstable non-linear systems~\citep{jain2021nonlinear}. These results again learn
from a single trajectory.
We believe that elements of our analysis can be ported over to 
offer many-trajectory bounds for these
particular classes of non-linear systems.

\section{Problem formulation}
\label{sec:problem}

\paragraph{Notation.}
The real eigenvalues of a Hermitian matrix $M \in \C^{k \times k}$
are
$\lambda_{\max}(M) = \lambda_1(M) \ge \dots \ge \lambda_k(M) = \lambda_{\min}(M)$.
For a square matrix $M \in \C^{k \times k}$,
$M^*$ denotes its conjugate transpose, and
$\rho(M)$ denotes its spectral radius: $\rho(M) = \max\{ \abs{\lambda} \mid \lambda \text{ is an eigenvalue of } M \}$.
The space of $n \times n$ real-valued symmetric positive semidefinite
(resp.\ positive definite) matrices is denoted $\sfSym^n_{\geq 0}$
(resp.\ $\sfSym^n_{> 0}$). 
The non-negative (resp.\ positive) orthant in $\R^n$ is denoted as 
$\R^n_{\geq 0}$ (resp.\ $\R^n_{> 0}$), and
$\mathbb{S}^{n-1}$ denotes the unit sphere in $\R^n$.
Finally, the set of positive integers is denoted
by $\N_{+}$.

\subsection{Linear regression from sequences}
\label{sec:problem:regression}

\paragraph{Regression model.}
A \vocab{covariate sequence} is an indexed set $\{x_t\}_{t \ge 1} \subset \R^n$.
Any distribution $\Px$ over covariate sequences is assumed to have bounded second moments,
i.e., that $\E[x_t x_t^\T]$ exists and is finite for all $t \ge 1$.
Also for such a distribution $\Px$, let
$\Pxi[\Px]$ be a distribution over
\vocab{observation noise} sequences $\{\xi_t\}_{t \ge 1} \subset \R^p$.
Denoting by $\{\calF_t\}_{t \geq 0}$ the filtration
with $\calF_t = \sigma(\{x_k\}_{k=1}^{t+1}, \{\xi_k\}_{k=1}^{t})$,
we assume that 
$\{\xi_t\}_{t \geq 1}$ is a $\sigma_\xi$-sub-Gaussian martingale difference sequence (MDS), i.e., 
for $t \geq 1$:
\begin{align*}
    \E[ \ip{v}{\xi_t} \mid \calF_{t-1}] = 0, \quad
    \E[ \exp(\lambda \ip{v}{\xi_t}) \mid \calF_{t-1} ] \leq \exp(\lambda^2 \norm{v}_2^2\sigma_\xi^2/2) \:\:\text{a.s.}\:\: \forall \lambda \in \R,  v \in \R^{p}.
\end{align*}
Given a \vocab{ground truth model} $W_\star \in \R^{p \times n}$,
define the \vocab{observations} (a.k.a.\ ``responses'' or ``labels''):
\begin{align}
    y_t = W_\star x_t + \xi_t, \quad t \geq 1. \label{eq:y_observation_model}
\end{align}
Denote by $\Pxy[\Px,\Pxi]$ the joint distribution
over covariates and their observations $\{(x_t, y_t)\}_{t \ge 1}$.

\paragraph{Regression task.}
Fix a ground truth model $W_\star \in \R^{p \times n}$,
a covariate distribution $\Px$,
an observation noise model $\Pxi$,
a training horizon $T$, and a test horizon $\Tnew$.
Draw $m$ independent sequences $\{(x_t^{(i)}, y_t^{(i)})\}_{i \in [m], t \ge 1}$
from $\Pxy[\Px,\Pxi]$,
and call their length-$T$ prefixes
$\{(x_t^{(i)}, y_t^{(i)})\}_{i=1,t=1}^{m,T}$ the training \vocab{examples}.
From these examples, the regression task is to find a hypothesis
$\hat{f}_{m,T} : \R^{n} \rightarrow \R^{p}$
that matches ground truth predictions $f_{W_\star}(x) := W_\star x$
in expectation over unseen trajectories of length $\Tnew{}$.
Specifically, the excess \vocab{risk} of a hypothesis $\hat{f}$ is:
\begin{align}
    L(\hat{f}; \Tnew, \Px) :=
    \E_{\Px} \left[
        \frac{1}{\Tnew} \sum_{t=1}^{\Tnew} \norm{ \hat{f}(x_t) - f_{W_\star}(x_t) }^2_2 \right]. \label{eq:risk_def}
\end{align}
We say that the evaluation horizon $\Tnew$ is \vocab{strict} if $\Tnew \le T$
and \vocab{extended} if $\Tnew > T$.
When the hypothesis class is linear, meaning the hypotheses $\hat{f}$ are of the
form $\hat{f}(x) = \hat{W} x$ with $\hat{W} \in \R^{p \times n}$,
the risk expression \eqref{eq:risk_def} simplifies as follows.
For a positive definite matrix $\Sigma \in \R^{n \times n}$,
define the weighted square norm
$\norm{M}_\Sigma^2 := \Tr( M \Sigma M^\T )$ for $M \in \R^{p \times n}$.
Denoting, for $t \ge 1$:
\begin{align}
    \Sigma_t(\Px) := \E_{\Px}[x_t x_t^\T], \quad\quad
    \Gamma_{t}(\Px) := \frac{1}{t} \sum_{k=1}^{t} \Sigma_k(\Px),
    \label{eq:covariance-def}
\end{align}
we overload notation and write:
\begin{align}
    L(\hat{W}; \Tnew, \Px) =  \norm{\hat{W} - W_\star}_{\Gamma_{\Tnew}(\Px)}^2.
    \label{eq:risk_linear}
\end{align}
The risk \eqref{eq:risk_def}, being a notion of error averaged over time steps,
relates to that of \cite{ziemann2022single} in the study of learning dynamics
(the difference lies in whether the error norm is squared).

By allowing unequal training and test horizons $T \ne \Tnew$, we cover two
related scenarios at once:
system identification in linear dynamical systems (when $\Tnew = 1$) and
predicting past the end of a sequence (when $\Tnew > T$).
For the latter, the risk definition \eqref{eq:risk_def}
is closely related to a commonly studied notion of ``final step''
generalization (see e.g.\ \cite[Eq.~5]{kuznetsov2017mixing}, \cite[Def.~10]{mcdonald2017timeseries})
that measures the performance of a hypothesis at $\Tnew - T$ time steps beyond the training horizon:
$L_{\mathrm{end}}(\hat{f}; \Tnew, \Px) := \E_{\Px}[ \norm{\hat{f}(x_{\Tnew}) - f_{W_\star}(x_{\Tnew})}_2^2 ]$.
Linear hypotheses enjoy the identity $L_{\mathrm{end}}(\hat{W}; \Tnew, \Px) = \norm{\hat{W}-W_\star}_{\Sigma_{\Tnew}(\Px)}^2$.
In turn:
\begin{align*}
    L_{\mathrm{end}}(\hat{W};\Tnew,\Px) \geq L(\hat{W};\Tnew,\Px) \gtrsim L_{\mathrm{end}}(\hat{W};\floor{\Tnew/2},\Px).
\end{align*}
In other words,
provided the scale of the covariances $\Sigma_t(\Px)$ does not grow 
substantially over time $t$,
our risk definition $L$ is comparable to the final-step risk $L_{\mathrm{end}}$.

\paragraph{Minimax risk.}
To compare the hardness of learning across problem classes
(i.e., families of covariate distributions $\Px$),
we measure the \vocab{minimax rate} of the risk $L$---i.e.,
the behavior of the best estimator's worst-case risk
over valid problem instances---as a function of
the amount of training data $m, T$ and
other problem parameters such as $n$, $p$, $\sigma_\xi$, and $\Tnew$.
Recall that $\Pxy$ %
denotes the distribution over labeled trajectories $\{(x_t, y_t)\}_{t \geq 1}$.
For a collection of covariate sequence distributions $\calPx$,
the minimax risk over problem instances consistent with $\calPx$ is:
\begin{align}
\sfR(m, T, \Tnew; \calPx) :=
  \inf_{\mathsf{Alg}} 
  \sup_{\Px \in \calPx}
  \sup_{W_\star, \Pxi}
  \E_{\otimes_{i=1}^{m} \Pxy[\Px,\Pxi]}
    \left[ L \left(
      \mathsf{Alg}(\{ (x_t^{(i)}, y_t^{(i)} ) \}_{i=1,t=1}^{m,T}); \Tnew, \Px \right) \right],
\label{eq:minimax_risk_def}
\end{align}
where the infimum ranges over estimators $\mathsf{Alg} : (\R^{n} \times \R^{p})^{mT} \to (\R^n \to \R^p)$
that map training samples to hypotheses,
the supremum 
over $W_\star$ is over all $p \times n$ ground truth models, and
the supremum 
over $\mathsf{P}_\xi$ is over all $\sigma_\xi$-sub-Gaussian MDS
processes determining the observation noise.

\paragraph{The ordinary least-squares estimator.}
Much like its classical role in \IID{} learning, the
\vocab{ordinary least-squares} (OLS) estimator will
be key to bounding the minimax risk \eqref{eq:minimax_risk_def}
from above.
We define the OLS estimator to be the linear hypothesis
$\hat{W}_{m,T} \in \R^{p \times n}$ that satisfies:
\begin{align}
    \hat{W}_{m,T} \in \argmin_{W \in \R^{p \times n}} \sum_{i=1}^{m} \sum_{t=1}^{T} \norm{ W x_t^{(i)} - y_t^{(i)} }_2^2. \label{eq:ols_definition}
\end{align}
For $i=1, \dots, m$, let $X^{(i)}_{m,T} \in \R^{T \times n}$ be the data matrix for the $i$-th trajectory
(i.e., the $t$-th row of $X^{(i)}_{m,T}$ is $x_t^{(i)}$ for $t = 1, \dots, T$).
Define $Y^{(i)}_{m,T} \in \R^{T \times p}$ and $\Xi^{(i)}_{m,T} \in \R^{T \times p}$
analogously.
Put $X_{m,T} \in \R^{ mT \times n}$ as the vertical concatenation of $X_{m,T}^{(1)}, \dots, X_{m,T}^{(m)}$,
and similarly for $Y_{m,T} \in \R^{mT \times p}$ and $\Xi_{m,T} \in \R^{mT \times p}$.
Whenever $X_{m,T}$ has full column rank, 
then we can write $\hat{W}_{m,T}$ as:
\begin{align}
    \hat{W}_{m,T} = Y_{m,T}^\T X_{m,T} (X_{m,T}^\T X_{m,T})^{-1}. \label{eq:W_hat_expr}
\end{align}

\subsection{Problem classes}
\label{sec:prob:general-problems}

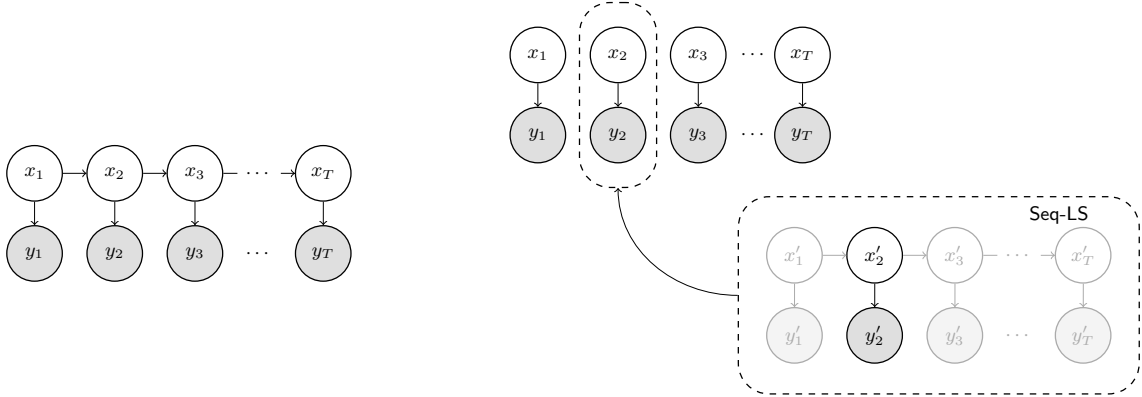
\begin{figure}
\centering
\centering
\begin{minipage}[b]{.42\columnwidth}%
\centering
\scalebox{.82}{%
\begin{tikzpicture}[
  y=-1cm,
  scale=1.3,
  every node/.style={font=\footnotesize},
  xx/.style={circle, draw=black, inner sep=0pt, minimum size=9mm, semithick},
  yy/.style={circle, draw=black, fill=gray!25, inner sep=0pt, minimum size=9mm, semithick},
  dd/.style={}
]

\node[xx] (x1) at (0, 0) {$x_{1}$};
\node[xx] (x2) at (1, 0) {$x_{2}$};
\node[xx] (x3) at (2, 0) {$x_{3}$};
\node[dd] (dx) at (2.8, 0) {$\cdots$};
\node[xx] (xT) at (3.6, 0) {$x_{T}$};

\node[yy] (y1) at (0, 1) {$y_{1}$};
\node[yy] (y2) at (1, 1) {$y_{2}$};
\node[yy] (y3) at (2, 1) {$y_{3}$};
\node[dd] (dy) at (2.8, 1) {$\cdots$};
\node[yy] (yT) at (3.6, 1) {$y_{T}$};

\draw [->] (x1) to (x2);
\draw [->] (x2) to (x3);
\draw [-] (x3) to (dx);
\draw [->] (dx) to (xT);

\draw [->] (x1) to (y1);
\draw [->] (x2) to (y2);
\draw [->] (x3) to (y3);
\draw [->] (xT) to (yT);
\end{tikzpicture}
}%
\vspace{1.5cm}
\subcaption{
The \DSR{} problem (\Cref{prob:dsr}): covariates $\{x_t\}$ are drawn from a sequence distribution,
and noisy observations $\{y_t\}$ are drawn conditioned on these covariates.}
\label{fig:problem-schematic:dsr}
\end{minipage}
\hfill
\begin{minipage}[b]{0.54\columnwidth}
\centering
\scalebox{.82}{%
\begin{tikzpicture}[
  y=-1cm,
  scale=1.3,
  every node/.style={font=\footnotesize},
  xx/.style={circle, draw=black, inner sep=0pt, minimum size=9mm, semithick},
  yy/.style={circle, draw=black, fill=gray!25, inner sep=0pt, minimum size=9mm, semithick},
  dd/.style={},
  xxl/.style={circle, draw=black!33, text=black!33, inner sep=0pt, minimum size=9mm, semithick},
  yyl/.style={circle, draw=black!33, text=black!33, fill=gray!9, inner sep=0pt, minimum size=9mm, semithick},
  ddl/.style={text=black!33},
  panel/.style={draw=black, dashed, inner sep=0pt,
      minimum size=9mm,rounded corners=0.5cm, semithick}
]

\node[xx] (xi1) at (0, 0) {$x_{1}$};
\node[xx] (xi2) at (1, 0) {$x_{2}$};
\node[xx] (xi3) at (2, 0) {$x_{3}$};
\node[dd] (dxi) at (2.7, 0) {$\cdots$};
\node[xx] (xiT) at (3.3, 0) {$x_{T}$};

\node[yy] (yi1) at (0, 1) {$y_{1}$};
\node[yy] (yi2) at (1, 1) {$y_{2}$};
\node[yy] (yi3) at (2, 1) {$y_{3}$};
\node[dd] (dyi) at (2.7, 1) {$\cdots$};
\node[yy] (yiT) at (3.3, 1) {$y_{T}$};

\draw [->] (xi1) to (yi1);
\draw [->] (xi2) to (yi2);
\draw [->] (xi3) to (yi3);
\draw [->] (xiT) to (yiT);

\node[panel, minimum height=3cm,   minimum width=1.25cm] (slice) at (1, .5) {};
\node[panel, minimum height=3.2cm, minimum width=6.5cm ] (tray)  at (5, 3) {};
\node[] (dsrlabel) at (6.5, 2) {\DSR{}};

\draw [->] (tray.west) to[in=-90, out=180] (slice.south);

\node[xxl] (x1) at (3.2+0, 2.5+0) {$x'_{1}$};
\node[xx] (x2) at (3.2+1, 2.5+0) {$x'_{2}$};
\node[xxl] (x3) at (3.2+2, 2.5+0) {$x'_{3}$};
\node[ddl] (dx) at (3.2+2.8, 2.5+0) {$\cdots$};
\node[xxl] (xT) at (3.2+3.6, 2.5+0) {$x'_{T}$};

\node[yyl] (y1) at (3.2+0, 2.5+1) {$y'_{1}$};
\node[yy] (y2) at (3.2+1, 2.5+1) {$y'_{2}$};
\node[yyl] (y3) at (3.2+2, 2.5+1) {$y'_{3}$};
\node[ddl] (dy) at (3.2+2.8, 2.5+1) {$\cdots$};
\node[yyl] (yT) at (3.2+3.6, 2.5+1) {$y'_{T}$};

\draw [->,black!33] (x1) to (x2);
\draw [->,black!33] (x2) to (x3);
\draw [-,,black!33] (x3) to (dx);
\draw [->,black!33] (dx) to (xT);

\draw [->,black!33] (x1) to (y1);
\draw [->] (x2) to (y2);
\draw [->,black!33] (x3) to (y3);
\draw [->,black!33] (xT) to (yT);
\end{tikzpicture}
}%
\subcaption{%
The corresponding baseline \PSR{} problem (\Cref{prob:psr}):
independent covariate-observation pairs $\{(x_t, y_t)\}$ are drawn, each from
the marginal distribution of the corresponding $t$'th step in the \DSR{} problem.}
\label{fig:problem-schematic:psr}
\end{minipage}%
\caption{%
Formulations of regression from sequential data, illustrated as graphical models.
Specifically, these graphs depict a simplified special case of our data model,
in which the observations $\{y_t\}$ across time are independent conditioned on the covariates $\{x_t\}$.
In our general definitions (\Cref{prob:dsr} and~\Cref{prob:psr}),
the observations $\{y_t\}$ can be conditionally interdependent,
via a martingale difference sequence on the observation noise (\Cref{sec:problem:regression}).
}
\label{fig:problem-schematic}
\end{figure}

We formalize linear regression from sequential data generally as follows:%
\begin{myprob}[\DSR{}]
\label{prob:dsr}
Assume a covariate sequence distribution $\Px$
in the linear regression model \eqref{eq:y_observation_model}.
Fix an evaluation horizon $\Tnew$.
On input $m$ labeled trajectories of length $T$
drawn from this model,
in the form of examples $\{(x_t^{(i)}, y_t^{(i)})\}_{i=1,t=1}^{m,T}$,
output a hypothesis $\hat{f}_{m,T}$
that minimizes excess risk $L(\hat{f}_{m,T}; \Tnew, \Px)$.
\end{myprob}
\noindent
Our topmost goal is to study the effect of learning from sequentially
dependent covariates \emph{in comparison with} learning in the classical \IID{} setup.
Linear regression is well understood in the latter setting.
Focusing on \emph{well-specified} linear regression further simplifies our presentation,
allowing us to isolate the effects of what interests us most---dependent covariates.
Generalizing the supervision aspect of \DSR{}
(say, to unrealizable and non-parametric regression, or to classification)
is left to future work. We return to discuss this in~\Cref{sec:conclusion}.

To study how dependent data affects learning,
we need to establish an ``independent data'' baseline.
The natural comparison point for \DSR{} is to remove all correlations across
time. Namely, instead of drawing covariates sequentially from the distribution $\Px$,
consider learning separately from the marginals of $\Px$ at each time step.
The resulting decorrelated distribution generates \emph{independent} examples,
but typically not \IID{} ones.
We formalize linear regression from independent data generally as follows:
\begin{myprob}[\PSR{}]
\label{prob:psr}
Fix a sequence of distributions $\{\Pxt{t}\}_{t \ge 1}$.
Consider their product over time $\otimes_{t \ge 1} \Pxt{t}$
as the covariate sequence distribution in the linear regression model~\eqref{eq:y_observation_model}.
Fix an evaluation horizon $\Tnew$.
On input $m$ labeled trajectories of length $T$
drawn from this model,
in the form of examples $\{(x_t^{(i)}, y_t^{(i)})\}_{i=1,t=1}^{m,T}$,
output a hypothesis $\hat{f}_{m,T}$
that minimizes $L(\hat{f}_{m,T}; \Tnew, \otimes_{t \ge 1} \Pxt{t})$.
\end{myprob}
\noindent
This \PSR{} problem generalizes the canonical \IID{} learning setup slightly.
Existing theory can still characterize its minimax risk, provided the
covariances of the distributions $\{\Pxt{t}\}$ are roughly equal in scale across time $t$.
However, this equal-scale requirement rules out the marginals of interesting applications,
such as dynamical systems that are not stable or ergodic.
We therefore extend, in later sections, characterizations of the regression risk to
handle covariances that can scale \emph{polynomially} across time instead.

\subsection{Problem separations}
\label{sec:prob:separations}

To set up a baseline for a \DSR{} problem, we will specifically instantiate \PSR{} over its marginals.
Namely, for a sequence distribution $\Px$ over $\{x_t\}_{t \ge 1}$,
let $\marginal{\Px}{t}$ be the marginal distribution of $x_t$ at time $t \ge 1$,
and consider \PSR{} with covariates drawn from the sequence $\{ \marginal{\Px}{t} \}_{t \ge 1}$.
\Cref{fig:problem-schematic} illustrates such a \DSR{} problem
and the corresponding \PSR{} instance over its marginals.

This decorrelated baseline is a hypothetical benchmark:
in a practical context, collecting independent marginal data,
when nature only supplies its dependent form, can be expensive or infeasible.
However, we can expect that having such data on hand would make learning easier,
with risk rates that resemble \IID{} learning.
In what follows, we outline scenarios where a sequential learning problem
and its decorrelated baseline coincide in difficulty, and others in which they diverge.
We then outline the possible assumptions that would allow us to always relate the two.

\paragraph{The \IID{} special case.}
When $T = \Tnew = 1$,
the example trajectories $\{x_1^{(i)}\}_{i=1}^m$ are trivially a set of \IID{} covariates.
The problems \DSR{} and \PSR{} thus coincide, and reduce to the well-specified
random design linear regression problem over $m$ \IID{}
covariates.
It is well-known that under \IID{} data, and mild regularity conditions, the
minimax risk scales as
$\sigma_\xi^2 pn/m$,
and is achieved by the OLS
estimator~\citep{hsu14randomdesign,mourtada19exactminimax,wainwright2019book}.

\paragraph{Extending the horizon.}
Considering nontrivial horizons $T = \Tnew > 1$, both
\DSR{} and its corresponding \PSR{} baseline become more involved, 
but for different reasons.

The \PSR{} problem, as we show in \Cref{sec:results:lower_bounds},
is not generally learnable with polynomially many examples. Specifically,
the minimax rate scales exponentially in the dimension $n$
provided the trajectory count $m$ is constant.
To address this, we will require that the covariances
of its constituent distributions $\{\Pxt{t}\}$ grow at most polynomially with time $t$.
Under this constraint,
the problem's minimax risk again scales as the \IID{}-like rate $\sigma_\xi^2 pn/(mT)$
times, at most, a factor determined exponentially by the covariance growth.

The \DSR{} problem inherits the same growth limitation.
Even then, it is still not generally learnable without further assumptions 
on the dependence structure of covariates:
the minimax risk %
is otherwise bounded away from zero as the horizon $T$ tends to infinity,
provided the trajectory count $m$ is constant.
To realize this, consider $x_1 \sim N(0, I_n)$ and $x_{t} = x_{t-1}$ for $t \ge 2$,
a sequence of identical covariates whose marginals are all independent Gaussians.
The resulting dataset presents an underdetermined regression problem if $m < n$.
In essence, its covariates lack sufficient ``excitation'' across time.
To rein \DSR{} back in to the realm of learnability, one must:
\begin{enumerate}[label=(\alph*)]
\item make further modeling assumptions about covariates, or
\item introduce excitation via independent resets.
\end{enumerate}

For (a), as detailed in~\Cref{sec:related},
the most common modeling assumption considers sequences that mix rapidly to a stationary distribution.
Another avenue---recently active in the literature, and sometimes overlapping with the
mixing approach---considers sequences generated by linear dynamical systems.
Among these two, %
mixing implies risk bounds that tend to zero with $T$,
but only hold in the worst case after a burn-in time that scales proportionally to the mixing time
\citep{bresler2020leastsquaresmarkov}.
This prevents a characterization of minimax risk uniformly across the full range of
problem instances $\Px$ that mix, %
unless one caps the mixing time to a fixed constant.
Narrowing instead to LDS models in the sequel,
we manage to succinctly carve out a basic problem
family, with \emph{unbounded} mixing time, and to characterize
its minimax risk uniformly.
One still pays a price for sequential dependency,
as this minimax risk turns out to be larger than its \PSR{}
counterpart by a factor of the dimension $n$.

Turning in addition to (b), by introducing (sufficiently many) resets,
we can expand our data model substantially:
we manage to lift most of our LDS assumptions and extend to other dynamical systems.
Remarkably, we even show that for any controllable LDS---including ones
that are unstable and hence grow exponentially in time---having sufficiently many resets guarantees
that the risk exhibits, once again, the \IID{}-like behavior of
$\sigma_\xi^2 pn/(mT)$, up to mere constants.

\subsection{Linear dynamical trajectories}
\label{sec:problem:lds-trajectories}

Fix a \vocab{dynamics matrix} $A \in \R^{n \times n}$
and a \vocab{control matrix} $B \in \R^{n \times d}$.
Consider the $n$-dimensional trajectory $\{x_t\}_{t \ge 1}$
defined by the linear dynamical system:
\begin{align}
    x_t= A x_{t-1} + B w_t,
    \:\: \text{where } w_t \sim N(0, I_d),
    \:\: \text{for } t \ge 1,
\label{eq:lds_definition}
\end{align}
taking $x_0 = 0$ by convention.
We assume that the noise process $\{w_t\}_{t \ge 1}$ is independent across time, i.e.,
that $w_t \perp w_{t'}$ whenever $t \neq t'$. 
Overloading notation,
let the matrix $\Sigma_t(A, B) := \sum_{k=0}^{t-1} A^k BB^\T (A^k)^\T$
denote the covariance of $x_t$,
and let the matrix $\Gamma_t(A, B) := \frac 1 t \sum_{k=1}^t \Sigma_k(A, B)$
denote the average covariance.
Denote by $\PxA{A,B}$ the distribution over the trajectory $\{x_t\}_{t \ge 1}$,
and let $\{x_t^{(i)}\}_{t \ge 1}$ for $i \ge 1$ denote independent draws from $\PxA{A,B}$.
When $B=I_n$, we use the respective shorthand notation
$\Sigma_t(A)$, $\Gamma_t(A)$, and $\Px^A$.

Modeling regression covariates as linear dynamical trajectories gives us the \DLR{} problem,
a specialization of \DSR{} (\Cref{prob:dsr}):
\begin{myprob}[\DLR{}]
Assume a dynamics matrix $A \in \R^{n \times n}$,
a control matrix $B \in \R^{n \times d}$,
and a corresponding linear dynamical covariate distribution $\PxA{A,B}$
in the linear regression model \eqref{eq:y_observation_model}.
Fix an evaluation horizon $\Tnew$.
On input $m$ labeled trajectories of length $T$,
drawn from this model,
in the form of examples $\{(x_t^{(i)}, y_t^{(i)})\}_{i=1,t=1}^{m,T}$,
output a hypothesis $\hat{f}_{m,T}$
that minimizes $L(\hat{f}_{m,T}; \Tnew, \PxA{A,B})$.
\end{myprob}

Let $\PxAt{A,B}{t}$ be the marginal distribution of $x_t$ under $\PxA{A,B}$ at each $t \ge 1$.
The natural decorrelated baseline for \DLR{}
is a corresponding specialization of \PSR{} (\Cref{prob:psr}) to LDS trajectories:%
\begin{myprob}[\PLR{}]
Assume a dynamics matrix $A \in \R^{n \times n}$,
a control matrix $B \in \R^{n \times d}$,
and a corresponding trajectory distribution $\PxA{A,B}$.
Consider covariates drawn independently from its marginals,
i.e., assume the linear regression model \eqref{eq:y_observation_model}
under the covariate sequence distribution $\otimes_{t \geq 1} \PxAt{A,B}{t}$.
Fix an evaluation horizon $\Tnew$.
On input $m$ labeled trajectories of length $T$,
drawn from this model,
in the form of examples $\{(x_t^{(i)}, y_t^{(i)})\}_{i=1,t=1}^{m,T}$,
output a hypothesis $\hat{f}_{m,T}$
that minimizes $L(\hat{f}_{m,T}; \Tnew, \otimes_{t \geq 1} \PxAt{A,B}{t})$.
\end{myprob}

\paragraph{Learning dynamical systems.}
\LDSLS{} generalizes \vocab{linear system identification},
the problem of recovering the dynamics $A$ from data.
The reduction follows by setting $W_\star = A$
and $\xi_t^{(i)} = B w_{t+1}^{(i)}$,
so that $y_t^{(i)} = x_{t+1}^{(i)}$.
Note that when $B$ has full row rank, 
the squared parameter error in the weighted $BB^\T$ norm $\norm{\cdot}_{BB^\T}$
is simply the risk $L(\hat{A};\Tnew, \PxA{A,B})$ when $\Tnew = 1$.
Recent related work typically assumes that $B$ indeed has full row rank,
but in later sections we touch on the more general case where
this is not required, so long as the pair $(A, B)$ is controllable.
Bounds in operator norm are also easily obtainable
from our proof techniques.
However, our lower bounds will not inform the system identification
problem specifically;
our hardness results rely on decoupling $W_\star$ from $A$ and
$\xi_t^{(i)}$ from $w_{t+1}^{(i)}$, whereas this reduction naturally ties them.

\section{Trajectory small-ball definition and examples}
\label{sec:traj_sb}

We establish risk upper bounds by studying the behavior
of the ordinary least-squares estimator.
The key technical definition that drives the analysis is a ``small-ball'' condition
on covariate sequences:
\begin{mydef}[Trajectory small-ball (\TSB{})]
\label{def:trajectory_small_ball}
Fix a trajectory length $T \in \N_+$, 
a parameter $k \in \{1, \dots, T\}$,
positive definite matrices $\{\Psi_j\}_{j=1}^{\floor{T/k}} \subset \sfSym^{n}_{> 0}$, 
and constants $\csb \geq 1$, $\alpha \in (0, 1]$.
The distribution $\Px$ satisfies the
$(T, k, \{\Psi_j\}_{j=1}^{\floor{T/k}}, \csb, \alpha)$-\vocab{trajectory-small-ball (\TSB{})}
condition if:
\begin{enumerate}
    \item $\frac{1}{\floor{T/k}} \sum_{j=1}^{\floor{T/k}} \Psi_j \preccurlyeq \Gamma_T(\Px)$,
    \item $\{x_t\}_{t \geq 1}$ is adapted to a filtration $\{\calF_t\}_{t \geq 1}$, and
    \item for all $v \in \R^n \setminus \{0\}$, $j \in \{1, \dots, \floor{T/k}\}$ and $\varepsilon > 0$:
\begin{align}
    \Pr_{\{x_t\} \sim \Px}\left\{
      \frac{1}{k} \sum_{t=(j-1)k+1}^{jk} \ip{v}{ x_t}^2 \leq \varepsilon \cdot v^\T \Psi_j v
      ~\Bigg|~ \calF_{(j-1)k} \right\} \leq (\csb\varepsilon)^\alpha \:\:
      \textrm{a.s.}\label{eq:trajectory_small_ball}
\end{align}
\end{enumerate}
Above, $\calF_0$ is understood to be the minimal $\sigma$-algebra.
Additionally, the distribution $\Px$ satisfies the
$(T, k, \Psi, \csb, \alpha)$-\TSB{} condition
if it satisfies $(T, k, \{\Psi_j\}_{j=1}^{\floor{T/k}}, \csb, \alpha)$-\TSB{} with $\Psi_j = \Psi$.
Finally, we call the parameter $k$ the \vocab{excitation window}.
\end{mydef}
In \Cref{def:trajectory_small_ball}, we typically
consider the matrices $\Psi_j$ to be the sharpest almost-sure lower bound
that we can specify (in the Loewner order) on the quantity
$\E[ \frac{1}{k} \sum_{t=(j-1)k+1}^{jk} x_tx_t^\T \mid \calF_{(j-1)k} ]$.
\Cref{sec:results:upper:small_ball_examples} lists examples
of covariate sequence distributions $\Px$ that satisfy the 
\TSB{} condition.

\Cref{def:trajectory_small_ball} draws inspiration from
the block martingale small-ball condition from
\cite{simchowitz18learning}. 
There are two main differences:
(a) we consider the small-ball probability
of the \emph{entire} block $\frac{1}{k} \sum_{t=(j-1)k+1}^{jk} \ip{v}{x_t}^2$
at once, instead of
the \emph{average} of small-ball probabilities:
\begin{align*}
\frac{1}{k} \sum_{t=(j-1)k+1}^{jk} \Pr \left\{ \ip{v}{x_t}^2 \leq \varepsilon \cdot v^\T \Psi_j v \mid \calF_{(j-1)k} \right\},
\end{align*}
and
(b) equation \eqref{eq:trajectory_small_ball} is required to hold
at all scales $\varepsilon > 0$, instead of at a single resolution. 
We need the first modification (a) to prove optimal rates
under many trajectories without assuming stability or ergodicity.
Furthermore, condition \eqref{eq:trajectory_small_ball} is implied by
a bound on the average of small-ball probabilities
(\Cref{stmt:avg_small_ball_implies_block}).
We need the second modification (b) in order to 
bound the expected value of the OLS risk.
Under weaker conditions, we would only have risk bounds
that hold with high probability, as detailed in the
following remark:

\begin{myremark}
\label{rem:mixing-epsilon}
\normalfont
In \Cref{sec:app:high_prob_upper_bounds}, we consider the following
modification to \Cref{def:trajectory_small_ball}, 
where we instead suppose that
\eqref{eq:trajectory_small_ball} holds for \emph{some} fixed $\varepsilon$
(such that the inequality's right-hand side is strictly less than one), rather than for all $\varepsilon$;
we refer to this modification as \vocab{weak trajectory small-ball} 
(\Cref{def:weak_trajectory_small_ball}). 
As described above, a consequence of the weak trajectory small-ball 
condition is that the main OLS risk bounds now
hold with high probability (i.e., polylogarithmic in $1/\delta$)
rather than in expectation (\Cref{stmt:upper_bound_main_high_prob}).\footnote{
Such a high probability bound does not, in turn,
imply a bound on the expected risk via integration over the tail.
The reason is that the high probability bound (\Cref{stmt:upper_bound_main_high_prob}) requires that the number
of data points $mT$ grow as $\log(1/\delta)$, where
$\delta$ is the failure probability.}
A key upshot (\Cref{stmt:phi_mixing_implies_weak_small_ball}), however,
is that this change allows
for an ergodic covariate sequence with $\phi$-mixing time bounded by $\tmix$
to be considered (weak) trajectory small-ball (with excitation window $k \asymp \tmix$),
provided the stationary distribution $\mu$ satisfies
a standard (weak) small-ball condition~\cite{mendelson2015learningwithoutconc,koltchinskii2015smin,oliveria2016lowertail}:
\begin{align*}
    \sup_{v \in \mathbb{S}^{n-1}} \Pr_\mu\{ \ip{v}{x}^2 \leq \varepsilon \cdot \E_\mu[\ip{v}{x}^2] \} < 1 \text{ for some } \varepsilon > 0.
\end{align*}
This in turn yields upper bounds for \DSR\ in the few trajectories
($m \lesssim n$) regime of the following form: if $mT \geq \tilde{\Omega}(\tmix n)$,
then
\begin{align*}
    L(\hat{W}_{m,T}; T, \Px) \leq \tilde{O}\left( \sigma_\xi^2 \frac{pn}{mT} \right),
\end{align*}
with high probability.
This statement generalizes the risk bound for a single ergodic trajectory from
\cite[Theorem 1]{bresler2020leastsquaresmarkov}
to the ordinary least-squares estimator~\eqref{eq:ols_definition}.
We can interpret the condition on $mT$ as a ``burn-in time'' requirement.
Meanwhile, at least in the single-trajectory ($m=1$) setting, \cite[Theorem 9]{bresler2020leastsquaresmarkov}
tells us that that such a burn-in assumption ($T \gtrsim \tmix n$) is necessary for a non-trivial
risk guarantee.
\end{myremark}

\subsection{Examples of trajectory small-ball distributions}
\label{sec:results:upper:small_ball_examples}

We now turn to specific examples of distributions $\Px$ which satisfy
the trajectory small-ball condition.
First is the example introduced in 
\Cref{sec:prob:separations}, where $x_1$ is drawn from a multivariate Gaussian
and subsequently copied as $x_{t} = x_{t-1}$ for all $t \geq 2$:
\begin{restatable}[Copies of a Gaussian draw]{myexample}{examplesimple}
\label{example:simple}
Let $\Sigma \in \sfSym^n_{> 0}$, and 
let $\Px$ denote the process
$x_1 \sim N(0, \Sigma)$ and $x_{t} = x_{t-1}$ for $t \geq 2$.
Fix any $T \in \N_{+}$.
Then $\Px$ satisfies the $(T, T, \Sigma, e, \frac 1 2)$-\TSB{} condition.
\end{restatable}
Note that this process only satisfies the trajectory small-ball condition
with excitation window $k=T$.
In other words, the conditional distribution $x_{t+k} \mid x_t$ for $k \geq 1$
(a Dirac distribution on $x_t$) 
contains no excitation as needed for learning.
This example can actually be generalized to arbitrary 
Gaussian processes indexed by time:
\begin{restatable}[Gaussian processes]{myexample}{examplegaussianprocess}
\label{example:gaussian_processes}
Let $\Px$ be a Gaussian process indexed by time, i.e., 
for every finite index set $I \subset \N_{+}$, the
collection of random variables $(x_t)_{t \in I}$
is jointly Gaussian. Let $\Tnd := \inf\{ t \in \N_{+} \mid \det(\E[x_tx_t^\T]) \neq 0 \}$,
and suppose $\Tnd$ is finite.
Fix a $T \in \N_{+}$ satisfying $T \geq \Tnd$.
Then $\Px$ satisfies the $(T, T, \Gamma_T(\Px), 2e, \frac{1}{2})$-\TSB{} condition.
\end{restatable}

Our next example involves independent, but not identically distributed, covariates:
\begin{restatable}[Independent Gaussians]{myexample}{exampleindependentgaussians}
\label{example:independent_gaussians}
Let $\{\Sigma_t\}_{t \geq 1} \subset \sfSym^n_{>0}$, and
let $\Px = \otimes_{t \geq 1} N(0, \Sigma_t)$.
Fix a $T \in \N_{+}$. Then
$\Px$ satisfies the $(T, 1, \{\Sigma_t\}_{t=1}^{T}, e, \frac 1 2)$-\TSB{}
condition.
\end{restatable}
\noindent
\Cref{example:independent_gaussians} 
allows us to select $k=1$, reflecting the independence
of the covariates across time.

We can also craft an example around a process that does not mix,
but that still exhibits an excitation window of $k=2$:
\begin{restatable}[Alternating halfspaces]{myexample}{examplepartition}
\label{example:partition}
Suppose that $n \geq 4$ is even, and let 
$u_1, \dots, u_n$ be a fixed orthonormal basis of $\R^n$.
Put $U_0 = \Span(u_1, \dots, u_{n/2})$ and $U_1 = \Span(u_{n/2+1}, \dots, u_n)$.
Let $i_1 \sim \mathrm{Bern}(\frac 1 2)$, $i_{t+1} = i_t \mod 2$ for $t \in \N_{+}$,
and let $\Px$ denote the process with
conditional distribution $x_t \mid i_t$
uniform over the spherical measure on
$U_{i_t} \cap \mathbb{S}^{n-1}$.
For any $T \geq 2$, the process $\Px$ satisfies the
$(T, 2, I_n / (2n), e, \frac 1 2)$-\TSB{} condition.
\end{restatable}
\noindent
To see that the covariate distribution $\{x_t\}$ does not mix, observe that
the marginal distribution for all $t$ is uniform on $\mathbb{S}^{n-1}$,
whereas the conditional distribution $x_{t+k} \mid x_t$
for any $k \in \N_{+}$ is either uniform on
$U_0 \cap \mathbb{S}^{n-1}$ or uniform on $U_1 \cap \mathbb{S}^{n-1}$.
Although it does not mix at all, the trajectory supplies ample excitation for learning
in any mere two steps.

Even for a process that does mix, it may exhibit an excitation window 
far smaller than its mixing time. The following sets up such an example,
where again where sufficient excitation is provided with $k=2$ steps:
\begin{restatable}[Normal subspaces]{myexample}{mixingchain}
\label{example:mixing_chain}
Suppose that $n \geq 3$.
Let $u_1, \dots, u_n$ be a fixed orthonormal basis in $\R^n$,
and let $U_{\neg i} := \Span( \{u_j\}_{j \neq i} )$
for $i \in \{1, \dots, n\}$.
Consider the Markov chain $\{i_t\}_{t \geq 1}$ defined by
$i_1 \sim \mathrm{Unif}(\{1, \dots, n\})$, and
$i_{t+1} \mid i_t \sim \mathrm{Unif}(\{1,\dots,n\}\setminus\{i_t\})$.
Let $\Px$ denote the process with
conditional distribution $x_t \mid i_t$ uniform
over the spherical measure on $U_{\neg i_t} \cap \mathbb{S}^{n-1}$.
For any $T \geq 2$, the process $\Px$ satisfies the
$(T, 2, I_n / (4n-4), e, \frac 1 2)$-\TSB{} condition.
\end{restatable}
\noindent
In this example, a straightforward computation (detailed in
\Cref{stmt:mixing_time_simple_chain}) shows that
the mixing time $\tmix(\varepsilon)$
of the Markov chain $\{i_t\}_{t \geq 1}$ scales
as $\log_{n}(1/\varepsilon)$.\footnote{
For concreteness, given a discrete-time Markov chain
over a finite state-space $S$ with transition matrix $P$
and stationary distribution $\pi$, we define the mixing time as:
$\tmix(\varepsilon) := \inf\{ k \in \N \mid \sup_{\mu \in \calP(S)} \tvnorm{ \mu P^k - \pi} \leq \varepsilon \}$.
Here, $\calP(S)$ denotes the set of all probability distributions over $S$,
and $\tvnorm{\cdot}$ denotes the total variation norm over distributions.} 
In most analyses which rely on mixing time arguments, 
one requires that the mixing time resolution $\varepsilon$ tends to zero
as either the amount of data 
and/or probability of success increases;
as a concrete example, \cite[Eq.~3.2]{duchi2012ergodicmd}
suggests to set $\varepsilon = 1/\sqrt{T}$, where $T$ is the number of
samples drawn from the underlying distribution.
On the other hand, the trajectory small-ball
condition in \Cref{example:mixing_chain}
holds with a short excitation window of length $k=2$,
independently of $T$.

Next we consider linear dynamical systems.
As setup, we first define the notion of controllability for a pair
of dynamics matrices $(A, B)$:
\begin{mydef}[Controllability]
\label{def:controllability}
Let $(A, B)$ be a pair of matrices with $A \in \R^{n \times n}$ and $B \in \R^{n \times d}$. For $k \in \{1, \dots, n\}$, we say that 
$(A, B)$ is \emph{$k$-step controllable} if the matrix:
\begin{align*}
\begin{bmatrix} B & AB & A^2 B & \cdots & A^{k-1} B \end{bmatrix} \in \R^{n \times kd}
\end{align*}
has full row rank. 
\end{mydef}
The classical definition of controllability in linear
systems \citep[cf.][Chapter~25]{rugh1996linear}
is equivalent to $n$-step controllability. \Cref{def:controllability}
allows the system to be controllable in fewer than $n$ steps.
Also note that $k$ is restricted to $\{1, \dots, n\}$, since
if a system is not $n$-step controllable, it will not be
$n'$-step controllable for any $n' > n$ (by the Cayley-Hamilton theorem).
A few special cases of interest to note are as follows.
If $B$ has rank $n$, then $(A, B)$ is trivially one-step controllable
for any $A$.
On the other hand, if $(A, B)$ are in canonical controllable form
(i.e., $A$ is the companion matrix associated
with the polynomial $p(z) = a_0 + a_1 z + \dots + a_{n-1} z^{n-1} + z^n$
and $B$ is the $n$-th standard basis vector),
then $(A, B)$ is $n$-step controllable. The latter corresponds directly to
the state-space representation of 
autoregressive processes of order $n$, e.g.\ $\mathsf{AR}(n)$.

\begin{restatable}[Linear dynamical systems]{myexample}{examplelds}
\label{stmt:lds_traj_small_ball}
Let $(A, B)$ with $A \in \R^{n \times n}$ and $B \in \R^{n \times d}$ be
$\kcont$-step-controllable (\Cref{def:controllability}). 
Let $\PxA{A,B}$ be the linear dynamical system defined in \eqref{eq:lds_definition}.
Fix any $T,k \in \N_+$ satisfying $T \geq k \geq \kcont$. Then, $\PxA{A,B}$ satisfies the
$(T, k, \Gamma_k(A, B), e, \frac 1 2)$-\TSB{} condition.
\end{restatable}

In all of the examples so far,
the time-$t$ marginal distribution of covariates $x_t$ has either been
a multivariate Gaussian or a spherical measure. 
To underscore the generality of the small-ball method, we can create
additional examples where this is not the case.
In what follows, we consider Volterra series~\citep{mathews2000polynomialSP},
which generalize the classical Taylor series to causal sequences.
Analogous to how polynomials can approximate continuous functions arbitrarily
well on a compact set, Volterra series can approximate 
signals that depend continuously (and solely) on their history 
over a bounded set of inputs~\citep[cf.][Section~1.5]{rugh1981nonlinear}.
\begin{restatable}[Degree-$D$ Volterra series]{myexample}{examplegeneralvolterra}
\label{example:volterra}
Fix a $D \in \N_{+}$.
Let $\{c_{i_1, \dots, i_d}^{(d, \ell)}\}_{i_1,\dots,i_d \in \N}$
for $d \in \{1, \dots, D\}$ and $\ell \in \{1, \dots, n\}$ denote
arbitrary rank-$d$ arrays.
Let $\{w_t^{(\ell)}\}_{t \geq 0}$ be \IID\ $N(0, 1)$ random variables
for $\ell \in \{1, \dots, n\}$.
Consider the process $\Px$ where for $t \geq 1$, the $\ell$-th coordinate of $x_t$,
denoted $(x_t)_{\ell}$, is:
\begin{align}
    (x_t)_{\ell} = \sum_{d=1}^{D} \sum_{i_1,\dots,i_d=0}^{t-1} c_{i_1,\dots,i_d}^{(d,\ell)} \prod_{d'=1}^{d} w_{t-i_{d'}-1}^{(\ell)}. \label{eq:volterra_series}
\end{align}
Let $\Tnd := \inf\{ t \in \N_{+} \mid \det(\Gamma_t(\Px)) \neq 0 \}$,
and suppose $\Tnd$ is finite.
There is a constant $c_D > 0$, depending only on $D$, such that
for any $T \geq \Tnd$, $\Px$ satisfies the
$(T,T,\Gamma_T(\Px),c_D,1/(2D))$-\TSB{} condition.
\end{restatable}
The main idea behind \Cref{example:volterra} is that,
while $x_t$ is certainly not Gaussian, 
the quadratic form
$\sum_{t=1}^{T} \ip{v}{x_t}^2$ 
is a degree at most $2D$ polynomial in $\{w_t^{(\ell)}\}_{t=0}^{T-1}$.
It will hence exhibit anti-concentration, according to a landmark result from
\cite{carbery2001anticonc}.
The same result
actually provides an immediate extension of this
example---as well as the previous examples---to noise distributions
with log-concave densities, such as Laplace or uniform noise.

We next present a special case of the Volterra series, where
we can choose the excitation window $k$ in the small-ball definition strictly
between the endpoints $1$ and $T$.
To set up, a few more definitions are needed:
\begin{mydef}
\label{def:rank_d_array}
Fix an integer $d \in \N_{+}$.
A rank-$d$ array
of coefficients $\{c_{i_1,\dots,i_d}\}_{i_1,\dots,i_d \in \N}$ is called:
\begin{enumerate}[label=(\alph*)]
    \item \vocab{symmetric} if $c_{i_1, \dots, i_d} = c_{\pi(i_1, \dots, i_d)}$
    for any permutation $\pi$ of indices $i_1, \dots, i_d \in \N$,
    \item \vocab{traceless} if $c_{i,\dots,i} = 0$ for all $i \in \N$, and
    \item \vocab{non-degenerate} if there exists an $k_{\mathsf{nd}} \in \N_{+}$ such that
    the following set is non-empty:    
    \begin{align*}
        \{ (i_1, \dots, i_d) \mid c_{i_1, \dots, i_d} \neq 0, i_1, \dots, i_d \in \{0, \dots, k_{\mathsf{nd}} - 1\} \}.
    \end{align*}
\end{enumerate}
The smallest $k_{\mathsf{nd}}$ 
such that $\{c_{i_1,\dots,i_d}\}$ is the \vocab{non-degeneracy index}.
\end{mydef}
\begin{restatable}[Degree-$2$ Volterra series]{myexample}{exampletwovolterra}
\label{example:volterra_degree_two}
Consider the following process $\Px$.
Let $\{c_{i,j}^{(\ell)}\}_{i,j \geq 0}$ for $\ell \in \{1, \dots, n\}$ be
symmetric, traceless, non-degenerate arrays (\Cref{def:rank_d_array}).
Let $\{w_t^{(\ell)}\}_{t \geq 0}$ be \IID\ $N(0, 1)$ random variables
for $\ell \in \{1, \dots, n\}$.
For $t \geq 1$, the $\ell$-th coordinate of $x_t$, denoted $(x_t)_{\ell}$, is:
\begin{align}
    (x_t)_{\ell} = \sum_{i=0}^{t-1} \sum_{j=i}^{t-1} c_{i,j}^{(\ell)} w_{t-i-1}^{(\ell)} w_{t-j-1}^{(\ell)}. \label{eq:toy_vector_quadratic}
\end{align}
Let $k_{\mathsf{nd}} \in \N_{+}$ denote the smallest non-degeneracy index for all
$n$ arrays.
There is a universal positive constant $c$ such that
for any $T$ and $k$ satisfying $T \geq k \geq k_{\mathsf{nd}}$, 
$\Px$ satisfies the $(T,k,\Gamma_k(\Px),c,\frac{1}{4})$-\TSB{} condition.
\end{restatable}
The assumptions pulled in from \Cref{def:rank_d_array} help simplify the
construction of an almost sure lower bound for conditional covariances
$\E[ \frac{1}{k} \sum_{t=(j-1)k+1}^{jk} x_tx_t^\T \mid \calF_{(j-1)k}]$,
to establish that \Cref{example:volterra_degree_two} satisfies
the trajectory small-ball condition.
We believe that generalizations to higher degree Volterra series with $k$
strictly between $1$ and $T$
are possible by more involved calculations.

Of course, many other examples are possible. To help in recognizing them,
the following statement
shows that condition \eqref{eq:trajectory_small_ball}
in the trajectory small-ball definition can be verified by
separately establishing small-ball probabilities
for the conditional distributions:

\begin{restatable}[Average small-ball implies trajectory small-ball]{myprop}{avgsmallballimpliesblock}
\label{stmt:avg_small_ball_implies_block}
Fix $T \in \N_+$, $k \in \{1, \dots, T\}$, $\{\Psi_j\}_{j=1}^{\floor{T/k}} \subset \sfSym^{n}_{> 0}$, and
$\alpha,\beta \in (0, 1)$.
Let $\Px$ be a covariate distribution, with $\{x_t\}_{t \geq 1}$ adapted
to a filtration $\{\calF_t\}_{t \geq 1}$. Suppose
for all $v \in \R^n \setminus \{0\}$ and  
$j \in \{1, \dots, \floor{T/k}\}$:
\begin{align}
    \frac{1}{k} \sum_{t=(j-1)k+1}^{jk} \Pr_{x_t \sim \Px}\left\{ \ip{v}{x_t}^2 \leq \alpha \cdot v^\T \Psi_j v ~\Big|~ \calF_{(j-1)k} \right\} \leq \beta \:\: \textrm{a.s.}, \label{eq:avg_weak_small_ball_inequality}
\end{align}
where $\calF_0$ is the minimal $\sigma$-algebra.
Then, 
for all $v \in \R^n \setminus \{0\}$,
$j \in \{1, \dots, \floor{T/k}\}$, and $\e \in (0, \alpha)$
\begin{align}
    \Pr_{\{x_t\} \sim \Px}\left\{
      \frac{1}{k} \sum_{t=(j-1)k+1}^{jk} \ip{v}{ x_t}^2 \leq \e \cdot v^\T \Psi_j v
      ~\Bigg|~ \calF_{(j-1)k} \right\} \leq \frac{\beta}{1-\e/\alpha}\:\:
      \textrm{a.s.} \label{eq:avg_weak_small_ball_implication}
\end{align}
\end{restatable}
\noindent
An immediate corollary of \Cref{stmt:avg_small_ball_implies_block} is the following:
suppose that for all 
$v \in \R^n \setminus \{0\}$, 
$j \in \{1, \dots, \floor{T/k}\}$, 
and $\e > 0$, 
\begin{align}
    \frac{1}{k} \sum_{t=(j-1)k+1}^{jk} \Pr_{x_t \sim \Px}\left\{ \ip{v}{x_t}^2 \leq \varepsilon \cdot v^\T \Psi_j v ~\Big|~ \calF_{(j-1)k} \right\} \leq (\csb \varepsilon)^\alpha \:\: \textrm{a.s.} \label{eq:avg_small_ball_inequality}
\end{align}
Then, 
the $(T, k, \{\Psi_j\}_{j=1}^{\floor{T/k}}, 2^{1+1/\alpha} \csb, \alpha)$-\TSB\ condition holds.
Equation~\eqref{eq:avg_small_ball_inequality} can be easier to verify
than \eqref{eq:trajectory_small_ball}, since the former allows
one to reason about each conditional distribution individually,
whereas the latter requires reasoning about the entire excitation
window altogether.

The following two sections present upper and lower bounds for
learning from trajectories, involving various instances of the
trajectory small-ball assumption where applicable.
All main results are summarized in \Cref{tab:results}.

\section{Risk upper bounds}
\label{sec:results:upper}
\label{sec:results:upper_bounds}

The trajectory small-ball definition 
allows us to carve out conditions for learnability.
A key quantity for what follows is the minimum eigenvalue of
the ratio of two positive definite matrices:
\begin{align}
    \ulam(A, B) := \lambda_{\min}(B^{-1/2} A B^{-1/2}), \quad A,B \in \sfSym^n_{> 0}. \label{eq:ulam_defn}
\end{align}

Our various upper bounds statements build on the following general lemma:
\begin{restatable}[General OLS upper bound]{mylemma}{upperboundgeneral}
\label{stmt:upper_bound_general}
There are universal positive constants $c_0$ and $c_1$ such that the following holds.
Suppose that $\Px$ satisfies
the $(T,k,\{\Psi_j\}_{j=1}^{\floor{T/k}},\csb,\alpha)$-\TSB{} condition
(\Cref{def:trajectory_small_ball}).
Put $S := \floor{T/k}$ and $\Gamma_T := \Gamma_T(\Px)$.
Fix any $\uGam \in \sfSym^{n}_{> 0}$ satisfying
$\frac{1}{S} \sum_{j=1}^{S} \Psi_j \preccurlyeq \uGam \preccurlyeq \Gamma_T$,
and let $\uMu(\{\Psi_j\}_{j=1}^{S}, \uGam)$
denote the geometric mean of the minimum eigenvalues $\{\ulam(\Psi_j, \uGam)\}_{j=1}^{S}$, i.e.,
\begin{align}
    \uMu( \{\Psi_j\}_{j=1}^{S}, \uGam)  := \left[\prod_{j=1}^{S} \ulam(\Psi_j, \uGam) \right]^{1/S}. \label{eq:uMu_defn}
\end{align}
Suppose that:
\begin{align}
    n \geq 2, \quad \frac{mT}{kn} \geq \frac{c_0}{\alpha} \log\left(\frac{\max\{e,\csb\}}{ \alpha \ulam(\uGam, \Gamma_T) \uMu(\{\Psi_j\}_{j=1}^{S}, \uGam)}\right). \label{eq:mTkn_requirements}
\end{align}
Then, for any $\Gamma' \in \sfSym^n_{> 0}$:
\begin{align}
\E[ \norm{ \hat{W}_{m,T} - W_\star }^2_{\Gamma'} ]
  \leq c_1 \csb \sigma_\xi^2
    \cdot \frac{pn}{ mT \alpha  \ulam(\uGam,\Gamma')  \uMu(\{\Psi_j\}_{j=1}^{S}, \uGam) }
    \cdot \log\left(\frac{ \max\{e,\csb\}}{\alpha \ulam(\uGam, \Gamma_T) \uMu(\{\Psi_j\}_{j=1}^{S}, \uGam)}\right)
    .
    \label{eq:general_risk_bound}
\end{align}
\end{restatable}
The proof of \Cref{stmt:upper_bound_general}
blends ideas from the
analysis of random design linear regression
\citep{hsu14randomdesign,oliveria2016lowertail,mourtada19exactminimax}
with techniques from linear system identification with full state observation~\citep{simchowitz18learning,sarkar2019sysid,faradonbeh2018unstable,dean2020lqr}.
Note that \Cref{stmt:upper_bound_general} makes no explicit
assumptions on the ergodicity of the process $\Px$.
The role of $\Px$ is instead succinctly captured 
by the trajectory small-ball condition,
together with the minimum eigenvalue quantities that appear in the bound.
The proof of \Cref{stmt:upper_bound_general} also yields, with some
straightforward modifications, bounds on
the risk that hold with high probability; we only present
bounds in expectation for simplicity.
Finally, if the square norm $\norm{X}_{M}^2$ is defined to be
$\lambda_{\max}(X M X^\T)$ instead of $\Tr(X M X^\T)$, then
\eqref{eq:general_risk_bound} holds with the expression $p + n$ replacing $pn$ in the numerator.

\begin{table}
\centering
\begin{tabular}{||c c c c c||} 
 \hline
 \textbf{Problem} & \textbf{Many?} & \textbf{Upper} & \textbf{Lower} & \textbf{Assumptions} \\
 \hline\hline
 \DSR & Y & $\frac{pn}{mT}$ & $\checkmark$ & \TSB $^\textrm{(a)}$\\ 
 \DSR & N & $\frac{pn}{mT}$ & $e^{-T/(\tmix n)} + \frac{pn}{T}$ & Ergodicity of covariates$^\textrm{(b)}$\\
 \LDSLS & Y & $\frac{pn}{mT}$ & $\checkmark$ & $\kcont$-step controllability$^\textrm{(c)}$\\
 \LDSLS & N & $\condNum \frac{p n^2}{m^2 T}$ & $\frac{pn^2}{m^2T}$ & Marginal stability, etc.$^\textrm{(d)}$ \\
 \PSR & N & $\frac{pn}{mT}$ & $\checkmark$ & Small-ball, poly variance growth$^\textrm{(e)}$ \\
 \PLR & N & $\condNum \frac{pn}{mT}$ & $\frac{pn}{mT}$ & Diagonalizable$^\textrm{(f)}$ \\
 \hline\hline
 \Sysid & Y & $\frac{n^2}{mT \lambda_{\min}(\Gamma_T)}$ & - & Same as \LDSLS\ (Y)$^\textrm{(g)}$ \\ 
 \Sysid & N & $\condNum \frac{n^2}{mT \lambda_{\min}(\Gamma_{mT/n})}$ & $\frac{n^2}{T^2}$ & Same as \LDSLS\ (N)$^\textrm{(h)}$ \\
 \hline
\end{tabular}
\captionsetup{singlelinecheck=off}
\caption[]{Summary of main results
presented in \Cref{sec:results:upper_bounds}
and \Cref{sec:results:lower_bounds}.
All upper/lower bounds shown suppress constant and polylogarithmic factors.
The \textbf{Many?}~column indicates whether the bounds apply
in the many trajectories regime (Y) where $m \gtrsim n$,
or the few trajectories regime (N) where $m \lesssim n$.
A checkmark ($\checkmark$) in the \textbf{Lower} column indicates
that the lower bound matches the upper bound, up to polylogarithmic factors.
The \Sysid\ problem is classic linear system identification: recover the unknown dynamics
matrix $A$ from linear dynamical trajectories,
with error measured in squared Frobenius norm (see \Cref{eq:lds_definition} and the discussion
at the end of \Cref{sec:problem:lds-trajectories}).
Elaborations on assumptions:
\begin{enumerate}[label=\textbf{(\alph*)}]
    \item Upper bound follows from \Cref{stmt:upper_bound_seq_ls_many_traj},
    treating as $O(1)$ all constants related to trajectory small-ball (\Cref{def:trajectory_small_ball}).
    Lower bound follows from \IID\ linear regression being a special case.
    \item Upper bound follows from 
    combining (i) \Cref{stmt:upper_bound_main_high_prob},
    a general OLS high probability upper bound that
    utilizes a simple modification (\Cref{def:weak_trajectory_small_ball}) to our trajectory small-ball definition,
    with (ii) \Cref{stmt:phi_mixing_implies_weak_small_ball},
    which shows that a $\phi$-mixing (\Cref{def:phi_mixing})
    covariate process (where the marginal distributions also fulfill
    a weak small-ball condition \eqref{eq:weak_small_ball_marginal})
    satisfies our modified trajectory small-ball condition.
    The upper bound holds in high probability instead of in expectation,
    and requires a burn-in time that satisfies $T \gtrsim \tmix n \log(1/\delta)$,
    where $\delta$ denotes the failure probability. 
    The lower bound is from 
    \cite[Theorems 1 and 3]{bresler2020leastsquaresmarkov}, and holds
    for the single trajectory ($m=1$).
    \item Upper bound follows from \Cref{stmt:upper_bound_many_trajectories}; see \Cref{def:controllability} for definition of $\kcont$-step controllability. Lower bound follows again from \IID\ linear
    regression being a special case.
    \item Upper bound follows from 
    \Cref{thm:sublinear_trajectories_bound}, under
    \Cref{assume:marginal_stability} (marginal stability), 
    \Cref{assume:diagonalizability} (diagonalizability), and
    \Cref{assume:one_step_controllability} (one-step controllability).
    The condition number of the diagonalizing factor is denoted by
    $\condNum$ (\Cref{def:condition_number}).
    Lower bound follows from \Cref{stmt:lower_bound_main},
    and is realized by a decoupled noise sequence (\Cref{def:gaussian_observation_noise}).
    \item Upper bound follows from \Cref{stmt:upper_bound_ind_seq_ls},
    treating as $O(1)$ constants relating to small-ball (\Cref{eq:small_ball_independent})
    and variance growth (\Cref{eq:variance_growth_condition}).
    The necessity of the variance growth condition
    is 
    shown in \Cref{stmt:ind_seq_ls_lower_bound}. 
    Note that in the many trajectories regime, the 
    \DSR\ upper bound applies.
    Lower bound again follows from \IID\ linear regression.
    \item Upper bound follows from
    \Cref{stmt:upper_bound_ind_lds_ls}; 
    $\condNum$ is the condition number of the diagonalizing factor (\Cref{def:condition_number}).
    Lower bound again follows from \IID\ linear regression.
    \item Upper bound follows from \Cref{stmt:upper_bound_parameter_recovery};
    $\Gamma_T$ is the $T$-step average covariance matrix (\Cref{eq:covariance-def}).
    \item 
    Upper bound follows from \Cref{stmt:upper_bound_parameter_recovery_few_trajs};
    $\condNum$ is the
    condition number of the diagonalizing factor (\Cref{def:condition_number}),
    and $\Gamma_{mT/n}$ is the $mT/n$-step average covariance matrix (\Cref{eq:covariance-def}).
    Lower bound applies to the single trajectory ($m=1$) setting and follows from \cite[Theorem 2.3]{simchowitz18learning}.
    (Technically, their bound applies to the operator, instead of 
    Frobenius norm, but the proof can be adjusted to apply.)
    Specializing the upper bound to one trajectory ($m=1$), drawn from the lower bound's hard instance,
    implies a gap of $n^3/T^2$ (upper) versus $n^2/T^2$ (lower)
    as noted in \cite[Section 2.2]{simchowitz18learning}.
\end{enumerate}
}
\label{tab:results}
\end{table}

As long as the process $\Px$ satisfies the trajectory small-ball condition
with excitation window $k=T$,
\Cref{stmt:upper_bound_general} (with $\Psi_1 = \uGam = \Gamma_T(\Px)$)
immediately yields the following
result for learning from many trajectories in the \DSR\ problem:
\begin{restatable}[Upper bound for \DSR{}, many trajectories]{mythm}{upperboundseqlsmanytraj}
\label{stmt:upper_bound_seq_ls_many_traj}
There are univeral positive constants $c_0$ and $c_1$
such that the following holds.
Suppose that $\Px$ satisfies
the trajectory small-ball condition (\Cref{def:trajectory_small_ball})
with parameters $(T, T, \Gamma_T(\Px), \csb, \alpha)$.
If:
\begin{align*}
    n \geq 2, \quad m \geq  \frac{c_0 n}{\alpha}\log\left(\frac{\max\{e,\csb\}}{\alpha}\right), 
\end{align*}
then, for any $\Gamma' \in \sfSym^n_{> 0}$:
\begin{align}
    \E[\norm{\hat{W}_{m,T} - W_\star}^2_{\Gamma'}] \leq c_1 \csb \sigma_\xi^2 \cdot \frac{pn}{mT\alpha \ulam(\Gamma_T(\Px), \Gamma')} \cdot \log\left(\frac{\max\{e, \csb\}}{\alpha}\right). \label{eq:T_T_traj_small_ball}
\end{align}
\end{restatable}
This result provides the upper bound for the summary statement \Cref{thm:simple-rate-many-traj-in-window}.
To interpret the bound \eqref{eq:T_T_traj_small_ball}, 
suppose that $\csb$ and $\alpha$ are universal constants.
Then, the requirement on $m$ simplifies to $m \gtrsim n$.
Under any strict evaluation horizon $\Tnew \leq T$, taking $\Gamma' = \Gamma_{\Tnew}(\Px)$, 
the risk $\E[L(\hat{W}_{m,T};\Tnew,\Px)]$ scales as $\sigma_\xi^2 pn / (mT)$.
The lower bound for \Cref{thm:simple-rate-many-traj-in-window} follows from the fact that \IID\ linear regression is a special case of \DSR{}.

Meanwhile, to obtain guarantees for parameter recovery,
consider taking $\Gamma' = I_n$. Then \Cref{stmt:upper_bound_seq_ls_many_traj}
implies that the parameter error
$\E[\norm{\hat{W}_{m,T}-W_\star}_F^2]$
scales as $\sigma_\xi^2 pn / [mT \cdot \lambda_{\min}(\Gamma_T(\Px))]$.
Note that operator norm bounds on parameters also hold,
with the expression $p+n$ replacing $pn$ in the bound.

\Cref{stmt:upper_bound_general} also 
yields a bound for \PSR, assuming polynomial growth
of the time-$t$ covariances $\Sigma_t$ \eqref{eq:covariance-def}.
To state the result, let $\phi : [1, \infty) \times [0, \infty) \rightarrow [1, \infty)$ be defined as:
\begin{align}
    \phi(a, x) := \begin{cases} 1 &\text{if } x \leq 1, \\
    ax &\text{otherwise.}
    \end{cases}
\end{align}
Note that $1 \leq \phi(a, x) \leq \max\{ax, 1\}$.
\begin{restatable}[Upper bound for \PSR{}]{mythm}{upperboundindseqls}
\label{stmt:upper_bound_ind_seq_ls}
There are universal positive constants $c_0$ and $c_1$ 
such that the following holds.
Fix any sequence of distributions $\{\Pxt{t}\}_{t \geq 1}$,
and let $\Sigma_t := \E_{x_t \sim \Pxt{t}}[x_tx_t^\T]$ for $t \in \N_{+}$.
Suppose there exists $\csb > 0$ and $\alpha \in (0, 1]$ such 
that for all $v \in \R^n \setminus \{0\}$, $\varepsilon > 0$ and $t \in \N_{+}$:
\begin{align}
    \Pr_{x_t \sim \Pxt{t}}\left\{ \ip{v}{x_t}^2 \leq \varepsilon \cdot v^\T \Sigma_t v \right\} \leq (\csb \varepsilon)^{\alpha}. \label{eq:small_ball_independent}
\end{align}
Furthermore, suppose there exists a 
$c_\beta \geq 1$ and $\beta \geq 0$ such that
for all $s, t \in \N_{+}$ satisfying $s \leq t$:
\begin{align}
    \frac{1}{\ulam(\Sigma_s, \Sigma_t)} \leq c_\beta (t/s)^\beta. \label{eq:variance_growth_condition}
\end{align}
If:
\begin{align*}
    n \geq 2, \quad mT \geq \frac{c_0 n}{\alpha} \left( \beta + \log\left(\frac{\max\{e,\csb\} c_\beta}{\alpha}\right) \right),
\end{align*}
then, for $\Px = \otimes_{t \geq 1} \Pxt{t}$:
\begin{align}
    \E[L(\hat{W}_{m,T};\Tnew, \Px)] \leq c_1 \csb \sigma_\xi^2 c_\beta e^\beta \cdot \frac{pn}{mT\alpha} \cdot \phi\left(c_\beta(\beta+1), 
    (\Tnew/T)^\beta
    \right) \left[ \beta + \log\left(\frac{\max\{e,\csb\}c_\beta}{\alpha}\right)\right]. \label{eq:risk:ind_seq_ls}
\end{align}
\end{restatable}

Consider specializing \Cref{stmt:upper_bound_ind_seq_ls}
to the case when $\Sigma_t = \Sigma$ for all $t \in \N_{+}$.
Doing so yields random design linear regression from $mT$
covariates drawn \IID\ from $\Pxt{1}$.
The growth condition \eqref{eq:variance_growth_condition} is trivially
satisfied with $c_\beta = 1$ and $\beta = 0$.
The small-ball assumption \eqref{eq:small_ball_independent}
simplifies to $\Pr_{x_1 \sim \Pxt{1}}\{ \abs{\ip{v}{x_1}} \leq \varepsilon \norm{v}_{\Sigma} \} \leq (\sqrt{\csb} \varepsilon)^{2\alpha}$ for all $v \neq 0$ and $\varepsilon > 0$,
which matches \cite[Assumption~1]{mourtada19exactminimax}
up to a minor redefinition of the constants $\csb, \alpha$.
Treating $\csb$ and $\alpha$ as constants,
the conclusion of \Cref{stmt:upper_bound_ind_seq_ls}
in this setting is that 
$\E[\norm{\hat{W}_{m,T} - W_\star}_{\Sigma}^2] \lesssim \sigma_\xi^2 pn/(mT)$ as long as $n \geq 2$ and $mT \gtrsim n$,
which recovers \cite[Proposition~2]{mourtada19exactminimax}.

On the other hand, \Cref{stmt:upper_bound_ind_seq_ls}
does not require that the covariates are drawn \IID\ from the 
same distribution, 
allowing the time-$t$ covariances $\Sigma_t$ to grow polynomially.
As an example, suppose that $\Sigma_t = t^\beta \cdot I_n$ for some $\beta > 0$.
In this case, $1/\ulam(\Sigma_s, \Sigma_t) = (t/s)^\beta$, so we can take 
$c_\beta=1$ in \eqref{eq:variance_growth_condition}.
Again treating $\csb$ and $\alpha$ as constants and taking $\Tnew \leq T$, we have
$\E[ L(\hat{W}_{m,T}; \Tnew, \Px) ] \lesssim \sigma_\xi^2 \beta e^\beta \cdot pn/(mT)$
as long as $mT \gtrsim \beta n$.
If $\beta$ is also considered a constant, then we further
have the risk bound
$\E[ L(\hat{W}_{m,T}; \Tnew, \Px) ] \lesssim \sigma_\xi^2 pn/(mT)$.
This matches the minimax rate for \IID\ linear regression.

It is natural to ask
if the covariance growth condition \eqref{eq:variance_growth_condition}
is needed under strict evaluation horizons $\Tnew \le T$.\footnote{Some regularity
is needed when under extended evaluations $\Tnew > T$, otherwise the risk could be arbitrarily large.}
In \Cref{sec:results:lower_bounds}, we show that
if the covariances are set to $\Sigma_t = 2^t \cdot I_n$
and $\Pxt{t} = N(0, \Sigma_t)$ (satisfying \eqref{eq:small_ball_independent}),
then the minimax risk $\mathsf{R}(m, T, T; \{\otimes_{t \geq 1} \Pxt{t}\})$ 
must scale at least $2^{cn/m}/T$ whenever $m \lesssim n$, for some positive constant $c$.
Sub-exponential
growth rates are therefore necessary for polynomial sample complexity.
Determining the optimal dependence of $\beta$
in \eqref{eq:risk:ind_seq_ls} is left to future work.

Note that \Cref{stmt:upper_bound_ind_seq_ls} is most interesting either when
trajectories are few ($m \lesssim n$)
or evaluations are extended ($\Tnew > T$).
When
$m \gtrsim n$ and $\Tnew \leq T$, one can usually apply
\Cref{stmt:upper_bound_seq_ls_many_traj}
with $\Gamma' = \Gamma_{\Tnew}(\Px)$ instead,
and avoid placing any requirements on the growth of covariances.

Considering any of the small-ball examples in \Cref{sec:results:upper:small_ball_examples},
recall that when the excitation window $k$ and the horizon $T$ are equal,
\Cref{stmt:upper_bound_seq_ls_many_traj} 
provides an upper bound on the risk of OLS estimation
for the corresponding \DSR\ problem.
Specifically, for \Cref{example:simple} 
and \Cref{stmt:lds_traj_small_ball} with $k=T$, 
if $\Tnew = T$ and trajectories are abundant ($m \gtrsim n$), then
the OLS estimator's rate $\sigma_\xi^2 pn/(mT)$
matches its behavior in \IID\ linear regression.
Meanwhile, for the degree-$D$ Volterra series (\Cref{example:volterra}),
we require that $m \gtrsim c_D \cdot n $,
and the OLS risk bound scales as $\sigma_\xi^2 c'_D \cdot pn/(mT)$,
for constants $c_D$ and $c'_D$ that only depend on $D$.

In order to cover scenarios in which trajectories may be relatively scarce,
namely $m \lesssim n$, we need additional structure.
More technically, when the small-ball condition is satisfied with 
$k < T$, one needs to further control the various eigenvalues
that appear in \Cref{stmt:upper_bound_general} in order to bound the risk of
OLS.
Specifically for \PSR{}, a covariate growth assumption suffices: \Cref{example:independent_gaussians}
combined with \Cref{stmt:upper_bound_ind_seq_ls} yields an OLS risk bound.
Furthermore, both \Cref{example:partition} and \Cref{example:mixing_chain}
can be immediately combined with \Cref{stmt:upper_bound_general}, 
since the matrices $\Psi_j$ in these examples are bounded above and
below by $\Gamma_T(\Px)$ up to universal constant factors.
But arbitrarily large risk can still be realized in the general \DSR{} problem, even
when the trajectory small-ball condition is satisfied.
To study the behavior of OLS across all regimes of trajectory count $m$, example dimensions $p$ and $n$,
and trajectory lengths $T$ and $\Tnew$,
we focus specifically on linear dynamical systems and the \DLR{} problem
for our remaining upper bounds.

\subsection{Upper bounds for linear dynamical system}
\label{sec:upper_bounds:lds}

In this section, we focus exclusively on dynamics $\PxA{A,B}$ described
by a linear dynamical system \eqref{eq:lds_definition}.
As discussed previously, in order to apply 
\Cref{stmt:upper_bound_general} in the few trajectories regime
when $m \lesssim n$ (or when $m \gtrsim n$ and $\Tnew > T$), we must
(a) show that the process $\PxA{A,B}$ satisfies the 
trajectory small-ball condition, and (b) bound the various
eigenvalues which appear in \Cref{stmt:upper_bound_general}.
\Cref{stmt:lds_traj_small_ball} establishes that $\PxA{A,B}$
satisfies the $(T, k, \Gamma_k(A, B), e, \frac 1 2)$-\TSB{}
condition, as long as $(A, B)$ is $\kcont$-step controllable
and $k \geq \kcont$, thus taking care of (a). To handle (b),
we introduce additional assumptions on the dynamics matrices $(A, B)$:

\begin{myassump}[Marginal instability]
\label{assume:marginal_stability}
The dynamics matrix $A$ in \LDSLS{} is \emph{marginally unstable}.
That is, $\rho(A) \leq 1$, where $\rho(A)$ denotes the spectral radius of $A$.
\end{myassump}

\begin{myassump}[Diagonalizability]
\label{assume:diagonalizability}
The dynamics matrix $A$ in \LDSLS{} is \vocab{complex diagonalizable} as $A = S D S^{-1}$,
where $S \in \C^{n \times n}$ is invertible
and $D \in \C^{n \times n}$ is a diagonal matrix comprising the eigenvalues of $A$. 
\end{myassump}

\begin{myassump}[One-step controllability]
\label{assume:one_step_controllability}
The control matrix $B$ in \LDSLS{} has full row rank,
i.e., $\rank(B) = n$. Equivalently, the pair $(A, B)$ is 
one-step controllable (\Cref{def:controllability}).
\end{myassump}

Assumption~\ref{assume:marginal_stability} is fairly standard
in the literature. Going beyond the regime
$\rho(A) = 1 + \varepsilon$, where $\varepsilon \lesssim 1/T$,
requires additional technical assumptions
on the dynamics matrix $A$ that we choose to avoid in the interest
of simplicity; the OLS estimator is in general not a consistent estimator when $\rho(A) > 1$ and $m=1$ (cf.~\cite{phillips2013inconsistent,sarkar2019sysid}).
The condition $\rho(A) \leq 1$ is often referred to as
\vocab{marginal stability} in other work. We choose to
call it marginally \emph{unstable} instead, to emphasize the fact that 
such systems, namely at $\rho(A)=1$, may not be ergodic and that the state can grow 
unbounded (e.g.\ have magnitude roughly $t^n$ at time $t$).

Diagonalizability (\Cref{assume:diagonalizability})
is less standard in the literature. 
We use it together with \Cref{assume:marginal_stability}
and \Cref{assume:one_step_controllability} to establish that
$\ulam(k, t; A, B) := \ulam(\Gamma_k(A, B), \Gamma_t(A, B)) \gtrsim c \cdot k/t$ whenever $k \leq t$, where $c$ is a constant that depends only on $A$ and $B$
(and not $k$ and $t$).
In previous work on linear system identification,  
the term $\ulam(k, t; A, B)$ only appears under a logarithm, and so
coarser analyses in the general case can
still establish polynomial rates
(cf.\ \cite[Proposition~A.1]{simchowitz18learning} and
\cite[Proposition~7.6]{sarkar2019sysid}).\footnote{Note, however,
that without diagonalizability,
\cite[Corollary~A.2]{simchowitz18learning} can only guarantee a $\sqrt{n^2/T}$ rate
for the operator norm of the parameter error in general, and this is likely not optimal.}
However, by allowing
for evaluation lengths $\Tnew > 1$, 
the dependence on $\ulam(k, t; A, B)$ is no longer entirely confined under a logarithm
(cf.~\Cref{stmt:upper_bound_general}).
A sharp characterization is hence critical for deriving optimal rates.
In \Cref{sec:beyond-diag-results}, we conjecture the correct scaling of
$\ulam(k, t; A, B)$ as a function of the ratio $k/t$ and the largest Jordan block size of $A$,
based on numerical simulation.

One-step controllability (\Cref{assume:one_step_controllability})
is also an assumption commonly made in linear system identification.
It is clear that some form of controllability is needed, otherwise
learning may be impossible (e.g.~consider the extreme case of $B=0$).
General multi-step controllability does not suffice either:
\cite[Theorem~2]{tsiamis2021lowerbounds} show that under a single trajectory ($m=1$),
$n$-step controllability (where $n$ remains the state dimension) does
not ensure finite risk, and even a more robust controllability definition \citep[Definition~3]{tsiamis2021lowerbounds} 
cannot ensure risk bounds better than exponential in the dimension $n$.
Considering these barriers, we simply choose to rely on one-step
controllability in the few-trajectory setting ($m \lesssim n$).

Finally, we introduce a condition number quantity that will feature commonly in our bounds:
\begin{mydef}
\label{def:condition_number}
For dynamics matrices $(A, B)$ in \LDSLS{}
satisfying \Cref{assume:diagonalizability}
and \Cref{assume:one_step_controllability},
the \vocab{condition number} $\condNum(A, B)$ is defined as:
$\condNum(A, B) := \frac{\lambda_{\max}(S^{-1} BB^\T S^{-*})}{\lambda_{\min}(S^{-1} BB^\T S^{-*})}$.
\end{mydef}

\subsubsection{Many trajectories}

Our first result instantiates \Cref{stmt:upper_bound_seq_ls_many_traj}
in the special case of $\Gamma' = I_n$, which yields a sharp bound for
parameter recovery without requiring stability of
the dynamics matrix $A$:
\begin{restatable}[Parameter recovery upper bound for \LDSLS{}, many trajectories]{mythm}{upperboundparameterrecovery}
\label{stmt:upper_bound_parameter_recovery}
There are universal positive constants $c_0$ and $c_1$
such that the following holds for any instance of \LDSLS{}.
Suppose that $(A, B)$ is $\kcont$-step controllable,
If $n \geq 2$, $m \geq c_0 n$, and $T \ge \kcont$, then:
\begin{align*}
    \E[\norm{\hat{W}_{m,T} - W_\star}_F^2]
    \leq c_1 \sigma_\xi^2
        \cdot \frac{pn}{m T \cdot \lambda_{\min}(\Gamma_T(A, B))}.
\end{align*}
\end{restatable}

\Cref{stmt:upper_bound_parameter_recovery} 
improves on existing linear system identification results in the following way:
it replaces stability assumptions on the dynamics matrix $A$
with a simpler assumption of relatively many trajectories ($m \gtrsim n$),
and it guarantees a rate that is inversely proportional
to the \emph{total} number of examples $mT$ instead of 
only one example per trajectory.
In other words, our analysis does not need to ``discard'' the data within a trajectory,
which is the case in~\cite[Proposition~1.1]{dean2020lqr}.
Additionally, although OLS is generally not a consistent estimator
from one trajectory ($m=1$) if the dynamics $A$ are unstable,
the results of \cite{dean2020lqr} imply consistency as $m \to \infty$,
i.e., that $\hat{W}_{m,T}$ converges in probability to $W_\star$ as $m \to \infty$.
\Cref{stmt:upper_bound_parameter_recovery}
adds that, provided $m \gtrsim n$,
OLS is consistent under
unstable systems as $T \to \infty$ as well,
even if the trajectory count $m$ remains finite.
We will return to parameter recovery from relatively few trajectories $(m \lesssim n)$ by this section's end.

We now look beyond an evaluation horizon of length one, and
consider the setting with many trajectories ($m \gtrsim n$).
As noted previously, in order to handle an 
arbitrary evaluation horizon $\Tnew$ (in particular those that extend past the training horizon $T$), 
some constraint on the admissible dynamics matrices 
is needed to ensure that the minimax risk remains finite.
Without assumptions, the quantity $\ulam(\Gamma_T(A, B), \Gamma_{\Tnew}(A, B))$,
whose inverse inevitably bounds the risk \eqref{eq:risk_def} from below,
can be arbitrarily small
whenever $\Tnew > T$, resulting in arbitrarily large risk.
We will use our stated assumptions from the beginning of this
section.
The following specializes \Cref{stmt:upper_bound_seq_ls_many_traj}
to \DLR{}:
\begin{restatable}[Risk upper bound for \LDSLS{}, many trajectories]{mythm}{upperboundmanytrajectories}
\label{stmt:upper_bound_many_trajectories}
There are universal positive constants $c_0$ and $c_1$ such that the following holds for any instance of \LDSLS{}.
Suppose that $(A, B)$ is $\kcont$-step controllable.
If $n \geq 2$, $m \geq c_0 n$, $T \geq \kcont$, and the evaluation horizon is strict ($\Tnew \leq T$), then:
\begin{align*}
    \E[ L(\hat{W}_{m,T};\Tnew, \PxA{A,B})]
    \leq c_1 \sigma_\xi^2
        \cdot \frac{p n}{m T}.
\end{align*}
On the other hand, suppose that $(A, B)$ satisfies 
\Cref{assume:marginal_stability},
\Cref{assume:diagonalizability}, and
\Cref{assume:one_step_controllability},
with $\condNum := \condNum(A, B)$ (\Cref{def:condition_number}). 
If $n \geq 2$,
$m \geq c_0 n$, and the evaluation horizon is extended
($\Tnew > T$), then:
\begin{align*}
    \E[ L(\hat{W}_{m,T};\Tnew, \PxA{A,B})]
    \leq c_1 \sigma_\xi^2
        \cdot \frac{p n}{m T}
        \cdot \condNum\frac{\Tnew}{T}.
\end{align*}
\end{restatable}
Setting $\Tnew = T$, 
\Cref{stmt:upper_bound_many_trajectories} states that the risk of
\DLR{} in the many trajectories regime satisfies
$\E[ L(\hat{W}_{m,T};T,\PxA{A,B}) ] \lesssim \sigma_\xi^2 pn/(mT)$.
This rate matches the corresponding
independent baseline \PLR{} in the many trajectories regime.
To see this, first observe that the marginal distribution
$\PxAt{A,B}{t}$ at time $t \in \N_{+}$ is $N(0, \Sigma_t(A, B))$.
Hence, the covariate distribution for \PLR{} corresponds
to the product distribution $\otimes_{t \geq 1} N(0, \Sigma_t(A, B))$,
which is an instance of a Gaussian process.
Therefore, \Cref{example:gaussian_processes} combined with 
\Cref{stmt:upper_bound_seq_ls_many_traj}
yields that
the \PLR{} problem also has a risk bound
that scales as $\sigma_\xi^2 pn/(mT)$ whenever $m \gtrsim n$.
Put differently, the dependent structure of the covariate distribution $\PxA{A,B}$
in \DLR{}
does not add any statistical overhead to the learning problem
(compared to the independent learning problem \PLR{}), as long as
$m \gtrsim n$.

\subsubsection{Few trajectories}
\label{sec:upper_bounds:lds:few_trajectories}

We now cover the regime in which relatively few training trajectories are available ($m \lesssim n$).
Our first result bounds the OLS risk for the \DLR{} problem:
\begin{restatable}[Risk upper bound for \LDSLS{}, few trajectories]{mythm}{fewtrajrate}
\label{thm:sublinear_trajectories_bound}
There are universal positive constants $c_0$, $c_1$, and $c_2$ such that the following holds for any instance of \LDSLS{}.
Suppose that $(A, B)$
satisfies 
\Cref{assume:marginal_stability},
\Cref{assume:diagonalizability}, and
\Cref{assume:one_step_controllability},
with $\condNum := \condNum(A, B)$ (\Cref{def:condition_number}). 
If $n \geq 2$, $m \leq c_0 n$, and $mT \geq c_1 n\log(\max\{\condNum n/m, e\})$, then:
\begin{align*}
    \E[L(\hat{W}_{m,T}; \Tnew, \PxA{A,B})]
    \leq c_2 \sigma_\xi^2  
        \cdot \frac{p n \log(\max\{\condNum n/m, e\})}{mT} 
        \cdot \phi\left( \gamma,  \frac{c_1 n \log(\max\{\condNum n/m, e\})}{m} \cdot \frac{\Tnew}{T} \right).
\end{align*}
\end{restatable}
To interpret \Cref{thm:sublinear_trajectories_bound},
consider $\condNum$ a constant and suppose that $\Tnew = T$.
Then \Cref{thm:sublinear_trajectories_bound} states that
$\E[L(\hat{W}_{m,T};T,\PxA{A,B})] \lesssim \sigma_\xi^2 \cdot pn/(mT) \cdot n\log^2(n/m) / m$.
We now see that this \DLR{} risk is an extra $n \log^2(n/m)/m$ factor larger than the
risk of the baseline problem \PLR{}:

\begin{restatable}[Risk upper bound for \PLR{}]{mythm}{upperboundindldsls}
\label{stmt:upper_bound_ind_lds_ls}
There are universal positive constants $c_0$ and $c_1$ such that the following holds for any instance of \PLR{}.
Suppose that $(A, B)$
satisfies 
\Cref{assume:marginal_stability},
\Cref{assume:diagonalizability}, and
\Cref{assume:one_step_controllability},
with $\condNum := \condNum(A, B)$ (\Cref{def:condition_number}). 
If $n \geq 2$ and $mT \geq c_0 n \log(\max\{\condNum, e\})$,
then:
\begin{align*}
    \E[ L(\hat{W}_{m,T}; \Tnew, \otimes_{t \geq 1} \PxAt{A,B}{t}) ]
    \leq c_1 \sigma_\xi^2 \cdot \frac{pn \condNum \log(\max\{\condNum, e\})  }{mT} 
    \cdot \phi\left(\condNum, \frac{\Tnew}{T}\right).
\end{align*}
\end{restatable}
\noindent
Treating $\condNum$ as a constant and setting $\Tnew = T$,
\Cref{stmt:upper_bound_ind_lds_ls} states that
$\E[L(\hat{W}_{m,T};T, \otimes_{t \geq 1} \PxAt{A,B}{t})]$
scales as $\sigma_\xi^2 pn/(mT)$, matching the risk of
\IID\ linear regression up to constant factors.
In \Cref{sec:results:lower_bounds}, we will see that the result of
\Cref{thm:sublinear_trajectories_bound} is sharp up to constants, and therefore the \DLR{} problem is fundamentally more difficult than its corresponding
baseline problem \PLR{} when trajectories are relatively scarce.

We conclude with our final upper bound,
using our assumptions to generalize \cite[Theorem~2.1]{simchowitz18learning} to the
few-trajectory setting:
\begin{restatable}[Parameter recovery upper bound for \LDSLS{}, few trajectories]{mythm}{paramrecoveryfewtrajrate}
\label{stmt:upper_bound_parameter_recovery_few_trajs}
There are universal positive constants $c_0$, $c_1$, and $c_2$ such that the following holds for any instance of \LDSLS{}.
Suppose that $(A, B)$
satisfies 
\Cref{assume:marginal_stability},
\Cref{assume:diagonalizability}, and
\Cref{assume:one_step_controllability},
with $\condNum := \condNum(A, B)$ (\Cref{def:condition_number}). 
If $n \geq 2$, and $mT \geq c_0 n \log(\max\{\condNum n/m, e\})$,
then:
\begin{align*}
    \E[ \norm{\hat{W}_{m,T} - W_\star}_F^2 ]
    \leq c_1 \sigma_\xi^2
      \cdot \frac{ pn \log(\max\{\condNum n/m, e\})}{mT \cdot \lambda_{\min}(\Gamma_{k_\star}(A, B))},
    \quad k_\star := \bigfloor{\frac{c_2 T}{n/m \cdot \log(\max\{\condNum n/m, e\})}}.
\end{align*}
\end{restatable}
\noindent
\Cref{stmt:upper_bound_parameter_recovery_few_trajs} complements \Cref{stmt:upper_bound_parameter_recovery};
together they cover parameter recovery across all problem regimes.
Again, operator norm bounds also hold with $p + n$ in place of $pn$.

\subsection{Comparison to learning from trajectories of multiple unknown systems}
\label{sec:results:upper:comparison}

As mentioned in \Cref{sec:related}, \cite{chen2022learning,modi2022learning} both study the 
setup where a learner observes multiple independent
trajectories from $K$ different unknown
linear dynamical systems. The task is to identify the
parameters of the $K$ underlying systems. This is more general than the setting we consider, which is recovered by fixing $K=1$. 
However, specializing these rates to our setting
yield either unnecessary requirements, suboptimal bounds, or both.

To see this, first, if we specialize
\cite[Theorem 1]{chen2022learning} to our setup, we generate
unnecessary assumptions.
Specifically, Theorem 1 requires strict stability, 
one-step controllability, and
$mT \gtrsim \max\{ n^3, 1/(1-\rho) \}$, where $\rho$ is the spectral radius of $A$.
In comparison, \Cref{stmt:upper_bound_parameter_recovery} 
only requires 
$\kcont$-step controllability, $T \geq \kcont$, and $m \gtrsim n$.
However, note that Theorem 1, like \Cref{stmt:upper_bound_parameter_recovery}, does 
have the property that the parameter error (in operator norm) 
scales as $\sqrt{n/(mT)}$, reflecting that all collected datapoints contribute to reducing error.

Next, we specialize \cite[Theorem 2]{modi2022learning}.
Theorem 2 bounds the error of an estimation procedure which outputs $m$ different
estimates $\{\hat{A}_i\}_{i=1}^{m}$, one for each observed trajectory (cf.~Eq.~(3)).
Specifically, it gives an upper bound on the quantity
$\frac{1}{m}\sum_{i=1}^{m} \norm{\hat{A}_i - A_i}_F^2$, where $A_i$ is the dynamics matrix
associated with the $i$-th trajectory. 
To specialize this to our setting, we average the estimates and apply Jensen’s inequality followed by Theorem 2.
This yields the bound $\| \hat{A} - A \|_F^2 \lesssim 1/T + n^2 / (mT)$, where $\hat{A} := \frac{1}{m}\sum_{i=1}^{m} \hat{A}_i$ is the
averaged estimate.
We see that, for a fixed $T$,
as $m \rightarrow \infty$, the rate tends to $1/T$ instead of zero 
(compared with the $n^2/(mT)$ bound from \Cref{stmt:upper_bound_parameter_recovery}). 
Additionally, 
Theorem 2 requires both that the dynamics are one-step controllable and that
the spectral radius of $A$ is bounded by $1 + O(1/T)$.

\section{Risk lower bounds}
\label{sec:results:lower_bounds}

Our lower bounds rely on the following statement, that
the expected trace inverse covariance---a classic quantity in asymptotic
statistics---bounds the minimax risk from below:
\begin{restatable}[Expected trace of inverse covariance bounds risk from below]{mylemma}{traceinvlowerboundsminimaxrisk}
\label{thm:trace_inv_lower_bounds_minimax_risk}
Fix $m, T \in \N_{+}$ and a set of covariate distributions $\calP_x$. 
Suppose that for every $\Px \in \calP_x$,
the data matrix $X_{m,T} \in \R^{mT \times n}$ drawn from $\otimes_{i=1}^{m} \Px$
has full column rank almost surely.
The minimax risk $\mathsf{R}(m, T, \Tnew; \calP_x)$ satisfies:
\begin{align*}
    \mathsf{R}(m, T, \Tnew; \calP_x) \geq \sigma_\xi^2 p \cdot \sup_{\Px \in \calP_x} \E_{\otimes_{i=1}^{m} \Px}\left[ \Tr\left( \Gamma_{\Tnew}^{1/2}(\Px) (X_{m,T}^\T X_{m,T})^{-1} \Gamma_{\Tnew}^{1/2}(\Px) \right) \right].
\end{align*}
\end{restatable}
\noindent
\Cref{thm:trace_inv_lower_bounds_minimax_risk} is well known, possibly considered folklore;
we state and prove it for completeness.
Our proof is inspired by a recent argument from \cite{mourtada19exactminimax}.
It smooths over problem instances according to a Gaussian prior,
and analytically characterizes the posterior distribution
of the parameter $W_\star$ under a simple Gaussian observation model
detailed in \Cref{sec:lower_bound_proof_sketch}.

Our first lower bound underscores the need to make variance growth
assumptions \eqref{eq:variance_growth_condition},
in \Cref{stmt:upper_bound_ind_seq_ls},
for \PSR{} in the few trajectories ($m \lesssim n$) regime:
\begin{restatable}[Need for growth assumptions in \PSR{} when $m \lesssim n$]{mythm}{indseqlslowerbound}
\label{stmt:ind_seq_ls_lower_bound}
There exists universal constant 
$c_0$, $c_1$, and $c_2$ such that the following holds.
Suppose that 
$\Px = \otimes_{t \geq 1} N(0, 2^t \cdot I_n)$,
$n \geq 6$, 
$mT \geq n$, 
and $m \leq c_0 n$.
Then:
\begin{align*}
    \mathsf{R}(m,T,T;\{\Px\}) \geq c_1 \sigma_\xi^2 \cdot \frac{p \cdot 2^{c_2 n/m}}{T} .
\end{align*}
\end{restatable}
\Cref{stmt:ind_seq_ls_lower_bound} states that if the variances $\Sigma_t$
are allowed to grow exponentially in $t$, 
then the minimax risk of \PSR{} scales exponentially in $n/m$ when $m \lesssim n$.
Thus, some sub-exponential growth assumption is necessary
in order to have the risk scale polynomially in $n/m$.

We now turn to a lower bound for \DLR{}.
We consider
two particular hard instances for \DLR{} dynamics matrices $(A, B)$,
where we set $B = I_n$ and vary $A$.
The first instance corresponds to \IID\ covariates,
i.e.,\ $A = 0_{n \times n}$.
The second instance corresponds
to an isotropic Gaussian random walk, i.e.,\ $A = I_n$.
These two hard instances satisfy
\Cref{assume:marginal_stability}, 
\Cref{assume:diagonalizability}, and
\Cref{assume:one_step_controllability}.
Together they show that our upper bounds are sharp up to logarithmic factors, 
treating the condition number $\condNum(A, B)$ 
from \Cref{def:condition_number} as a constant:
\begin{restatable}[Risk lower bound]{mythm}{lowerboundmain}
\label{stmt:lower_bound_main}
There are universal positive constants $c_0$, $c_1$, and $c_2$ such that the following holds.
Recall that $\PxA{I_n}$ (resp.\ $\PxA{0_{n \times n}}$)
denotes the covariate distribution for a linear dynamical system
with $A=I_n$ and $B=I_n$ (resp.\ $A=0_{n \times n}$ and $B=I_n$).
If $T \geq c_0$, $n \geq c_1$, and $mT \geq n$, then:
\begin{align*}
    \mathsf{R}(m, T, \Tnew; \{\PxA{0_{n \times n}}, \PxA{I_n}\}) \geq c_2 \sigma_\xi^2 \cdot \frac{pn}{mT} \cdot \max\left\{ \frac{n \Tnew}{m T}, \frac{\Tnew}{T}, 1 \right\}.
\end{align*}
\end{restatable}
We can interpret this lower bound by a breakdown of the quantity
$\varphi := \max\{n\Tnew/(mT), \Tnew/T, 1\}$ across various regimes.
When trajectories are limited ($m \lesssim n$),
$\varphi \asymp \max\{ n\Tnew/(mT), 1\}$, and therefore
the minimax risk is bounded below by
$\sigma_\xi^2 \cdot pn/(mT) \cdot \max\{ n\Tnew/(mT), 1\}$.
This is the same rate prescribed by the OLS upper bound of
\Cref{thm:sublinear_trajectories_bound}, up to 
the condition number $\condNum(A, B)$ and logarithmic factors
in $n/m$. We have thus justified the
summary statement \Cref{thm:simple-rate-few-traj}.
On the other hand, under many trajectories ($m \gtrsim n$),
$\varphi \asymp \max\{\Tnew/T,1\}$
and the minimax risk is bounded below by 
$\sigma_\xi^2 \cdot pn/(mT) \cdot \max\{\Tnew/T,1\}$.
By \Cref{stmt:upper_bound_many_trajectories}, the OLS risk is
bounded above by the same quantity times $\condNum(A, B)$,
justifying the summary statement \Cref{thm:simple-rate-many-traj}.

\section{Key proof ideas}
\label{sec:proof_ideas}

In this section, we highlight some of the key ideas
behind our results. Proofs of the upper bounds
are in \Cref{sec:appendix:upper_bound_proofs},
and proofs of the lower bounds are in \Cref{sec:appendix:lower_bounds}.

\paragraph{Additional notation.}
For $r \in \mathbb{N}_+$ and $M \in \R^{n \times n}$,
let $J_r \in \R^{r \times r}$ denote the
Jordan block of size $r$ with ones along its diagonal,
let $\mathsf{BDiag}(M, r) \in \R^{nr \times nr}$ denote the block 
diagonal matrix with diagonal blocks $M$,
and let
$\mathsf{BToep}(M, r) \in \R^{nr \times nr}$ denote the block Toeplitz matrix with first column
$(I_n, M^\T, \dots, (M^{r-1})^\T)^\T$.

\subsection{Upper bounds}

The proof of 
\Cref{stmt:upper_bound_general}
decomposes the risk using a standard basic inequality, 
which we now describe.
While \Cref{stmt:upper_bound_general} is stated quite generally, for simplicity of exposition we  
restrict ourselves in this section to the case when
the matrix parameters $\{ \Psi_j\}_{j=1}^{S}$ in \Cref{def:trajectory_small_ball} are all set to $\uGam$.
Under this simplification, we have that
$\uMu(\{\Psi_j\}_{j=1}^{S}, \uGam) = 1$.

Equation~\eqref{eq:y_observation_model} yields the identity
$Y_{m,T} = X_{m,T} W_\star^\T + \Xi_{m,T}$.
Plugging this relationship into the formula~\eqref{eq:ols_definition} for $\hat{W}_{m,T}$
gives $\hat{W}_{m,T} - W_\star = \Xi_{m,T}^\T X_{m,T} (X_{m,T}^\T X_{m,T})^{-1}$.
Define the whitened version of $X_{m,T}$ as
$\tilde{X}_{m,T} := X_{m,T} \uGam^{-1/2}$.
From these definitions and after some basic manipulations,
for any $\Gamma' \in \sfSym^n_{> 0}$:
\begin{align}
    \norm{\hat{W}_{m,T} - W_\star}^2_{\Gamma'} 
    \leq \min\{ n, p \} \frac{ \opnorm{ (\tilde{X}_{m,T}^\T \tilde{X}_{m,T})^{-1/2} \tilde{X}_{m,T}^\T \Xi_{m,T} }^2 }{ \lambda_{\min}( \tilde{X}_{m,T}^\T \tilde{X}_{m,T} ) \cdot \ulam(\uGam, \Gamma')  }. \label{eq:basic_inequality}
\end{align}
This decomposes the analysis into
two parts: 
(a) upper-bounding 
$\opnorm{ (\tilde{X}_{m,T}^\T \tilde{X}_{m,T})^{-1/2} \tilde{X}_{m,T}^\T \Xi_{m,T} }^2$, which is a self-normalized
martingale term, and (b)
lower-bounding the term
$\lambda_{\min}( \tilde{X}_{m,T}^\T \tilde{X}_{m,T} )$.
The analysis for the martingale term is fairly standard
\citep[cf.][Corollary~1]{abbasiyadkori2011selfnormalized},
so for the remainder of this section we focus on the
minimum eigenvalue bound, which contains much of what is novel in our analysis.

We first demonstrate 
how the trajectory small-ball definition (\Cref{def:trajectory_small_ball})
can be used to establish \vocab{pointwise} convergence of the quadratic form
$\chi(v) := \sum_{i=1}^{m} \sum_{t=1}^{T} \ip{v}{\tilde{x}_t^{(i)}}^2$ 
for $v \in \mathbb{S}^{n-1}$, where $\tilde{x}_t^{(i)} := \uGam^{-1/2} x_t^{(i)}$
is a whitened state vector.
Specifically, we show that 
for a fixed $v \in\mathbb{S}^{n-1}$,
the probability of the event
$\{\chi(v) \leq \psi \cdot \varepsilon \}$
is small, for $\psi,\varepsilon$ to be specified.

The key idea is that for any non-negative random variable $X$
satisfying $\Pr\{ X \leq \e \} \leq (c \e)^\alpha$ for all $\e > 0$,
the moment generating function satisfies
$\E[ \exp(-\eta X) ] \leq (c/\eta)^\alpha$
for all $\eta > 0$ (cf.~\Cref{stmt:small_ball_to_mgf}).
Hence, by condition \eqref{eq:trajectory_small_ball} 
from \Cref{def:trajectory_small_ball}, for any $\eta > 0$:
\begin{align*}
    \E\left[\exp\left( -\frac{\eta}{k}\sum_{t=(j-1)k+1}^{jk} \ip{v}{\tilde{x}_t^{(i)}}^2  \right) ~\Bigg|~ \calF_{(j-1)k} \right] \leq \left(\frac{\csb}{\eta}\right)^\alpha \:\: \textrm{a.s.}, \:\: i = 1, \dots m, \:\: j = 1, \dots S.
\end{align*}
By a Chernoff bound, the tower property of conditional expectation, 
and the independence of the trajectories $\{\tilde{x}_t^{(i)}\}_{t \geq 1}$
and $\{\tilde{x}_t^{(i')}\}_{t \geq 1}$ when $i \neq i'$,
for any $\psi > 0$:
\begin{align*}
    \Pr\left( \frac{1}{k}\sum_{i=1}^{m}\sum_{t=1}^{T} \ip{v}{\tilde{x}_t^{(i)}}^2 \leq \zeta \right) &\leq \inf_{\eta \geq 0} e^{\eta\zeta} \E \exp\left( -\frac{\eta}{k}\sum_{i=1}^{m}\sum_{t=1}^{T} \ip{v}{\tilde{x}_t^{(i)}}^2 \right) \\
    &\leq \inf_{\eta \geq 0} e^{\eta\zeta} \left(\frac{\csb}{\eta}\right)^{mS\alpha} \\
    &= \exp\left( -mS\alpha\left( \log\left(\frac{mS\alpha}{\csb \zeta}\right) - 1 \right)\right).
\end{align*}
Now with a change of variables $t := \log\left(\frac{mS\alpha}{\csb \zeta}\right) - 1$,
we obtain:
\begin{align}
    \Pr\left( \sum_{i=1}^{m}\sum_{t=1}^{T} \ip{v}{\tilde{x}_t^{(i)}}^2 \leq \frac{mT \alpha}{2\csb} e^{-(t+1)} \right) \leq \exp(-mS\alpha t) \quad \forall t > 0. \label{eq:pointwise_lower_tail}
\end{align}
The key upshot of \eqref{eq:pointwise_lower_tail} is that 
it controls tail probability at all scales. 
This control is needed in order to bound the expected value of \eqref{eq:basic_inequality} by integration.
At this point, it remains to upgrade \eqref{eq:pointwise_lower_tail}
from pointwise to uniform 
over $\mathbb{S}^{n-1}$.
A natural approach is to
use standard covering and union bound arguments,
as is done in \cite{simchowitz18learning}.
However, straightforward covering argument yields
un-necessary logarithmic factors in the covariate dimension $n$.
In order to circumvent this issue,  
we utilize the PAC-Bayes argument from \cite{mourtada19exactminimax} (which itself
is an extension of \cite{oliveria2016lowertail}) to establish uniform
concentration. 
The details are given in \Cref{sec:appendix:general_OLS_proof}.

\subsection{Lower bounds}
\label{sec:lower_bound_proof_sketch}

\subsubsection{Observation noise behind \Cref{thm:trace_inv_lower_bounds_minimax_risk}}
Our definition of minimax risk
$\mathsf{R}(m, T, \Tnew; \calP_x)$ in \eqref{eq:minimax_risk_def}
involves a supremum over the worst case
$\sigma_\xi$-sub-Gaussian MDS distribution that
models the observation noise.
The proof of \Cref{thm:trace_inv_lower_bounds_minimax_risk} bounds this supremum from below
by considering a noise model that
decouples the observation noise $\{\xi_t\}_{t \geq 1}$ from the randomness
that drives the trajectory $\{x_t\}_{t \geq 1}$:
\begin{mydef}[Gaussian observation noise]
\label{def:gaussian_observation_noise}
The \emph{Gaussian observation noise model} holds when
$\xi_t \sim N(0, \sigma_\xi^2 I_p)$, $\xi_t \perp \xi_{t'}$ if $t \neq t'$, and the process $\{\xi_t\}_{t \geq 1}$ is independent from the process $\{x_t\}_{t \geq 1}$.
\end{mydef}

Decoupling the noise processes %
orthogonalizes the two problems simultaneously present in \DSR{}:
learning the dynamics of covariates and learning the responses from covariates.
\Cref{def:gaussian_observation_noise} draws attention to the latter.
It will unfortunately exclude us from addressing linear system identification
specifically with our lower bound,
but it allows a sharp and simple characterization of the minimax risk in general.
The proof of \Cref{thm:trace_inv_lower_bounds_minimax_risk}
is given in \Cref{sec:appendix:lower_bounds:trace_inv_proof}.

\subsubsection{An analysis of non-isotropic gramian matrices}
\label{sec:proof_ideas:gramian}

A key technical challenge for our analysis lies in constructing a sharp lower bound
on the expected trace inverse of a gramian matrix formed by random non-isotropic Gaussian
random vectors. Specifically, for integers $q, n \in \N_{+}$ with $q \geq n$, 
and for a fixed positive definite matrix $\Sigma \in \sfSym^q_{> 0}$, we are interested
in a lower bound on the quantity $\E \Tr((W^\T \Sigma W)^{-1})$,
where $W \in \R^{q \times n}$ has \IID\ $N(0, 1)$ entries.
The matrix $W^\T \Sigma W$ is equal in distribution to the gramian matrix $Y \in \R^{n \times n}$ of the vectors
$g_1, \dots, g_n \in \R^q$, which are drawn \IID\ from $N(0, \Sigma)$, i.e., 
$Y_{ij} = \ip{g_i}{g_j}$.

The main tool we use to analyze $\E \Tr((W^\T \Sigma W)^{-1})$ is 
the convex Gaussian min-max theorem (CGMT) from \cite{thrampoulidis14gmt},
which allows us to bound from below the expected trace inverse by studying
a two dimensional min-max game that is more amenable to analysis.
The key idea is to cast the expected trace inverse as a least-norm optimization problem, and apply CGMT to the value of the optimization problem.
We believe the following result to be of independent
interest.
\begin{restatable}{mylemma}{traceinvlowerbound}\label{lemma:trace_inv_lower_bound}
Let $q, n$ be positive integers with $q \geq n$ and $n \geq 2$.
Let $W \in \R^{q \times n}$ have \IID\ $N(0, 1)$ entries, and let $\Sigma \in \R^{q \times q}$ be positive definite.
Let $g \sim N(0, I_q)$ and $h \sim N(0, I_{n-1})$, with $g$ and $h$ independent. 
Also, let $\{e_i\}_{i=1}^{q}$ be the standard basis vectors in $\R^q$.
We have:
\begin{align}
    \E \Tr((W^\T \Sigma W)^{-1}) \geq \frac{n}{\sum_{i=1}^{q} \E\min_{\beta \geq 0} \max_{\tau \geq 0}\left[ -\frac{\beta \norm{h}_2}{\tau} + \norm{\beta g - e_i}^2_{(\Sigma^{-1} + \beta \norm{h}_2 \tau I_q)^{-1}} \right]}. \label{eq:inverse_gram_matrix_min_max}
\end{align}
\end{restatable}

The proof of \Cref{lemma:trace_inv_lower_bound} appears
in \Cref{sec:appendix:lower_bounds:gramians}.
We now discuss how to analyze the two-dimensional min-max game appearing
in \Cref{lemma:trace_inv_lower_bound}.
We first start by heuristically replacing it with a \vocab{stylized problem},
where the random quantities which appear in \eqref{eq:inverse_gram_matrix_min_max}
are replaced by their expected scaling:
\begin{align}
    \mathsf{SP}(\Sigma, n) :=
    \sum_{i=1}^{q} \min_{\beta \geq 0} \max_{\tau \geq 0}
      \underbrace{\left[
        -\frac{\beta \sqrt{n}}{\tau} +
        \beta^2 \Tr((\Sigma^{-1} + \beta\sqrt{n} \tau  I_q)^{-1}) +
        (\Sigma^{-1} + \beta\sqrt{n}\tau  I_q)^{-1}_{ii}
      \right]}_{=:\: \ell_i(\beta,\tau)}. \label{eq:stylized_problem_general}
\end{align}
While \eqref{eq:stylized_problem_general} is not a valid upper bound on the value of the min-max game appearing in \eqref{eq:inverse_gram_matrix_min_max}, analyzing 
\eqref{eq:stylized_problem_general} is simpler and gives the correct intuition;
we give a rigorous upper bound in \Cref{stmt:Z_i_helper}.

We start by observing that if $\beta = 0$, then regardless of the choice
of $\tau$, $\ell_i(0, \tau) = \Sigma_{ii}$,
and therefore $\sum_{i=1}^{q} \ell_i(0, \tau_i) = \Tr(\Sigma)$ for any $\{\tau_i\}_{i=1}^{q} \subset \R^q_{\geq 0}$.
On the other hand, if $\beta > 0$, then
$\ell_i(\beta, \tau)$
tends to $-\infty$ as $\tau \rightarrow 0^+$
and to $0$ as $\tau \rightarrow \infty$.
Therefore, if we can show that
there exists a $v \in (0, \Tr(\Sigma))$, such that
every 
set of points 
$\{(\beta_i,\tau_i)\}_{i=1}^{q} \subset \R_{> 0}^{2}$ 
satisfying:\footnote{The conditions given in \eqref{eq:first_order_conditions} are \emph{not}
in general necessary first-order optimality conditions for a nonconvex/nonconcave game~\citep[see e.g.][Proposition 21]{jin2020localoptimality}.
However,
since for every $\beta > 0$, the function $\tau \mapsto \ell_i(\beta, \tau)$ has only strictly concave stationary points  (\Cref{stmt:local_strict_concave_max}), these conditions are necessary for this
particular problem.}
\begin{align}
    \frac{\partial \ell_i}{\partial \beta}(\beta_i, \tau_i) = \frac{\partial \ell_i}{\partial \tau}(\beta_i, \tau_i) = 0, \quad i=1, \dots, q, \label{eq:first_order_conditions}
\end{align}
also satisfies 
$v = \sum_{i=1}^{q} \ell_i(\beta_i, \tau_i)$,
then $\mathsf{SP}(\Sigma, n) = v$.

To uncover the critical points,
we define the functions $f$ and $q_i$, for $i=1, \dots, q$, as:
\begin{align*}
    f(x) := -x\sqrt{n} + x^2 \Tr((\Sigma^{-1} + x \sqrt{n} I_q)^{-1}), \quad
    q_i(x) := (\Sigma^{-1} + x \sqrt{n} I_q)^{-1}_{ii}.
\end{align*}
With these definitions, we can write:
\begin{align*}
    \ell_i(\beta,\tau) = \frac{1}{\tau^2} f(\beta\tau) + q_i(\beta \tau).
\end{align*}
Calculating
$\frac{\partial \ell_i}{\partial \beta}(\beta,\tau) = \frac{\partial \ell_i}{\partial \tau}(\beta,\tau) = 0$ yields,
for $\tau \neq 0$:
\begin{align}
    0 &= \frac{\partial \ell_i}{\partial \tau}(\beta, \tau) = \tau^{-2} f'(\beta \tau) \beta - 2 \tau^{-3} f(\beta\tau) + q_i'(\beta \tau) \beta, \label{eq:critical_point_one} \\
    0 &= \frac{\partial \ell_i}{\partial \beta}(\beta, \tau) = \tau^{-2} f'(\beta\tau)\tau + q_i'(\beta\tau) \tau. \label{eq:critical_point_two}
\end{align}
The second condition \eqref{eq:critical_point_two}
implies that $q_i'(\beta\tau) = - \tau^{-2}f'(\beta\tau)$.
Plugging this condition into \eqref{eq:critical_point_one} implies that
$f(\beta\tau) = 0$, and hence $\ell_i(\beta, \tau) = q_i(\beta\tau)$ for the critical point $(\beta, \tau)$.
We now study the positive roots of the equation $f(x) = 0$, or equivalently:
\begin{align*}
    x \sqrt{n} = x^2  \Tr((\Sigma^{-1} + x\sqrt{n} I_q)^{-1}).
\end{align*}
Using the variable substitution $y := x \sqrt{n}$, we have, when $y > 0$, the equivalent problem:
\begin{align*}
    \psi(y; \Sigma) := y \Tr((\Sigma^{-1} + y I_q)^{-1}) = n.
\end{align*}
Observe that $\psi(0; \Sigma) = 0$ and $\lim_{y\rightarrow\infty} \psi(y; \Sigma) = q$.
Furthermore, $\psi(y; \Sigma)$ is continuous and monotonically increasing with $y$.
Therefore, as long as $q > n$, there is exactly one $\bar{y} \in (0, \infty)$ such that $\psi(\bar{y}; \Sigma) = n$,
or equivalently there is exactly one $\bar{x} \in (0, \infty)$ such that
$\psi(\bar{x} \sqrt{n}; \Sigma) = n$.
Such a quantity $\bar{x}$ supplies the curve of critical points
$\mathsf{Crit}(\bar{x}) := \{(\beta,\tau) \in \R_{> 0}^2 \mid \beta\tau = \bar{x} \}$.
Note that $\mathsf{Crit}(\bar{x})$ is the set of critical points for
every $\ell_i(\beta, \tau)$, $i=1, \dots, q$.
Furthermore, 
for any $(\beta_\star, \tau_\star) \in \mathsf{Crit}(\bar{x})$
and $i \in \{1, \dots, q\}$, we have that
$\ell_i(\beta_\star, \tau_\star) = q_i(\beta_\star \tau_\star) = (\Sigma^{-1} + \bar{x}\sqrt{n} I_q)^{-1}_{ii}$.
Therefore:
\begin{align*}
    \{(\beta_i, \tau_i)\}_{i=1}^{T} \subset \mathsf{Crit}(\bar{x}) \Longrightarrow \sum_{i=1}^{q} \ell_i(\beta_i, \tau_i) = \Tr((\Sigma^{-1} + \bar{x}\sqrt{n} I_q)^{-1}) \in (0, \Tr(\Sigma)),
\end{align*}
and thus:
\begin{align}
    \mathsf{SP}(\Sigma, n) = \frac{\sqrt{n}}{\bar{x}}, \:\: \textrm{with $\bar{x}$ the solution to} \:\: \psi(\bar{x}\sqrt{n}; \Sigma) = n. \label{eq:SP_solution}
\end{align}
In light of \eqref{eq:SP_solution},
\Cref{lemma:trace_inv_lower_bound} then suggests that:
\begin{align}
    \E \Tr((W^\T \Sigma W)^{-1}) \gtrapprox \frac{n}{\mathsf{SP}(\Sigma, n)} = \bar{x} \sqrt{n}, \label{eq:heuristic_approx_v1}
\end{align}
where the $\gtrapprox$ notation indicates the heuristic nature
of replacing the expected min-max game appearing in the bound  
\eqref{eq:inverse_gram_matrix_min_max} with the approximation
\eqref{eq:stylized_problem_general}.

If we briefly check \eqref{eq:heuristic_approx_v1} in the simple case when $\Sigma = I_q$,
we see that:
\begin{align*}
    n = \psi(\bar{x}\sqrt{n};I_q) = \bar{x}\sqrt{n} \frac{q}{1 + \bar{x}\sqrt{n}} \Longrightarrow \bar{x}\sqrt{n} = \frac{n}{q} (1 + \bar{x}\sqrt{n}) \geq \frac{n}{q}.
\end{align*}
Hence, \eqref{eq:heuristic_approx_v1} yields that
$\E \Tr((W^\T W)^{-1}) \gtrapprox n/q$,
which is the correct scaling; the exact result is
$\E \Tr((W^\T W)^{-1}) = n/(q-n-1)$ for $q \geq n + 2$.

\subsubsection{Ideas behind \Cref{stmt:ind_seq_ls_lower_bound}}

We let $X_{m,T}$ denote the data matrix associated with $m$ \IID\ copies of $\{x_t\}_{t=1}^{T}$, with $x_t \sim N(0, 2^t \cdot I_n)$ and 
$x_{t} \perp x_{t'}$ for $t \neq t'$.
We also define 
$\Gamma_T := \frac{1}{T} \sum_{t=1}^{T} 2^t \cdot I_n = \frac{2}{T} (2^T - 1) \cdot I_n$,
and observe that $\Gamma_T \succcurlyeq \frac{2^T}{T} \cdot I_n$.
By \Cref{lemma:trace_inv_lower_bound}, it suffices to lower bound the quantity
$\E \Tr(\Gamma_T^{1/2} (X_{m,T}^\T X_{m,T})^{-1} \Gamma_T^{-1/2})$.
Since each column of $X_{m,T}$ is independent,
the matrix
$X_{m,T}2^{-T/2}$ has the same distribution as
$\mathsf{BDiag}(\Theta^{1/2}, m)W$,
where $\Theta \in \R^{T \times T}$ is diagonal,
$\Theta_{ii} = 2^{i-T}$ for $i \in \{1, \dots, T\}$, and 
$W \in \R^{mT \times n}$ has \IID\ $N(0, 1)$ entries.
In other words, we have:
\begin{align*}
    \E \Tr(\Gamma_T^{1/2} (X_{m,T}^\T X_{m,T})^{-1} \Gamma_T^{-1/2}) \geq \frac{1}{T} \E \Tr((W^\T \mathsf{BDiag}(\Theta, m) W)^{-1}).
\end{align*}
By the arguments in \Cref{sec:proof_ideas:gramian}, we have:
\begin{align*}
    \E \Tr((W^\T \mathsf{BDiag}(\Theta, m) W)^{-1}) \gtrapprox \frac{n}{\mathsf{SP}(\mathsf{BDiag}(\Theta, m), n)},
\end{align*}
where the notation $\gtrapprox$ indicates the heuristic nature of the inequality as explained previously.
From \eqref{eq:SP_solution}, we want to find $\bar{x}$ such that:
\begin{align*}
    n = \psi(\bar{x}\sqrt{n}; \mathsf{BDiag}(\Theta, m)) = \bar{x}\sqrt{n} \cdot m \sum_{j=0}^{T-1} \frac{1}{2^j + \bar{x}\sqrt{n}}.
\end{align*}
While solving this equation exactly for $\bar{x}\sqrt{n}$ is not tractable, we can estimate a lower bound on $\bar{x}\sqrt{n}$ quite easily.
For any integer $T_c \in \{0, \dots, T\}$, we have the following estimate:
\begin{align*}
    \frac{n}{m} =  \bar{x}\sqrt{n}\sum_{j=0}^{T-1} \frac{1}{2^j + \bar{x}\sqrt{n}} \leq T_c + 2 \bar{x}\sqrt{n} \cdot 2^{-T_c}.
\end{align*}
Let us first assume that $\bar{x}\sqrt{n} \in [1, 2^{T-1}]$, so that $\ceil{\log_2(\bar{x}\sqrt{n})} \in \{0, \dots, T\}$.
Setting $T_c = \ceil{\log_2(\bar{x}\sqrt{n})}$ then yields
the lower bound $\bar{x}\sqrt{n} \geq 2^{n/m-3}$.
On the other hand, if $\bar{x}\sqrt{n} > 2^{T-1}$, then since we assume $mT \geq c_1 n$,
we also have $\bar{x}\sqrt{n} > 2^{c_1n/m - 1}$.
Finally, if $\bar{x}\sqrt{n} < 1$, we have:
\begin{align*}
    \frac{n}{m} = \bar{x}\sqrt{n} \sum_{j=0}^{T-1} \frac{1}{2^j + \bar{x}\sqrt{n}} < \sum_{j=0}^{T-1} \frac{1}{2^j + \bar{x}\sqrt{n}} \leq 2 \Longrightarrow m \geq n/2.
\end{align*}
This yields a contradiction, since by assumption $m \leq c_2 n$, if $c_2 < 1/2$,
so we must have $\bar{x}\sqrt{n} \geq 2^{c' n/m - 3}$ with $c' = \min\{1, c_1\}$.
Now by \eqref{eq:SP_solution} and \eqref{eq:heuristic_approx_v1}:
\begin{align*}
    \mathsf{SP}(\mathsf{BDiag}(\Theta, m), n) = \frac{n}{\bar{x}\sqrt{n}} \leq n 2^{-c'n/m+3} \Longrightarrow   \E \Tr(\Gamma_T^{1/2} (X_{m,T}^\T X_{m,T})^{-1} \Gamma_T^{-1/2}) \gtrapprox \frac{2^{c' n/m}}{T}.
\end{align*}
We make this argument rigorous in \Cref{sec:appendix:ind_seq_ls_lower_bound}.

\subsubsection{Ideas behind \Cref{stmt:lower_bound_main}}
\label{sec:lower_bound_proof_sketch:few_trajectories}

We focus here on the hard instance when $A=I_n$ and $m \lesssim n$, since the cases when $A=0_{n\times n}$ 
or $A=I_n$ and $m \gtrsim n$
are straightforward applications of Jensen's inequality
and some basic manipulations (see \Cref{stmt:minimax_jensen_rate}).

The proof used by \Cref{stmt:lower_bound_main} when $A=I_n$ and $m \lesssim n$ is
actually a special case of a general proof indexed by 
the largest Jordan block size of the hard instance.
For a maximum Jordan block size $r$, the hard instances are
$A = \mathsf{BDiag}(J_r, n/r)$,
where we assume for simplicity that $r$ divides $n$;
this reduces to $A=I_n$ when $r=1$.
We associate two important matrices with these hard instances.
To define them, let $\calI_r := \{1, 1 + r, \dots, 1 + (T-1)r\}$,
and let $E_{\calI_r} \in \R^{T \times Tr}$ denote 
the linear operator that extracts the coordinates in $\calI_r$.
The following matrices then play a key role in our analysis:
\begin{align}
    \Psi_{r,T,\Tnew} := \mathsf{BDiag}(\Gamma_{\Tnew}^{-1/2}(J_r), T) \mathsf{BToep}(J_r, T), \quad \Theta_{r,T,\Tnew} := E_{\calI_r} \Psi_{r,T,\Tnew} \Psi_{r,T,\Tnew}^\T E_{\calI_r}^\T. \label{eq:Theta_T_r_def}
\end{align}
The next step is to use a simple decoupling argument
(see \Cref{lemma:decoupling_blocks}) to argue that,
for $A = \mathsf{BDiag}(J_r, d)$:
\begin{align*}
    \E \Tr(\Gamma_{\Tnew}^{1/2} (X_{m,T}^\T X_{m,T})^{-1} \Gamma_{\Tnew}^{1/2}) \geq \E \Tr((W^\T \mathsf{BDiag}(\Theta_{r,T,\Tnew}, m) W)^{-1}),
\end{align*}
where $W \in \R^{mT \times d}$ has \IID\ $N(0, 1)$ entries.
This positions us to use the arguments in \Cref{sec:proof_ideas:gramian} again.
We first focus on the $r=1$ case.
We reduce the problem to assuming $\Tnew = T$,
by observing that since $\Gamma_t(I_n) = \frac{t+1}{2} \cdot I_n$ for any $t \in \N_{+}$, then
$\Theta_{1,T,\Tnew} = \frac{T+1}{\Tnew+1} \cdot \Theta_{1,T,T}$.
Therefore, 
\begin{align}
    \E \Tr((W^\T \mathsf{BDiag}(\Theta_{1,T,\Tnew}, m) W)^{-1}) &= \frac{\Tnew + 1}{T+1} \cdot \E \Tr((W^\T \mathsf{BDiag}(\Theta_{1,T,T}, m) W)^{-1}) \label{eq:Tnew_equals_T_wlog_r_equals_one_lower_bound} \\
    &\gtrapprox \frac{\Tnew + 1}{T+1} \cdot \frac{n}{\mathsf{SP}(\mathsf{BDiag}(\Theta_{1,T,T}, m), n)} \nonumber ,
\end{align}
where again the $\gtrapprox$ notation highlights the heuristic nature of the bound, used to build intuition.

To proceed, let $L_T \in \R^{T \times T}$ be the lower triangular matrix of all ones
and define $S_T := (L_TL_T^\T)^{-1}$.
A computation yields that $\Theta_{1,T,T}^{-1} = \frac{T+1}{2} S_T$. 
Note that we can write $S_T$ as a rank-one
perturbation to a tri-diagonal matrix.
Specifically, $S_T = \Tri{2}{-1}{T} - e_Te_T^\T$,
where $\Tri{a}{b}{T}$ denotes the symmetric $T \times T$ tri-diagonal matrix
with $a$ on the diagonal and $b$ on the lower and upper off-diagonals.
By the standard formula for the eigenvalues of a tri-diagonal matrix, we have that
$\lambda_{T-k+1}(\Tri{2}{-1}{T}) = 2\left(1 - \cos\left(\frac{k\pi}{T+1}\right)\right) \asymp k^2/T^2$.
In \Cref{sec:appendix:lower_bounds:eigenvalues},
we apply the work of \cite{kulkarni99tridiagonal}
to show that the rank-one perturbation is negligible:
$\lambda_{T-k+1}(S_T) \asymp k^2/T^2$ as well.
Therefore $\lambda_{T-k+1}(\Theta_{1,T,T}^{-1}) \asymp k^2/T$.
With this bound, we have:
\begin{align*}
    n &= \psi(\bar{x}\sqrt{n}; \mathsf{BDiag}(\Theta_{1,T,T}, m)) = \bar{x}\sqrt{n} \cdot m \sum_{i=1}^{T} \frac{1}{\lambda_i(\Theta_{1,T,T}^{-1}) + \bar{x}\sqrt{n}} \\
    &\lesssim \bar{x}\sqrt{n} \cdot m\sum_{i=1}^{T} \frac{1}{ i^2/T + \bar{x}\sqrt{n}} 
    \leq \bar{x}\sqrt{n} \cdot m\int_0^T \frac{1}{x^2/T + \bar{x}\sqrt{n}} \,\rmd x 
    \lesssim \sqrt{\bar{x}\sqrt{n}} \cdot m \sqrt{T}.
\end{align*}
This implies that $\bar{x}\sqrt{n} \gtrsim n^2/(m^2T)$, and therefore by \eqref{eq:SP_solution} and \eqref{eq:heuristic_approx_v1}:
\begin{align*}
    \mathsf{SP}(\mathsf{BDiag}(\Theta_{1,T,T}, m), n) = \frac{n}{\bar{x}\sqrt{n}} \lesssim \frac{m^2 T}{n} \Longrightarrow \E \Tr(\Gamma_{\Tnew}^{1/2} (X_{m,T}^\T X_{m,T})^{-1} \Gamma_{\Tnew}^{1/2}) \gtrapprox \frac{\Tnew}{T} \cdot \frac{n^2}{m^2 T}.
\end{align*}
We make this argument rigorous in \Cref{sec:appendix:lower_bounds:r_equals_one}.

\subsubsection{Beyond diagonalizability} 
\label{sec:proof_ideas:beyond_diag}

\begin{figure}[t]
\centering
\begin{minipage}{.5\textwidth}
  \centering
  \includegraphics[width=0.95\linewidth]{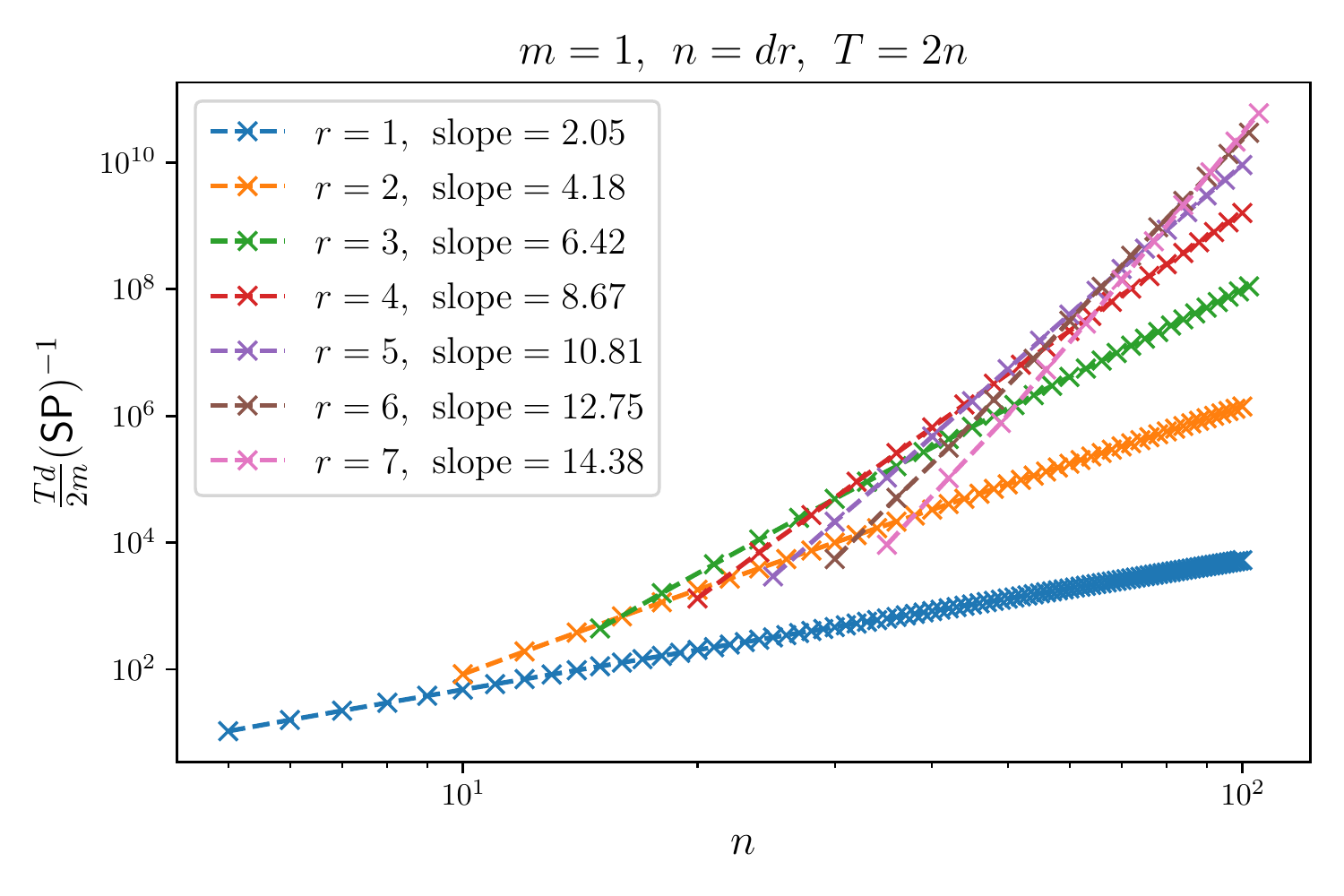}
  \subcaption{Plot of $\frac{Td}{2m} (\mathsf{SP})^{-1}$ versus $n$.}
  \label{fig:fix_m_vary_n}
\end{minipage}%
\hfill
\begin{minipage}{.5\textwidth}
  \centering
  \includegraphics[width=0.95\linewidth]{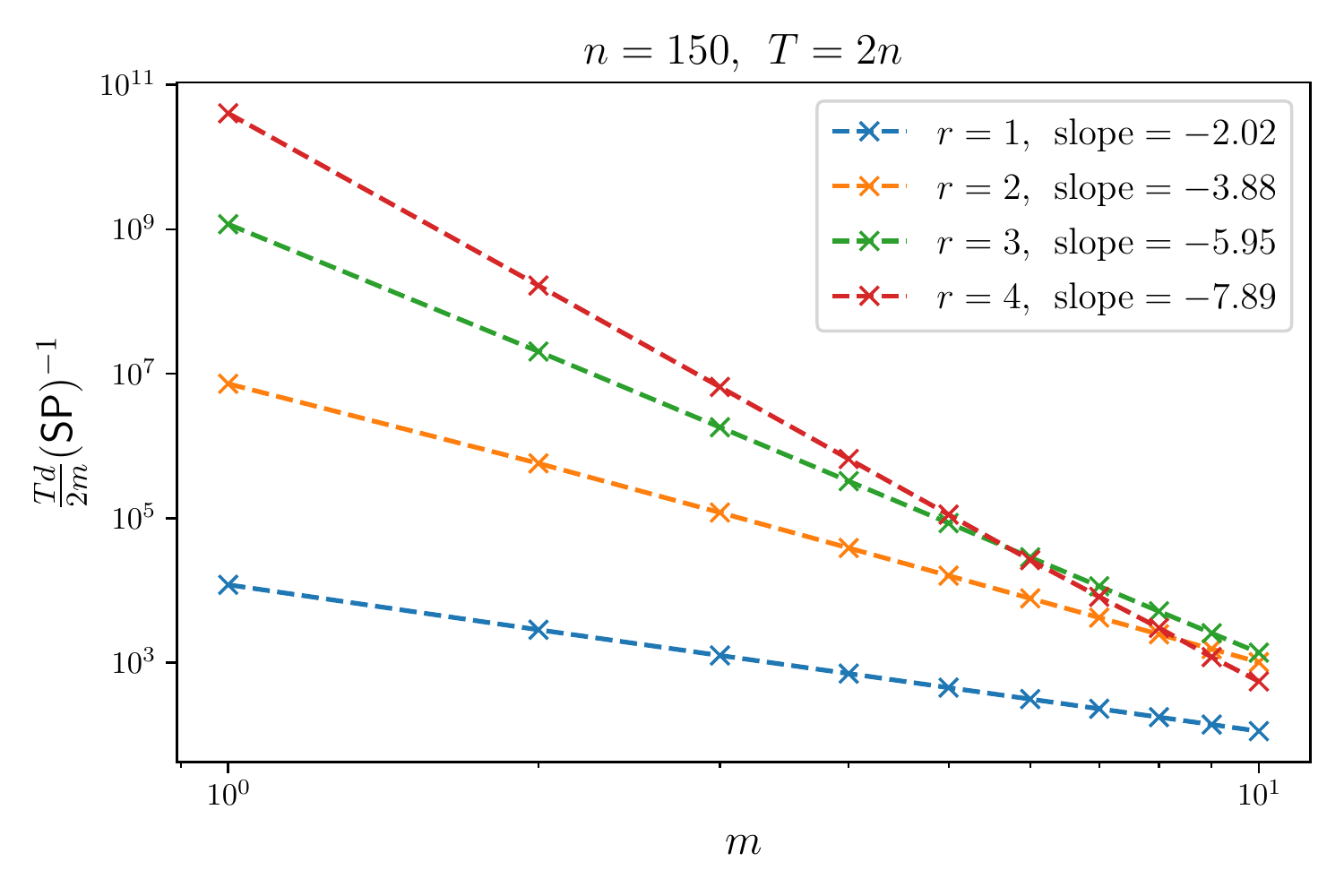}
  \subcaption{Plot of $\frac{Td}{2m} (\mathsf{SP})^{-1}$ versus $m$.}
  \label{fig:fix_n_vary_m}
\end{minipage}
\caption{Plot of $\frac{Td}{2m} (\mathsf{SP})^{-1}$ versus $n$ in (\subref{fig:fix_m_vary_n}) and versus $m$ in (\subref{fig:fix_n_vary_m}), both on a log-log scale.
For (\subref{fig:fix_m_vary_n}), $m$ and $p$ are fixed to one,
$d$ is fixed to $n/r$, and $T$ is fixed to $2n$.
For (\subref{fig:fix_n_vary_m}), $n$ is fixed to $150$, $p$ is fixed to one, and $T$ is fixed to $2n$.
In the legends, the slope of the line (in log-log space) computed via
linear regression is shown. 
Based on these plots, we conjecture that $\frac{Td}{2m}(\mathsf{SP})^{-1} \gtrsim c_r (d/m)^{2r}$, where $c_r$  depends only on $r$.}
\label{fig:risk_lower_bound}
\end{figure}

When $r \geq 2$, the analytic complexity 
of characterizing the solution to $n = \psi(\bar{x}\sqrt{n}; \mathsf{BDiag}(\Theta_{r,T,\Tnew}, m))$
increases significantly.
Nevertheless, we can still solve for $\bar{x}\sqrt{n}$ by numerical root finding,
to look at the scaling patterns 
for small values of $r$ and $\Tnew=T$.
This computation (\Cref{fig:risk_lower_bound}) leads 
us to conjecture
a general bound of
$\mathsf{R}(m,T,T;\{\PxA{\mathsf{BDiag}(J_r,n/r)}\}) \gtrapprox c_rn^{2r}/(m^{2r} T)$
when $m \lesssim n$,
where $c_r$ is a constant depending only on $r$.
A complete and precise statement is given in \Cref{sec:beyond-diag-results}.

\section{Numerical simulation}
\label{sec:experiments}

We conduct a simple numerical simulation illustrating the benefits of
multiple trajectories on learning. We construct a family of 
\DLR\ problem instances, parameterized by a scalar $\rho \in (0, \infty)$ as follows.
The covariate distribution $\Px$ is 
the linear dynamical system $x_{t+1} = A x_t + w_t$ with:
\begin{align*}
    A = U \diag( \underbrace{\rho, \dots, \rho}_{\floor{n/2} \textrm{ times}}, -\rho, \dots, -\rho ) U^\T, \quad U \sim \mathrm{Unif}(O(n)), \quad w_t \sim N(0, I/4).
\end{align*}
Here, $O(n)$ denotes the set of $n \times n$ orthonormal matrices.
By construction, $\rho$ is the spectral radius of $A$.
The labels $y_t$ are set as $y_{t} = x_{t+1}$, so that
the ground truth $W_\star \in \R^{n \times n}$ is equal to $A$.

We compare the risk of the OLS estimator \eqref{eq:ols_definition}
on the \DLR\ problem instance, compared with its risk on the corresponding
\PLR\ baseline. Specifically, 
we plot the ratio between OLS excess risks $\E[L(\cdot; T, \Px)]$ on the two problem instances ($\Px$),
respectively.
We fix the covariate dimension $n=5$ and
the trajectory horizon length $T=10n$, and vary the number of trajectories
$m \in \{1, \dots, 10\}$. \Cref{fig:multitraj} shows the result of this experiment, where we also vary $\rho \in \{0.98, 0.99, 1.0, 1.01, 1.02\}$.
The error bars are plotted over $1000$ trials. All computations are implemented 
using \texttt{jax}~\cite{jax2018github}, and run with \texttt{float64} precision on a single machine.
\begin{figure}[htb]
    \centering
    \includegraphics[width=0.6\linewidth]{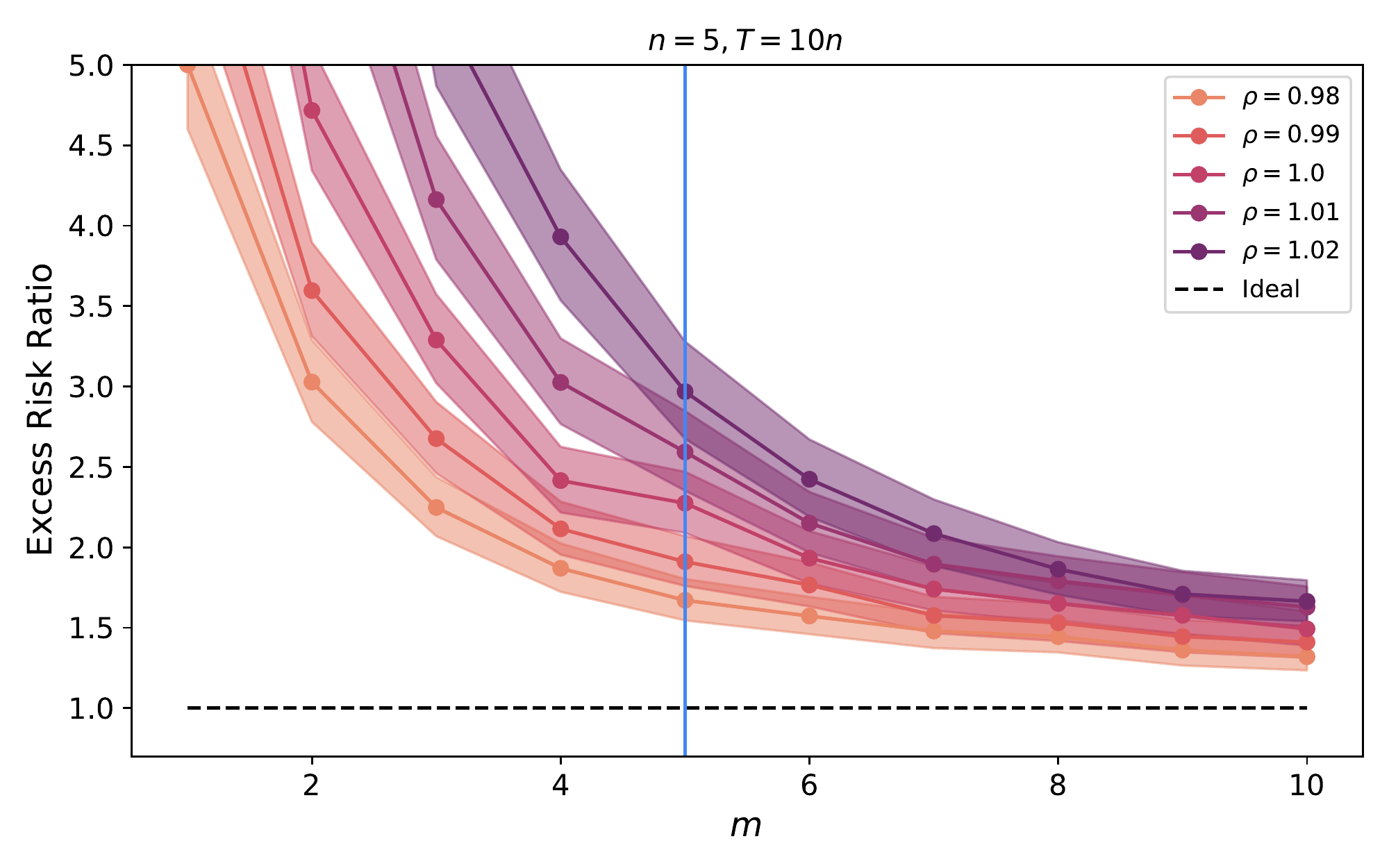}
    \caption{
    Plot of the ratio of excess risk for 
    \DLR\ problem instances over its corresponding
    \PLR\ baseline instance, as a function of the number of 
    trajectories $m$, holding both covariate dimension $n$ and
    horizon length $T$ fixed.
    The vertical blue line marks the transition between few trajectories ($m \leq n$) and many
    trajectories ($m \geq n$).}
    \label{fig:multitraj}
\end{figure}

In \Cref{fig:multitraj}, we see that for the few trajectories regime ($m \leq n$) appearing to the left of the
vertical blue line, the instability of
the covariate process plays an outstanding role in determining the
value of the ratio. On the other hand, for the
many trajectories regime ($m \geq n$) appearing to the right of the blue line, the ratios quickly converge to a constant no greater than two (at $m=10$). This behavior is consistent with 
\Cref{stmt:upper_bound_many_trajectories}.
Finally, \Cref{stmt:lower_bound_main} suggests that the scaling
behavior of the $\rho=1$ curve with respect to $m$ is on the order of
$1/m$.
\section{Concluding remarks}
\label{sec:conclusion}

Having sharply characterized the worst-case excess risk of \DSR{} and \DLR{}, we see more
precisely the trade-offs---or arguably the lack thereof---presented by resetting %
a system, or by simply observing parallel runs from one, where possible.
After sufficient resets, one learns roughly as though examples were independent altogether
(as reflected in the \PSR{} and \PLR{} baselines).

In addition to the theoretical upshot that it presents,
this phenomenon seems encouraging insofar as the setup may describe reality:
one does not learn to ride a bicycle by witnessing thousands of unrelated pedal strokes,
nor by watching one cyclist endure the entire Tour de France,
but rather by seeing and attempting many moderate rides and maneuvers.

We see a number of future directions for research,
primarily in further
charting out the reach of the \IID{}-like phenomenon in learning from multiple sequences.
Our work offers the trajectory small-ball criterion (\Cref{def:trajectory_small_ball})
as a vehicle for proving that this phenomenon occurs,
or otherwise for bounding the minimax rate from above.
What other notable sequential processes,
outside of those covered in~\Cref{sec:results:upper:small_ball_examples},
can we capture as trajectory small-ball instances?
One might look to covariate sequences generated by, say,
input-to-state stable (ISS) non-linear systems,
stochastic polynomial difference equations, 
or various Markov decision processes.

On the flip side, when must we necessarily pay a price for dependent data?
One answer from our work is that a necessary gap between
independent and sequentially dependent learning
appears when there are insufficiently many trajectories ($m \lesssim n$).
As outlined in~\Cref{sec:proof_ideas:beyond_diag} and~\Cref{sec:beyond-diag-results}, we conjecture that this gap can be
made much wider, namely by considering non-diagonalizable linear dynamical systems.
That said, other pertinent problems may exhibit gaps as well.
Finding them would help inform where the limits of learning from sequential data lie.

On the regression side, one might look to move beyond a well-specified linear regression model,
extend to other loss functions,
analyze regularized least-squares estimators in place of OLS,
or consider a more adversarial analysis (e.g.\ measuring regret rather than risk, in an online setting).

\section*{Acknowledgements}
We thank Vikas Sindhwani for organizing a lecture series on
learning and control, during which this work was prompted,
and for encouraging us to pursue the ensuing questions.
We also thank Sumeet Singh for pointing out 
that the claimed first-order optimality conditions
in \eqref{eq:critical_point_one} and \eqref{eq:critical_point_two}
require further conditions to hold for nonconvex/nonconcave games, e.g.,\ strictly concave stationary points of the inner maximization problem.
M.\ Soltanolkotabi is supported by the Packard Fellowship in Science and Engineering, a Sloan Research Fellowship in Mathematics, an NSF-CAREER under award \#1846369, DARPA Learning with Less Labels (LwLL) and FastNICS programs, and NSF-CIF awards \#1813877 and \#2008443.

\bibliographystyle{alpha}
\bibliography{paper}

\appendix

\allowdisplaybreaks

\section{Beyond diagonalizability: a conjecture for the general case}
\label{sec:beyond-diag-results}

Recall that various results in
\Cref{sec:upper_bounds:lds} required a diagaonalizability
assumption (\Cref{assume:diagonalizability}) on the dynamics matrix $A$,
specifically in the many trajectories regime when $\Tnew > T$
(\Cref{stmt:upper_bound_many_trajectories}), 
or in the few trajectories regime (\Cref{sec:upper_bounds:lds:few_trajectories}).
In this section, we conjecture how removing
the diagonalizability assumption would affect the results.
For simplicity, we focus on the few trajectories regime, and
further assume that $\Tnew = T$.
Building on potential extensions of this paper's analysis, and numerical evidence
detailed in \Cref{sec:proof_ideas}, we conjecture the following
extensions of
\Cref{thm:sublinear_trajectories_bound} and \Cref{stmt:lower_bound_main}:
\begin{myconjecture}[Risk for \LDSLS{} with few trajectories under non-diagonalizable systems]
\label{conj:general-few-trajectories}
There are universal positive constants $c_0$, $c_1$, $c_2$, $c_3$,
and a universal mapping $\varphi : \mathbb{N}_{+} \rightarrow \R_{>0}$
such that the following holds for any instance of \LDSLS{} satisfying \Cref{assume:marginal_stability} (marginal stability)
and \Cref{assume:one_step_controllability} (one-step controllability).
Let $A = S J S^{-1}$ denote the Jordan normal form of the dynamics matrix $A$.
Define
$\condNum := \frac{\lambda_{\max}(S^{-1} BB^\T S^{-*})}{\lambda_{\min}(S^{-1} BB^\T S^{-*})}$,
and let $r$ be the size of the largest Jordan block in $J$.
If $n \geq c_0$, $m \leq c_1 n$, and $mT \geq c_2 n$, then:
\begin{align}
    \E[L(\hat{W}_{m,T}; \Tnew, \PxA{A,B})]
    \leq c_3 \sigma_\xi^2 \varphi(r) \condNum
    \cdot \frac{pn^{2r}} {m^{2r} T}.
    \label{eq:conjecture_upper_bound}
\end{align}
Additionally,
there exist universal positive constants $c'_0$, $c'_1$, $c'_2$, $c'_3$, and $c'_4$ such that the following is true.
Suppose $\calA \subseteq \R^{n \times n}$ is any set
containing all $n \times n$ matrices with Jordan blocks of size at most $r$.
Let $T \geq c'_0$, $n \geq c'_1$, $mT \geq c'_2 n$, and $m \leq c'_3 n$.
Then:
\begin{align}
    \mathsf{R}(m, T, T; \{ \PxA{A} \mid A \in \calA \})
    \geq c'_4 \sigma_\xi^2 \varphi(r) \condNum
    \cdot \frac{pn^{2r}}{m^{2r} T}.
    \label{eq:conjecture_lower_bound}
\end{align}
\end{myconjecture}

\Cref{stmt:upper_bound_general}
provides a viable path towards proving the upper bound
\eqref{eq:conjecture_upper_bound} from \Cref{conj:general-few-trajectories}
up to logarithmic factors
in the regime of constant Jordan block size $r$,
by reducing the problem to understanding
the scaling of $\ulam(k,t;A,B) = \ulam(\Gamma_k(A, B), \Gamma_t(A, B))$ when $k \leq t$.
Our analysis uses diagonalizability (\Cref{assume:diagonalizability}) 
of the dynamics matrices to show that $\ulam(k,t;A,B) \gtrsim \gamma^{-1} \cdot k/t$.
Without such an assumption, analyzing $\ulam(k,t;A,B)$ is
substantially more involved.
\begin{figure}[htb]
    \centering
    \includegraphics[width=0.6\linewidth]{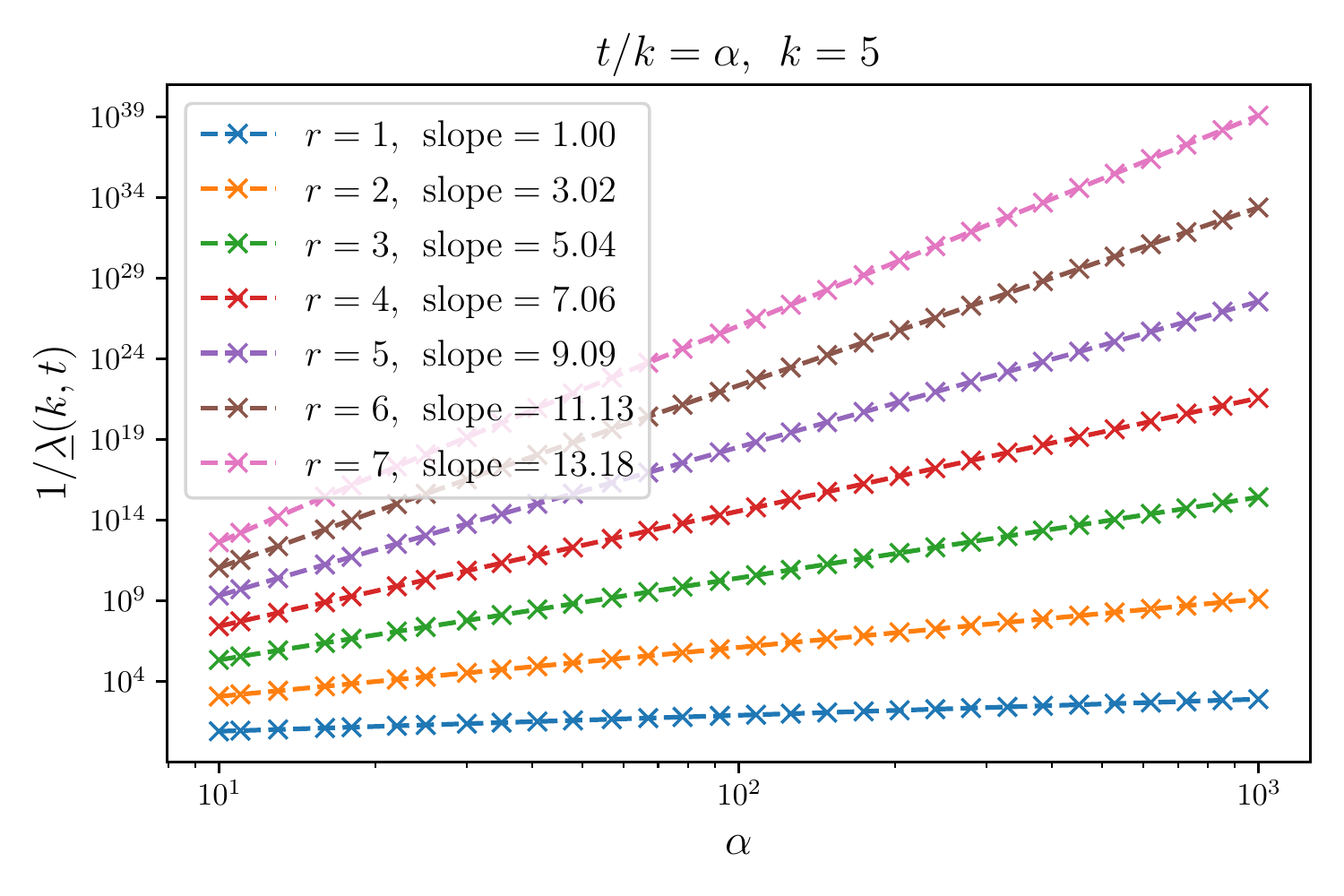}
    \caption{
    A plot of the ratio $\alpha$ versus $1/\ulam(k,t)$ with $k$ fixed to $5$ and $t$ fixed to $k \alpha$.
    Here, $\ulam(k,t) := \ulam(k,t;J_r,I_r)$.
    The slope of the line
    (in log-log space) computed via linear regression is reported.
    We conjecture that in general, $\ulam(k,t;A,B) \gtrsim c_{r} \gamma^{-1} \cdot  (k/t)^{2r-1}$.}
    \label{fig:ulam_T_k}
\end{figure}
A numerical simulation (\Cref{fig:ulam_T_k})
suggests that $\ulam(k,t;A,B) \gtrsim c_{r} \gamma^{-1} \cdot (k/t)^{2r-1}$ is
the general rate for dynamics matrices $A$ with Jordan
blocks at most size $r$, where $c_{r}$ is a constant depending only on $r$.
Assuming this scaling to be correct and plugging the rate
into \Cref{stmt:upper_bound_general} yields
\eqref{eq:conjecture_upper_bound} up to logarithmic factors.
Partial progress towards analyzing $\ulam(k,t;A,B)$ 
was made in \cite[Proposition~7.6]{sarkar2019sysid}, where
it is shown that $\ulam(k,t;A,B) \gtrsim c_{r} \gamma^{-1} \cdot (k/t)^{r^2}$,
with $1/c_{r}$ depending exponentially on $r$.
We do not conjecture a form for the
mapping $\varphi(r)$;
$\ulam(k,t;A,B)$ becomes numerically ill-conditioned when
$r$ is large, hindering simulation with large blocks.

On the other hand, the analytic arguments in \Cref{sec:lower_bound_proof_sketch:few_trajectories} 
combined with the numerical evidence in \Cref{fig:risk_lower_bound}
suggest that the bound \eqref{eq:conjecture_lower_bound} holds 
(up to the condition number factor $\condNum$). 
The one caveat is that, even if we were to analytically
characterize the eigenvalues of $\Theta_{r,T,T}$ for all $r$, our proof strategy
would most likely not be able to give a sharp characterization 
of the leading constant $\varphi(r)$ in the lower bound.
This is because our
proof inherently exploits the independence between 
decoupled subsystems, and does not tackle the harder challenge 
of understanding the coupling effects within a Jordan block.

We conclude this section by noting that \Cref{conj:general-few-trajectories} 
does not include any logarithmic factors in
the upper bound rate \eqref{eq:conjecture_upper_bound},
and includes the condition number factor $\gamma$
in the lower bound \eqref{eq:conjecture_lower_bound}.
In other words, \Cref{conj:general-few-trajectories}
applied to the special case of $r=1$ conjectures 
that \Cref{thm:sublinear_trajectories_bound} is loose by $\log^2(\gamma n/m)$,
and that \Cref{stmt:lower_bound_main} is loose by a factor of $\condNum$.

\section{Analysis for upper bounds}
\label{sec:appendix:upper_bound_proofs}

\subsection{Preliminaries}

We collect various technical results
which we will use in the proof of the upper bounds.
The first result gives us a bound
on the functional inverse of $T \mapsto T/\log{T}$.
\begin{myprop}[{\cite[Lemma~A.4]{simchowitz18learning}}]
\label{prop:invert_log_t_over_t}
For $\alpha \geq 1$, $T \geq 2\alpha \log(4\alpha)$ implies that $T \geq \alpha \log{T}$.
\end{myprop}

The next two results study various properties of functions
involving $\ulam$.
\begin{myprop}
\label{stmt:ulam_concave_first_arg}
For $A \in \sfSym^{n}_{> 0}$, the map $X \mapsto \ulam(X, A)$ is concave over symmetric matrices.
\end{myprop}
\begin{proof}
Observe that $\ulam(X, A) = \lambda_{\min}(A^{-1/2} X A^{-1/2}) = \inf\{ \ip{X}{A^{-1/2} vv^\T A^{-1/2}} \mid v \in \mathbb{S}^{n-1}\}$ 
is the pointwise infimum over a set of linear functions in $X$, and is therefore concave.
\end{proof}

\begin{myprop}
\label{stmt:Psi_leq_one}
Fix $T \in \N_{+}$,
$\{ \Psi_t\}_{t=1}^{T} \subset \sfSym^n_{> 0}$, and $\Gamma \in \sfSym^n_{> 0}$.
Suppose that $\frac{1}{T} \sum_{t=1}^{T} \Psi_t \preccurlyeq \Gamma$.
Then $\left[ \prod_{t=1}^{T} \ulam(\Psi_t, \Gamma) \right]^{1/T} \leq 1$.
\end{myprop}
\begin{proof}
We have that:
\begin{align*}
    \left[ \prod_{t=1}^{T} \ulam(\Psi_t, \Gamma) \right]^{1/T} &\leq \frac{1}{T} \sum_{t=1}^{T} \ulam(\Psi_t, \Gamma) &&\text{using the AM-GM inequality} \\
    &\leq \ulam\left(\frac{1}{T} \sum_{t=1}^{T} \Psi_t, \Gamma \right) &&\text{using \Cref{stmt:ulam_concave_first_arg} and Jensen's inequality} \\
    &\leq \ulam(\Gamma, \Gamma) &&\text{since } \frac{1}{T} \sum_{t=1}^{T} \Psi_t \preccurlyeq \Gamma \\
    &= 1.
\end{align*}
\end{proof}

The next result relates the anti-concentration properties of a non-negative
random variable to its moment generating function on
$(-\infty, 0)$.
\begin{myprop}[{\cite[Lemma~7]{mourtada19exactminimax}}]
\label{stmt:small_ball_to_mgf}
Let $X$ be a non-negative random variable. Suppose
there exists an $\alpha \in (0, 1]$ and positive constant
$c$ such that:
\begin{align*}
    \Pr(X \leq t) \leq (ct)^{\alpha} \quad \forall t > 0.
\end{align*}
Then: 
\begin{align*}
    \E[\exp(-\eta X)] \leq (c/\eta)^\alpha \quad \forall \eta > 0.
\end{align*}
\end{myprop}

The next few results involve various properties of 
Gaussian and spherical distributions.

\begin{myprop}[{\cite[Lemma~6.2]{magnus1978gaussianforms}}]
\label{stmt:gaussian_fourth_moment}
For $w \sim N(0, I)$ and symmetric matrices $A,B$:
\begin{align*}
    \E[ w^\T A w w^\T B w ] = 2\ip{A}{B} + \Tr(A)\Tr(B).
\end{align*}
\end{myprop}

\begin{myprop}[{\cite[Lemma~2.2]{dasgupta2003jl}}]
\label{stmt:small_ball_unit_sphere}
Let $n \geq 2$ and $v \in \R^n \setminus \{0\}$ be fixed.
Suppose that $\psi$ is drawn uniformly at random from the uniform measure
over $\mathbb{S}^{n-1}$. We have that for all $\varepsilon > 0$:
\begin{align*}
    \Pr\left\{ \ip{v}{\psi}^2 \leq \frac{\varepsilon}{n} \norm{v}_2^2 \right\} \leq (e \varepsilon)^{1/2}.
\end{align*}
\end{myprop}

Next, we state a classic result which gives
us anti-concentration of arbitrary Gaussian (more generally any log-concave distribution) polynomials of bounded degree.
\begin{mythm}[{\cite[Theorem~8]{carbery2001anticonc}}]
\label{thm:carbery_and_wright}
Fix an integer $d \in \N_{+}$. There exists a universal positive constant $c$
such that the following is true.
Let $p : \R^n \rightarrow \R$ be a degree $d$ polynomial,
and let $\varepsilon > 0$.
We have:
\begin{align*}
    \Pr\{ \abs{p(x)} \leq \varepsilon \cdot \E\abs{p(x)} \} \leq c \cdot d \varepsilon^{1/d}, \quad x \sim N(0, I_n).
\end{align*}
For the case when $d=2$ and $p$ is non-negative, we can take
$c=\sqrt{e/2}$.
\end{mythm}
\Cref{thm:carbery_and_wright} can be further specialized as follows. Suppose $w \sim N(0, I)$,
$x$ is fixed, and $\bmattwo{Q_{11}}{Q_{12}}{Q_{12}^\T}{Q_{22}}$
is positive semidefinite. Then:
\begin{align}
    \Pr\left\{  \cvectwo{x}{w}^\T  \bmattwo{Q_{11}}{Q_{12}}{Q_{12}^\T}{Q_{22}} \cvectwo{x}{w} \leq \varepsilon \cdot \Tr(Q_{22}) \right\} \leq (e \varepsilon)^{1/2} \quad \forall \varepsilon > 0. \label{eq:cw_quadratic_explicit_constant}
\end{align}
Both \eqref{eq:cw_quadratic_explicit_constant} and
the explicit constant in \Cref{thm:carbery_and_wright} for $d=2$ and $p$ non-negative
can be derived by bounding the MGF of various Gaussian quadratic forms; see e.g.~\cite{tu2023gaussian}.

Next, we state a well-known result from \cite{abbasiyadkori2011selfnormalized},
which yields an anytime bound for the size of a
self-normalized martingale difference sequence (MDS).
\begin{mylemma}[{\cite[Theorem~3]{abbasiyadkori2011selfnormalized}}]
\label{lemma:yasin}
Fix a $\delta \in (0, 1)$
and positive definite matrix $V \in \R^{d \times d}$.
Let $\{x_t\}_{t \geq 1} \subset \R^d$ be a stochastic process
adapted to the filtration $\{\calF_t\}_{t \geq 1}$.
Let $\{\eta_t\}_{t \geq 1} \subset \R$ be a martingale difference
sequence adapted to $\{\calF_t\}_{t \geq 2}$.
Suppose there exists $R > 0 $ such that $\E[\exp(\lambda \eta_t) \mid \calF_t] \leq \exp(\lambda^2 R^2/2)$ a.s. for all $\lambda \in \R$ and $t \geq 1$.
Define $V_t := \sum_{k=1}^{t} x_kx_k^\T$ for $t \geq 1$.
With probability at least $1-\delta$,
\begin{align*}
    \bignorm{\sum_{k=1}^{t} \eta_k x_k }_{(V_t + V)^{-1}} \leq \sqrt{2 R^2 \log\left( \frac{\det(V_t + V)^{1/2} \det(V)^{-1/2} }{\delta} \right)} \quad \forall t \geq 1.
\end{align*}
\end{mylemma}

\Cref{lemma:yasin} is generalized to vector-valued self-normalized MDS via a covering argument:
\begin{myprop}[{\cite[Proposition~8.2]{sarkar2019sysid}}]
\label{prop:yasin_vector}
Fix a $\delta \in (0, 1)$
and positive definite matrix $V \in \R^{d \times d}$.
Let $\{x_t\}_{t \geq 1} \subset \R^d$ be a stochastic process
adapted to the filtration $\{\calF_t\}_{t \geq 1}$.
Let $\{\eta_t\}_{t \geq 1} \subset \R^p$ be a stochastic process adapted to $\{\calF_t\}_{t \geq 2}$.
Suppose that for every fixed $v \in \mathbb{S}^{p-1}$, for every $t \geq 1$:
\begin{enumerate}[label=(\alph*)]
    \item $\E[ \ip{v}{\eta_t} \mid \calF_t ] = 0$ a.s.
    \item $\E[ \exp(\lambda \ip{v}{\eta_t}) \mid \calF_t ] \leq \exp(\lambda^2 R^2/2)$ a.s. for every $\lambda \in \R$.
\end{enumerate}
Define $V_t := \sum_{k=1}^{t} x_kx_k^\T$ for $t \geq 1$.
With probability at least $1-\delta$,
for all $t \geq 1$:
\begin{align*}
    \bignorm{ \sum_{k=1}^{t} \eta_k x_k^\T (V_t+V)^{-1/2}}_{\mathrm{op}} \leq 2 \sqrt{2 R^2 \log\left( \frac{5^p\det(V_t + V)^{1/2} \det(V)^{-1/2} }{\delta} \right)}.
\end{align*}
\end{myprop}

The next result assumes $V_t$ is invertible in order to simplify
\Cref{prop:yasin_vector}.
\begin{myprop}
\label{prop:yasin_vector_easier}
Under the same hypothesis of \Cref{prop:yasin_vector}, we have
with probability at least $1-\delta$, for all $t \geq 1$:
\begin{align*}
    \ind\{V_t \succcurlyeq V\} \bignorm{ \sum_{k=1}^{t} \eta_k x_k^\T V_t^{-1/2}}_{\mathrm{op}} \leq 4 \cdot \ind\{V_t \succcurlyeq V\} \sqrt{R^2 \log\left( \frac{5^p\det(V_t + V)^{1/2} \det(V)^{-1/2} }{\delta} \right)}.
\end{align*}
\end{myprop}
\begin{proof}
Observe that when $V_t \succcurlyeq V$, we have:
\begin{align*}
    2 V_t \succcurlyeq V_t + V \Longrightarrow V_t^{-1} \preccurlyeq 2 (V_t + V)^{-1}.
\end{align*}
For two positive definite matrices $M_1$ and $M_2$ satisfying
$M_1 \preccurlyeq M_2$,
and any matrix $N$,
\begin{align*}
    \opnorm{N M_1^{1/2}} = \sqrt{\lambda_{\max}(N M_1 N^\T)} \leq \sqrt{\lambda_{\max}(N M_2 N^\T)} = \opnorm{N M_2^{1/2}}.
\end{align*}
Therefore,
\begin{align*}
    \ind\{V_t \succcurlyeq V\} \bignorm{ \sum_{k=1}^{t} \eta_k x_k^\T V_t^{-1/2}}_{\mathrm{op}} &\leq 2 \cdot \ind\{V_t \succcurlyeq V\} \bignorm{ \sum_{k=1}^{t} \eta_k x_k^\T (V_t+V)^{-1/2}}_{\mathrm{op}} \\
    &\leq 4 \cdot \ind\{V_t \succcurlyeq V\} \sqrt{R^2 \log\left( \frac{5^p\det(V_t + V)^{1/2} \det(V)^{-1/2} }{\delta} \right)},
\end{align*}
where the last inequality holds for every $t$ with probability at least $1-\delta$ by \Cref{prop:yasin_vector}.
\end{proof}

\subsection{Examples of trajectory small-ball}

In this section, we prove that the examples listed in
\Cref{sec:results:upper:small_ball_examples} satisfying the
trajectory small-ball condition (\Cref{def:trajectory_small_ball}).

\examplesimple*
\begin{proof}
When $k=T$ and $\uGam = I_n$, the condition
\eqref{eq:trajectory_small_ball} simplifies to:
\begin{align*}
    \sup_{v \in \mathbb{S}^{n-1}} \Pr\left\{ \frac{1}{T} \sum_{t=1}^{T} \ip{v}{ x_t}^2 \leq \varepsilon \right\} \leq (\csb \varepsilon)^{\alpha} \quad \forall \varepsilon > 0.
\end{align*}
Since $x_1 = x_2 = \dots = x_T$, this further simplifies to:
\begin{align*}
    \sup_{v \in \mathbb{S}^{n-1}} \Pr\left\{ \ip{v}{x_1}^2 \leq \varepsilon \right\} \leq (\csb \varepsilon)^{\alpha} \quad \forall \varepsilon > 0.
\end{align*}
Since $\ip{v}{x_1} \sim N(0, 1)$,
\Cref{eq:cw_quadratic_explicit_constant} yields that
$\Pr_{X \sim N(0, 1)}\{ X^2 \leq \varepsilon \} \leq (e \varepsilon)^{1/2}$,
so we can take $\csb = e$ and $\alpha = 1/2$.
\end{proof}

\examplegaussianprocess*
\begin{proof}
Since the covariates $(x_1, \dots, x_T)$ are jointly Gaussian,
we can write,
\begin{align*}
    \begin{bmatrix} x_1 \\ \vdots \\ x_T \end{bmatrix} = \begin{bmatrix} \mu_1 \\ \vdots \\ \mu_T \end{bmatrix} + \begin{bmatrix} M_1 \\ \vdots \\ M_T \end{bmatrix} w,
\end{align*}
where $\mu_1, \dots, \mu_T \in \R^n$ 
and $M_1, \dots, M_T \in \R^{n \times nT}$ are fixed, and
$w \sim N(0, I_{nT})$.
For any $v \in \R^n$,
\begin{align*}
    \frac{1}{T}\sum_{t=1}^{T} \ip{v}{x_t}^2 = \frac{1}{T}\sum_{t=1}^{T} \ip{v}{\mu_t + M_t w}^2.
\end{align*}
This is a degree 2 non-negative polynomial in $w$, and therefore
by \Cref{thm:carbery_and_wright}, for all $\varepsilon > 0$:
\begin{align*}
    \Pr\left\{ \frac{1}{T}\sum_{t=1}^{T} \ip{v}{x_t}^2 \leq \varepsilon \E\left[ \frac{1}{T} \sum_{t=1}^{T} \ip{v}{x_t}^2 \right] \right\} \leq (2e \varepsilon)^{1/2}.
\end{align*}
\end{proof}

\examplepartition*
\begin{proof}
For $i \in \{0, 1\}$, let
$\psi_i$ be uniform on the uniform measure over $U_i \cap \mathbb{S}^{n-1}$,
let $P_{U_i}$ denote the orthogonal projector onto $U_i$,
and let $v_i = P_{U_i} v$.

Fix any $v \in \R^n \setminus \{0\}$. We observe that for any $t \in \N_{+}$,
$\ip{v}{x_{t+1}}^2 + \ip{v}{x_{t+2}}^2 \mid i_t$ is equal
in distribution to $\ip{v}{\psi_0}^2 + \ip{v}{\psi_1}^2$, which itself is equal in distribution to $\ip{v_0}{\psi_0}^2 + \ip{v_1}{\psi_1}^2$.
Suppose first that $\norm{v_0}_2 \geq \norm{v_1}_2$. Then,
since $\norm{v}_2^2 = \norm{v_0}_2^2 + \norm{v_1}_2^2 \leq 2 \norm{v_0}_2^2$:
\begin{align*}
    \left\{ \ip{v_0}{\psi_0}^2 + \ip{v_1}{\psi_1}^2 \leq \frac{\varepsilon}{n} \norm{v}_2^2 \right\} &\subseteq \left\{ \ip{v_0}{\psi_0}^2 + \ip{v_1}{\psi_1}^2 \leq \frac{2\varepsilon}{n} \norm{v_0}_2^2 \right\} \\
    &\subseteq \left\{ \ip{v_0}{\psi_0}^2 \leq \frac{2\varepsilon}{n} \norm{v_0}_2^2 \right\}.
\end{align*}
Writing $\alpha_0 = (\ip{u_1}{v}, \dots, \ip{u_{n/2}}{v}) \in \R^{n/2}$, 
by a change of coordinates 
we have that $\norm{\alpha_0}_2^2 = \norm{v_0}_2^2$, and that
$\ip{v_0}{\psi_0}$ is equal in distribution to
$\ip{\alpha_0}{\zeta_0}$, where $\zeta_0$ is uniform on $\mathbb{S}^{n/2-1}$.
Since we assumed $\norm{v_0}_2 \geq \norm{v_1}_2$, we must have that $\alpha_0 \neq 0$. Hence by \Cref{stmt:small_ball_unit_sphere},
\begin{align*}
    \Pr\left\{\ip{v_0}{\psi_0}^2 + \ip{v_1}{\psi_1}^2 \leq \frac{\varepsilon}{n} \norm{v}_2^2\right\} \leq \Pr\left\{\ip{\alpha_0}{\zeta_0}^2 \leq \frac{2\varepsilon}{n} \norm{\alpha_0}_2^2 \right\} \leq (e \varepsilon)^{1/2}.
\end{align*}
Note that if $\norm{v_1}_2 > \norm{v_0}_2$, an identical argument
yields the same bound.
Hence, letting $\calF_t = \sigma(x_1, \dots, x_t)$,
we have shown that for all $t \geq 0$:
\begin{align*}
    \Pr\left\{ \frac{1}{2} \sum_{\ell=1}^{2} \ip{v}{x_{t+\ell}}^2 \leq \varepsilon \cdot v^\T \left(\frac{1}{2n} I_n\right) v \,\bigg|\, \calF_t\right\} \leq (e \varepsilon)^{1/2},
\end{align*}
from which the claim follows.
\end{proof}

\mixingchain*
\begin{proof}

Fix any $v \in \R^n \setminus \{0\}$,
and for $i \in \{1, \dots, n\}$,
let $v_i = P_{U_{\neg i}} v$, where $P_{U_{\neg i}}$
is the orthogonal projector onto $U_{\neg i}$
Let $\{\psi_i\}_{i=1}^{n}$ be independent random variables, where
each $\psi_i$ is uniform on the 
uniform measure over $U_{\neg i} \cap \mathbb{S}^{n-1}$.

Let indices $j,k \in \{1, \dots, n\}$ with $j \neq k$.
We first observe that
since $j \neq k$, we have that $U_{\neg j}^\perp = \Span(u_j) \subset U_{\neg k}$.
Therefore:
\begin{align*}
    \norm{v}_2^2 = \norm{v_j}_2^2 + \norm{P^\perp_{U_{\neg j}} v_j}_2^2 \leq \norm{v_j}_2^2 + \norm{v_k}_2^2.
\end{align*}
Hence, assuming that
$\norm{v_j}_2 \geq \norm{v_k}_2$, we have:
\begin{align*}
    \left\{ \ip{v_j}{\psi_j}^2 + \ip{v_k}{\psi_k}^2 \leq \frac{\varepsilon}{2(n-1)} \norm{v}_2^2 \right\} &\subseteq \left\{ \ip{v_j}{\psi_j}^2 + \ip{v_k}{\psi_k}^2 \leq \frac{\varepsilon}{n-1} \norm{v_j}_2^2 \right\} \\
    &\subseteq \left\{ \ip{v_j}{\psi_j}^2 \leq \frac{\varepsilon}{n-1} \norm{v_j}_2^2 \right\}.
\end{align*}
Writing $\alpha_j = ( \ip{u_i}{v} )_{i \neq j} \in \R^{n-1}$,
by a change of coordinates we have that $\norm{\alpha_j}_2^2 = \norm{v_j}_2^2$,
and that $\ip{v_j}{\psi_j}$ is equal in distribution to $\ip{\alpha_j}{\zeta_j}$,
where $\zeta_j$ is uniform on $\mathbb{S}^{n-2}$.
Since we assumed $\norm{v_j}_2 \geq \norm{v_k}_2$, we must have that $\alpha_j \neq 0$. Hence by \Cref{stmt:small_ball_unit_sphere},
\begin{align*}
    \Pr\left\{ \ip{v_j}{\psi_j}^2 + \ip{v_k}{\psi_k}^2 \leq \frac{\varepsilon}{2(n-1)} \norm{v}_2^2 \right\} \leq \Pr\left\{ \ip{\alpha_j}{\zeta_j}^2 \leq \frac{\varepsilon}{n-1} \norm{\alpha_j}_2^2  \right\} \leq (e \varepsilon)^{1/2}.
\end{align*}
On the other hand if $\norm{v_k}_2 > \norm{v_j}_2$, an identical argument yields the same bound.

Now, for any $i \in \{1, \dots, n\}$ and $t \in \N_{+}$:
\begin{align*}
    &~~~~\Pr\left\{ \ip{v}{x_{t+1}}^2 + \ip{v}{x_{t+2}}^2 \leq \frac{\varepsilon}{2(n-1)} \norm{v}_2^2 \,\bigg|\, i_t=i \right\} \\
    &= \sum_{j \neq i, k \neq j} \Pr\left\{ \ip{v}{x_{t+1}}^2 + \ip{v}{x_{t+2}}^2 \leq \frac{\varepsilon}{2(n-1)} \norm{v}_2^2 \,\bigg|\, i_t=i, i_{t+1}=j, i_{t+2}=k  \right\} \Pr\{i_{t+1}=j,i_{t+2}=k \mid i_t=i\} \\
    &= \sum_{j \neq i, k \neq j} \Pr\left\{ \ip{v_j}{\psi_j}^2 + \ip{v_k}{\psi_k}^2 \leq \frac{\varepsilon}{2(n-1)} \norm{v}_2^2 \right\} \Pr\{i_{t+1}=j,i_{t+2}=k \mid i_t=i\} \\
    &\leq (e \varepsilon)^{1/2} \sum_{j \neq i, k \neq j}  \Pr\{i_{t+1}=j,i_{t+2}=k \mid i_t=i\} \\
    &= (e \varepsilon)^{1/2}.
\end{align*}
Note we also have that
$\Pr\left\{ \ip{v}{x_{1}}^2 + \ip{v}{x_{2}}^2 \leq \frac{\varepsilon}{2(n-1)} \norm{v}_2^2 \right\} \leq (e \varepsilon)^{1/2}$ by a nearly identical argument.
Hence, letting $\calF_t = \sigma(x_1, \dots, x_t)$, we have shown that
for all $t \geq 0$:
\begin{align*}
    \Pr\left\{ \frac{1}{2} \sum_{\ell=1}^{2} \ip{v}{x_{t+\ell}}^2 \leq \varepsilon \cdot v^\T \left( \frac{1}{4(n-1)} I_n \right) v \,\bigg|\, \calF_t \right\} \leq (e \varepsilon)^{1/2},
\end{align*}
from which the claim follows.
\end{proof}

For the next claim, recall that the mixing time of a Markov
chain over state-space $S$ with transition matrix $P$ and stationary
distribution $\pi$ is defined as:
\begin{align*}
    \tmix(\varepsilon) = \inf\left\{ k \in \N \,\bigg|\, \sup_{\mu \in \calP(S)} \tvnorm{ \mu P^k - \pi} \leq \varepsilon \right\}.
\end{align*}
Here, $\calP(S)$ denotes the set of distributions over $S$, and $\tvnorm{\cdot}$ is the total-variation norm
over distributions.

\begin{myprop}
\label{stmt:mixing_time_simple_chain}
Let $n \geq 2$.
Consider the Markov chain $\{i_t\}_{t \geq 1}$
where $i_1 \sim \mathrm{Unif}(\{1, \dots, n\})$ and
$i_{t+1} \mid i_t \sim \mathrm{Unif}(\{1, \dots, n\} \setminus \{i_t\})$.
We have that:
\begin{align*}
    \tmix(\varepsilon) = \inf\left\{ k \in \N \,\bigg|\, (n-1)^{-k} \leq \frac{2\varepsilon}{1-1/n} \right\}.
\end{align*}
\end{myprop}
\begin{proof}
Let $\ind \in \R^n$ denote the all ones vector.
The transition matrix for this Markov chain is:
\begin{align*}
    P = \frac{1}{n-1} (\ind \ind^\T - I_n),
\end{align*}
and its stationary distribution is uniform over $\{1, \dots, n\}$.
Note that for $j \geq 1$,
$(\ind\ind^\T)^j = n^{j-1} \ind\ind^\T$.
Since $\ind\ind^\T$ and $I_n$ commute, by the binomial theorem
we have that:
\begin{align*}
    P^k &= \frac{1}{(n-1)^k} \sum_{j=0}^{k} \binom{k}{j} (\ind \ind^\T)^{k-j} (-1)^j \\
    &= \frac{1}{(n-1)^k} \left[ \sum_{j=0}^{k-1} \binom{k}{j} n^{k-j-1} (-1)^j \ind\ind^\T + (-1)^k I_n \right] \\
    &= \frac{1}{(n-1)^k} \left[ \frac{1}{n} \left( (n-1)^k - (-1)^k \right) \ind\ind^\T + (-1)^k I_n \right] \\
    &= \frac{1}{n} \ind\ind^\T + \frac{(-1)^k}{(n-1)^k} \left[ I_n - \frac{1}{n} \ind\ind^\T \right].
\end{align*}
Now, let $\mu \in \R^n_{\geq 0}$ satisfy $\mu^\T \ind = 1$.
We have:
\begin{align*}
    \bignorm{\mu^\T P^k - \frac{1}{n} \ind^\T}_1 &= \frac{1}{(n-1)^k} \bignorm{\mu - \frac{1}{n} \ind}_1.
\end{align*}
It is straightforward to check that
$\sup_{\mu \in \R^n_{\geq 0}, \mu^\T \ind = 1} \bignorm{\mu - \frac{1}{n} \ind}_1 = 1-\frac{1}{n}$,
from which the claim follows,
since the TV distance between two distributions
$\mu$, $\nu$ is $\tvnorm{\mu-\nu} = \frac{1}{2} \norm{\mu-\nu}_1$.
\end{proof}

\examplelds*
\begin{proof}
Let $\Gamma_k$ be shorthand for $\Gamma_k(\Px)$
and $\Sigma_k$ be shorthand for $\Sigma_k(\Px)$.
Let $w = (w_1, \dots, w_{k}) \in \R^{nk}$ denote the vertical
concatenation of the process noise variables.
Let $M_t := \rvectwo{A^t}{\Phi_t} \in \R^{n \times n(k+1)}$ denote the matrix
such that $x_t = M_t \cvectwo{x}{w}$.
With this notation, for any $v \in \R^n$:
\begin{align*}
    \frac{1}{k} \sum_{t=1}^{k} \ip{v}{x_t}^2 = \cvectwo{x}{w}^\T \left( \frac{1}{k} \sum_{t=1}^{k} M_t^\T vv^\T M_t \right) \cvectwo{x}{w}.
\end{align*}
By \Cref{eq:cw_quadratic_explicit_constant},
for any $\varepsilon > 0$,
\begin{align*}
    \Pr\left\{ \cvectwo{x}{w}^\T \left( \frac{1}{k} \sum_{t=1}^{k} M_t^\T vv^\T M_t \right) \cvectwo{x}{w} \geq \varepsilon \cdot \Tr\left( \frac{1}{k} \sum_{t=1}^{k} \Phi_t^\T vv^\T \Phi_t \right) \right\} \leq (e \varepsilon)^{1/2}.
\end{align*}
On the other hand:
\begin{align*}
     \Tr\left( \frac{1}{k} \sum_{t=1}^{k} \Phi_t^\T vv^\T \Phi_t \right) = v^\T \left( \frac{1}{k} \sum_{t=1}^{k}  \Phi_t \Phi_t^\T \right) v = v^\T \left( \frac{1}{k} \sum_{t=1}^{k} \Sigma_t \right) v = v^\T \Gamma_k v.
\end{align*}
Because we assumed that $k \geq k_c$, then $\Gamma_k$ is invertible.
Thus, we can take $\csb = e$ and $\alpha = 1/2$.
\end{proof}

\begin{myprop}
\label{stmt:toy_quadratic_filter}
Consider the scalar stochastic process $\{x_t\}_{t \geq 1}$ defined by:
\begin{align*}
    x_t = \sum_{i=0}^{t-1} \sum_{j=0}^{t-1} c_{i,j} w_{t-i-1} w_{t-j-1},
\end{align*}
where $\{c_{i,j}\}_{i,j \geq 0}$ are the coefficients which describe the dynamics,
and 
$\{w_t\}_{t \geq 0}$ are \IID\ $N(0, 1)$ random variables.
Let $\{\calF_t\}_{t \geq 1}$ denote the filtration
defined as $\calF_t := \sigma(w_0, \dots, w_{t-1})$, so that $x_t$ is
$\calF_t$-measurable.
Suppose that $\{c_{i,j}\}_{i,j \geq 0}$ is
symmetric and traceless.
For every $t \geq 1$ and $k \geq 0$, almost surely we have:
\begin{align*}
    \E[ x_{t + k}^2 \mid \calF_k ] \geq \E[ x_t^2 ] + (\E[ x_{t+k} \mid \calF_k])^2.
\end{align*}
\end{myprop}
\begin{proof}
For $t \in \N_+$, define the
symmetric matrices $M_{t} \in \R^{t \times t}$ with $(M_{t})_{ii} = 0$
and $(M_{t})_{ij} = c_{(i-1),(j-1)}$.
With this notation and with $\bar{w}_t \sim N(0, I_t)$, we can write $x_t$ as:
\begin{align*}
    x_t = \bar{w}_t^\T M_t \bar{w}_t.
\end{align*}
Therefore, by \Cref{stmt:gaussian_fourth_moment} and the
assumption that $\Tr(M_t) = 0$:
\begin{align*}
    \E[x_t^2] = \E(\bar{w}_t^\T M_t \bar{w}_t)^2 
    = 2\norm{M_t}_F^2 + \Tr(M_t)^2 = 2\norm{M_t}_F^2.
\end{align*}
Now, partition $M_{t+k}$ as:
\begin{align*}
     M_{t+k} = \bmattwo{M_t}{D_{t,k}}{D^\T_{t,k}}{E_{t,k}}.
\end{align*}
Let $\bar{v}_k = (w_{k-1}, \dots, w_0)$. Given $\calF_k$, we can write $x_{t+k}$ as:
\begin{align*}
    x_{t+k} = \cvectwo{\bar{w}_t}{\bar{v}_k}^\T \bmattwo{M_t}{D_{t,k}}{D^\T_{t,k}}{E_{t,k}}\cvectwo{\bar{w}_t}{\bar{v}_k}.
\end{align*}
With this notation:
\begin{align*}
    \E[x_{t+k} \mid \calF_k] =  \bar{v}_k^\T E_{t,k} \bar{v}_k, \quad \E[x_{t+k}^2 \mid \calF_k] &= \E_{\bar{w}_t}\left(\cvectwo{\bar{w}_t}{\bar{v}_k}^\T \bmattwo{M_t}{D_{t,k}}{D^\T_{t,k}}{E_{t,k}}\cvectwo{\bar{w}_t}{\bar{v}_k}\right)^2.
\end{align*}
Expanding the square:
\begin{align*}
    \left(\cvectwo{\bar{w}_t}{\bar{v}_k}^\T \bmattwo{M_t}{D_{t,k}}{D^\T_{t,k}}{E_{t,k}}\cvectwo{\bar{w}_t}{\bar{v}_k}\right)^2 &= ( \bar{w}_t^\T M_t \bar{w}_t + 2 \bar{w}_t^\T D_{t,k}\bar{v}_k + \bar{v}_k^\T E_{t,k} \bar{v}_k)^2 \\
    &= ( \bar{w}_t^\T M_t \bar{w}_t)^2 + 4 \bar{w}_t^\T M_t \bar{w}_t  \bar{w}_t^\T D_{t,k}\bar{v}_k +
    2 \bar{w}_t^\T M_t \bar{w}_t  \bar{v}_k^\T E_{t,k} \bar{v}_k \\
    &\qquad+ 4 (\bar{w}_t^\T D_{t,k} \bar{v}_k)^2 + 4 \bar{w}_t^\T D_{t,k} \bar{v}_k \bar{v}_k^\T E_{t,k} \bar{v}_k  + ( \bar{v}_k^\T E_{t,k} \bar{v}_k )^2.
\end{align*}
Using \Cref{stmt:gaussian_fourth_moment} again:
\begin{align*}
     \E[x_{t+k}^2 \mid \calF_k ] &= \E_{\bar{w}_t}\left(\cvectwo{\bar{w}_t}{\bar{v}_k}^\T \bmattwo{M_t}{D_{t,k}}{D^\T_{t,k}}{E_{t,k}}\cvectwo{\bar{w}_t}{\bar{v}_k}\right)^2 \\
     &= \E_{\bar{w}_t}( \bar{w}_t^\T M_t \bar{w}_t)^2 + 2 \Tr(M_t) \bar{v}_k^\T E_{t,k}\bar{v}_k + 4 \norm{D_{t,k}\bar{v}_k}_2^2 + ( \bar{v}_k^\T E_{t,k} \bar{v}_k )^2 \\
     &= 2\norm{M_t}_F^2 + 4 \norm{D_{t,k}\bar{v}_k}_2^2 + ( \bar{v}_k^\T E_{t,k} \bar{v}_k )^2 \geq 2\norm{M_t}_F^2 + ( \bar{v}_k^\T E_{t,k} \bar{v}_k )^2 .
\end{align*}
To complete the proof, we recall that $\E[x_t^2] = 2\norm{M_t}_F^2$ and
$\E[x_{t+k} \mid \calF_k] =  \bar{v}_k^\T E_{t,k} \bar{v}_k$.
\end{proof}

\exampletwovolterra*
\begin{proof}
Fix a $v \in \R^n$.
The relation \eqref{eq:toy_vector_quadratic} shows that
$\ip{v}{x_t}^2$ is a degree four polynomial in $\{w_i^{(\ell)}\}_{i=0,\ell=1}^{t-1,n}$. 
Let $\calF_t = \sigma(\{w_i^{(\ell)}\}_{i=0,\ell=1}^{t-1,n})$,
so that $x_t$ is $\calF_t$-measurable.
By \Cref{thm:carbery_and_wright}, 
there exists a univeral positive constant $c > 0$ such that
for any $s \geq 0$,
\begin{align*}
    \Pr\left\{ \frac{1}{k} \sum_{t=1}^{k} \ip{v}{x_{t+s}}^2 \leq \varepsilon \E\left[ \frac{1}{k} \sum_{t=1}^{k} \ip{v}{x_{t+s}}^2 \,\bigg|\, \calF_s \right] \,\bigg|\, \calF_s \right\} \leq (c\varepsilon)^{1/4} \quad \textrm{a.s.}
\end{align*}
To conclude the proof, we need to lower bound
$ \E\left[ \frac{1}{k} \sum_{t=1}^{k} \ip{v}{x_{t+s}}^2 \,\bigg|\, \calF_s \right]$.
For any $t \geq 1$,
\begin{align*}
    \E\left[ \ip{v}{x_{t+s}}^2 \,\bigg|\, \calF_s \right] &= \E\left[ \left( \sum_{\ell=1}^{n} v_\ell \cdot (x_{t+s})_{\ell} \right)^2 \,\bigg|\, \calF_s \right] \\
    &\stackrel{(a)}{=} \sum_{\ell=1}^{n} v_\ell^2 \cdot \E[ (x_{t+s})_{\ell}^2 \mid \calF_s ] + \sum_{\ell_1 \neq \ell_2}^{n} v_{\ell_1} v_{\ell_2} \cdot \E[ (x_{t+s})_{\ell_1} \mid \calF_s ] \cdot \E[ (x_{t+s})_{\ell_2}] \mid \calF_s ]  \\
    &\stackrel{(b)}{\geq} \sum_{\ell=1}^{n} v_\ell^2 \cdot \E[ (x_t)_\ell^2 ] + \sum_{\ell=1}^{n} v_\ell^2 \cdot (\E[ (x_{t+s})_\ell \mid \calF_s ])^2 \\
    &\qquad + \sum_{\ell_1 \neq \ell_2}^{n} v_{\ell_1} v_{\ell_2} \cdot \E[ (x_{t+s})_{\ell_1} \mid \calF_s ] \cdot \E[ (x_{t+s})_{\ell_2}] \mid \calF_s ] \\
    &=  \sum_{\ell=1}^{n} v_\ell^2 \cdot \E[ (x_t)_\ell^2 ] + \left( \sum_{\ell=1}^{n} v_\ell \cdot \E[ (x_{t+s})_\ell \mid \calF_s ]\right)^2 \\
    &\geq \sum_{\ell=1}^{n} v_\ell^2 \cdot \E[ (x_t)_\ell^2 ] \stackrel{(c)}{=} v^\T \Sigma_t(\Px) v. 
\end{align*}
Above, (a) follows
since each coordinate of $x_t$ is independent by definition,
(b) follows from \Cref{stmt:toy_quadratic_filter},
and (c) follows since $\E[x_t] = 0$ and each coordinate is independent,
so $\E[ (x_t)_{\ell_1} (x_t)_{\ell_2} ] = \E[(x_t)_{\ell_1}] \E[(x_t)_{\ell_2}] = 0$
for $\ell_1 \neq \ell_2$.
Hence, we have shown:
\begin{align*}
     \E\left[ \frac{1}{k} \sum_{t=1}^{k} \ip{v}{x_{t+s}}^2 \,\bigg|\, \calF_s \right] \geq v^\T \left( \frac{1}{k} \sum_{t=1}^{k} \Sigma_t(\Px) \right) v = v^\T \Gamma_k(\Px) v.
\end{align*}
Note that because we assume that $k \geq k_{\mathsf{nd}}$,
the covariances $\Sigma_t(\Px)$ are all invertible (and hence so is $\Gamma_k(\Px)$).
The claim now follows.
\end{proof}

\examplegeneralvolterra*
\begin{proof}
Fix a $v \in \R^n$.
The definition \eqref{eq:volterra_series}
expresses $\ip{v}{x_t}$ as a degree at most $D$ polynomial in the noise 
variables $\{w_t^{(\ell)}\}$.
By \Cref{thm:carbery_and_wright}, there exists a
positive constant $c_D$, only depending on $D$,
such that:
\begin{align*}
    \Pr\left\{ \frac{1}{T} \sum_{t=1}^{T} \ip{v}{x_t}^2 \leq \varepsilon \E\left[  \frac{1}{T} \sum_{t=1}^{T} \ip{v}{x_t}^2 \right] \right\} \leq (c_D \varepsilon)^{1/(2D)}.
\end{align*}
Since $T \geq T_{\mathsf{nd}}$, the matrix $\Gamma_T(\Px)$ is invertible.
The claim now follows.
\end{proof}

\subsection{Proof of \Cref{stmt:avg_small_ball_implies_block}}

\avgsmallballimpliesblock*
\begin{proof}
The following proof builds on the argument given
in \cite[Section~E.1]{simchowitz18learning}.
We note that a similar style of proof is used in
\cite[Lemma~15]{bartlett2020benign}.

Define the shorthand notation $\Pr_{t}\{\cdot\} := \Pr\{\,\cdot \mid \calF_t \}$, and similarly $\E_{t}[\cdot] := \E[\,\cdot \mid \calF_t]$.
Now fix a $v \in \R^n \setminus \{0\}$, $j \in \{1, \dots, \floor{T/k}\}$.
Markov's inequality yields that:
\begin{align*}
    \frac{1}{k} \sum_{t=(j-1)k+1}^{jk} \ip{v}{x_t}^2 \geq \alpha v^\T \Psi_j v  \cdot \frac{1}{k} \sum_{t=(j-1)k+1}^{jk} \ind\{ \ip{v}{x_t}^2 > \alpha v^\T \Psi_j v \},
\end{align*}
and therefore for all $\e > 0$:
\begin{align*}
    \Pr_{(j-1)k}\left\{  \frac{1}{k} \sum_{t=(j-1)k+1}^{jk} \ip{v}{x_t}^2 \leq \varepsilon \cdot v^\T \Psi_j v \right\} \leq \Pr_{(j-1)k}\left\{ \frac{1}{k} \sum_{t=(j-1)k+1}^{jk} \ind\{ \ip{v}{x_t}^2 > \alpha v^\T \Psi_j v \} \leq \e/\alpha \right\}.
\end{align*}
Define $Z_j := \frac{1}{k} \sum_{t=(j-1)k+1}^{jk} \ind\{ \ip{v}{x_t}^2 > \alpha v^\T \Psi_j v \}$, and observe that $Z_j \in [0, 1]$.
By \eqref{eq:avg_weak_small_ball_inequality}, we have:
\begin{align*}
    \E_{(j-1)k}[Z_j] \geq 1 - \beta.
\end{align*}
On the other hand:
\begin{align*}
    \E_{(j-1)k}[Z_j] &= \E_{(j-1)k}[ Z_j \ind\{ Z_j > \e/\alpha \}] + \E_{(j-1)k}[ Z_j \ind \{Z_j \leq \e/\alpha \}] \\
    &\leq \Pr_{(j-1)k}\{ Z_j > \e/\alpha \} + \e/\alpha \cdot  \Pr_{(j-1)k}\{Z_j \leq \e/\alpha \} && \text{since } Z_j \leq 1 \\
    &= 1 - (1-\e/\alpha) \Pr_{(j-1)k}\{Z_j \leq \e/\alpha \}.
\end{align*}
Combining both these inequalities, and further
restricting $\e \in (0, \alpha)$, we obtain,
\begin{align*}
    \Pr_{(j-1)k}\{Z_j \leq \e/\alpha\} \leq \frac{\beta}{1-\e/\alpha},
\end{align*}
which implies \eqref{eq:avg_weak_small_ball_implication}.
\end{proof}

\subsection{General ordinary least-squares estimator upper bound}
\label{sec:appendix:general_OLS_proof}

In this section, we supply the proof of
\Cref{stmt:upper_bound_general}.
We first start with a result which bounds the 
minimum eigenvalue of the empirical covariance matrix.

\begin{mylemma}[Minimum eigenvalue bound via trajectory small-ball]
\label{stmt:small_ball_to_min_eval}
Suppose that $\Px$ satisfies
the $(T,k,\{\Psi_j\}_{j=1}^{\floor{T/k}},\csb,\alpha)$-trajectory-small-ball condition
(\Cref{def:trajectory_small_ball}).
Put $S := \floor{T/k}$, and $\Gamma_T := \Gamma_T(\Px)$.
Fix any $\uGam \in \sfSym^{n}_{> 0}$ satisfying
$\frac{1}{S} \sum_{j=1}^{S} \Psi_j \preccurlyeq \uGam \preccurlyeq \Gamma_T$.
Define 
$\tilde{X}_{m,T} := X_{m,T} \uGam^{-1/2}$, and:
\begin{align}
    \uMu(\{\Psi_j\}_{j=1}^{S}, \uGam) := \left[\prod_{j=1}^{S} \ulam(\Psi_j, \uGam) \right]^{1/S}. \label{eq:geo_mean_evals}
\end{align}
Suppose that:
\begin{align}
    n \geq 2, \quad \frac{mT}{kn} \geq \frac{32}{\alpha} \log\left(  \frac{320 \csb}{ \alpha \ulam(\uGam, \Gamma_T) \uMu(\{\Psi_j\}_{j=1}^{S}, \uGam) } \right). \label{eq:mS_assump1}
\end{align}
For any $t \geq 0$,
with probability at least $1 - 2e^{-t}$, the following statements 
simultaneously hold:
\begin{align}
    \Tr(\tilde{X}_{m,T}^\T \tilde{X}_{m,T}) &\leq  \frac{mTn e^t}{\ulam(\uGam, \Gamma_T)}, \quad
    \lambda_{\min}(\tilde{X}_{m,T}^\T \tilde{X}_{m,T}) \geq \frac{mT\alpha \uMu(\{\Psi_j\}_{j=1}^{S}, \uGam)}{8e\csb} \exp\left(-\frac{16kn}{mT\alpha} t\right). \label{eq:traj_sb_min_eval_bound}
\end{align}
\end{mylemma}
\begin{proof}
The proof uses the PAC-Bayes argument for uniform convergence
from \cite{mourtada19exactminimax}. 
The first step is to construct a family of random variables,
indexed by both $v \in \mathbb{S}^{n-1}$ and a scale parameter
$\eta > 0$, such that its moment generating function is pointwise
bounded by one.
For notational brevity, let:
\begin{align*}
    \ulam := \ulam(\uGam, \Gamma_T), \quad \uMu := \uMu(\{\Psi_j\}_{j=1}^{S}, \uGam).
\end{align*}
Since $\uGam \preccurlyeq \Gamma_T$ by assumption, we have
$\ulam \in (0, 1]$.
Similarly, since
$\frac{1}{S} \sum_{j=1}^{S} \Psi_j \preccurlyeq \uGam$,
we also have $\uMu \in (0, 1]$ by \Cref{stmt:Psi_leq_one}.

The trajectory small-ball condition \eqref{eq:trajectory_small_ball}
implies for any $v \in \mathbb{S}^{n-1}$, $j \in \{1, \dots, S\}$, and
$\varepsilon > 0$,
\begin{align*}
    \Pr\left\{ \frac{1}{k}\sum_{t=(j-1)k+1}^{jk} \ip{v}{\uGam^{-1/2} x_t}^2 \leq \varepsilon \ulam(\Psi_j, \uGam) \,\Bigg|\, \calF_{(j-1)k} \right\} \leq (\csb \varepsilon)^\alpha.
\end{align*}
Using a change of variables $\varepsilon \gets \varepsilon/\ulam(\Psi_j,\uGam)$,
\begin{align*}
    \Pr\left\{ \frac{1}{k}\sum_{t=(j-1)k+1}^{jk} \ip{v}{\uGam^{-1/2} x_t}^2 \leq \varepsilon \,\Bigg|\, \calF_{(j-1)k} \right\} \leq \left(\csb/\ulam(\Psi_j,\uGam) \cdot \varepsilon \right)^\alpha.
\end{align*}
By \Cref{stmt:small_ball_to_mgf}, for any $\eta > 0$,
\begin{align}
    \E\left[ \exp\left( -\frac{\eta}{k}\sum_{t=(j-1)k+1}^{jk} \ip{v}{\uGam^{-1/2} x_t}^2 + \alpha \log\left(\frac{\eta   \ulam(\Psi_j, \uGam) }{\csb} \right) \right) \,\Bigg|\, \calF_{(j-1)k} \right] \leq 1 \quad \textrm{a.s.} \label{eq:laplace_transform_single}
\end{align}
For $i \in \{1, \dots, m\}$
and $j \in \{1, \dots, S\}$,
define the random variables
$Z_j^{(i)}(v; \eta)$, $Z^{(i)}(v; \eta)$,
and $Z(v;\eta)$:
\begin{align*}
    Z_j^{(i)}(v;\eta) &:= -\frac{\eta}{k} \sum_{t=(j-1)k+1}^{jk} \ip{v}{\uGam^{-1/2} x_t^{(i)}}^2 + \alpha \log\left(\frac{\eta   \ulam(\Psi_j, \uGam)}{\csb} \right), \\
    Z^{(i)}(v;\eta) &:= \sum_{j=1}^{S} Z_j^{(i)}(v; \eta), \\
    Z(v, \eta) &:= \sum_{i=1}^{m} Z^{(i)}(v;\eta).
\end{align*}
We first claim that $\E[\exp(Z(v;\eta))] \leq 1$ for every $v \in \mathbb{S}^{n-1}$ 
and $\eta > 0$.
Since $Z^{(i)}(v;\eta)$ is independent of $Z^{(i')}(v;\eta)$ 
whenever $i \neq i'$, 
we have that:
\begin{align*}
    \E[\exp(Z(v;\eta))] = \E\left[ \exp\left(\sum_{i=1}^{m} Z^{(i)}(v;\eta)\right)\right] = \prod_{i=1}^{m} \E[ \exp(Z^{(i)}(v;\eta)) ].
\end{align*}
Furthermore, by repeated applications of the tower property and \eqref{eq:laplace_transform_single},
for every $i \in \{1, \dots, m\}$,
\begin{align*}
    \E[\exp(Z^{(i)}(v; \eta))] &= \E\left[ \exp\left( \sum_{j=1}^{S} Z_j^{(i)}(v;\eta) \right) \right] \\
    &=\E\left[ \exp\left(\sum_{j=1}^{S-1} Z^{(i)}_j(v;\eta)\right) \E[ \exp(Z^{(i)}_{S}(v;\eta)) \mid \calF_{(S-1)k}]\right] \\
    &\leq \E\left[ \exp\left(\sum_{j=1}^{S-1} Z^{(i)}_j(v;\eta)\right) \right] \\
    &~\vdots \\
    &\leq 1.
\end{align*}
Hence $\E[\exp(Z(v;\eta))] \leq 1$ for every $v \in \mathbb{S}^{n-1}$
and $\eta > 0$.

Let us now import some notation from \cite{mourtada19exactminimax}.
First, let $\pi$ denote the spherical measure on $\mathbb{S}^{n-1}$,
and let $\rho_{v,\gamma}$ denote the uniform measure over the
spherical cap 
\begin{align*}
    \calC(v,\gamma) := \{ w \in \mathbb{S}^{n-1} \mid \norm{v-w}_2 \leq \gamma \}.
\end{align*}
Next, let $F_{v,\gamma}(\Sigma) := \int_{\calC(v,\gamma)} \ip{w}{\Sigma w} \,\rmd\rho_{v,\gamma}$ for any symmetric matrix $\Sigma$.

Fix any positive $t, \eta$.
For two measures $\mu$ and $\nu$ with 
$\mu$ absolutely continuous w.r.t.\ $\nu$, 
let $\mathsf{KL}(\mu, \nu) := \E_{\mu} \log\left(\frac{\rmd\mu}{\rmd\nu}\right)$ denote the KL-divergence
between $\mu$ and $\nu$.
By the PAC-Bayes deviation bound (cf.~\cite{catoni2007pacbayes}),
there exists an event $\calE_{t,1}$
with probability at least $1-e^{-t}$, 
such that on $\calE_{t,1}$, we have 
for every $v \in \mathbb{S}^{n-1}$ and $\gamma > 0$,
\begin{align}
    &-\frac{\eta}{k} F_{v,\gamma}\left(\uGam^{-1/2} \sum_{i=1}^{m} \sum_{t=1}^{kS} x_t^{(i)} (x_t^{(i)})^\T \uGam^{-1/2}\right) + mS \alpha  \log\left(\frac{\eta \uMu}{\csb} \right) \leq \mathsf{KL}(\rho_{v,\gamma},\pi) + t. \label{eq:pac_bayes_bound}
\end{align}
Next, by \cite[Sections~5.3 and 5.4]{mourtada19exactminimax}, we can write $F_{v,\gamma}$
in terms of a scalar function $\phi$ such that:
\begin{align}
    F_{v,\gamma}(\Sigma) = (1-\phi(\gamma))\ip{v}{\Sigma v} + \phi(\gamma) \frac{1}{n} \Tr(\Sigma), \quad \phi(\gamma) \in \left[0, \frac{n}{n-1} \gamma^2\right]. \label{eq:F_v_gamma}
\end{align}
Furthermore, 
for every $v \in \mathbb{S}^{n-1}$ and $\gamma > 0$,
the KL-divergence term can be upper bounded by:
\begin{align}
    \mathsf{KL}(\rho_{v,\gamma},\pi) \leq n\log\left(1 + \frac{2}{\gamma}\right). \label{eq:KL_upper_bound}
\end{align}
Therefore on $\calE_{t,1}$,
plugging \eqref{eq:F_v_gamma} and \eqref{eq:KL_upper_bound} into \eqref{eq:pac_bayes_bound},
\begin{align*}
    \lambda_{\min}\left( \tilde{X}_{m,T}^\T \tilde{X}_{m,T} \right) &\geq \frac{k}{\eta (1-\phi(\gamma))}\left[ mS\alpha \log\left(\frac{\eta\uMu}{\csb}\right) - n \log\left(1 + \frac{2}{\gamma}\right) - t \right] \\
    &\qquad- \frac{\phi(\gamma)}{1-\phi(\gamma)} \frac{1}{n} \Tr\left( \tilde{X}_{m,T}^\T \tilde{X}_{m,T} \right).
\end{align*}
Restricting $\gamma \in [0, 1/2]$, we have 
from \eqref{eq:F_v_gamma} that
$0 \leq \phi(\gamma) \leq \frac{n}{n-1} \gamma^2 \leq 2 \gamma^2 \leq 1/2$.
Hence, $1-\phi(\gamma) \in [1/2, 1]$.
Furthermore, $1+2/\gamma \leq 5/(4\gamma^2)$.
Therefore,
\begin{align*}
    \lambda_{\min}\left( \tilde{X}_{m,T}^\T \tilde{X}_{m,T} \right) \geq \frac{k }{\eta}\left[ mS\alpha\log\left(\frac{\eta\uMu}{\csb}\right) - n \log\left(\frac{5}{4\gamma^2}\right) - t \right] -  \frac{4\gamma^2}{n} \Tr\left( \tilde{X}_{m,T}^\T \tilde{X}_{m,T}  \right).
\end{align*}
Define the non-negative random variables $\psi_i := \sum_{t=1}^{T} \norm{\uGam^{-1/2} x_t^{(i)}}_2^2$, for $i=1, \dots, m$.
It is straightforward to verify that 
$\Tr( \tilde{X}_{m,T}^\T \tilde{X}_{m,T} ) = \sum_{i=1}^{m} \psi_i$.
By Markov's inequality, for any $\beta > 0$:
\begin{align*}
    \Pr\left( \Tr(\tilde{X}_{m,T}^\T \tilde{X}_{m,T}) > \beta \right) &= \Pr\left( \sum_{i=1}^{m} \psi_i > \beta \right) \leq \frac{\E\left[ \sum_{i=1}^{m} \psi_i \right]}{\beta} \\
    &= 
     \frac{mT \Tr(\uGam^{-1} \Gamma_T)}{\beta} \leq \frac{mTn \lambda_{\max}(\Gamma_T^{1/2} \uGam^{-1} \Gamma_T^{1/2})}{\beta} = \frac{mTn}{\ulam \beta}.
\end{align*}
Therefore, setting $\beta = \frac{e^t mTn}{\ulam}$, there exists an event $\calE_{t,2}$ such that $\Pr(\calE_{t,2}^c) \leq e^{-t}$ and on $\calE_{t,2}$,
\begin{align*}
    \Tr(\tilde{X}_{m,T}^\T \tilde{X}_{m,T}) \leq \frac{e^{t} mTn}{\ulam}.
\end{align*}
Therefore on $\calE_{t,1} \cap \calE_{t,2}$, which we assume holds for the
remainder of the proof, we have:
\begin{align*}
    \lambda_{\min}\left( \tilde{X}_{m,T}^\T \tilde{X}_{m,T} \right) \geq \frac{k}{\eta}\left[ mS\alpha \log\left(\frac{\eta\uMu}{\csb}\right) - n \log\left(\frac{5}{4\gamma^2}\right) - t \right] - \frac{4 mTe^{t}}{\ulam} \gamma^2
\end{align*}

Next, we further restrict $\frac{\eta \uMu}{\csb} \geq e$ so that $\log(\eta\uMu/\csb) \geq 1$.
Now consider, for positive constants $A, B, C$, the function
$x \mapsto A \log(B/x) + Cx$
on the domain $(0, \infty)$.
The derivative vanishes at $x = A/C$, and the function attains
a minimum value of $A (1 + \log(BC/A))$ with this choice of $x$.
Let us set:
\begin{align*}
    A \gets \frac{k n}{\eta}, \quad B \gets \frac{5}{4}, \quad C \gets \frac{4 mT e^{t}}{\ulam}, \quad     
    x \gets \gamma^2.
\end{align*}
Then by choosing $\gamma^2 = \frac{k n \ulam}{4  \eta mT e^{t}}$, we have that:
\begin{align*}
    \frac{k n}{\eta} \log\left(\frac{5}{4\gamma^2}\right) + \frac{4 mT e^{t}}{\ulam} \gamma^2 = \frac{k n}{\eta} \left[ 1 + \log\left( \frac{5 m T e^{t} \eta }{ kn \ulam }  \right)\right].
\end{align*}
Note that this choice of $\gamma$ satisfies $\gamma \in [0, 1/2]$, since:
\begin{align*}
    \frac{k n \ulam}{4  \eta mT e^{t}} \leq \frac{1}{4} &\Longleftarrow \frac{k n}{\eta mT} \leq 1 &&\text{since } t \geq 0 \text{ and } \ulam \leq 1 \\
    &\Longleftarrow \frac{k n \uMu}{e \csb mT} \leq 1 &&\text{since } \eta \geq e\csb/\uMu \\
    &\Longleftrightarrow \frac{n \uMu}{e \csb } \leq \frac{mT}{k} \\
    &\Longleftarrow \frac{n}{e\csb} \leq \frac{mT}{k} &&\text{since } \uMu \leq 1,
\end{align*}
and the last condition holds by \eqref{eq:mS_assump1}.
With this choice of $\gamma$, 
we have:
\begin{align}
    &~~~~\lambda_{\min}\left( \tilde{X}_{m,T}^\T \tilde{X}_{m,T} \right) \nonumber \\
    &\geq \frac{k}{\eta} \left[ mS\alpha \log\left(\frac{\eta \uMu}{\csb}\right) - t - n\left( 1 + \log\left( \frac{5  m T e^{t} \eta}{k n \ulam} \right)\right) \right] \nonumber \\
    &= \frac{k}{\eta} \left[ \left(mS \alpha - n \right) \log\left(\frac{\eta\uMu}{\csb}\right) -t - n\left(1 + \log\left( \frac{5 \csb  mT e^{t}}{kn \ulam \uMu} \right)\right) \right] \nonumber \\
    &\geq \frac{k}{\eta} \left[ \frac{mS\alpha}{2} \log\left(\frac{\eta\uMu}{\csb}\right) -t - n\left(1 + \log\left( \frac{5 \csb  mT e^{t}}{kn \ulam \uMu} \right)\right) \right]  &&\text{since } mS\geq 2n/\alpha \nonumber \\
    &=\frac{k}{\eta} \left[ \frac{mS\alpha}{2} \log\left(\frac{\eta\uMu}{\csb}\right) - (1 + n)t - n\left(1 + \log\left( \frac{5 \csb  mT}{kn \ulam \uMu} \right)\right) \right] \nonumber \\
    &\geq \frac{k}{\eta} \left[ \frac{mS\alpha}{2} \log\left(\frac{\eta\uMu}{\csb}\right) - 2nt - n\left(1 + \log\left( \frac{5 \csb  mT}{kn \ulam \uMu} \right)\right) \right] \nonumber \\
    &\geq \frac{k}{\eta} \left[ \frac{mS\alpha}{4} \log\left(\frac{\eta\uMu}{\csb}\right) - 2n t \right] &&\text{since } \frac{mS\alpha}{4n} \geq 1 + \log\left( \frac{5 \csb  mT}{kn \ulam \uMu} \right) \nonumber \\
    &= \frac{kmS\alpha}{4\csb/\uMu} \left[\frac{\log(\eta\uMu/\csb) - \frac{8nt}{mS\alpha}}{\eta\uMu/\csb}\right].
    \label{eq:before_eta_optimization}
\end{align}
It remains to optimize over $\eta \in [e \csb/\uMu, \infty)$.
For any $G \in \R$, the function $\eta' \mapsto \frac{\log{\eta'} - G}{\eta'}$
on $(0, \infty)$
attains a maximum of $\exp(-1-G)$ at $\eta' = \exp(1 + G)$.
Hence, setting $\eta = \frac{\csb}{\uMu} \exp( 1 + 8nt/(mS\alpha) )$, which satisfies
$\eta \geq e \csb/\uMu$,
we have:
\begin{align*}
    \lambda_{\min}\left( \tilde{X}_{m,T}^\T \tilde{X}_{m,T} \right) &\geq \frac{kmS\alpha \uMu}{4e\csb} \exp\left( - \frac{8nt}{mS\alpha} \right) \\
    &\geq \frac{mT\alpha \uMu}{8e\csb} \exp\left(-\frac{16kn}{mT\alpha} t \right) &&\text{since } S \geq T/(2k).
\end{align*}
The claim now follows by gathering the requirements on 
the quantity $\frac{mT}{kn}$ and simplifying as in \eqref{eq:mS_assump1}.
\end{proof}

\begin{mycorollary}
\label{stmt:invertible_almost_surely}
Assume the hypothesis of \Cref{stmt:small_ball_to_min_eval} hold.
Then, $\tilde{X}_{m,T}^\T \tilde{X}_{m,T}$ is invertible almost surely.
\end{mycorollary}
\begin{proof}
For any $t \geq 0$, define the event $\calE_t$ as:
\begin{align*}
    \calE_t := \left\{  \lambda_{\min}\left( \tilde{X}_{m,T}^\T \tilde{X}_{m,T} \right) < \frac{mT\alpha\uMu}{8e\csb} \exp\left(-\frac{16kn}{mT\alpha} t\right) \right\}.
\end{align*}
The event $\{\lambda_{\min}(\tilde{X}_{m,T}^\T \tilde{X}_{m,T}) = 0\}$
is the intersection $\bigcap_{t = 1}^{\infty} \calE_t$.
By \Cref{stmt:small_ball_to_min_eval}, we have that
$\Pr(\calE_t) \leq 2 e^{-t}$.
Since the events $\calE_{t'} \subseteq \calE_t$ whenever $t' \geq t$,
by continuity of measure from above,
\begin{align*}
    \Pr(\lambda_{\min}(\tilde{X}_{m,T}^\T \tilde{X}_{m,T}) = 0) = \Pr\left( \bigcap_{t=1}^{\infty} \calE_t \right) = \lim_{t \rightarrow \infty} \Pr(\calE_t) \leq \lim_{t \rightarrow \infty} 2 e^{-t} = 0.
\end{align*}
\end{proof}

We are now ready to restate and prove \Cref{stmt:upper_bound_general}.
\upperboundgeneral*
\begin{proof}
For notational brevity, let:
\begin{align*}
    \ulam := \ulam(\uGam, \Gamma_T), \quad \uMu := \uMu(\{\Psi_j\}_{j=1}^{S}, \uGam).
\end{align*}

We choose $c_0 \geq 64$ such that \eqref{eq:mTkn_requirements} 
implies \eqref{eq:mS_assump1},
and also so that $\frac{mT}{kn} \geq 64/\alpha$.
By \Cref{stmt:invertible_almost_surely},
$X_{m,T}$ has full column rank almost surely, hence:
\begin{align*}
    \hat{W}_{m,T} - W_\star = \Xi_{m,T}^\T X_{m,T} (X_{m,T}^\T X_{m,T})^{-1}.
\end{align*}
Put $\tilde{X}_{m,T} := X_{m,T} \uGam^{-1/2}$.
With this decomposition, we have:
\begin{align*}
     \norm{ \hat{W}_{m,T} - W_\star }_{\Gamma'}^2 
    &= \norm{\Xi_{m,T}^\T X_{m,T} (X_{m,T}^\T X_{m,T})^{-1}}^2_{\Gamma'} \\
    &= \norm{\Xi_{m,T}^\T X_{m,T} (X_{m,T}^\T X_{m,T})^{-1} \uGam^{1/2}}_{\uGam^{-1/2} \Gamma' \uGam^{-1/2}}^2 \\
    &\leq \lambda_{\max}(\uGam^{-1/2} \Gamma' \uGam^{-1/2})  \norm{\Xi_{m,T}^\T X_{m,T} (X_{m,T}^\T X_{m,T})^{-1} \uGam^{1/2}}_F^2 \\
    &= \lambda_{\max}(\uGam^{-1/2} \Gamma' \uGam^{-1/2}) \norm{ (\tilde{X}_{m,T}^\T \tilde{X}_{m,T})^{-1} \tilde{X}_{m,T}^\T \Xi_{m,T} }_F^2 \\
    &\leq \min\{ n, p \} \lambda_{\max}(\uGam^{-1/2} \Gamma' \uGam^{-1/2}) \opnorm{ (\tilde{X}_{m,T}^\T \tilde{X}_{m,T})^{-1} \tilde{X}_{m,T}^\T \Xi_{m,T} }^2 \\
    &\leq \min\{ n, p \} \lambda_{\max}(\uGam^{-1/2} \Gamma' \uGam^{-1/2}) \frac{ \opnorm{ (\tilde{X}_{m,T}^\T \tilde{X}_{m,T})^{-1/2} \tilde{X}_{m,T}^\T \Xi_{m,T} }^2 }{\lambda_{\min}(\tilde{X}_{m,T}^\T \tilde{X}_{m,T} )} \\
    &= \min\{ n, p \} \frac{ \opnorm{ (\tilde{X}_{m,T}^\T \tilde{X}_{m,T})^{-1/2} \tilde{X}_{m,T}^\T \Xi_{m,T} }^2 }{\ulam(\uGam, \Gamma') \cdot \lambda_{\min}(\tilde{X}_{m,T}^\T \tilde{X}_{m,T} )}.
\end{align*}

Fix any $t > 0$.
By \Cref{stmt:small_ball_to_min_eval},
there exists 
an event $\calE_{t,1}$ with probability at least $1 - 2e^{-t}$,
such that on $\calE_{t,1}$ we have:
\begin{align*}
    \Tr(\tilde{X}_{m,T}^\T \tilde{X}_{m,T}) \leq \frac{mnT e^t}{\ulam}, \quad \lambda_{\min}(\tilde{X}_{m,T}^\T \tilde{X}_{m,T}) \geq \frac{m T \alpha \uMu}{8e\csb} \exp\left(-\frac{16kn}{mT\alpha} t\right).
\end{align*}
Recall by our cohice of $c_0$, we have $mT/k \geq 64 n/\alpha$.
Hence, on $\calE_{t,1}$,
\begin{align*}
    \lambda_{\min}(\tilde{X}_{m,T}^\T \tilde{X}_{m,T}) \geq \zeta_t := \frac{m T \alpha \uMu}{8e\csb} \exp(-t/4).
\end{align*}
We now apply \Cref{prop:yasin_vector_easier}
with $V \gets M_t := \zeta_t I_n$ and:
\begin{align*}
    &x_{1}, \dots, x_{T}, x_{T+1}, \dots, x_{2T}, \dots, x_{(m-1)T+1}, \dots, x_{mT} \gets \\
    &\qquad \uGam^{-1/2} x_1^{(1)}, \dots, \uGam^{-1/2} x_T^{(1)}, \uGam^{-1/2}x_{1}^{(2)}, \dots, \uGam^{-1/2}x_{T}^{(2)}, \dots, \uGam^{-1/2} x_{1}^{(m)}, \dots, \uGam^{-1/2} x_{T}^{(m)},
\end{align*}
to conclude that there exists an event $\calE_{t,2}$ with
probability at least $1-e^{-t}$ such that on $\calE_{t,2}$:
\begin{align*}
    &~~~~\ind\{\tilde{X}_{m,T}^\T \tilde{X}_{m,T} \succcurlyeq M_t\} \opnorm{ (\tilde{X}_{m,T}^\T \tilde{X}_{m,T})^{-1/2} \tilde{X}_{m,T}^\T \Xi_{m,T} }^2 \\
    &\leq 16 \sigma_\xi^2 \left[ p \log{5} + \frac{1}{2}\log\det\left(I_n + \zeta_t^{-1} \tilde{X}_{m,T}^\T \tilde{X}_{m,T}\right) + t \right] \\
    &\leq 32 \sigma_\xi^2 \left[ p + \log\det\left(I_n + \zeta_t^{-1} \tilde{X}_{m,T}^\T \tilde{X}_{m,T}\right) + t \right] \\
    &\leq 32 \sigma_\xi^2 \left[ p + n \log(1 + \zeta_t^{-1} \Tr(\tilde{X}_{m,T}^\T \tilde{X}_{m,T})/n) + t \right].
\end{align*}
Above, the last inequality holds since
$\log\det(X) \leq n \log(\Tr(X)/n)$ for any $X \in \sfSym^n_{\geq 0}$
by the AM-GM inequality.
By \Cref{prop:invert_log_t_over_t}, whenever $t \geq 8 \log{16}$, we have
$t \leq e^{t/4}$.
Furthermore, for any $t \geq 0$ we have $1 \leq e^{t/4}$.
Therefore, for $t \geq 8\log{16}$,
on $\calE_{t,1} \cap \calE_{t,2}$:
\begin{align*}
     &~~~~\frac{ \opnorm{ (\tilde{X}_{m,T}^\T \tilde{X}_{m,T})^{-1/2} \tilde{X}_{m,T}^\T \Xi_{m,T} }^2 }{ \lambda_{\min}(\tilde{X}_{m,T}^\T \tilde{X}_{m,T} )} \\
     &\leq \frac{256e\csb}{mT\alpha} e^{t/4} \sigma_\xi^2 \left[ p + n \log\left( 1 + \frac{8e\csb}{\alpha \ulam \uMu} e^{(1+1/4)t} \right) + t \right] \\
     &\leq \frac{256e\csb}{mT\alpha} e^{t/4} \sigma_\xi^2 \left[ p + n \log\left( \frac{16 e \csb}{\alpha \ulam \uMu}\right) + n(1+1/4)t + t \right] \\
     &\leq \frac{256e\csb}{mT\alpha} e^{t/4} \sigma_\xi^2 \left[ p + n \log\left( \frac{16 e \csb}{\alpha \ulam \uMu}\right) + 3nt \right] \\
     &\leq \frac{256e\csb}{mT\alpha} \sigma_\xi^2 \left[ p + n \log\left( \frac{16 e \csb}{\alpha \ulam \uMu}\right) + 3n \right] e^{t/2} .
\end{align*}
Define the random variable $Z$ as:
\begin{align*}
    Z :=   \frac{ \opnorm{ (\tilde{X}_{m,T}^\T \tilde{X}_{m,T})^{-1/2} \tilde{X}_{m,T}^\T \Xi_{m,T} }^2 }{ \lambda_{\min}(\tilde{X}_{m,T}^\T \tilde{X}_{m,T} )} \left( \frac{256e\csb}{mT\alpha} \sigma_\xi^2 \left[ p + n \log\left( \frac{16 e \csb}{\alpha \ulam \uMu}\right) + 3n \right] \right)^{-1} .
\end{align*}
We have shown that:
\begin{align*}
    \Pr(Z > e^{t/2}) \leq 3e^{-t} \:\: \forall t \geq 8 \log{16} \Longleftrightarrow
    \Pr(Z > s) \leq 3 s^{-2} \:\: \forall s \geq 16^4.
\end{align*}
Hence,
\begin{align*}
    \E[Z] = \int_0^\infty \Pr(Z > s) \,\rmd s \leq 16^4 + 3\int_{16^4}^{\infty} s^{-2} \,\rmd s = 16^4 + 3/16^4. 
\end{align*}
That is, for some universal positive $c_1$,
\begin{align}
    \E[\norm{\hat{W}_{m,T} - W_\star}^2_{\Gamma'}] \leq c_1 \sigma_\xi^2 \min\{n,p\} \csb \left[\frac{p + n\log\left(\frac{\max\{e,\csb\}}{\alpha \ulam \uMu}\right)}{ mT \alpha    \ulam(\uGam,\Gamma') \uMu }\right]. \label{eq:bound_with_min}
\end{align}
Now, if $p \leq n$, \eqref{eq:bound_with_min}
is upper bounded by:
\begin{align*}
    c_1 \sigma_\xi^2 p \csb \left[\frac{p + n \log\left(\frac{ \max\{e,\csb\}}{\alpha \ulam \uMu}\right) }{ mT \alpha \ulam(\uGam,\Gamma') \uMu }\right] \leq 2c_1 \sigma_\xi^2 p \csb \left[\frac{ n \log\left(\frac{\max\{e,\csb\}}{\alpha \ulam \uMu}\right) }{ mT  \alpha \ulam(\uGam,\Gamma') \uMu }\right]. 
\end{align*}
On the other hand, if $p > n$, \eqref{eq:bound_with_min}
is upper bounded by:
\begin{align*}
    c_1 \sigma_\xi^2 n \csb \left[\frac{p + n \log\left(\frac{ \max\{e,\csb\}}{\alpha \ulam \uMu}\right) }{ mT \alpha \ulam(\uGam,\Gamma') \uMu }\right] < 2 c_1 \sigma_\xi^2 n \csb \left[\frac{p \log\left(\frac{\max\{e,\csb\}}{\alpha \ulam \uMu}\right) }{ mT  \alpha \ulam(\uGam,\Gamma')  \uMu }\right].
\end{align*}
\end{proof}

\subsection{Proof of \Cref{stmt:upper_bound_ind_seq_ls}}

\upperboundindseqls*
\begin{proof}
Equation~\eqref{eq:small_ball_independent}
shows that $\Px$ satisfies the
$(T, 1, \{\Sigma_t\}_{t=1}^{T}, \csb, \alpha)$-\TSB{} condition.
Let $\Gamma_t := \frac{1}{t} \sum_{k=1}^{t} \Sigma_k$ for $t \in \N_{+}$.
For any $s,t \in \N_{+}$ with $s \leq t$,
\begin{align*}
    \ulam(\Gamma_s, \Gamma_t) &\geq \ulam(\Gamma_s, \Sigma_t) &&\text{since } \Gamma_t \preccurlyeq \Sigma_t \\
    &\geq \frac{1}{s} \sum_{k=1}^{s} \ulam(\Sigma_k, \Sigma_t) &&\text{using \Cref{stmt:ulam_concave_first_arg} and Jensen's inequality} \\
    &\geq \frac{1}{c_\beta s} \sum_{k=1}^{s} (k/t)^\beta &&\text{using \eqref{eq:variance_growth_condition}} \\
    &\geq \frac{1}{c_\beta (\beta+1)} (s/t)^\beta &&\text{since } x \mapsto x^\beta \text{ is increasing}.
\end{align*}

Next, the growth condition \eqref{eq:variance_growth_condition}
implies that:
\begin{align*}
    \uMu(\{\Sigma_t\}_{t=1}^{T}, \Gamma_T) &= \left[ \prod_{t=1}^{T} \ulam(\Sigma_t, \Gamma_T) \right]^{1/T} \\
    &\geq \left[ \prod_{t=1}^{T} \ulam(\Sigma_t, \Sigma_T) \right]^{1/T} &&\text{since } \Gamma_T \preccurlyeq \Sigma_T \\
    &\geq \left[ \prod_{t=1}^{T} \frac{1}{c_\beta} (t/T)^\beta \right]^{1/T} &&\text{using \eqref{eq:variance_growth_condition}} \\
    &= \frac{1}{c_\beta T^\beta} (T!)^{\beta/T} \\
    &\geq \frac{1}{c_\beta e^\beta} &&\text{since } T! \geq (T/e)^T.
\end{align*}
We now apply \Cref{stmt:upper_bound_general}
with $\uGam = \Gamma_T$.
In doing so, the requirement \eqref{eq:mTkn_requirements}
simplifies to:
\begin{align*}
    n \geq 2, \quad \frac{mT}{n} \geq \frac{c_0}{\alpha} \log\left( \frac{\max\{e, \csb\} c_\beta e^\beta}{\alpha} \right).
\end{align*}
We first assume that $\Tnew \leq T$, in which case
\eqref{eq:general_risk_bound}
yields:
\begin{align*}
    \E[L(\hat{W}_{m,T};\Tnew, \Px)] \leq c_1 \csb \sigma_\xi^2 \cdot \frac{pn}{mT\alpha} \cdot c_\beta e^\beta \cdot \log\left( \frac{\max\{e,\csb\} c_\beta e^\beta}{\alpha} \right).
\end{align*}
On the other hand, when $\Tnew > T$, we have
$\ulam(\Gamma_T, \Gamma_{\Tnew}) \geq \frac{1}{c_b(\beta+1)} (T/\Tnew)^\beta$, and 
\eqref{eq:general_risk_bound} yields:
\begin{align*}
    \E[L(\hat{W}_{m,T};\Tnew, \Px)] \leq c_1 \csb \sigma_\xi^2 \cdot \frac{pn}{mT\alpha} \cdot c_\beta e^\beta \cdot c_\beta(\beta+1)\left(\frac{\Tnew}{T}\right)^\beta \cdot \log\left( \frac{\max\{e,\csb\} c_\beta e^\beta}{\alpha} \right).
\end{align*}

\end{proof}

\subsection{Proofs for linear dynamical systems}

\subsubsection{Control of ratios of covariance matrices}

\begin{myprop}
\label{prop:ratio_covariances}
Let $(A, B)$ be the dynamics matrices for an \LDSLS{} 
instance, and suppose $(A, B)$ satisfy
\Cref{assume:marginal_stability},
\Cref{assume:diagonalizability},
and \Cref{assume:one_step_controllability}.
Put $\Sigma_t := \Sigma_t(A, B)$ for $t \in \N_{+}$ and
$\condNum := \condNum(A, B)$.
For any integers $T_1, T_2$ satisfying $1 \leq T_2 \leq T_1$,
\begin{align*}
    \lambda_{\min}(\Sigma_{T_1}^{-1/2} \Sigma_{T_2} \Sigma_{T_1}^{-1/2}) \geq \frac{1}{\condNum}\frac{T_2}{T_1}.
\end{align*}
\end{myprop}
\begin{proof}
Observe that for any $t \geq 1$,
\begin{align*}
    \Sigma_t = \sum_{k=0}^{t-1} A^k BB^* (A^k)^* =  \sum_{k=0}^{t-1} S D^k S^{-1} BB^* S^{-*} (D^k)^* S^{*}.
\end{align*}
By \Cref{assume:one_step_controllability}, we have that
$BB^*$ is invertible, and hence $S^{-1} BB^* S^{-*}$ is also invertible.
Therefore we have the following lower and upper bound on $\Sigma_t$:
\begin{align}
    \lambda_{\min}(S^{-1}BB^*S^{-*}) \cdot S \left(\sum_{k=0}^{t-1} D^k (D^k)^* \right) S^* \preccurlyeq \Sigma_t \preccurlyeq     \lambda_{\max}(S^{-1}BB^*S^{-*}) \cdot S \left(\sum_{k=0}^{t-1} D^k (D^k)^* \right) S^*.
\end{align}
Now recall that for two square matrices $X,Y$, the eigenvalues of $XY$ coincide with the eigenvalues of $YX$.
Letting $Q_t := \sum_{k=0}^{t-1} D^k (D^k)^*$, we have:
\begin{align*}
    \lambda_{\min}( \Sigma_{T_1}^{-1/2} \Sigma_{T_2} \Sigma_{T_1}^{-1/2} ) &\geq \lambda_{\min}(S^{-1} BB^* S^{-*}) \lambda_{\min}(\Sigma_{T_1}^{-1/2} S Q_{T_2} S^{*} \Sigma_{T_1}^{-1/2}) \\
    &= \lambda_{\min}(S^{-1} BB^* S^{-*}) \lambda_{\min}( (SQ_{T_2}S^*)^{1/2} \Sigma_{T_1}^{-1} (SQ_{T_2}S^*)^{1/2}) \\
    &\geq \frac{\lambda_{\min}(S^{-1} BB^* S^{-*})}{\lambda_{\max}(S^{-1} BB^* S^{-*})} \lambda_{\min}( (SQ_{T_2}S^*)^{1/2} (S^{-*} Q_{T_1}^{-1} S^*) (S Q_{T_2} S^*)^{1/2}) \\
    &= \frac{\lambda_{\min}(S^{-1} BB^* S^{-*})}{\lambda_{\max}(S^{-1} BB^* S^{-*})} \lambda_{\min}( Q_{T_2} Q_{T_1}^{-1} ).
\end{align*}
Let $\lambda \in \C$ be an eigenvalue of $A$.
We have
\begin{align*}
    \sum_{k=0}^{t-1} \abs{\lambda}^{2k} = \begin{cases}
    \frac{1 - \abs{\lambda}^{2t}}{1-\abs{\lambda}^2} &\text{if } \abs{\lambda} < 1, \\
    t &\text{if } \abs{\lambda} = 1.
   \end{cases}
\end{align*}
Therefore, $(Q_{T_2} Q_{T_1}^{-1})_{ii}$ is:
\begin{align*}
    (Q_{T_2} Q_{T_1}^{-1})_{ii} = \begin{cases} \frac{1-\abs{\lambda_i}^{2T_2}}{1-\abs{\lambda_i}^{2T_1}} = \frac{1-(\abs{\lambda_i}^{2T_1})^{T_2/T_1}}{1-\abs{\lambda_i}^{2T_1}} &\text{if } \abs{\lambda_i} < 1, \\
    T_2/T_1 &\text{if } \abs{\lambda_i} = 1.
    \end{cases}
\end{align*}
Note that $\inf_{x \in (0, 1)} \frac{1-x^c}{1-x} = c$ for $c \in [0, 1]$.
Therefore, we can lower bound:
\begin{align*}
    \lambda_{\min}(Q_{T_2} Q_{T_1}^{-1}) \geq \frac{T_2}{T_1}.
\end{align*}
The claim now follows.
\end{proof}

\begin{myprop}
\label{prop:lower_bound_ratio}
Let $(A, B)$ be the dynamics matrices for an \LDSLS{} 
instance, and suppose $(A, B)$ satisfy
\Cref{assume:marginal_stability},
\Cref{assume:diagonalizability}, and
\Cref{assume:one_step_controllability}.
Put $\Gamma_t := \Gamma_t(A, B)$ for $t \in \N_{+}$ and
$\condNum := \condNum(A, B)$.
For any integers $k,t \in \N_{+}$ satisfying $k \leq t$, we have:
\begin{align*}
    \ulam(\Gamma_k, \Gamma_t) \geq \frac{1}{8\condNum} \frac{k}{t}.
\end{align*}
\end{myprop}
\begin{proof}
Let $\Sigma_t := \Sigma_t(A, B)$ for $t \in \N_{+}$.
We first consider the case when $k \geq 2$.
Observe that $\Gamma_t \preccurlyeq \Sigma_t$.
Furthermore, for any $k \geq 2$, we have:
\begin{align*}
    \Gamma_k = \frac{1}{k} \sum_{k'=1}^{k} \Sigma_{k'} \succcurlyeq \frac{1}{k} \sum_{k'=\floor{k/2}}^{k} \Sigma_{\floor{k/2}} = \frac{k - \floor{k/2}+1}{k} \Sigma_{\floor{k/2}} \succcurlyeq \frac{1}{2} \Sigma_{\floor{k/2}}. 
\end{align*}
Therefore,
\begin{align*}
    \ulam(\Gamma_k, \Gamma_t) = \lambda_{\min}(\Gamma_t^{-1/2} \Gamma_k \Gamma_t^{-1/2}) \geq \frac{1}{2} \lambda_{\min}( \Sigma_t^{-1/2} \Sigma_{\floor{k/2}} \Sigma_t^{-1/2} ) \stackrel{(a)}{\geq} \frac{1}{2\condNum} \frac{\floor{k/2}}{t} \geq \frac{1}{8\condNum} \frac{k}{t}.
\end{align*}
Above, (a) follows from \Cref{prop:ratio_covariances}.
When $k=1$, we have $\Gamma_1 = \Sigma_1$, and therefore by \Cref{prop:ratio_covariances}:
\begin{align*}
    \ulam(\Gamma_1, \Gamma_t) = \ulam(\Sigma_1, \Gamma_t) \geq \ulam(\Sigma_1, \Sigma_t) = \lambda_{\min}(\Sigma_t^{-1/2} \Sigma_1 \Sigma_t^{-1/2}) \geq \frac{1}{\condNum} \frac{1}{t}.
\end{align*}
The claim now follows.
\end{proof}

\begin{myfact}
\label{stmt:Gamma_t_is_monotonic}
Let $(A, B)$ be the dynamics matrices for an \LDSLS{} instance. 
For any $s, t \in \N_{+}$ with $s \leq t$:
\begin{align*}
    \Gamma_s(A, B) \preccurlyeq \Gamma_t(A, B).
\end{align*}
\end{myfact}

\begin{myprop}
\label{stmt:uMu_lower_bound_lds_marginals}
Let $(A, B)$ be the dynamics matrices for an \LDSLS{} 
instance, and suppose $(A, B)$ satisfy
\Cref{assume:marginal_stability},
\Cref{assume:diagonalizability}, and
\Cref{assume:one_step_controllability}.
Put $\Gamma_t := \Gamma_t(A, B)$ for $t \in \N_{+}$,
$\Sigma_t := \Sigma_t(A, B)$ for $t \in \N_{+}$, and
$\condNum := \condNum(A, B)$.
For any $T$, we have:
\begin{align*}
    \left[ \prod_{t=1}^{T} \ulam(\Sigma_t, \Gamma_T) \right]^{1/T} \geq \frac{1}{8e\condNum}.
\end{align*}
\end{myprop}
\begin{proof}
By \Cref{prop:lower_bound_ratio},
we have that $\ulam(\Gamma_t,\Gamma_T) \geq \frac{1}{8\condNum} \frac{t}{T}$
for all $t \in \{1, \dots, T\}$. Therefore,
since $\ulam(\Sigma_t, \Gamma_T) \geq \ulam(\Gamma_t, \Gamma_T)$,
and since $n! \geq (n/e)^n$ for all $n \in \N_{+}$,
\begin{align*}
    \left[ \prod_{t=1}^{T} \ulam(\Sigma_t, \Gamma_T) \right]^{1/T} \geq \frac{(T!)^{1/T}}{8\condNum T} \geq \frac{1}{8e \condNum}.
\end{align*}
\end{proof}

\subsubsection{Many trajectory results}

\begin{mylemma}
\label{stmt:upper_bound_many_trajs_helper}
There are universal positive constants $c_0$ and $c_1$ such that the
following holds for any instance of \LDSLS{}.
Suppose that $(A, B)$ is $\kcont$-step controllable.
If $n \geq 2$ and $m \geq c_0 n$, then
for any $\Gamma' \in \sfSym^n_{> 0}$:
\begin{align}
    \E[ \norm{\hat{W}_{m,T} - W_\star}_{\Gamma'}^2 ]  \leq c_1 \sigma_\xi^2 \cdot \frac{pn}{mT \cdot \ulam(\Gamma_T(A, B), \Gamma')}.
    \label{eq:lds_ls_many_traj_general_rate}
\end{align}
\end{mylemma}
\begin{proof}
Let $\Gamma_T := \Gamma_T(A, B)$.
By \Cref{stmt:lds_traj_small_ball},
\LDSLS{} satisfies the $(T, T, \Gamma_T, e, 1/2)$-\TSB{}
condition.
We therefore invoke \Cref{stmt:upper_bound_general} with $k = T$ and $\uGam = \Gamma_T$.
In this case,
$\uMu$ from \eqref{eq:uMu_defn} simplifies to 
$\uMu = \ulam(\Gamma_T, \Gamma_T) = 1$,
and the requirement \eqref{eq:mTkn_requirements}
simplifies to $n \geq 2$ and
$m \geq c_0 n$.
Finally, the rate \eqref{eq:general_risk_bound} simplifies to
\eqref{eq:lds_ls_many_traj_general_rate}.
\end{proof}

\upperboundparameterrecovery*
\begin{proof}
Follows by invoking \Cref{stmt:upper_bound_many_trajs_helper}
with $\Gamma' = I_n$.
\end{proof}

\upperboundmanytrajectories*
\begin{proof}
Let $\Gamma_t := \Gamma_t(A, B)$ for $t \in \N_{+}$.
Invoking \Cref{stmt:upper_bound_many_trajs_helper} with $\Gamma' = \Gamma_{\Tnew}$
yields the bound:
\begin{align*}
    \E[L(\hat{W}_{m,T};\Tnew, \PxA{A,B})] \leq c_1 \sigma_\xi^2 \cdot \frac{pn}{mT \cdot \ulam(\Gamma_T, \Gamma_{\Tnew})}.
\end{align*}
If $\Tnew \leq T$, then $\ulam(\Gamma_T, \Gamma_{\Tnew}) \geq 1$
since $\Gamma_{T} \succcurlyeq \Gamma_{\Tnew}$ by 
\Cref{stmt:Gamma_t_is_monotonic}.
On the other hand, if $\Tnew > T$,
by \Cref{prop:lower_bound_ratio},
$\ulam(\Gamma_T, \Gamma_{\Tnew}) \geq \frac{1}{8\condNum} \frac{T}{\Tnew}$.
The claim now follows.
\end{proof}

\subsubsection{Few trajectory results}

\fewtrajrate*
\begin{proof}
Let $\Gamma_t := \Gamma_t(A, B)$ for all $t \in \N_{+}$.
By \Cref{stmt:lds_traj_small_ball},
for any $k \in \{1, \dots, T\}$, \LDSLS{}
satisfies the $(T, k, \Gamma_k, e, 1/2)$-\TSB{} condition.
We will apply \Cref{stmt:upper_bound_general} with $\uGam = \Gamma_k$.
The quantity $\uMu$ from \eqref{eq:uMu_defn}
simplifies to $\uMu = \ulam(\Gamma_k, \Gamma_k) = 1$.
By \Cref{prop:lower_bound_ratio},
we have that $\ulam(\Gamma_k, \Gamma_T) \geq \frac{1}{8\condNum} \frac{k}{T}$.
Hence the requirement \eqref{eq:mTkn_requirements}
simplifies to $n \geq 2$ and
\begin{align}
    \frac{mT}{kn} \geq c \log\left(\condNum' \frac{T}{k}\right), \quad \gamma' := \max\{e, \gamma\} \label{eq:requirement_v1}
\end{align}
for some universal positive constant $c$.
Thus, for \eqref{eq:requirement_v1} to hold, it suffices to require:
\begin{align}
    \frac{T}{k} \geq \max\left\{ \frac{2 c n}{m} \log{\condNum'}, \frac{2cn}{m} \log\left(\frac{T}{k}\right)\right\}. \label{eq:requirement_v2}
\end{align}
As long as $2c n/m \geq 1$, then by \Cref{prop:invert_log_t_over_t},
\begin{align*}
    \frac{T}{k} \geq \frac{4cn}{m} \log\left(\frac{8 c n}{m}\right) \Longrightarrow \frac{T}{k} \geq \frac{2cn}{m} \log\left(\frac{T}{k}\right).
\end{align*}
Hence, for \eqref{eq:requirement_v2} to hold, it suffices to require
\begin{align}
    \frac{T}{k} \geq \frac{4c n}{m} \log\left(\frac{8 c\condNum' n}{m}\right). \label{eq:requirement_v4}
\end{align}
Based on \eqref{eq:requirement_v4}, we choose $k$ as:
\begin{align}
    k = \bigfloor{\frac{T}{4cn/m \cdot \log(8 c\condNum' n/m)}}. \label{eq:choice_of_k}
\end{align}
To ensure that $k \geq 1$, we need to ensure that:
\begin{align}
    mT \geq 4c n \log(8 c\condNum' n/m).
\end{align}
On the other hand, since $2cn/m \geq 1$, we have that:
\begin{align*}
     \frac{4c n}{m} \log\left(\frac{8c \condNum' n}{m}\right) \geq 1,
\end{align*}
which ensures that $k \leq T$.
Thus, our choice of $k$ from \eqref{eq:choice_of_k}
ensures that \eqref{eq:requirement_v1} holds.
We now ready to invoke \Cref{stmt:upper_bound_general}
with $\Gamma' = \Gamma_{\Tnew}$,
and conclude for a universal $c'$:
\begin{align}
    \E[ L(\hat{W}_{m,T};\Tnew, \PxA{A,B}) ] \leq c' \sigma_\xi^2 \cdot \frac{p n \log(e/\ulam(\Gamma_k, \Gamma_T)) }{mT \ulam(\Gamma_k, \Gamma_{\Tnew})}. \label{eq:bound_general_v1}
\end{align}

First, we assume that $\Tnew \leq k$.
By \Cref{stmt:Gamma_t_is_monotonic} we have $\Gamma_k \succcurlyeq \Gamma_{\Tnew}$, and therefore
$\ulam(\Gamma_k, \Gamma_{\Tnew}) \geq 1$.
Equation \eqref{eq:bound_general_v1} yields:
\begin{align}
     \E[ L(\hat{W}_{m,T};\Tnew, \PxA{A,B}) ] \leq c' \sigma_\xi^2   \cdot \frac{ pn \log(e/\ulam(\Gamma_k, \Gamma_T)) }{mT }
    \label{eq:bound_Tnew_leq_T}
\end{align}
By \Cref{prop:lower_bound_ratio},
\begin{align}
    \ulam(\Gamma_k, \Gamma_T) \geq \frac{1}{8\condNum} \frac{1}{T} \bigfloor{\frac{T}{4cn/m \cdot \log(8c\condNum' n/m)}} \geq \frac{m}{64c \condNum n \log(8 c \condNum' n/m)}. \label{eq:bound_v1}
\end{align}
Plugging \eqref{eq:bound_v1} into \eqref{eq:bound_Tnew_leq_T},
and using the inequalities
$\log{x} \leq x$ for $x > 0$ and $\phi(a, x) \geq 1$
for all $a \geq 1$ yields, for another universal $c''$:
\begin{align*}
    \E[ L(\hat{W}_{m,T}; \Tnew, \PxA{A,B}) ] &\leq c' \sigma_\xi^2 \cdot \frac{pn}{mT} \cdot \log(e \cdot 64c\condNum n/m \cdot  \log(8c \condNum' n/m)) \\
    &\leq c' \sigma_\xi^2 \cdot \frac{pn \log( 512e \cdot (c \condNum' n/m)^2 )}{mT}  \\
    &\leq c'' \sigma_\xi^2 \cdot \frac{pn \log(\max\{\condNum n/m, e\}  )}{mT}  \\
    &\leq c'' \sigma_\xi^2 \cdot \frac{pn \log(\max\{\condNum n/m, e\}  )}{mT} \cdot \phi\left( \condNum, c_1 \frac{n \log(\max\{\condNum n/m, e\})}{m} \frac{\Tnew}{T}\right).
\end{align*}
On the other hand, if $\Tnew > k$, 
then by \Cref{prop:lower_bound_ratio},
\begin{align}
    \ulam(\Gamma_k, \Gamma_{\Tnew}) \geq \frac{1}{8\condNum} \frac{1}{\Tnew} \bigfloor{\frac{T}{4cn/m \cdot \log(8c\condNum' n/m)}} \geq \frac{m}{64c \condNum n \log(8c \condNum' n/m)} \cdot \frac{T}{\Tnew}. \label{eq:bound_v2}
\end{align}
Plugging \eqref{eq:bound_v1} and \eqref{eq:bound_v2} into 
\eqref{eq:bound_general_v1} and using again the inequality
$\log{x} \leq x$ for $x > 0$ yields,
for a universal $c'''$:
\begin{align*}
    &~~~~\E[ L(\hat{W}_{m,T}; \Tnew, \PxA{A,B}) ] \\
    &\leq c' \sigma_\xi^2 \cdot \frac{pn}{mT} \cdot \log(e \cdot 64c\condNum n/m \cdot  \log(8c \condNum' n/m)) \cdot 64c\condNum n / m \cdot \log(8c \condNum' n/m ) \cdot \frac{\Tnew}{T} \\
    &\leq c''' \sigma_\xi^2 \cdot \frac{pn\log(\max\{\condNum n/m, e\})}{mT} \cdot \condNum \frac{n\log(\max\{\condNum n/m, e\}) }{m} \cdot \frac{\Tnew}{T}.
\end{align*}
Furthermore, when $\Tnew > k$, by choosing $c_1$ sufficiently large:
\begin{align*}
    8 c \frac{n \log(8c\condNum' n/m)}{m} \frac{\Tnew}{T} > 1 %
    &\Longrightarrow c_1 \frac{n \log(\max\{\condNum n/m, e\})}{m} \frac{\Tnew}{T} > 1 \\
    &\Longrightarrow \condNum c_1 \frac{n \log(\max\{\condNum n/m, e\})}{m} \frac{\Tnew}{T} = \phi\left( \condNum, c_1 \frac{n \log(\max\{\condNum n/m, e\})}{m} \frac{\Tnew}{T}\right).
\end{align*}
The claim now follows.
\end{proof}

\upperboundindldsls*
\begin{proof}
Let $\Gamma_t := \Gamma_t(A, B)$ and $\Sigma_t := \Sigma_t(A, B)$
for $t \in \N_{+}$.
From \Cref{example:independent_gaussians},
we have that 
\PLR{} satisfies the $(T, 1, \{ \Sigma_t \}_{t=1}^{T}, e, 1/2)$-\TSB{} condition.
We will apply \Cref{stmt:upper_bound_general} with
$\uGam = \Gamma_T$, $k=1$, and $\Gamma' = \Gamma_{\Tnew}$.
By \Cref{stmt:uMu_lower_bound_lds_marginals}, we have that:
\begin{align*}
    \uMu(\{\Sigma_t\}_{t=1}^{T}, \Gamma_T) \geq \frac{1}{8e \condNum}.
\end{align*}
The requirement \eqref{eq:mTkn_requirements} simplifies to
$n \geq 2$ and $\frac{mT}{n} \geq c \log(\max\{\condNum,e\})$ 
for a universal constant $c$.
By \Cref{stmt:upper_bound_general}, 
for a universal $c'$:
\begin{align*}
    \E[L(\hat{W}_{m,T}; \Tnew, \PxA{A,B})] \leq c' \sigma_\xi^2 \cdot \frac{pn \log(\max\{\condNum,e\})}{mT \cdot \ulam(\Gamma_T, \Gamma_{\Tnew}) } \cdot 8 e \condNum.
\end{align*}
If $\Tnew \leq T$, then $\ulam(\Gamma_T, \Gamma_{\Tnew}) \geq 1$
since $\Gamma_T \succcurlyeq \Gamma_{\Tnew}$ by \Cref{stmt:Gamma_t_is_monotonic}.
On the other hand, if $\Tnew > T$, then 
by \Cref{prop:lower_bound_ratio},
$\ulam(\Gamma_T, \gamma_{\Tnew}) \geq \frac{1}{8\condNum} \frac{T}{\Tnew}$.
The claim now follows.
\end{proof}

\paramrecoveryfewtrajrate*
\begin{proof}
The proof is identical to that of \Cref{thm:sublinear_trajectories_bound}
until \eqref{eq:bound_general_v1}, after which we set $\Tnew = 1$
from which the result follows.
\end{proof}

\subsection{High probability upper bounds}
\label{sec:app:high_prob_upper_bounds}

\subsubsection{Weak trajectory small ball}

We first present a modified definition of trajectory small-ball
(cf.~\Cref{def:trajectory_small_ball}) which we will use to establish
high probability bounds.
\begin{mydef}[Weak trajectory small-ball (\wTSB{})]
\label{def:weak_trajectory_small_ball}
Fix a trajectory length $T \in \N_+$, 
a parameter $k \in \{1, \dots, T\}$,
positive definite matrices $\{\Psi_j\}_{j=1}^{\floor{T/k}} \subset \sfSym^{n}_{> 0}$, 
and constants $\alpha, \beta \in (0, 1)$.
The distribution $\Px$ satisfies the
$(T, k, \{\Psi_j\}_{j=1}^{\floor{T/k}}, \alpha, \beta)$-\vocab{weak-trajectory-small-ball (\wTSB{})}
condition if:
\begin{enumerate}
    \item $\frac{1}{\floor{T/k}} \sum_{j=1}^{\floor{T/k}} \Psi_j \preccurlyeq \Gamma_T(\Px)$,
    \item $\{x_t\}_{t \geq 1}$ is adapted to a filtration $\{\calF_t\}_{t \geq 1}$, and
    \item for all $v \in \R^n \setminus \{0\}$, $j \in \{1, \dots, \floor{T/k}\}$:
\begin{align}
    \Pr_{\{x_t\} \sim \Px}\left\{
      \frac{1}{k} \sum_{t=(j-1)k+1}^{jk} \ip{v}{x_t}^2 \leq \alpha \cdot v^\T \Psi_j v
      ~\Bigg|~ \calF_{(j-1)k} \right\} \leq \beta \:\:
      \textrm{a.s.}\label{eq:weak_trajectory_small_ball}
\end{align}
\end{enumerate}
\end{mydef}
The main difference between \Cref{def:weak_trajectory_small_ball}
vs.\ \Cref{def:trajectory_small_ball} is the third condition
\eqref{eq:weak_trajectory_small_ball}, which only needs to hold for a 
\emph{fixed} resolution $\alpha$ and failure probability $\beta$.
By contrast, in \Cref{def:trajectory_small_ball}, the condition
must hold of for \emph{all} resolutions---there denoted by
$\e$---with failure probabilities that tend to zero as
the resolution $\e \to 0$
(cf.~\eqref{eq:trajectory_small_ball}).

\subsubsection{Ordinary least squares bounds}

\begin{mylemma}[Minimum eigenvalue bound via weak trajectory small-ball]
\label{stmt:weak_small_ball_to_min_eval}
Suppose that $\Px$ satisfies
the $(T,k,\{\Psi_j\}_{j=1}^{\floor{T/k}},\alpha,\beta)$-\wTSB{} condition
(\Cref{def:weak_trajectory_small_ball}).
Put $S := \floor{T/k}$ and $\Gamma_t := \Gamma_t(\Px)$
for $t \in \N_{+}$.
Fix any $\uGam \in \sfSym^{n}_{> 0}$ satisfying
$\frac{1}{S} \sum_{j=1}^{S} \Psi_j \preccurlyeq \uGam \preccurlyeq \Gamma_T$,
and define the constants:
\begin{align}
    C_S := \frac{\frac{1}{S}\sum_{j=1}^{S} \ulam(\Psi_j, \uGam)^2}{\left(\frac{1}{S}\sum_{j=1}^{S} \ulam(\Psi_j, \uGam)\right)^2}, \quad \bar\mu := \frac{1}{S} \sum_{j=1}^{S} \ulam(\Psi_j, \uGam). \label{eq:weak_small_ball_growth_conditions}
\end{align}
(Note that $1 \leq C_S \leq S$ always).
Fix $\delta \in (0, 1)$, and 
suppose that:
\begin{align*}
    n \geq 2, \quad \frac{mT}{kn} \geq \frac{64C_S}{1-\beta}\log\left( \frac{1280 C_S}{\alpha (1-\beta)  \ulam(\uGam, \Gamma_T) \bar\mu \delta}\right).
\end{align*}
With probability at least $1-\delta$,
the following events simulatenously hold:
\begin{align}
    \lambda_{\min}\left( \uGam^{-1/2}\sum_{i=1}^{m}\sum_{t=1}^{T} x_t^{(i)} (x_t^{(i)})^\T \uGam^{-1/2}\right) &\geq \frac{\alpha (1-\beta) mT \bar\mu}{8}, \label{eq:high_prob_min_eval_bound} \\
    \Tr\left( \uGam^{-1/2}\sum_{i=1}^{m}\sum_{t=1}^{T} x_t^{(i)} (x_t^{(i)})^\T \uGam^{-1/2} \right) &\leq \frac{2mTn }{\ulam(\uGam, \Gamma_T) \cdot \delta}. \nonumber
\end{align}
\end{mylemma}
\begin{proof}
The proof proceeds quite similarly to the proof of 
\Cref{stmt:small_ball_to_min_eval}.
Thus, we focus mostly on the parts that differ.
For notational brevity, let:
\begin{align*}
    \beta' := 1 - \beta, \quad \ulam := \ulam(\uGam, \Gamma_T), \quad  \ulam_j := \ulam(\Psi_j, \uGam).
\end{align*}
Since $\uGam \preccurlyeq \Gamma_T$ by assumption, we have
$\ulam \in (0, 1]$.

The first step, in preparation for applying the PAC-Bayes deviation
inequality, is to construct a family of random variables with moment generating function upper bounded by one.
To do this, we utilize the weak trajectory small-ball condition \eqref{eq:weak_trajectory_small_ball}, which
implies for any $v \in \mathbb{S}^{n-1}$ and $j \in \{1, \dots, S\}$:
\begin{align*}
    \Pr\left\{ \frac{1}{k}\sum_{t=(j-1)k+1}^{jk} \ip{v}{\uGam^{-1/2} x_t}^2 \leq \alpha \ulam(\Psi_j, \uGam) \,\Bigg|\, \calF_{(j-1)k} \right\} \leq \beta.
\end{align*}
Let $\tilde{x}_t := \uGam^{-1/2} x_t$ be the whitened vector.
Define the random indicator variables
for $i=1, \dots, m$ and $j=1, \dots, S$:
\begin{align*}
    B_j^{(i)} &:= \ind\left\{ \frac{1}{k}\sum_{t=(j-1)k+1}^{jk} \ip{v}{\tilde{x}_t^{(i)}}^2 \geq \alpha \ulam(\Psi_j, \uGam) \right\}.
\end{align*}
By Markov's inequality:
\begin{align*}
    \sum_{t=1}^{T} \ip{v}{\tilde{x}_t^{(i)}}^2 
    \geq k\alpha \sum_{j=1}^{S} \ulam_j \ind\{ B_j^{(i)} = 1 \}.
\end{align*}
Hence for any $\eta > 0$ and $v \in \mathbb{S}^{n-1}$:
\begin{align*}
    \E\exp\left( -\eta \sum_{t=1}^{T} \ip{v}{\tilde{x}_t^{(i)}}^2 \right) \leq \E \exp\left( -\eta k \alpha\sum_{j=1}^{S} \ulam_j \ind\{ B_j^{(i)} = 1 \} \right).
\end{align*}
Now observe:
\begin{align*}
    \E[\exp(-\eta k \alpha \ulam_j \ind\{B_j^{(i)}=1\}) \mid \calF_{(j-1)k}] &= e^{-\eta k \alpha \ulam_j} \Pr(B_j^{(i)}=1 \mid \calF_{(j-1)k}) + \Pr(B_j^{(i)}=0 \mid \calF_{(j-1)k}) \\
    &= (e^{-\eta k \alpha \ulam_j} - 1) \Pr(B_j^{(i)}=1 \mid \calF_{(j-1)k}) + 1 \\
    &\leq (e^{-\eta k \alpha \ulam_j} - 1) \beta' + 1 \\
    &\stackrel{(a)}{\leq} 1 + \left( -\eta k \alpha \ulam_j + \frac{1}{2} \eta^2 k^2 \alpha^2 \ulam_j^2 \right)\beta' \\
    &\stackrel{(b)}{\leq} \exp\left( \left( -\eta k \alpha \ulam_j + \frac{1}{2} \eta^2 k^2 \alpha^2 \ulam_j^2 \right)\beta' \right).
\end{align*}
Above, we used the facts
(a) for $x > 0$, we have
$e^{-x} - 1 \leq -x + \frac{x^2}{2}$,
and (b) for $x \in \R$, we have $1 + x \leq e^x$.
Hence by the tower property:
\begin{align*}
    \E\exp\left( -\eta\sum_{t=1}^{T} \ip{v}{\tilde{x}_t^{(i)}}^2\right) 
    &\leq \exp\left( \sum_{j=1}^{S} \left( -\eta k \alpha \ulam_j + \frac{1}{2} \eta^2 k^2 \alpha^2 \ulam_j^2 \right)\beta' \right) \\
    &\leq \exp\left( - \eta k \alpha \beta' \sum_{j=1}^{S} \ulam_j + \frac{1}{2}\eta^2 k^2 \alpha^2 \beta' \sum_{j=1}^{S} \ulam_j^2 \right) \\
    &= \exp\left(  -\eta k \alpha \beta' \left( \sum_{j=1}^{S} \ulam_j\right) \left( 1 - \frac{\eta k \alpha}{2} \frac{\sum_{j=1}^{S} \ulam_j^2}{\sum_{j=1}^{S} \ulam_j}  \right)\right).
\end{align*}
Now, let us set
\begin{align*}
    \eta = \frac{1}{k\alpha} \frac{\sum_{j=1}^{S} \ulam_j}{\sum_{j=1}^{S} \ulam_j^2} = \frac{1}{k\alpha} \cdot \frac{1}{\bar\mu} \cdot \frac{1}{C_S},
\end{align*}
from which we conclude:
\begin{align*}
    \E\exp\left( -\eta\sum_{t=1}^{T} \ip{v}{\tilde{x}_t^{(i)}}^2\right) 
    &\leq \exp\left( -\frac{\beta'}{2} \frac{\left(\sum_{j=1}^{S} \ulam_j\right)^2}{\sum_{j=1}^{S} \ulam_j^2} \right) = \exp\left(-\frac{S \beta'}{2} \frac{\left(\frac{1}{S}\sum_{j=1}^{S} \ulam_j\right)^2}{\frac{1}{S}\sum_{j=1}^{S} \ulam_j^2}\right) \\
    &= \exp\left(-\frac{S\beta'}{2C_S}\right).
\end{align*}
By independence across the $m$ trajectories:
\begin{align*}
    \E \exp\left( -\eta \sum_{i=1}^{m} \sum_{t=1}^{T} \ip{v}{\tilde{x}_t^{(i)}}^2 + \frac{mS \beta'}{2C_S} \right) \leq 1.
\end{align*}
As desired, we have constructed a family of random variables indexed by 
$v \in \mathbb{S}^{n-1}$, with MGF bounded by one.

Using the PAC-Bayes arguments from 
\Cref{stmt:small_ball_to_min_eval}
followed by Markov's inequality, 
with probability at least $1-2e^{-t}$, for all 
$v \in \mathbb{S}^{n-1}$ and $\gamma \in [0, 1/2]$:
\begin{align}
    \sum_{i,t} \ip{\tilde{x}_t^{(i)}}{v}^2 \geq \frac{1}{\eta}\left[ \frac{mS\beta'}{2C_S} - n\log\left(\frac{5}{4\gamma^2}\right) - t\right] - 
    \frac{4\gamma^2 e^t mT}{\ulam}. \label{eq:high_prob_min_eval_ineq_1}
\end{align}
Choosing $\gamma^2 = \frac{n \ulam}{4  \eta mT e^{t}}$, we have that:
\begin{align}
    \frac{n}{\eta} \log\left(\frac{5}{4\gamma^2}\right) + \frac{4 mT e^{t}}{\ulam} \gamma^2 = \frac{n}{\eta} \left[ 1 + \log\left( \frac{5 m T e^{t} \eta }{ n \ulam }  \right)\right]. \label{eq:high_prob_min_eval_ineq_2}
\end{align}
Note that this choice of $\gamma$ satisfies $\gamma \in [0, 1/2]$, since:
\begin{align*}
    \frac{n \ulam}{4 \eta mT e^{t}} \leq \frac{1}{4} &\Longleftarrow \frac{n}{\eta mT} \leq 1 &&\text{since } t \geq 0 \text{ and } \ulam \leq 1.
\end{align*}
The RHS above is ensured by:
\begin{align*}
    \frac{mT}{kn} \geq \alpha \frac{\sum_{j=1}^{S} \ulam_j^2}{\sum_{j=1}^{S} \ulam_j} = \alpha C_S \cdot \bar\mu \Longleftarrow \frac{mT}{kn} \geq C_S &&\text{since } \alpha, \bar\mu \leq 1.
\end{align*}
If we further enforce that:
\begin{align}
    \frac{mS\beta'}{4C_S} \geq (n+1)\left[t + \log\left(\frac{5mT\eta}{n\ulam}\right)\right], \label{eq:high_prob_min_eval_data_req}
\end{align}
then combining \eqref{eq:high_prob_min_eval_ineq_1}
with \eqref{eq:high_prob_min_eval_ineq_2}:
\begin{align*}
     \sum_{i,t} \ip{\tilde{x}_t^{(i)}}{v}^2 &\geq \frac{1}{\eta}\left[ \frac{mS\beta'}{2C_S} - t - n - n\log\left(\frac{5mT e^t \eta}{n\ulam}\right) \right] \\
     &= \frac{1}{\eta}\left[ \frac{mS\beta'}{2C_S} - (n+1)t - n - n \log\left( \frac{5mT\eta}{n\ulam} \right) \right] \\
     &\geq \frac{1}{\eta}\left[ \frac{mS\beta'}{2C_S} - (n+1)t - (n+1) \log\left(\frac{5mT\eta}{n\ulam}\right) \right] \\
     &\geq \frac{mS\beta'}{4\eta C_S} = \frac{\alpha \beta' mT \bar\mu}{8}.
\end{align*}
For \eqref{eq:high_prob_min_eval_data_req}, it suffices that:
\begin{align*}
    \frac{mT}{kn} \geq \frac{32 C_S t}{\beta'}, \quad \frac{mT}{kn} \geq \frac{16 C_S}{\beta'} \log\left( \frac{5}{\ulam \alpha} \frac{\sum_{j=1}^{S} \ulam_j}{\sum_{j=1}^{S} \ulam_j^2} \frac{mT}{kn} \right) = \frac{16 C_S}{\beta'} \log\left(\frac{5}{\ulam \alpha \bar\mu C_S}\right) + \frac{16 C_S}{\beta'} \log\left(\frac{mT}{kn}\right).
\end{align*}
For the RHS inequality, by \Cref{prop:invert_log_t_over_t} 
it suffices that:
\begin{align*}
    \frac{mT}{kn} \geq \frac{32 C_S}{\beta'} \max\left\{\log\left( \frac{5}{\ulam \alpha \bar\mu C_S} \right), 0\right\}, \quad \frac{mT}{kn} \geq \frac{64 C_S}{\beta'} \log\left( \frac{128 C_S}{\beta'} \right).
\end{align*}
Note that $x \log(1/x) \leq 1/e$ for all $x > 0$, and therefore:
\begin{align*}
    C_S\log\left( \frac{5}{\ulam \alpha \bar\mu C_S} \right) &= C_S\log\left(\frac{5}{\alpha \ulam \bar\mu}\right) + C_S\log\left(\frac{1}{C_S} \right) \leq C_S\log\left(\frac{5}{\alpha \ulam \bar\mu}\right) + 1
    \leq 2 C_S \log\left(\frac{5}{\alpha \ulam \bar\mu}\right).
\end{align*}
Hence it suffices that:
\begin{align*}
    \frac{mT}{kn} \geq \frac{64C_S}{\beta'} \max\left\{\log\left(\frac{5}{\alpha \ulam \bar\mu} \right),   \log\left( \frac{128 C_S}{\beta'} \right)\right\}.
\end{align*}
The claim now follows by simplifying all the required inequalities
for the quantity $mT/kn$.
\end{proof}

To contrast the effects of the \wTSB{} assumption from those
of the \TSB{} assumption,
let us compare \Cref{stmt:weak_small_ball_to_min_eval}
to its counterpart \Cref{stmt:small_ball_to_min_eval}.
The minimum eigenvalue
bound \eqref{eq:high_prob_min_eval_bound} from \Cref{stmt:weak_small_ball_to_min_eval}
differs from the corresponding \TSB{}
bound \eqref{eq:traj_sb_min_eval_bound} in the role of the
eigenvalues of the matrices $\{ \Psi_j\}_{j=1}^{\floor{T/k}}$
from the small-ball definition. 
However, due to the differing requirements on the amount of data
$mT$, neither result is necessarily sharper than the other, as
detailed in the following remark:
\begin{myremark}
\normalfont
When the matrices $\{ \Psi_j\}_{j=1}^{\floor{T/k}}$ 
from the trajectory small-ball definition
vary across $j$, 
both \Cref{stmt:small_ball_to_min_eval}
and \Cref{stmt:weak_small_ball_to_min_eval}
yield different dependencies
on the eigenvalues $\{\ulam(\Psi_j, \uGam)\}_{j=1}^{\floor{T/k}}$.
In particular,
\Cref{stmt:small_ball_to_min_eval}
yields a minimum eigenvalue bound scaling as
$mT \uMu$, where $\uMu$ is the
\emph{geometric mean} of the eigenvalues
$\{\ulam(\Psi_j, \uGam)\}$, whereas
\Cref{stmt:weak_small_ball_to_min_eval}
yields a bound scaling as $mT \bar\mu$, 
where $\bar\mu$ is the \emph{arithmetic
mean} of the eigenvalues.
By the AM-GM inequality, we have that
$\bar\mu \geq \uMu$, so the latter bound
is stronger than the former.
However, \Cref{stmt:weak_small_ball_to_min_eval}
has a stronger requirement on 
the amount of data, requiring that
$mT \gtrsim kn C_S$, where $C_S \in [1, S]$ 
is defined in \eqref{eq:weak_small_ball_growth_conditions}, whereas \Cref{stmt:small_ball_to_min_eval} 
has the weaker requirement that $mT \gtrsim kn$.
In the worst case when $C_S \asymp S$, 
then the $mT \gtrsim kn C_S$ requirement
simplifies to the many trajectories 
assumption $m \gtrsim n$.
Thus, the qualitative behavior of these two bounds
are not necessarily comparable.
\end{myremark}

Meanwhile, although neither bound is strictly sharper than
the other, if we assume polynomial growth of the $\{\Psi_j\}$
matrices, then the two bounds are roughly on par:
\begin{myremark}
\normalfont
When the matrices $\{\Psi_j\}$ exhibit low degree polynomial
growth, both \Cref{stmt:small_ball_to_min_eval}
and \Cref{stmt:weak_small_ball_to_min_eval}
yield similar qualitative behavior.
Concretely, let us suppose that
$k=1$, $\Psi_j = j^p \cdot I$ for $j \in [T]$,
and $\uGam = \frac{1}{T} \sum_{j=1}^{T} \Psi_j$.
Then, $\bar\mu = 1$, whereas
$\uMu \geq \frac{1+p}{e^p}$.
Thus, if we consider $p$ as constant, then $\bar\mu \asymp \uMu$.
\end{myremark}

We now state our general OLS upper bound under 
the weak trajectory small-ball condition.
\begin{mylemma}[General OLS upper bound, high probability]
\label{stmt:upper_bound_main_high_prob}
There are universal positive constants $c_0$ and $c_1$ such that the following holds.
Suppose that $\Px$ satisfies
the $(T,k,\{\Psi_j\}_{j=1}^{\floor{T/k}},\alpha,\beta)$-\wTSB{} condition
(\Cref{def:weak_trajectory_small_ball}).
Put $S := \floor{T/k}$ and $\Gamma_t := \Gamma_t(\Px)$
for $t \in \N_{+}$.
Fix any $\uGam \in \sfSym^{n}_{> 0}$ satisfying
$\frac{1}{S} \sum_{j=1}^{S} \Psi_j \preccurlyeq \uGam \preccurlyeq \Gamma_T$,
and the constants:
\begin{align*}
    C_S := \frac{\frac{1}{S}\sum_{j=1}^{S} \ulam(\Psi_j, \uGam)^2}{\left(\frac{1}{S}\sum_{j=1}^{S} \ulam(\Psi_j, \uGam)\right)^2}, \quad \bar\mu := \frac{1}{S} \sum_{j=1}^{S} \ulam(\Psi_j, \uGam).
\end{align*}
Fix $\delta \in (0, 1/e)$.
Suppose that:
\begin{align*}
    n \geq 2, \quad \frac{mT}{kn} \geq \frac{c_0 C_S}{1-\beta}\log\left( \frac{C_S}{\alpha (1-\beta)  \ulam(\uGam, \Gamma_T) \bar\mu \delta}\right).
\end{align*}
Then, for any $\Gamma' \in \sfSym^n_{> 0}$,
with probability at least $1-\delta$:
\begin{align*}
    \norm{\hat{W}_{m,T} - W_\star}^2_{\Gamma'} \leq c_1 \sigma_\xi^2 \left[ \frac{pn \log\left(\frac{1}{\alpha(1-\beta) \ulam(\uGam, \Gamma_T) \bar\mu \delta}\right)}{ \ulam(\uGam, \Gamma') \alpha(1-\beta) mT \bar\mu }\right].
\end{align*}
\end{mylemma}
\begin{proof}
Put $\beta' := 1-\beta$ and $\tilde{X}_{m,T} := X_{m,T} \uGam^{-1/2}$.
By the arguments in the proof of \Cref{stmt:upper_bound_general},
\begin{align*}
    \norm{ \hat{W}_{m,T} - W_\star }_{\Gamma'}^2 
    &\leq \min\{ n, p \} \frac{ \opnorm{ (\tilde{X}_{m,T}^\T \tilde{X}_{m,T})^{-1/2} \tilde{X}_{m,T}^\T \Xi_{m,T} }^2 }{\ulam(\uGam, \Gamma') \cdot \lambda_{\min}(\tilde{X}_{m,T}^\T \tilde{X}_{m,T} )}.
\end{align*}
Put $M := (\alpha \beta' mT \bar\mu/8) \cdot I := \zeta \cdot I$.
By \Cref{prop:yasin_vector_easier}, 
with probability at least $1-\delta/2$:
\begin{align*}
    &~~~~\ind\{\tilde{X}_{m,T}^\T \tilde{X}_{m,T} \succcurlyeq M\} \opnorm{ (\tilde{X}_{m,T}^\T \tilde{X}_{m,T})^{-1/2} \tilde{X}_{m,T}^\T \Xi_{m,T} }^2 \\
    &\leq 16 \sigma_\xi^2 \left[ p \log{5} + \frac{1}{2}\log\det\left(I_n + \zeta^{-1} \tilde{X}_{m,T}^\T \tilde{X}_{m,T}\right) + \log(2/\delta) \right] \\
    &\leq 32 \sigma_\xi^2 \left[ p + \log\det\left(I_n + \zeta^{-1} \tilde{X}_{m,T}^\T \tilde{X}_{m,T}\right) + \log(2/\delta) \right] \\
    &\leq 32 \sigma_\xi^2 \left[ p + n \log(1 + \zeta^{-1} \Tr(\tilde{X}_{m,T}^\T \tilde{X}_{m,T})/n) + \log(2/\delta) \right].
\end{align*}
Now, by \Cref{stmt:weak_small_ball_to_min_eval}, with
probability at least $1-\delta/2$, we also have:
\begin{align*}
    \lambda_{\min}(\tilde{X}_{m,T}^\T \tilde{X}_{m,T}) \geq \zeta, \quad \Tr(\tilde{X}_{m,T}^\T \tilde{X}_{m,T}) \leq \frac{4mTn}{\ulam \delta}.
\end{align*}
On both events:
\begin{align*}
    \opnorm{ (\tilde{X}_{m,T}^\T \tilde{X}_{m,T})^{-1/2} \tilde{X}_{m,T}^\T \Xi_{m,T} }^2 &\leq 32\sigma_\xi^2 \left[ p + n \log\left(1 + \frac{32}{\alpha\beta' \ulam \bar\mu \delta}\right) + \log(2/\delta) \right] \\
    &\leq 64 \sigma_\xi^2 \left[ p + n \log\left(\frac{33}{\alpha\beta'\ulam\bar\mu\delta}\right)\right].
\end{align*}
Combining the inequalities:
\begin{align*}
    \norm{\hat{W}_{m,T} - W_\star}^2_{\Gamma'} \leq 512 \sigma_\xi^2 \min\{n,p\} \left[\frac{ p + n \log\left(\frac{33}{\alpha\beta'\ulam\bar\mu\delta}\right) }{ \ulam(\uGam, \Gamma') \alpha\beta' mT \bar\mu }\right] \leq 1024 \sigma_\xi^2 \left[\frac{pn \log\left(\frac{33}{\alpha\beta' \ulam \bar\mu \delta}\right)}{ \ulam(\uGam, \Gamma') \alpha\beta' mT \bar\mu }\right].
\end{align*}
By a union bound, both events hold with probability
at least $1-\delta$, which concludes the proof.

\end{proof}

\subsubsection{Mixing implies weak trajectory small-ball}

One advantage of \Cref{def:weak_trajectory_small_ball} is that
it is implied by the standard notions of $\phi$-mixing in the literature (see e.g.~\cite{mohri2008rademachermixing,duchi2012ergodicmd,kuznetsov2017mixing}).
In this section, we prove this reduction.
First, we state the definition of $\phi$-mixing.
\begin{mydef}[$\phi$-mixing covariate sequence]
\label{def:phi_mixing}
Let $\{x_t\}_{t \geq 1}$ be a covariate sequence which is adapted
to a filtration $\{\calF_t\}_{t \geq 1}$.
Define the function $\phi(k)$ as:
\begin{align}
    \phi(k) := \sup_{t \in \N_+} \sup_{B \in \calF_t} \tvnorm{ \Pr_{x_{t+k}}(\cdot \mid B) - \Pr_{x_{t+k}} }. \label{eq:phi_mixing}
\end{align}
The process $\{x_t\}_{t \geq 1}$ is called
\emph{$\phi$-mixing} if $\lim_{k \to \infty} \phi(k) = 0$.
We also let $\bar{\phi}(k)$ denote the \vocab{upper envelope} of $\phi(k)$, i.e., $\bar{\phi}(k) := \sup_{k' \geq k} \phi(k)$.
\end{mydef}

The following result shows that a $\phi$-mixing covariate sequence where
each marginal distribution is weakly small-ball satisfies the
weak trajectory small-ball condition.
\begin{myprop}
\label{stmt:phi_mixing_implies_weak_small_ball}
Fix $\alpha \in (0, 1)$ and $\beta \in (0, 1/4)$.
Suppose that 
the covariate sequence $\{x_t\}_{t \geq 1}$ is $\phi$-mixing,
and that
for every $t \in \N_+$ and $v \in \R^n \setminus \{0\}$ we have:
\begin{align}
    \Pr_{x_t}\{ \ip{v}{x_t}^2 \leq \alpha v^\T \Sigma_t v \} \leq \beta, \quad \Sigma_t := \E[x_tx_t^\T]. \label{eq:weak_small_ball_marginal}
\end{align}
Let $\kmix := \inf\{ k \in \N_+ \mid \bar{\phi}(k) \leq \beta \}$
and assume that $T \geq 2\kmix$.
Put $S := \floor{T/(2\kmix)}$ and suppose that $\{\Psi_j\}_{j=1}^{S}$ satisfies:
\begin{align*}
    \Psi_j \preccurlyeq \frac{1}{4} \Sigma_t \quad\forall j \in [S], \, t \in [\kmix(2j - 1) + 1, 2j \kmix].
\end{align*}
Then, $\Px$ satisfies the
$\left(T, 2\kmix, \{\Psi_j\}_{j=1}^{T/(2\kmix)}, \alpha, \frac{4}{3}\left(\frac{1}{2}+\beta\right)\right)$-\wTSB\ condition (cf.~\Cref{def:weak_trajectory_small_ball}).
\end{myprop}
\begin{proof}
Fix $j \in [S]$.
Since $\Psi_j \preccurlyeq \frac{1}{4}\Sigma_t$
for all $t \in [\kmix(2j - 1) + 1, 2j \kmix]$, we have:
\begin{align*}
    \Psi_j \preccurlyeq \frac{1}{4\kmix}\sum_{t=\kmix(2j-1)+1}^{2j\kmix} \Sigma_t.
\end{align*}
Hence,
\begin{align*}
    \frac{1}{S} \sum_{j=1}^{S} \Psi_j \preccurlyeq \frac{1}{4S  \kmix} \sum_{j=1}^{S}\sum_{t=\kmix(2j-1)+1}^{2j\kmix} \Sigma_t \preccurlyeq \frac{1}{4S \kmix} \sum_{t=1}^{T} \Sigma_t \preccurlyeq \Gamma_T.
\end{align*}
Above, the last inequality holds since $\floor{T/(2\kmix)} \geq T/(4\kmix)$.
By definition of $\phi$-mixing (cf.~\Cref{def:phi_mixing})
and the upper envelope $\bar\phi$, 
for any $j \in [S]$ and $t \geq \kmix(2j-1) + 1$:
\begin{align*}
    \Pr_{x_t}\left\{ \ip{v}{x_t}^2 \leq 4\alpha \cdot v^\T \Psi_j v ~\Big|~ \calF_{(j-1)2\kmix} \right\} &\leq \Pr_{x_t}\left\{ \ip{v}{x_t}^2 \leq 4\alpha \cdot v^\T \Psi_j v \right\} + \beta \\
    &\leq \Pr_{x_t}\left\{ \ip{v}{x_t}^2 \leq \alpha \cdot v^\T \Sigma_t v \right\} + \beta \\
    &\leq 2\beta.
\end{align*}
Therefore:
\begin{align*}
    &\frac{1}{2\kmix} \sum_{t=(j-1)2\kmix+1}^{2j\kmix} \Pr_{x_t}\left\{ \ip{v}{x_t}^2 \leq 4\alpha \cdot v^\T \Psi_j v ~\Big|~ \calF_{(j-1)2\kmix} \right\} \\
    &\leq \frac{1}{2\kmix} [ \kmix + 2\beta \kmix ] = \frac{1}{2} + \beta.
\end{align*}
The claim now follows from \Cref{stmt:avg_small_ball_implies_block}.
\end{proof}

We conclude by noting that $\phi$-mixing is a stronger notion of mixing
than $\beta$-mixing, where \eqref{eq:phi_mixing} is only required to hold in
expectation. We leave to future work an analysis that only relies on 
the weaker $\beta$-mixing.
\section{Analysis for lower bounds}
\label{sec:appendix:lower_bounds}

\subsection{Preliminaries}

Here, we collect the necessary auxiliary results we
will use to prove the lower bound.
The first result is an instance of the well-known fact that the conditional mean 
is the estimator which minimizes the mean squared error.
\begin{myprop}
\label{prop:mean_minimizes_least_squares}
Let $T \in \N_{+}$ and $\{\Pxt{t}\}_{t=1}^{T}$ be a sequence of distributions over $\R^n$ with finite second moments $\Sigma_t := \E_{x_t \sim \Pxt{t}}[x_t x_t^\T]$.
Let $P_W$ be any arbitrary distribution on $\R^{p \times n}$.
Put $\Gamma_T := \frac{1}{T} \sum_{t=1}^{T} \Sigma_t$.
We have:
\begin{align*}
    \inf_{\hat{W}} \E_{W \sim P_W}\left[ \frac{1}{T}\sum_{t=1}^{T} \E_{x_t \sim \Pxt{t}} \norm{\hat{W}(x_t) - Wx_t}_2^2 \right] = \E_{W \sim P_W} \norm{\E_{W' \sim P_W}[W'] - W}_{\Gamma_T}^2,
\end{align*}
where the infimum ranges over measurable functions $\hat{W} : \R^n \rightarrow \R^{p}$.
\end{myprop}
\begin{proof}
Let $\mu_T := \frac{1}{T} \sum_{t=1}^{T} \Pxt{t}$ denote the uniform mixture distribution,
so that
\begin{align*}
    \frac{1}{T}\sum_{t=1}^{T} \E_{x_t \sim \Pxt{t}} \norm{\hat{W}(x_t) - Wx_t}_2^2 = \E_{\bar{x} \sim \mu_T} \norm{\hat{W}(\bar{x}) - W \bar{x}}_2^2.
\end{align*}
By repeated applications of Fubini's theorem,
\begin{align*}
    \inf_{\hat{W}} \E_{W \sim P_W} \E_{\bar{x} \sim \mu_T} \norm{\hat{W}(\bar{x}) - W\bar{x}}_2^2 &= \inf_{\hat{W}} \E_{\bar{x} \sim \mu_T} \E_{W \sim P_W} \norm{\hat{W}(\bar{x}) - W\bar{x}}_2^2 \\
    &= \E_{\bar{x} \sim \mu_T}\left[ \inf_{\hat{y} \in \R^p} \E_{W \sim P_W} \norm{\hat{y} - W\bar{x}}_2^2 \right] \\
    &= \E_{\bar{x} \sim \mu_T} \E_{W \sim P_W} \norm{ \E_{W'\sim P_W}[W'] \bar{x} - W \bar{x} }_2^2 \\
    &= \E_{W \sim P_W} \E_{\bar{x} \sim \mu_T} \norm{ \E_{W'\sim P_W}[W'] \bar{x} - W \bar{x} }_2^2 \\
    &= \E_{W \sim P_W} \norm{\E_{W'\sim P_W}[W'] - W}_{\Gamma_T}^2.
\end{align*}
\end{proof}

The next result is a simple fact which states that
if a function is strictly increasing and
concave on an interval, then any root of the
function is lower bounded by the root of the
linear approximation at any point in the interval.
\begin{myprop}
\label{prop:linear_approx_root_lower_bound}
Let $f : I \rightarrow \R$ be a $C^1(I)$ function
that is strictly increasing and concave on an interval $I \subseteq \R$.
Suppose that $f$ has a (unique) root $x_0 \in I$.
For any $x \in I$, we have that:
\begin{align*}
    x - \frac{f(x)}{f'(x)} \leq x_0.
\end{align*}
\end{myprop}
\begin{proof}
Because $f$ is concave on $I$, we have that:
\begin{align*}
    0 = f(x_0) \leq f(x) + f'(x)(x_0 - x).
\end{align*}
Next, because $f$ is strictly increasing on $I$, we have that $f'(x) > 0$.
The claim now follows by re-arranging the previous inequality.
\end{proof}

The next result states that 
the trace inverse of any positive definite matrix
is lower bounded by the trace inverse of any
priciple submatrix. 
The claim is immediate from Cauchy's eigenvalue interlacing theorem, but we give a
more direct proof.
\begin{restatable}{myprop}{traceinvselector}\label{prop:trace_inv_selector_lower_bound}
Let $M \in \R^{q \times n}$ have full column rank.
Let $I \subseteq \{1, \dots, n\}$ be any index set, and let
$E_I : \R^n \rightarrow \R^{\abs{I}}$ denote any linear map which
extracts the coordinates associated to $I$.
We have:
\begin{align*}
    \Tr((M^\T M)^{-1}) \geq \Tr((E_I M^\T M E_I^\T)^{-1}).
\end{align*}
\end{restatable}
\begin{proof}
Fix a $z \in \R^n$.
Since $M$ has full column rank, we have
that $(M^\T)^{\dag} = M (M^\T M)^{-1}$.
Therefore,
\begin{align*}
    \min_{c \in \R^q : M^\T c = z} \norm{c}_2^2 = \norm{(M^\T)^{\dag} z}_2^2 = z^\T (M^\T M)^{-1} z.
\end{align*}
Taking expectation with $z \sim N(0, I_n)$,
\begin{align*}
    \Tr((M^\T M)^{-1}) = \E_{z \sim N(0, I_n)}\left[ \min_{c \in \R^q : M^\T c = z} \norm{c}_2^2 \right].
\end{align*}
On the other hand, we have that:
\begin{align*}
    \min_{c \in \R^q : M^\T c = z} \norm{c}_2^2 \geq \min_{c \in \R^q : E_I M^\T c = E_I z} \norm{c}_2^2.
\end{align*}
This is clear because
for any $c \in \R^q$ satisfying $M^\T c = z$,
the equality $E_I M^\T c = E_I z$ trivially holds.
This means we have the following set inclusion:
\begin{align*}
    \{ c \in \R^q \mid M^\T c = z \} \subseteq \{ c \in \R^q \mid E_I M^\T c = E_I z \}.
\end{align*}
Therefore, minimizing any function over the
first set will be lower bounded by minimizing
the same function over the second set.
From this inclusion, we conclude for any index set $I$:
\begin{align*}
    \Tr((M^\T M)^{-1}) &= \E_{z \sim N(0, I_n)}\left[ \min_{c \in \R^q : M^\T c = z} \norm{c}_2^2 \right] \geq \E_{z \sim N(0, I_n)}\left[ \min_{c \in \R^q : E_I M^\T c = E_I z} \norm{c}_2^2 \right] \\
    &= \E_{z \sim N(0, I_{\abs{I}})}\left[ \min_{c \in \R^q : E_I M^\T c = z} \norm{c}_2^2 \right] = \Tr((E_I M^\T M E_I^\T )^{-1}).
\end{align*}
\end{proof}

Next, we state well-known upper and lower tail bounds
for chi-squared random variables.
\begin{mylemma}[{\cite[Lemma~1]{laurent00adaptive}}]
\label{lemma:chi_squared_tail_bounds}
Let $g_1, \dots, g_D$ be \IID\ $N(0, 1)$ random variables, and let $a_1, \dots, a_D$ be non-negative
scalars.
For any $t > 0$, we have:
\begin{align*}
    \Pr\left\{ \sum_{i=1}^{D} a_i(g_i^2 - 1) \geq 2 \sqrt{t} \sqrt{\sum_{i=1}^{D} a_i^2} + 2 t \max_{i=1, \dots, D} a_i \right\} &\leq e^{-t}, \\ 
    \Pr\left\{ \sum_{i=1}^{D} a_i(g_i^2 - 1) \leq - 2 \sqrt{t} \sqrt{\sum_{i=1}^{D} a_i^2} \right\} &\leq e^{-t}.
\end{align*}
\end{mylemma}

Finally, we conclude with a convex extension of
Gordon's min-max theorem.
\begin{mythm}[{\cite[Theorem~II.1]{thrampoulidis14gmt}}]
\label{thm:gaussian_min_max}
Let $A \in \R^{m \times n}$, $g \in \R^m$, and $h \in \R^n$ have \IID\ $N(0, 1)$ entires and be independent of each other.
Suppose that $S_1 \subset \R^n$ and $S_2 \subset \R^m$ are non-empty compact convex sets,
and let $\psi : S_1 \times S_2 \rightarrow \R$ be a continuous, convex-concave function.
For every $t \in \R$, we have:
\begin{align*}
    \Pr\left\{ \min_{x \in S_1} \max_{y \in S_2} \left[ y^\T A x + \psi(x, y) \right] \geq t  \right\} \leq 2 \Pr\left\{ \min_{x \in S_1} \max_{y \in S_2} \left[\norm{x}_2 g^\T y + \norm{y}_2 h^\T x + \psi(x, y)\right] \geq t \right\}.
\end{align*}
\end{mythm}

\subsection{Proof of \Cref{thm:trace_inv_lower_bounds_minimax_risk}}
\label{sec:appendix:lower_bounds:trace_inv_proof}

We first prove the following intermediate result, which holds
under the Gaussian observation noise model (\Cref{def:gaussian_observation_noise}).
\begin{mylemma}
\label{stmt:minimax_risk_trace_inverse_one_dist}
Let $T \in \N_{+}$,
$\{\Pxt{t}\}_{t=1}^{T}$ be a sequence of distributions over $\R^n$ with finite second moments $\Sigma_t := \E_{x_t \sim \Pxt{t}}[x_t x_t^\T]$,
and $\sigma_{\xi} > 0$.
Let $P_X$ be a distribution on $\R^{q \times n}$ with $q \geq n$
such that for $X \sim P_X$,
$X^\T X$ is invertible almost surely.
For $W \in \R^{p \times n}$, let $P_W$ be the distribution over $\R^{q \times n} \times \R^{q \times p}$
with $(X, Y) \sim P_W$ satisfying $X \sim P_X$ and $Y \mid X = X W^\T + \Xi$,
where $\Xi \in \R^{q \times p}$ has \IID\ $N(0, \sigma_{\xi}^2)$ entries (and is independent of everything else).
Put $\Gamma_T := \frac{1}{T} \sum_{t=1}^{T} \Sigma_t$.
We have that:
\begin{align*}
    &\inf_{\hat{W}} \sup_{W \in \R^{p \times n}} \E_{(X,Y) \sim P_W} \left[ \frac{1}{T} \sum_{t=1}^{T} \E_{x_t \sim \Pxt{t}} \norm{ \hat{W}(X, Y, x_t) - Wx_t}_2^2 \right] \\
    &\qquad\geq \sigma_{\xi}^2 p \cdot \E_{X \sim P_X} \Tr(\Gamma_T^{1/2} (X^\T X)^{-1} \Gamma_T^{1/2}),
\end{align*}
where the infimum ranges over all measurable functions 
$\hat{W} : \R^{q \times n} \times \R^{q \times p} \times \R^{n} \rightarrow \R^{p}$.
\end{mylemma}
\begin{proof}
The proof extends the Bayesian argument from \cite[Theorem~1]{mourtada19exactminimax}.
Let $p_\lambda$ be any prior distribution over $\R^{p \times n}$.
Let $\mu_T := \frac{1}{T} \sum_{t=1}^{T} \Pxt{t}$ denote the uniform mixture.
Bounding the minimax risk from below by the Bayes risk:
\begin{align*}
    &~~~\,\inf_{\hat{W}} \sup_{W \in \R^{p \times n}} \E_{(X,Y) \sim P_W} \E_{\bar{x} \sim \mu_T} \norm{ \hat{W}(X, Y, \bar{x}) - W\bar{x}}_2^2 \\
    &\geq \inf_{\hat{W}} \E_{W_\lambda \sim p_\lambda} \E_{(X,Y) \sim P_{W_\lambda}} \E_{\bar{x} \sim \mu_T} \norm{ \hat{W}(X, Y, \bar{x}) - W_\lambda \bar{x}}_2^2 \\
    &= \inf_{\hat{W}} \E_{(X,Y)} \E_{W_\lambda \mid (X, Y)} \E_{\bar{x} \sim \mu_T} \norm{ \hat{W}(X, Y, \bar{x}) - W_\lambda \bar{x}}_2^2 &&\textrm{using Fubini's theorem} \\
    &= \E_{(X,Y)} \inf_{\hat{W}_{X,Y}}  \E_{W_\lambda \mid (X, Y)} \E_{\bar{x} \sim \mu_T} \norm{ \hat{W}_{X,Y}(\bar{x}) - W_\lambda \bar{x}}_2^2 &&\textrm{where }\hat{W}_{X,Y} \text{ maps } \R^n \rightarrow \R^p \\
    &= \E_{(X,Y)} \E_{W_\lambda \mid (X,Y)} \norm{ \E[W_\lambda \mid X,Y] - W_\lambda }_{\Gamma_T}^2 &&\textrm{using \Cref{prop:mean_minimizes_least_squares}}.
\end{align*}
Now let $W_\lambda \sim p_\lambda$ have \IID\ $N(0, 1/\lambda)$ entries for $\lambda > 0$.
Noting that 
\begin{align*}
    \vec(Y) = (I_p \otimes X) \vec(W_\lambda^\T) + \vec(\Xi),
\end{align*}
we see that the vector $\cvectwo{\vec(W_\lambda^\T)}{\vec(Y)}$ is jointly Gaussian conditioned on $X$:
\begin{align*}
    \cvectwo{\vec(W_\lambda^\T)}{\vec(Y)} \mid X \sim N\left( 0, \bmattwo{ \frac{1}{\lambda} I_{pn} }{ \frac{1}{\lambda}(I_p \otimes X^\T)  }{*}{ \frac{1}{\lambda}(I_p \otimes XX^\T) + \sigma_{\xi}^2 I_{qp} }  \right).
\end{align*}
Therefore, the distribution of $\vec(W_\lambda^\T) \mid X, Y$ is:
\begin{align*}
    \vec(W_\lambda^\T) \mid X,Y &~\sim N(\mu_\lambda, \Sigma_\lambda), \\
    \mu_\lambda &:= \frac{1}{\lambda}(I_p \otimes X^\T)\left[ \frac{1}{\lambda}(I_p \otimes XX^\T) + \sigma_{\xi}^2 I_{qp} \right]^{-1} \vec(Y), \\
    \Sigma_\lambda &:= \frac{1}{\lambda} I_{pn} - \frac{1}{\lambda^2} (I_p \otimes X^\T) \left[ \frac{1}{\lambda}(I_p \otimes XX^\T) + \sigma_{\xi}^2 I_{qp} \right]^{-1} (I_p \otimes X) .
\end{align*}
A generalization of the identity
$X^\T( \frac{1}{\lambda} XX^\T + \sigma_{\xi}^2 I_q)^{-1} = (\frac{1}{\lambda} X^\T X + \sigma_{\xi}^2 I_n)^{-1} X^\T$ yields:
\begin{align*}
     (I_p \otimes X^\T) \left[ \frac{1}{\lambda}(I_p \otimes XX^\T) + \sigma_{\xi}^2 I_{qp} \right]^{-1} &=  \left[ \frac{1}{\lambda}(I_p \otimes X^\T X) + \sigma_{\xi}^2 I_{np} \right]^{-1} (I_p \otimes X^\T).
\end{align*}
Therefore,
\begin{align*}
    &~~~~\E[ \vec(W_\lambda^\T) \mid X, Y] - \vec(W_\lambda^\T) \\
    &= \mu_\lambda - \vec(W_\lambda^\T) \\
    &= \left[ (I_p \otimes X^\T X) + \sigma_{\xi}^2 \lambda I_{np} \right]^{-1} (I_p \otimes X^\T) \vec(Y) - \vec(W_\lambda^\T) \\
    &= \left[ \left[ (I_p \otimes X^\T X) + \sigma_{\xi}^2 \lambda I_{np}\right]^{-1} (I_p \otimes X^\T X) - I_{np} \right] \vec(W_\lambda^\T) \\
    &\qquad+ \left[ (I_p \otimes X^\T X) + \sigma_{\xi}^2 \lambda I_{np} \right]^{-1} (I_p \otimes X^\T) \vec(\Xi).
\end{align*}
Observing that
\begin{align*}
    \norm{\E[W_\lambda \mid X,Y] - W_\lambda}_{\Gamma_T}^2 = \norm{\E[\vec(W_\lambda^\T) \mid X, Y] - \vec(W_\lambda^\T)}_{I_p \otimes \Gamma_T}^2,
\end{align*}
and defining $M_X(\lambda) := (I_p \otimes X^\T X) + \sigma_{\xi}^2 \lambda I_{np}$,
we have the following bias-variance decomposition:
\begin{align*}
    &~~~~\E_{X,\Xi,W_\lambda} \norm{\E[W_\lambda \mid X,Y] - W_\lambda}_{\Gamma_T}^2 \\
    &= \E_{X,\Xi,W_\lambda} \bignorm{\left[ M_X^{-1}(\lambda) (I_p \otimes X^\T X) - I_{np} \right] \vec(W_\lambda^\T) }_{I_p \otimes \Gamma_T}^2 \\
    &\qquad + \sigma_{\xi}^2 \E_{X} \Tr\left( (I_p \otimes \Gamma_T^{1/2}) M_X^{-1}(\lambda) (I_p \otimes X^\T X) M_X^{-1}(\lambda) (I_p \otimes \Gamma_T^{1/2})  \right) \\
    &\geq \sigma_{\xi}^2 \E_{X} \Tr\left( (I_p \otimes \Gamma_T^{1/2}) M_X^{-1}(\lambda) (I_p \otimes X^\T X) M_X^{-1}(\lambda) (I_p \otimes \Gamma_T^{1/2})  \right) .
\end{align*}
Since $\lambda \mapsto \Tr\left( (I_p \otimes \Gamma_T^{1/2}) M_X^{-1}(\lambda) (I_p \otimes X^\T X) M_X^{-1}(\lambda) (I_p \otimes \Gamma_T^{1/2}) \right)$
is non-negative and decreasing in $\lambda$ for $\lambda > 0$,
by the monotone convergence theorem:
\begin{align*}
    &~~~~\lim_{\lambda \rightarrow 0^+} \E_{X,\Xi,W_\lambda} \norm{\E[W_\lambda \mid X,Y] - W_\lambda}_{\Gamma_T}^2 \\
    &\geq \sigma_\xi^2 \lim_{\lambda \rightarrow 0^+} \E_{X} \Tr\left( (I_p \otimes \Gamma_T^{1/2}) M_X^{-1}(\lambda) (I_p \otimes X^\T X) M_X^{-1}(\lambda) (I_p \otimes \Gamma_T^{1/2})  \right) \\
    &=\sigma_\xi^2 \E_{X} \Tr\left( (I_p \otimes \Gamma_T^{1/2}) M_X^{-1}(0) (I_p \otimes X^\T X) M_X^{-1}(0) (I_p \otimes \Gamma_T^{1/2})  \right) \\
    &= \sigma_\xi^2 \E_{X} \Tr( (I_p \otimes \Gamma_T^{1/2} (X^\T X)^{-1} \Gamma_T^{1/2} ) ) \\
    &= \sigma_\xi^2 p \cdot \E_{X} \Tr( \Gamma_T^{1/2} (X^\T X)^{-1} \Gamma_T^{1/2} ).
\end{align*}
Since the first expression above lower bounds the minimax risk, 
this concludes the proof.
\end{proof}

We now restate and prove \Cref{thm:trace_inv_lower_bounds_minimax_risk}.
\traceinvlowerboundsminimaxrisk*
\begin{proof}
Fix a $\Px \in \calP_x$, and let $\{\Pxt{t}\}_{t=1}^{\Tnew}$ denote its marginal
distributions up to time $\Tnew$.
Let $\Pxi^{\mathsf{g}}$ denote the $\sigma_\xi$-MDS corresponding to the
Gaussian observation noise model (\Cref{def:gaussian_observation_noise}).
Note that for any hypothesis $f : \R^n \rightarrow \R^p$, we have from \eqref{eq:risk_def}:
\begin{align*}
    L(\hat{f}; \Tnew, \Px) =
    \E_{\Px} \left[
        \frac{1}{\Tnew} \sum_{t=1}^{\Tnew} \norm{ \hat{f}(x_t) - W_\star x_t }^2_2 \right] = \frac{1}{\Tnew} \sum_{t=1}^{\Tnew} \E_{x_t \sim \Pxt{t}} \norm{\hat{f}(x_t) - W_\star x_t}_2^2.
\end{align*}
By the definition of $\sfR(m,T,\Tnew;\calP_x)$ from \eqref{eq:minimax_risk_def}
and \Cref{stmt:minimax_risk_trace_inverse_one_dist}:
\begin{align*}
    \sfR(m,T,\Tnew;\calP_x) &\geq 
  \inf_{\mathsf{Alg}} 
  \sup_{W_\star}
  \E_{\otimes_{i=1}^{m} \Pxy[\Px,\Pxi^{\mathsf{g}}]}
    \left[ L \left(
      \mathsf{Alg}(\{ (x_t^{(i)}, y_t^{(i)} ) \}_{i=1,t=1}^{m,T}); \Tnew, \Px \right) \right] \\
      &\geq \sigma_\xi^2 p \cdot \E_{\otimes_{i=1}^{m} \Px}\left[ \Tr\left(\Gamma_{\Tnew}^{1/2}(\Px)(X_{m,T}^\T X_{m,T})^{-1} \Gamma_{\Tnew}^{1/2}(\Px) \right) \right].
\end{align*}
Since the bound above holds for any $\Px \in \calP_x$, we can take the 
supremum over $\Px \in \calP_x$, from the which the claim follows.
\end{proof}

\subsection{A general risk lower bound}

We now state a lower bound which applies
with an arbitrary number of trajectories.
\begin{restatable}{mylemma}{minimaxjensenrate}\label{stmt:minimax_jensen_rate}
Suppose that $\calP_x$ is any set containing $\PxA{0_{n \times n}}$ and
$\PxA{I_n}$. Let $mT \geq n$. Then:
\begin{align*}
    \mathsf{R}(m,T,\Tnew;\calP_x)
    \geq \frac{\sigma_\xi^2 }{2}
        \cdot \frac{pn}{mT}
        \cdot \max \left\{ \frac{\Tnew}{T}, 1 \right\}.
\end{align*}
\end{restatable}
\begin{proof}
Define $\zeta(A) := \E_{\otimes_{i=1}^{m} \PxA{A}} \left[\Tr\left(\Gamma_{\Tnew}^{1/2}(A) (X_{m,T}^\T X_{m,T})^{-1} \Gamma_{\Tnew}^{1/2}(A) \right)\right]$.
By \Cref{thm:trace_inv_lower_bounds_minimax_risk}:
\begin{align}
    \mathsf{R}(m,T,\Tnew;\calP_x) \geq \sigma_\xi^2 p \cdot 
    \max\{\zeta(0_{n \times n}), \zeta(I_n)\}. \label{eq:lower_bound_risk_zero_and_eye}
\end{align}

Next, for any $M \in \mathsf{Sym}^n_{\geq 0}$,
the function $X \mapsto \Tr(M^{1/2} X^{-1} M^{1/2})$ is convex
on the domain $\mathsf{Sym}^n_{> 0}$.
To see this, we define $f(X; v) := v^\T X^{-1} v$ for $X \in \mathsf{Sym}^n_{> 0}$.
We can write
$f(X;v)$ as $f(X;v) = \sup\left\{ -z^\T X z + 2 v^\T z \mid z \in \R^n \right\}$;
therefore $X \mapsto f(X;v)$ is convex on $\mathsf{Sym}^n_{> 0}$, since it is the pointwise supremum of an affine function in $X$.
Now we see that $X \mapsto \Tr(M^{1/2} X^{-1} M^{1/2})$
is convex, since
$\Tr(M^{1/2} X^{-1} M^{1/2}) = \sum_{i=1}^{n} f(X; M^{1/2} e_i)$,
which is the sum of convex functions.
Therefore by Jensen's inequality, 
whenever $X_{m,T}^\T X_{m,T}$ is invertible almost surely,
\begin{align*}
    \zeta(A) &= \E_{\otimes_{i=1}^{m} \PxA{A}} \left[\Tr\left( \Gamma_{\Tnew}^{1/2}(A) (X_{m,T}^\T X_{m,T})^{-1} \Gamma_{\Tnew}^{1/2}(A) \right)\right] \\
    &\geq \Tr\left( \Gamma_{\Tnew}^{1/2}(A) (\E_{\otimes_{i=1}^{m} \PxA{A}}[X_{m,T}^\T X_{m,T}])^{-1} \Gamma_{\Tnew}^{1/2}(A) \right) \\
    &= \Tr\left( \Gamma_{\Tnew}^{1/2}(A) (mT \cdot \Gamma_T(A))^{-1} \Gamma_{\Tnew}^{1/2}(A) \right) \\
    &= \frac{\Tr(\Gamma_{\Tnew}(A) \Gamma_T^{-1}(A))}{mT}.
\end{align*}

We first consider the case when $A = 0_{n \times n}$.
Under these dynamics, 
it is a standard fact that when $mT \geq n$, then
$X_{m,T}^\T X_{m,T}$ is invertible almost surely.
Furthermore, $\Gamma_t(0_{n \times n}) = I_n$ for all $t$, Hence,
$\zeta(0_{n \times n}) \geq \frac{n}{mT}$.

Next, we consider the case when $A = I_n$.
We first argue that as long as $mT \geq n$,
the matrix $X_{m,T}^\T X_{m,T}$ is invertible almost surely.
We write $x_t^{(i)} = \sum_{k=1}^{t} w_k^{(i)}$,
where $\{w_t^{(i)}\}_{i=1,t=1}^{m,T}$ are all \IID\ $N(0, I_n)$ vectors.
Let $p : \R^{mTn} \rightarrow \R$ be the polynomial
$p(\{w_t^{(i)}\}) = \det(X_{m,T}^\T X_{m,T})$.
The zero-set of $p$ is either all of $\R^{mTn}$, or
Lebesgue measure zero.
We will select $\{w_t^{(i)}\}$ so that $p(\{w_t^{(i)}\}) \neq 0$,
which shows that the zeros of this polynomial are not all of $\R^{mTn}$, and hence Lebesgue measure zero. 
Since the Gaussian measure on $\R^{mTn}$ is absolutely
continuous w.r.t.\ the Lebesgue measure on $\R^{mTn}$,
this implies that $\det(X_{m,T}^\T X_{m,T}) \neq 0$ almost surely.

To select $\{w_t^{(i)}\}$, we introduce some notation.
Let $e_i \in \R^n$ denote the $i$-th standard basis vector.
For any positive integer $k$,
let $U(k) \in \R^{k \times k}$ 
be the upper triangular matrix with ones for all
its non-zero entries. Let $S(k) = U(k) U(k)^\T$.
By construction, $S(k)$ is invertible since $U(k)$ is invertible.
We put $w_t^{(i)} = e_{(i-1)T + t} \cdot \ind\{ (i-1)T + t  \leq n \}$. We now claim that with this choice of $\{w_t^{(i)}\}$,
the matrix $X_{m,T}^\T X_{m,T}$ is invertible.

Suppose first that $T \geq n$.
Then we have that
$X_{m,T}^\T X_{m,T} = S(n)$, and therefore $\det(X_{m,T}^\T X_{m,T}) \neq 0$.
On the other hand, suppose that $T < n$. 
Because $mT \geq n$, then we have that:
\begin{align*}
    X_{m,T}^\T X_{m,T} = \mathsf{BDiag}(\underbrace{S(T), \dots, S(T)}_{\floor{n/T} \textrm{ times}}, S(n - T \floor{n/T})),
\end{align*}
where $\mathsf{BDiag}(M_1, \dots, M_k)$
denotes the block diagonal matrices with block diagonals
$M_1$, \dots, $M_k$.
Since $S(T)$ and $S(n-T\floor{n/T})$ are both invertible, so is
$X_{m,T}^\T X_{m,T}$ and therefore
$\det(X_{m,T}^\T X_{m,T}) \neq 0$.
Thus, $X_{m,T}^\T X_{m,T}$ is invertible almost surely.

Next, we note that
$\Sigma_t(I_n) = t \cdot I_n$ and
$\Gamma_t(I_n) = \left(\frac{1}{t} \sum_{k=1}^{t} k\right) \cdot I_n = \frac{t+1}{2} \cdot I_n$.
Hence we have 
$ \Gamma_{\Tnew}(I_n) \Gamma_T^{-1}(I_n) = \frac{\Tnew + 1}{T + 1} \cdot I_n \succcurlyeq \frac{\Tnew}{2T} \cdot I_n$,
and therefore $\zeta(I_n) \geq \frac{n}{2mT} \frac{\Tnew}{T}$.

Combining our bounds on $\zeta(0_{n \times n})$ and
$\zeta(I_n)$, we have the desired claim:
\begin{align*}
    \mathsf{R}(m,T,\Tnew;\calP_x) \geq \sigma_\xi^2 p \cdot \max\left\{ \frac{n}{mT}, \frac{n}{2mT} \frac{\Tnew}{T} \right\} \geq \frac{\sigma_\xi^2}{2} \cdot \frac{pn}{mT} \cdot \max\left\{\frac{\Tnew}{T}, 1\right\}.
\end{align*}
\end{proof}

\subsection{Non-isotropic random gramian matrices}
\label{sec:appendix:lower_bounds:gramians}

The goal of this subsection is to prove \Cref{lemma:trace_inv_lower_bound}, which gives a bound on the expected
trace inverse of a non-isotropic random gramian matrix.
We first prove an auxiliary lemma, which will be used
as a building block in the proof.
\begin{mylemma}
\label{lemma:min_max_upper_bound}
Fix any $x \in \R^q$.
Let $g \in \R^q$ and $h \in \R^n$ be random vectors with \IID\ $N(0, 1)$ entries,
and let $W \in \R^{q \times n}$ be a random matrix with \IID\ $N(0, 1)$ entries.
Let $\Sigma \in \R^{q \times q}$ be positive definite.
We have that:
\begin{align*}
    \E \min_{\alpha \in \R^n} \norm{ \Sigma^{1/2} W \alpha - x }_2^2 \leq \E \min_{\beta \geq 0} \max_{\tau \geq 0} \left[ -\frac{\beta\norm{h}_2}{\tau} + \norm{\beta g - \Sigma^{-1/2} x}_{(\Sigma^{-1} + \beta \norm{h}_2 \tau I_q)^{-1}}^2 \right] .
\end{align*}
\end{mylemma}
\begin{proof}
The proof invokes the convex Gaussian min-max lemma (\Cref{thm:gaussian_min_max})
via a limiting argument.
In what follows, 
let $\{\alpha_k\}_{k \geq 1}$ 
and $\{v_k\}_{k \geq 1}$ be any two positive, increasing sequences of scalars tending to $+\infty$.
It is clear that for every $W$,
\begin{align*}
    \lim_{k \rightarrow \infty} \min_{\norm{\alpha}_2 \leq \alpha_k} \norm{\Sigma^{1/2} W \alpha - x}_2^2 = \min_{\alpha \in \R^n} \norm{\Sigma^{1/2} W \alpha - x}_2^2.
\end{align*}
Since $\alpha = 0$ is always a feasible solution to 
$\min_{\norm{\alpha}_2 \leq \alpha_k} \norm{\Sigma^{1/2} W \alpha - x}_2^2$,
we have for every $k \geq 1$:
\begin{align*}
    0 \leq \min_{\norm{\alpha}_2 \leq \alpha_k} \norm{\Sigma^{1/2} W \alpha - x}_2^2 \leq \norm{x}_2^2.
\end{align*}
Therefore, by the dominated convergence theorem,
\begin{align}
        \E \min_{\alpha \in \R^n} \norm{ \Sigma^{1/2} W \alpha - x }_2^2 = \E  \lim_{k \rightarrow \infty} \min_{\norm{\alpha}_2 \leq \alpha_k} \norm{\Sigma^{1/2} W \alpha - x}_2^2 = \lim_{k \rightarrow \infty} \E \min_{\norm{\alpha}_2 \leq \alpha_k} \norm{ \Sigma^{1/2} W \alpha - x }_2^2. \label{eq:first_dominated_convergence}
\end{align}
We next state two variational forms which we will use:
\begin{align}
    \frac{1}{2}\norm{x}_2^2 &= \max_{v \in \R^q} \left\{ v^\T x - \frac{\norm{v}_2^2}{2} \right\}, \label{eq:x_sq_variation} \\
    \norm{x}_2 &= \min_{\tau \geq 0} \left\{ \frac{\norm{x}_2^2 \tau}{2} + \frac{1}{2\tau} \right\}. \label{eq:x_variation}
\end{align}
Using the first variational form \eqref{eq:x_sq_variation}, we have for every $W$ and $k_1 \geq 1$,
\begin{align*}
    \min_{\norm{\alpha}_2 \leq \alpha_{k_1}} \frac{1}{2}\norm{\Sigma^{1/2} W \alpha - x}_2^2 &= \min_{\norm{\alpha}_2 \leq \alpha_{k_1}} \max_{v \in \R^q} \left[  v^\T (\Sigma^{1/2} W \alpha - x) - \frac{\norm{v}_2^2}{2} \right] \\
    &= \min_{\norm{\alpha}_2 \leq \alpha_{k_1}} \max_{v \in \R^q} \left[ v^\T W \alpha - v^\T \Sigma^{-1/2} x - \frac{v^\T \Sigma^{-1} v}{2} \right] \\
    &= \min_{\norm{\alpha}_2 \leq \alpha_{k_1}}\max_{\norm{v}_2 \leq \opnorm{\Sigma W} \alpha_{k_1} + \norm{\Sigma^{1/2} x}_2} \left[ v^\T W \alpha - v^\T \Sigma^{-1/2} x - \frac{v^\T \Sigma^{-1} v}{2} \right] \\
    &= \lim_{k_2 \rightarrow \infty} \min_{\norm{\alpha}_2 \leq \alpha_{k_1}} \max_{\norm{v}_2 \leq v_{k_2}} \left[ v^\T W \alpha - v^\T \Sigma^{-1/2} x - \frac{v^\T \Sigma^{-1} v}{2} \right] .
\end{align*}
Observe that for every $k_2 \geq 1$,
\begin{align*}
    0 &\leq \min_{\norm{\alpha}_2 \leq \alpha_{k_1}} \max_{\norm{v}_2 \leq v_{k_2}} \left[ v^\T W \alpha - v^\T \Sigma^{-1/2} x - \frac{v^\T \Sigma^{-1} v}{2} \right] \\
    &\leq \max_{\norm{v}_2 \leq v_{k_2}} \left[ - v^\T \Sigma^{-1/2} x - \frac{v^\T \Sigma^{-1} v}{2} \right] \\
    &\leq \max_{v \in \R^q}  \left[ - v^\T \Sigma^{-1/2} x - \frac{v^\T \Sigma^{-1} v}{2} \right] 
    = \frac{1}{2}\norm{x}_2^2.
\end{align*}
Therefore, by \eqref{eq:first_dominated_convergence}
and another application of the dominated convergence theorem:
\begin{align}
    \E \min_{\alpha \in \R^n} \norm{ \Sigma^{1/2} W \alpha - x }_2^2 &= 
    \lim_{k_1 \rightarrow \infty} \E \min_{\norm{\alpha}_2 \leq \alpha_{k_1}} \norm{\Sigma^{1/2} W \alpha - x}_2^2 \nonumber \\
    &= \lim_{k_1 \rightarrow \infty} \E \lim_{k_2 \rightarrow \infty} \min_{\norm{\alpha}_2 \leq \alpha_{k_1}} \max_{\norm{v}_2 \leq v_{k_2}} \left[ v^\T W \alpha - v^\T \Sigma^{-1/2} x - \frac{v^\T \Sigma^{-1} v}{2} \right]  \nonumber \\
    &= \lim_{k_1 \rightarrow \infty} \lim_{k_2 \rightarrow \infty} \E \min_{\norm{\alpha}_2 \leq \alpha_{k_1}} \max_{\norm{v}_2 \leq v_{k_2}} \left[ v^\T W \alpha - v^\T \Sigma^{-1/2} x - \frac{v^\T \Sigma^{-1} v}{2} \right]. \label{eq:second_dominated_convergence}
\end{align}
We now apply
\Cref{thm:gaussian_min_max}
to the expectation on the RHS of \eqref{eq:second_dominated_convergence}:
\begin{align}
    &~~~~\E \min_{\norm{\alpha}_2 \leq \alpha_{k_1}} \max_{\norm{v}_2 \leq v_{k_2}} \left[ v^\T W \alpha - v^\T \Sigma^{-1/2} x - \frac{v^\T \Sigma^{-1} v}{2} \right] \nonumber \\
    &= \int_0^\infty \Pr\left\{ \min_{\norm{\alpha}_2 \leq \alpha_{k_1}} \max_{\norm{v}_2 \leq v_{k_2}} \left[ v^\T W \alpha - v^\T \Sigma^{-1/2} x - \frac{v^\T \Sigma^{-1} v}{2} \right] \geq t \right\}  \rmd t \nonumber \\
    &\stackrel{(a)}{\leq} 2 \int_0^\infty \Pr\left\{ \min_{\norm{\alpha}_2 \leq \alpha_{k_1}} \max_{\norm{v}_2 \leq v_{k_2}} \left[\norm{\alpha}_2 g^\T v + \norm{v}_2 h^\T \alpha - v^\T \Sigma^{-1/2} x - \frac{v^\T \Sigma^{-1} v}{2}  \right] \geq t \right\} \rmd t \nonumber \\
    &= 2 \E \min_{\norm{\alpha}_2 \leq \alpha_{k_1}} \max_{\norm{v}_2 \leq v_{k_2}} \left[\norm{\alpha}_2 g^\T v + \norm{v}_2 h^\T \alpha - v^\T \Sigma^{-1/2} x - \frac{v^\T \Sigma^{-1} v}{2}  \right] \nonumber \\
    &\leq 2 \E \min_{\norm{\alpha}_2 \leq \alpha_{k_1}} \max_{v \in \R^q} \left[\norm{\alpha}_2 g^\T v + \norm{v}_2 h^\T \alpha - v^\T \Sigma^{-1/2} x - \frac{v^\T \Sigma^{-1} v}{2}  \right] . \label{eq:gaussian_min_max_inequality}
\end{align}
Above, inequality (a) is 
an application of \Cref{thm:gaussian_min_max}.
Now for every $k_1$, $g$, and $h$,
define
\begin{align*}
    \psi_{k_1}(g, h) := \min_{\norm{\alpha}_2 \leq \alpha_{k_1}} \max_{v \in \R^q} \left[\norm{\alpha}_2 g^\T v + \norm{v}_2 h^\T \alpha - v^\T \Sigma^{-1/2} x - \frac{v^\T \Sigma^{-1} v}{2}  \right].
\end{align*}
For every $k_1$, $g$, and $h$, we have
\begin{align*}
    0 \leq \psi_{k_1}(g,h)  
    \leq \max_{v \in \R^q} \left[ - v^\T \Sigma^{-1/2} x - \frac{v^\T \Sigma^{-1} v}{2}  \right] 
    = \frac{\norm{x}^2_2}{2}.
\end{align*}
Furthermore, since $\{\alpha_k\}$ is an increasing sequence,
the sequence $ \{ \psi_{k}(g, h) \}_{k \geq 1}$ is
montonically decreasing.
Therefore, by the monotone convergence theorem,
\begin{align*}
    \lim_{k \rightarrow \infty} \psi_{k}(g,h) &= \inf\{ \psi_{k}(g,h) \mid k \in \N_{+} \} \\
    &= \min_{\alpha \in \R^n} \max_{v \in \R^q} \left[\norm{\alpha}_2 g^\T v + \norm{v}_2 h^\T \alpha - v^\T \Sigma^{-1/2} x - \frac{v^\T \Sigma^{-1} v}{2}  \right].
\end{align*}
Therefore by another application of the dominated convergence theorem, we have that:
\begin{align}
    &\lim_{k_1 \rightarrow \infty} \E \min_{\norm{\alpha}_2 \leq \alpha_{k_1}} \max_{v \in \R^q} \left[\norm{\alpha}_2 g^\T v + \norm{v}_2 h^\T \alpha - v^\T \Sigma^{-1/2} x - \frac{v^\T \Sigma^{-1} v}{2}  \right] \nonumber \\
    &= \E \lim_{k_1 \rightarrow \infty} \min_{\norm{\alpha}_2 \leq \alpha_{k_1}} \max_{v \in \R^q} \left[\norm{\alpha}_2 g^\T v + \norm{v}_2 h^\T \alpha - v^\T \Sigma^{-1/2} x - \frac{v^\T \Sigma^{-1} v}{2}  \right] \nonumber \\
    &= \E \min_{\alpha \in \R^n} \max_{v \in \R^q} \left[\norm{\alpha}_2 g^\T v + \norm{v}_2 h^\T \alpha - v^\T \Sigma^{-1/2} x - \frac{v^\T \Sigma^{-1} v}{2}  \right]. \label{eq:third_dominated_convergence}
\end{align}
Chaining together inequalities \eqref{eq:second_dominated_convergence},
\eqref{eq:gaussian_min_max_inequality}, \eqref{eq:third_dominated_convergence}, we have:
\begin{align}
    \E  \min_{\alpha \in \R^n} \norm{ \Sigma^{1/2} W \alpha - x }_2^2 \leq 2 \E \min_{\alpha \in \R^n} \max_{v \in \R^q} \left[\norm{\alpha}_2 g^\T v + \norm{v}_2 h^\T \alpha - v^\T \Sigma^{-1/2} x - \frac{v^\T \Sigma^{-1} v}{2}  \right]. \label{eq:min_max_upper_bound}
\end{align}
We now proceed to study the RHS of \eqref{eq:min_max_upper_bound}, which we denote by $(\mathsf{AO})$ (the \emph{auxiliary optimization} problem):
\begin{align*}
    (\mathsf{AO}) &:= \min_{\alpha \in \R^n} \max_{v \in \R^q} \left[ \norm{\alpha}_2 g^\T v + \norm{v}_2 h^\T \alpha - v^\T \Sigma^{-1/2} x - \frac{v^\T \Sigma^{-1} v}{2} \right] \\
    &= \min_{\beta \geq 0} \min_{\theta \in [-1,1]} \max_{v \in \R^q}\left[ \beta g^\T v + \beta \norm{v}_2 \norm{h}_2 \theta - v^\T \Sigma^{-1/2} x - \frac{v^\T \Sigma^{-1} v}{2} \right] \\
    &\stackrel{(a)}{=} \min_{\beta \geq 0} \max_{v \in \R^q} \left[ \beta g^\T v - \beta \norm{v}_2 \norm{h}_2 - v^\T \Sigma^{-1/2} x - \frac{v^\T \Sigma^{-1} v}{2} \right] \\
    &\stackrel{(b)}{=} \min_{\beta \geq 0} \max_{v \in \R^q} \max_{\tau \geq 0} \left[  \beta g^\T v - \beta\norm{h}_2 \norm{v}_2^2 \frac{\tau}{2} - \frac{\beta \norm{h}_2}{2 \tau} - v^\T \Sigma^{-1/2} x - \frac{v^\T \Sigma^{-1} v}{2}   \right] \\
    &= \min_{\beta \geq 0} \max_{\tau \geq 0} \max_{v \in \R^q}\left[  \beta g^\T v - \beta\norm{h}_2 \norm{v}_2^2 \frac{\tau}{2} - \frac{\beta \norm{h}_2}{2 \tau} - v^\T \Sigma^{-1/2} x - \frac{v^\T \Sigma^{-1} v}{2}   \right] \\
    &= \min_{\beta \geq 0} \max_{\tau \geq 0} \left[ -\frac{\beta\norm{h}_2}{2\tau} + \frac{1}{2} \norm{\beta g - \Sigma^{-1/2} x}_{(\Sigma^{-1} + \beta \norm{h}_2 \tau I_q)^{-1}}^2 \right] .
\end{align*}
Above, (b) holds by the variational form \eqref{eq:x_variation}.
The proof is now finished after justifying (a).
First, let $h_\beta(\theta, v)$ denote the term in the bracket,
so that
\begin{align*}
    (\mathsf{AO}) = \min_{\beta \geq 0} \min_{\theta \in [-1, 1]} \max_{v \in \R^q} h_\beta(\theta, v).
\end{align*}
Fix a $\beta \geq 0$.
By weak duality,
\begin{align*}
    \min_{\theta \in [-1, 1]} \max_{v \in \R^q} h_\beta(\theta, v) \geq \max_{v \in \R^q} \min_{\theta \in [-1, 1]} h_\beta(\theta, v) = \max_{v \in \R^q} h_\beta(-1, v).
\end{align*}
On the other hand,
\begin{align*}
    \min_{\theta \in [-1, 1]} \max_{v \in \R^q} h_\beta(\theta, v) \leq \min_{\theta \in [-1, 0]} \max_{v \in \R^q} h_\beta(\theta, v) = \max_{v \in \R^q} \min_{\theta \in [-1, 0]} h_\beta(\theta, v) = \max_{v \in \R^q} h_\beta(-1, v).
\end{align*}
The first equality above is Sion's minimax theorem, since the function $\theta \mapsto h_\beta(\theta, v)$ is
affine for every $v$ and the function $v \mapsto h_\beta(\theta, v)$ is
concave for $\theta \in [-1, 0]$.
Therefore,
\begin{align*}
    \min_{\theta \in [-1, 1]} \max_{v \in \R^q} h_\beta(\theta, v) = \max_{v \in \R^q} h_\beta(-1, v).
\end{align*}
\end{proof}

With \Cref{lemma:min_max_upper_bound} in hand,
we can now restate and prove \Cref{lemma:trace_inv_lower_bound}.
\traceinvlowerbound*
\begin{proof}
We rewrite $\E \Tr((W^\T \Sigma W)^{-1})$ in a way that 
is amenable to \Cref{lemma:min_max_upper_bound}.
Let $w_1 \in \R^{q}$ denote the first column of $W$, so that $W = \begin{bmatrix} w_1 & W_2 \end{bmatrix}$ with $W_2 \in \R^{q \times (n-1)}$.
We write:
\begin{align*}
    W^\T \Sigma W = \begin{bmatrix} \norm{w_1}^2_{\Sigma} & w_1^\T \Sigma W_2 \\
    W_2^\T \Sigma w_1 & W_2^\T \Sigma W_2 \end{bmatrix}.
\end{align*}
Using the block matrix inversion formula to compute
the $(1, 1)$ entry of
$(W^\T \Sigma W)^{-1}$:
\begin{align*}
    ((W^\T \Sigma W)^{-1})_{11} &= ( w_1^\T (\Sigma - \Sigma W_2^\T (W_2^\T \Sigma W_2)^{-1} W_2^\T \Sigma ) w_1)^{-1} \\
    &= (w_1^\T \Sigma^{1/2}( I - P_{\Sigma^{1/2} W_2}) \Sigma^{1/2} w_1)^{-1} \\
    &= (w_1^\T \Sigma^{1/2} P^\perp_{\Sigma^{1/2} W_2} \Sigma^{1/2} w_1)^{-1}.
\end{align*}
Since the columns of $W$ are all independent and identically distributed, this calculation shows that
the law of $((W^\T \Sigma W)^{-1})_{ii}$ is the
same as the law of $((W^\T \Sigma W)^{-1})_{11}$ for all $i=1, \dots, n$.
Therefore:
\begin{align*}
    \E \Tr((W^\T \Sigma W)^{-1}) &= \sum_{i=1}^{n} \E ((W^\T \Sigma W)^{-1})_{ii} = n \cdot \E (w_1^\T \Sigma^{1/2} P^\perp_{\Sigma^{1/2} W_2} \Sigma^{1/2} w_1)^{-1} \\
    &\geq \frac{n}{\E \Tr(\Sigma^{1/2} P^\perp_{\Sigma^{1/2} W_2} \Sigma^{1/2})}.
\end{align*}
The last inequality follows from Jensen's inequality combined with the independence of $w_1$ and $W_2$.
By decomposing
$\Tr( \Sigma^{1/2} P^{\perp}_{\Sigma^{1/2} W_2} \Sigma^{1/2} ) = \sum_{i=1}^{q} \norm{ P^\perp_{\Sigma^{1/2} W_2} \Sigma^{1/2}e_i}_2^2$
and observing that
\begin{align*}
    \norm{P^\perp_{\Sigma^{1/2} W_2} x}_2^2 = \min_{\alpha \in \R^{n-1}} \norm{ \Sigma^{1/2} W_2 \alpha - x }_2^2 \quad \forall x \in \R^q,
\end{align*}
we have the following identity:
\begin{align*}
    \E\Tr( \Sigma^{1/2} P^{\perp}_{\Sigma^{1/2} W_2} \Sigma^{1/2} ) = \sum_{i=1}^{q} \E\min_{\alpha_i \in \R^{n-1}} \norm{\Sigma^{1/2} W_2 \alpha_i - \Sigma^{1/2} e_i}_2^2.
\end{align*}
Invoking \Cref{lemma:min_max_upper_bound} with $x = \Sigma^{1/2} e_i$ for $i = 1, \dots, q$ yields
\begin{align*}
     \E\Tr( \Sigma^{1/2} P^{\perp}_{\Sigma^{1/2} W_2} \Sigma^{1/2} ) \leq   \sum_{i=1}^{q} \E \min_{\beta \geq 0} \max_{\tau \geq 0} \left[ -\frac{\beta\norm{h}_2}{\tau} +  \norm{\beta g - e_i}_{(\Sigma^{-1} + \beta \norm{h}_2 \tau I_q)^{-1}}^2 \right],
\end{align*}
where $g \sim N(0, I_q)$ and $h \sim N(0, I_{n-1})$.
The claim now follows.
\end{proof}

We conclude this section with the following technical result
which we will use in the sequel.
\begin{mylemma}
\label{stmt:Z_i_helper}
Let $q, n \in \N_{+}$ with $q \geq n$ and $n \geq 6$, and let $\Sigma \in \sfSym^{q}_{>0}$.
Let $g \sim N(0, I_q)$ and $h \sim N(0, I_{n-1})$ with $g$ and $h$ independent.
Define the random variables $Z_i$ for $i \in \{1, \dots, q\}$ as:
\begin{align}
    Z_i := \min_{\beta \geq 0} \max_{\tau \geq 0} \left[ -\frac{\beta\norm{h}_2}{2\tau} + \beta^2 \norm{g}^2_{(\Sigma^{-1} + \beta \norm{h}_2 \tau I_q)^{-1}} + (\Sigma^{-1} + \beta \norm{h}_2 \tau I_q)^{-1}_{ii} \right].
\end{align}
Let $\{\lambda_i\}_{i=1}^{q}$ denote the eigenvalues of $\Sigma^{-1}$ listed
in decreasing order. Define $n_1$ and the random function $p(y)$ as:
\begin{align}
    n_1 := \frac{n}{64}, \quad p(y) := \sum_{i=1}^{q} \frac{y}{\lambda_i + y} g_{i}^2 - \frac{n_1}{2}. \label{eq:n_1_and_p_y_def}
\end{align}
There exists an event $\calE$ (over the probability of $g$ and $h$)
such that the following statements hold:
\begin{enumerate}[label=(\alph*)]
    \item $\Pr(\calE^c) \leq e^{-n/128} + e^{-q/16}$.
    \item On $\calE$, there exists a unique root $y^* \in (0, \infty)$ such that $p(y^*) = 0$.
    \item The following bounds hold for $i \in \{1, \dots, q\}$:
    \begin{align}
        Z_i \leq \Sigma_{ii}, \quad \ind\{\calE\} Z_i \leq \ind\{\calE\} (\Sigma^{-1} + y^* I_q)^{-1}_{ii}. \label{eq:Z_i_inequalities}
    \end{align}
\end{enumerate}
\end{mylemma}
\begin{proof}
First, we observe that we can trivially upper bound the value of $Z_i$
by setting $\beta = 0$ and obtaining the bound $Z_i \leq \Sigma_{ii}$.
Furthermore, by the rotational invariance of $g$ and the fact that $g$ and $h$ are independent, we have that $Z_i$ is equal in distribution to:
\begin{align*}
    Z_i = \min_{\beta \geq 0} \max_{\tau \geq 0} \left[ -\frac{\beta\norm{h}_2}{2\tau} + \beta^2 \sum_{i=1}^{q} \frac{g_i^2}{\lambda_i + \beta\norm{h}_2 \tau} + (\Sigma^{-1} + \beta \norm{h}_2 \tau I_q)^{-1}_{ii} \right].
\end{align*}
Define the following events:
\begin{align*}
    \calE_h := \left\{ \norm{h}_2 \geq \sqrt{n}/8 \right\}, \quad \calE_g := \left\{ \sum_{i=1}^{q} g_{i}^2 \geq q/2 \right\},
\end{align*}
and put $\calE := \calE_h \cap \calE_g$.
Since $n \geq 6$, by a standard computation
we have that $\E\norm{h}_2 \geq \sqrt{n}/4$.
Therefore, by Gaussian concentration of Lipschitz functions
\citep[cf.][Chapter~2]{wainwright2019book},
$\Pr(\calE_h^c) \leq e^{-n/128}$.
Furthermore, \Cref{lemma:chi_squared_tail_bounds}
yields that $\Pr(\calE_g^c) \leq e^{-q/16}$.
By a union bound,
$\Pr(\calE^c) \leq  e^{-n/128} + e^{-q/16}$.

We now focus on upper bounding the quantity:
\begin{align*}
    \ind\{\calE\} Z_i \leq \ind\{\calE\} \min_{\beta \geq 0} \max_{\tau \geq 0} \underbrace{\left[ -\frac{\beta\sqrt{n}}{16\tau} + \beta^2 \sum_{i=1}^{q} \frac{g_i^2}{\lambda_i + \beta\sqrt{n} \tau/8} + (\Sigma^{-1} + \beta\sqrt{n} \tau/8 I_q)^{-1}_{ii} \right]}_{=: \ell_i(\beta, \tau)}.
\end{align*}

Let us bracket the value
of the game $\min_{\beta \geq 0} \max_{\tau \geq 0} \ell_i(\beta, \tau)$.
We previously noted that
$\ell_i(0, \tau) = \Sigma_{ii}$
for all $\tau \in [0, \infty)$.
Next, for any $\beta > 0$, 
$\lim_{\tau \rightarrow \infty} \ell_i(\beta, \tau) = 0$.
Hence, 
\begin{align*}
    \min_{\beta \geq 0} \max_{\tau \geq 0} \ell_i(\beta, \tau) \in [0, \Sigma_{ii}].
\end{align*}
Recalling from \eqref{eq:n_1_and_p_y_def} that $n_1 = n/64$ (so that $\sqrt{n_1} = \sqrt{n}/8$) and defining $f, q_i$ as:
\begin{align*}
    f(x) &:= -\frac{x\sqrt{n_1}}{2} + x^2 \sum_{i=1}^{q} \frac{g_{i}^2}{\lambda_i + x \sqrt{n_1}}, \\
    q_i(x) &:= (\Sigma^{-1} + x \sqrt{n_1})^{-1}_{ii},
\end{align*}
we have that $\ell_i(\beta, \tau) = \frac{1}{\tau^2} f(\beta \tau) + q_i(\beta \tau)$.

In order to sharpen our estimate for the value of the game,
we will study the positive critical points $(\beta, \tau) \in \R^2_{> 0}$
of the game
$\min_{\beta} \max_{\tau} \ell_i(\beta, \tau)$, i.e.,
the points $(\beta, \tau) \in \R^2_{> 0}$ satisfying
$\frac{\partial \ell_i}{\partial \beta}(\beta, \tau) = 0$
and $\frac{\partial \ell_i}{\partial \tau}(\beta, \tau) = 0$.
Note that in general for a nonconvex/nonconcave game,
this is \emph{not} a necessary first order optimality condition for the global min/max value~\citep[see e.g.][Proposition 21]{jin2020localoptimality}.
However, for every fixed $\beta > 0$, stationary points of the function $\tau \mapsto \ell_i(\beta, \tau)$ on $\R_{> 0}$ are strictly concave by \Cref{stmt:local_strict_concave_max}. Hence, by the implicit
function theorem (or alternatively \cite[Theorem 23]{jin2020localoptimality}), the first order stationarity conditions $\frac{\partial \ell_i}{\partial \beta}(\beta, \tau) = 0$
and $\frac{\partial \ell_i}{\partial \tau}(\beta, \tau) = 0$
are necessary for global min/max optimality.
For $\tau \neq 0$, this yields:
\begin{align*}
    0 &= \tau^{-2} f'(\beta \tau) \beta - 2 \tau^{-3} f(\beta \tau) + q_i'(\beta \tau) \beta, \\
    0 &= \tau^{-2} f'(\beta \tau) \tau + q_i'(\beta \tau) \tau.
\end{align*}
Together, these conditions imply that $f(\beta\tau) = 0$,
and that the value of the game at such a
critical point is $q_i(\beta \tau)$.
Thus, we are interested in the positive roots of $f(x) = 0$.
To proceed, recall the definition of $p(y)$ from \eqref{eq:n_1_and_p_y_def}:
\begin{align*}
    p(y) = \sum_{i=1}^{q} \frac{y}{\lambda_i + y} g_{i}^2 - \frac{n_1}{2}.
\end{align*}
Note that $y^*$ is a positive root of $p$ iff $y^*/\sqrt{n_1}$ is a positive root of $f$.
Since $q \geq n$ by assumption,
observe that on $\calE$:
\begin{align*}
\lim_{y \rightarrow \infty} p(y) = \sum_{i=1}^{q} g_{i}^2 - n_1/2 \geq q/2 - n_1/2 \geq n/2 - n/64 > 0.
\end{align*}
On the other hand, $p(0) = -n_1/2 < 0$. 
Since $p(y)$ is continuous and strictly increasing, 
on $\calE$ there exists a unique $y^* \in (0, \infty)$ such that
$p(y^*) = 0$.
Thus,
\begin{align*}
    \ind\{ \calE \} Z_i \leq  \ind\{ \calE \} (\Sigma^{-1} + y^* I_q)^{-1}_{ii}.
\end{align*}
\end{proof}

\begin{myprop}
\label{stmt:local_strict_concave_max}
Let $M, A$ be $n \times n$ positive definite matrices, and
let $\alpha, \beta$ be positive numbers.
Consider the function:
\begin{align*}
    f(\tau) := -\frac{\alpha}{\tau} + \ip{ (A + \beta \tau I)^{-1}}{M}.
\end{align*}
Suppose there exists a $\tau \in (0, \infty)$ satisfying
$f'(\tau) = 0$.
Then, $f''(\tau) < 0$.
\end{myprop}
\begin{proof}
A straightforward computation yields the following expressions for $f'(\tau)$ and $f''(\tau)$:
\begin{align*}
    f'(\tau) &= \alpha \tau^{-2} - \beta \ip{(A + \beta \tau I)^{-2}}{M}, \\
    f''(\tau) &= -2\alpha\tau^{-3} + 2\beta^2\ip{(A+\beta \tau I)^{-3}}{M}.
\end{align*}
The assumed condition $f'(\tau) = 0$ implies that:
\begin{align*}
    \alpha \tau^{-2} = \beta \ip{(A+\beta\tau I)^{-2}}{M} \Longrightarrow -2\alpha \tau^{-3} = -2\beta\tau^{-1} \ip{(A+\beta \tau I)^{-2}}{M}.
\end{align*}
Next, let $A = Q \Lambda Q^\T$ be the eigendecomposition of $A$, with $\Lambda = \diag(\{\lambda_i\}_{i=1}^{n})$.
For any integer $k$:
\begin{align*}
    \ip{(A+\beta \tau I)^{-k}}{M} = \Tr(M Q (\Lambda + \beta\tau I)^{-k} Q^\T) = \ip{ QMQ^\T}{ (\Lambda + \beta \tau I)^{-k}} = \sum_{i=1}^{n} \frac{(QMQ^\T)_{ii}}{(\lambda_i + \beta \tau)^k}.
\end{align*}
Now, since $M$ is positive definite, $(QMQ^\T)_{ii} > 0$ for all $i \in [n]$.
Furthermore, since $A$ is positive definite, $\lambda_i > 0$ for all $i \in [n]$. 
Hence plugging these expressions into the expression of $f''(\tau)$:
\begin{align*}
    f''(\tau) &= -2\beta^2 \sum_{i=1}^{n} \frac{(QMQ^\T)_{ii}}{(\lambda_i + \beta \tau)^2 \beta \tau} + 2\beta^2 \sum_{i=1}^{n}\frac{(QMQ^\T)_{ii}}{(\lambda_i + \beta\tau)^2 (\lambda_i + \beta\tau)} \\
    &< -2\beta^2 \sum_{i=1}^{n} \frac{(QMQ^\T)_{ii}}{(\lambda_i + \beta \tau)^2 \beta \tau} + 2\beta^2 \sum_{i=1}^{n}\frac{(QMQ^\T)_{ii}}{(\lambda_i + \beta\tau)^2 \beta \tau} \\
    &= 0.
\end{align*}
\end{proof}

\subsection{Proof of \Cref{stmt:ind_seq_ls_lower_bound}}
\label{sec:appendix:ind_seq_ls_lower_bound}

\indseqlslowerbound*
\begin{proof}
Let $\Gamma_T := \Gamma_T(\Px)$.
We have that $\Gamma_T = \frac{2}{T} (2^T-1) I_n \succcurlyeq \frac{2^T}{T} I_n$.
By \Cref{thm:trace_inv_lower_bounds_minimax_risk}:
\begin{align*}
    \mathsf{R}(m,T,T;\{\Px\}) \geq \sigma_\xi^2 p \cdot \E \Tr(\Gamma_T^{1/2} (X_{m,T}^\T X_{m,T})^{-1} \Gamma_T^{1/2}) \geq \frac{\sigma_\xi^2 p}{T} \cdot \E \Tr((2^{-T/2} X_{m,T}^\T X_{m,T} 2^{-T/2})^{-1}).
\end{align*}
Since each column of $X_{m,T}$ is independent,
the matrix
$X_{m,T}2^{-T/2}$ has the same distribution as
$\mathsf{BDiag}(\Theta^{1/2}, m)W$,
where $\Theta \in \R^{T \times T}$ is diagonal,
$\Theta_{ii} = 2^{i-T}$ for $i \in \{1, \dots, T\}$, and 
$W \in \R^{mT \times n}$ has \IID\ $N(0, 1)$ entries.
Let $\lambda_t = 2^{T-t}$ for $t \in \{1, \dots, T\}$.
With this notation:
\begin{align*}
    \E \Tr((2^{-T/2} X_{m,T}^\T X_{m,T} 2^{-T/2})^{-1})  = \E \Tr((W^\T \mathsf{BDiag}(\Theta, m) W)^{-1}).
\end{align*}

Let $\{g_{j}\}_{j=1}^{m}$ be independent
isotropic Gaussian random vectors in $\R^{T}$,
and let $h \sim N(0, I_{n-1})$ be independent from $\{g_j\}$.
Define the random variables $\{Z_i\}_{i=1}^{T}$ as:
\begin{align}
    Z_i := \min_{\beta \geq 0} \max_{\tau \geq 0} \left[ -\frac{\beta \norm{h}_2}{2\tau} +  \beta^2\sum_{j=1}^{m} \sum_{t=1}^{T} \frac{g_{j,t}^2}{\lambda_{t} + \beta \norm{h}_2 \tau} + \frac{1}{\lambda_i + \beta \norm{h}_2 \tau} \right].
\end{align}
By \Cref{lemma:trace_inv_lower_bound},
\begin{align*}
    \E \Tr((W^\T \mathsf{BDiag}(\Theta, m) W)^{-1}) \geq \frac{n}{2m} \left[ \sum_{i=1}^{T} \E[Z_i] \right]^{-1}.
\end{align*}
Next, define 
\begin{align*}
    n_1 := \frac{n}{64}, \quad
    p(y) := \sum_{j=1}^{m} \sum_{t=1}^{T} \frac{y}{\lambda_{t} + y} g_{j,t}^2 - \frac{n_1}{2}.
\end{align*}
Since $n \geq 6$ and $mT \geq n$, we
can invoke \Cref{stmt:Z_i_helper} to conclude there exists
an event $\calE_1$ (over the probability of $\{g_j\}$ and $h$)
such that:
\begin{enumerate}[label=(\alph*)]
    \item on $\calE_1$, 
    there exists a unique root $y^* \in (0, \infty)$
    such that $p(y^*) = 0$,
    \item the following inequalities holds:
    \begin{align}
        Z_i \leq \frac{1}{\lambda_i}, \quad \ind\{\calE_1\} Z_i \leq\ind\{\calE_1\} \frac{1}{\lambda_i + y^*}, \label{eq:bound_on_Z_i_good}
    \end{align}
    \item the following estimate holds:
    \begin{align*}
    \Pr(\calE_1^c) \leq e^{-n/128} + e^{-mT/16}.
    \end{align*}
\end{enumerate}

Now, let $c = 1/20$, and assume that
$c n_1 / m \geq 4$.
We can check easily that $\ceil{c n_1/m} \leq T$.
Fix a $\delta \in (0, e^{-2}]$ to be chosen later.
Define the integer $T_c := \ceil{c n_1/m} \in \{4, \dots, T\}$, 
and the events (over the probability of $\{g_j\}$ and $h$):
\begin{align*}
    \calE_2^{g,T_c} := \left\{ \sum_{j=1}^{m} \sum_{t=1}^{T_c} g_{j,t}^2 \leq 5 mT_c \right\}, \quad \calE_2^{g,+} := \left\{ \max_{t=1, \dots, T} \sum_{j=1}^{m} g_{j,t}^2 \leq 2 m + 4 \log\left( \frac{t^2 \pi^2}{6\delta} \right) \right\}.
\end{align*}
By \Cref{lemma:chi_squared_tail_bounds},
$\Pr((\calE_2^{g,T_c})^c) \leq e^{-mT_c}$.
Next, Gaussian concentration for Lipschitz functions \citep[cf][Chapter~2]{wainwright2019book} yields,
for any $\eta \in (0, 1)$:
\begin{align*}
    \max_{t=1, \dots, T} \Pr\left\{ \sqrt{\sum_{j=1}^{m} g_{j,t}^2} \geq \sqrt{m} + \sqrt{2\log(1/\eta)} \right\} \leq \eta.
\end{align*}
Hence by a union bound, and
the fact that 
$6\delta/\pi^2 \sum_{t=1}^{T} t^{-2} \leq 6\delta/\pi^2 \sum_{t=1}^{\infty} t^{-2} = \delta$,
we have that $\Pr((\calE_2^{g,+})^c) \leq \delta$.
Putting $\calE := \calE_1 \cup \calE_2^{g,T_c} \cup \calE_2^{g,+}$, we have:
\begin{align}
    \Pr(\calE^c) &\leq e^{-n/128} + e^{-mT/16} + e^{-mT_c} + \delta \nonumber \\
    &\leq e^{-n/128} + e^{-mT/16} + e^{-c n_1} + \delta. \label{eq:prob_bad_event}
\end{align}
Next, noting that $t/2 \geq \log_2\log((t+1)^2\pi^2/6)$ for all $t \geq 4$:
\begin{align}
    &~~~~\sum_{t=T_c}^{T-1} 2^{-t} \log((t+1)^2 \pi^2/(6\delta)) \nonumber \\
    &= \sum_{t=T_c}^{T-1} 2^{-t + \log_2\log((t+1)^2 \pi^2/6)} + \log(1/\delta) \sum_{t=T_c}^{T-1} 2^{-t} \nonumber \\
    &\leq \sum_{t=T_c}^{T-1} 2^{-t/2} + \log(1/\delta) \sum_{t=T_c}^{T-1} 2^{-t} && \text{since } T_c \geq 4 \nonumber \\
    &= \sqrt{2}/(\sqrt{2}-1) (2^{-T_c/2} - 2^{-T/2}) + 2 \log(1/\delta) (2^{-T_c} - 2^{-T}) \nonumber \\
    &\leq (4 + 2 \log(1/\delta)) 2^{-T_c/2} \nonumber \\
    &\leq 4 \log(1/\delta) 2^{-T_c/2} &&\text{since } \delta \in (0, e^{-2}). \label{eq:geometric_series_with_log_bound}
\end{align}
Now, on $\calE$:
\begin{align*}
    \frac{n_1}{2} &= \sum_{j=1}^{m} \sum_{t=1}^{T} \frac{y^*}{\lambda_{t} + y^*} g_{j,t}^2 &&\text{since } p(y^*) = 0 \\
    &\leq \sum_{j=1}^{m} \sum_{t=1}^{T_c} g_{j,t}^2 + y^* \sum_{j=1}^{m} \sum_{t=T_c}^{T-1} 2^{-t} g_{j,t+1}^2 \\
    &= \sum_{j=1}^{m} \sum_{t=1}^{T_c} g_{j,t}^2 + y^* \sum_{t=T_c}^{T-1} 2^{-t} \left[\sum_{j=1}^{m} g_{j,t+1}^2\right]  \\
    &\leq 5mT_c + y^* \sum_{t=T_c}^{T-1} 2^{-t} \left[ 2m + 4 \log((t+1)^2 \pi^2/(6 \delta)) \right] &&\text{using } \calE \\
    &\leq 5mT_c + 4m y^* 2^{-T_c} + 16 y^* \log(1/\delta) 2^{-T_c/2} &&\text{using \eqref{eq:geometric_series_with_log_bound}} \\
    &\leq 5mT_c + 18 m y^* \log(1/\delta) 2^{-T_c/2} &&\text{since } \delta \in (0, e^{-2}).
\end{align*}
This inequality implies the following
lower bound on $y^*$:
\begin{align*}
    y^* &\geq \frac{2^{cn_1/(2m)}}{18 \log(1/\delta)} \left[ \frac{n_1}{2m} - 5 c \frac{n_1}{m} - 5 \right] \\
    &=  \frac{2^{cn_1/(2m)}}{18 \log(1/\delta)} \left[ \frac{n_1}{4m} - 5 \right] &&\text{since } c = 1/20 \\
    &\geq  \frac{2^{cn_1/(2m)}}{144 \log(1/\delta)} \frac{n_1}{m} &&\text{since } cn_1/m \geq 4 \Longrightarrow n_1/(8m) \geq 5 \\
    &=: \underline{y}^*.
\end{align*}
We now bound,
\begin{align*}
    \sum_{i=1}^{T} \E[Z_i] &= \sum_{i=1}^{T} \left[\E[\ind\{\calE\} Z_i] + \E[\ind\{\calE^c\} Z_i] \right] \\
    &\leq \sum_{i=1}^{T}  \left(\E\left[ \ind\{\calE\} \frac{1}{\lambda_i + y^*}\right] + \Pr(\calE^c) \frac{1}{\lambda_i} \right) &&\text{using \eqref{eq:bound_on_Z_i_good}} \\
    &\leq \sum_{i=1}^{T} \frac{1}{\lambda_i + \underline{y}^*} + \Pr(\calE^c) \sum_{t=1}^{T} \frac{1}{\lambda_i} &&\text{since } y^* \geq \underline{y}^* \text{ on } \calE \\
    &\leq \sum_{t=0}^{T-1} \frac{1}{2^t + \underline{y}^*} + 2\left(e^{-n/128} + e^{-mT/16} + e^{-c n_1} + \delta \right) &&\text{using } \eqref{eq:prob_bad_event} \\
    &\leq \frac{T_c}{y^*} + 2 \cdot 2^{-T_c} + 2\left(e^{-n/128} + e^{-mT/16} + e^{-c n_1} + \delta \right) \\
    &\leq 288 c \log(1/\delta) 2^{-c n_1/(2m)} + 2 \cdot 2^{-cn_1/m} \\
    &\qquad+ 2\left( e^{-n/128} + e^{-mT/16} + e^{-c n_1} + \delta\right) .
\end{align*}
Since $c n_1/m \geq 4$,
we can choose $\delta = e^{-c n_1/(2m)} \in (0, e^{-2}]$ and
obtain:
\begin{align*}
    \sum_{i=1}^{T} \E[Z_i] &\leq 144c^2 \frac{n_1}{m} 2^{-c n_1/(2m)} + 2 \cdot 2^{-c n_1/m} + 2\left(e^{-n/128} + e^{-mT/16} + e^{-c n_1} + e^{-c n_1/(2m)} \right). 
\end{align*}
Since $1 \leq c n_1/(4m)$,
$mT \geq n$, and $m \geq 1$,
this inequality implies there
exists universal positive constants
$c_1, c_2$ such that:
\begin{align*}
    \sum_{i=1}^{T} \E[Z_i] \leq  \frac{c_1n}{m} 2^{-c_2 n/m}.
\end{align*}
Hence:
\begin{align*}
    \sfR(m,T,T;\{\Px\}) \geq \frac{\sigma_\xi^2 p}{T} \frac{n}{2m} \left[ \sum_{i=1}^{T} \E[Z_i] \right]^{-1} \geq \frac{\sigma_\xi^2 p}{T} \frac{n}{2m}  \frac{m}{c_1n} 2^{c_2 n/m} = \frac{\sigma^2_\xi p 2^{c_2 n/m}}{2c_1 T}.
\end{align*}
\end{proof}

\subsection{Block decoupling}
\label{sec:appendix:block_decoupling}

We now use a block decoupling argument to study lower bounds
on the risk.
The first step is the following result,
which bounds the risk from below by a 
particular random gramian matrix.
\begin{mylemma}
\label{lemma:decoupling_blocks}
Let $n = dr$ with both $d,r$ positive integers.
Define $\calI_r := \{1, 1 + r, \dots, 1 + (T-1)r\}$,
and let $E_{\calI_r} \in \R^{T \times Tr}$ denote 
the linear operator which extracts the coordinates in $\calI_r$,
so that $(E_{\calI_r} x)_i = x_{1 + (i-1)r}$ for $i=1, \dots, T$.
Recall the following definitions from Equation~\eqref{eq:Theta_T_r_def}:
\begin{align*}
    \Psi_{r,T,\Tnew} &= \mathsf{BDiag}(\Gamma_{\Tnew}^{-1/2}(J_r), T) \mathsf{BToep}(J_r, T) \in \R^{Tr \times Tr}, \\
    \Theta_{r,T,\Tnew} &= E_{\calI_r} \Psi_{r,T,\Tnew} \Psi_{r,T,\Tnew}^\T E_{\calI_r}^\T \in \R^{T \times T}.
\end{align*}
Then, for $A = \mathsf{BDiag}(J_r, d)$ we have:
\begin{align*}
    \E_{\otimes_{i=1}^{m} \PxA{A}} \left[\Tr\left(\Gamma_{\Tnew}^{1/2}(A) (X_{m,T}^\T X_{m,T})^{-1} \Gamma_{\Tnew}^{1/2}(A)\right)\right] \geq \E \Tr((W^\T \mathsf{BDiag}(\Theta_{r,T,\Tnew}, m) W)^{-1}),
\end{align*}
where $W \in \R^{mT \times d}$ is a matrix with independent $N(0, 1)$ entries.
\end{mylemma}
\begin{proof}
We apply \Cref{prop:trace_inv_selector_lower_bound}
with:
\begin{align*}
    M = X_{m,T} \Gamma_{\Tnew}^{-1/2}, \quad I = \{1, 1 + r, 1 + 2r, \dots, 1 + (d-1)r\}, \quad \abs{I} = d.
\end{align*}
Note that the block diagonal structure of $A$ yields 
the same block diagonal structure on $\Gamma_{\Tnew}$ and its inverse square root, specifically
$\Gamma_{\Tnew}(A) = \mathsf{BDiag}(\Gamma_{\Tnew}(J_r), d)$
and $\Gamma_{\Tnew}^{-1/2}(A) = \mathsf{BDiag}(\Gamma_{\Tnew}^{-1/2}(J_r), d)$.
Hence, it is not hard to see that the columns
of $M E_I^\T$ are not only independent, but also
identically distributed. Furthermore, the distribution
of each column obeys a multivariate Gaussian in $\R^{mT}$.
Hence, $ME_I^\T$ is equal in distribution to $Q_{m,T}^{1/2} W$,
where $W \in \R^{mT \times d}$ is a matrix of \IID\ Gaussians and
$Q_{m,T} \in \mathsf{Sym}^{mT}_{> 0}$ is a positive definite covariance matrix to be determined.
Furthermore, because $ME_I^\T$ contains the vertical
concatenation of $m$ independent trajectories, 
$Q_{m,T}$ itself is block diagonal:
\begin{align*}
    Q_{m,T} = \mathsf{BDiag}(Q_T, m), \quad Q_T \in \mathsf{Sym}^T_{> 0}.
\end{align*}
Let us now compute an expression for $Q_{T}$.
Consider the dynamics:
\begin{align*}
    x_{t+1}^r = J_r x_t^r + w_t^r, \quad x_0^r = 0, \quad w_t^r \sim N(0, \sigma_w^2 I_r).
\end{align*}
It is not hard to see that,
with $w_{0:T-1}^r = (w_0, \dots, w_{T-1}) \in \R^{rT}$,
\begin{align*}
    \begin{bmatrix} \Gamma_{\Tnew}^{-1/2}(J_r) x_1^r \\ \vdots \\ \Gamma_{\Tnew}^{-1/2}(J_r) x_T^r \end{bmatrix} = \Psi_{r,T,\Tnew} w_{0:T-1}^r.
\end{align*}
From this, we see that every column of $ME_I^\T$ 
is equal in distribution to $E_{\calI_r} \Psi_{r,T,\Tnew} w^r_{0:T-1}$, and therefore
has distribution
$N(0, E_{\calI_r} \Psi_{r,T,\Tnew} \Psi_{r,T,\Tnew}^\T E_{\calI_r}^\T)$.
Therefore:
\begin{align*}
    Q_T = E_{\calI_r} \Psi_{r,T,\Tnew} \Psi_{r,T,\Tnew}^\T E_{\calI_r}^\T = \Theta_{r,T,\Tnew}.
\end{align*}
The claim now follows.
\end{proof}

\subsection{Eigenvalue analysis of a tridiagonal matrix}
\label{sec:appendix:lower_bounds:eigenvalues}

For any $T \in \N_{+}$, recall that 
$L_T$ denotes the $T \times T$ lower triangle matrix
with ones in the lower triangle,
and $\Tri{a}{b}{T}$ denotes the symmetric $T \times T$ tri-diagonal matrix with
$a$ on the digonal and $b$ on the lower and upper off-diagonals.
In this section, we study the eigenvalues of $(L_TL_T^\T)^{-1}$,
which we denote by $S_T$:
\begin{align}
    S_T = (L_TL_T^\T)^{-1} = \Tri{2}{-1}{T} - e_Te_T^\T. \label{eq:S_T_defn}
\end{align}
Understanding the eigenvalues of this matrix will be necessary
in the proof of \Cref{thm:lower_bound_r_equals_1}.
The following result sharply characterizes the spectrum of
$S_T$ up to constant factors.
\begin{mylemma}
\label{cor:SST_lower_upper_bounds}
Suppose $T \geq 8$. For all $k=1, \dots, T$, we have that:
\begin{align*}
    0.02 \frac{k^2}{T^2} \leq \lambda_{T-k+1}(S_T) \leq \pi^2\frac{k^2}{T^2}.
\end{align*}
\end{mylemma}
\begin{proof}
We prove the upper bound in \Cref{prop:SST_upper_bounds},
and the lower bound in \Cref{prop:SST_lower_bounds}.
\end{proof}

The next result gives the necessary upper bounds on the eigenvalues
of $S_T$.
\begin{myprop}
\label{prop:SST_upper_bounds}
We have that:
\begin{align*}
    \lambda_{T-k+1}(S_T) \leq \pi^2 \frac{k^2}{T^2}, \quad k=1, \dots, T.
\end{align*}
\end{myprop}
\begin{proof}
By \eqref{eq:S_T_defn}, we immediately produce a semidefinite upper bound on $S_T$:
\begin{align*}
    S_T = \Tri{2}{-1}{T} - e_Te_T^\T \preccurlyeq \Tri{2}{-1}{T}.
\end{align*}
Therefore by the Courant min-max theorem, followed by the closed-form expression for the eigenvalues of $\Tri{2}{-1}{T}$, we have:
\begin{align*}
    \lambda_{T-k+1}(S_T) \leq \lambda_{T-k+1}(\Tri{2}{-1}{T}) = 2\left(1 - \cos\left(\frac{k\pi}{T+1}\right)\right), \quad k=1, \dots, T.
\end{align*}
Next, we have the following elementary lower bounds for $\cos(x)$
on $x \in [0, \pi]$:
\begin{align*}
    \cos(x) \geq \begin{cases} 
    1 - x^2/2 &\text{if } x \in [0, 2\pi/3], \\
    (x-\pi)^2/4 - 1 &\text{if } x \in [2\pi/3, \pi].
    \end{cases}
\end{align*}
Therefore, when $ k \in \left\{1, \dots, \bigfloor{\frac{2(T+1)}{3}} \right\}$, we immediately have that:
\begin{align*}
    \lambda_{T-k+1}(S_T) \leq 
    \pi^2 \frac{k^2}{(T+1)^2}.
\end{align*}
For the case when $k \in \{ \floor{\frac{2(T+1)}{3}} + 1, \dots, T\}$, we use the cosine lower bounds to bound:
\begin{align*}
    \lambda_{T-k+1}(S_T) &\leq 4 - \frac{\pi^2}{2} \left( 1 - \frac{k}{T+1}\right)^2 \\
    &\leq 4\left[ 1 - \left( 1 - \frac{k}{T+1} \right)^2 \right] \\
    &= 4 \left( \frac{k}{T+1} \right) \left( 2 - \frac{k}{T+1} \right) \\
    &= 4 \left( \frac{k}{T+1} \right) \left( \frac{2(T+1)-k}{T+1} \right) \\
    &\leq 4 \left( \frac{k}{T+1} \right) \left( \frac{ 3k - k }{T+1} \right) &&\text{since } k \geq 2(T+1)/3 \\
    &= 8 \frac{k^2}{(T+1)^2}.
\end{align*}
The claim now follows by taking the maximum over the upper bounds.
\end{proof}

We now move to the lower bound on $\lambda_{T-k+1}(S_T)$.
At this point, it would be tempting to use Weyl's inequalities,
which imply that 
$\lambda_i(S_T) \geq \lambda_i(\Tri{2}{-1}{T}) - 1$.
However, this bound becomes vacuous, since
$\lambda_{T}(\Tri{2}{-1}{T}) \lesssim 1/T^2$.
To get finer grained control, we need to use the
eigenvalue interlacing result of~\cite{kulkarni99tridiagonal}.
This is done in the following result:
\begin{myprop}
\label{prop:SST_lower_bounds}
Suppose that $T \geq 8$.
We have that
\begin{align*}
    \lambda_{T-k+1}(S_T) \geq 0.02 \frac{k^2}{T^2}, \quad k=1, \dots, T.
\end{align*}
\end{myprop}
\begin{proof}
The proof relies on the interlacing result
from~\cite[Theorem~4.1]{kulkarni99tridiagonal}.
However, the interlacing result does not cover the 
minimum eigenvalue of $S_T$, so we first explicitly derive a lower bound for $\lambda_{\min}(S_T)$. To do this, we note that:
\begin{align*}
    \lambda_{\min}(S_T) = \lambda_{\min}((L_TL_T^\T)^{-1}) = \frac{1}{\opnorm{L_T}^2}.
\end{align*}
Letting $l_i \in \R^T$ denote the $i$-th column of $L_T$,
by the variational form of the operator norm followed by Cauchy-Schwarz,
\begin{align*}
    \opnorm{L_T} = \max_{\norm{v}_2 = 1} \norm{L_T v}_2  \leq \max_{\norm{v}_2 = 1} \sum_{i=1}^{T} \norm{l_{i}}_2 \abs{v_i} \leq \sqrt{ \sum_{i=1}^{T} \norm{l_{i}}_2^2 }
    = \sqrt{\sum_{i=1}^{T} i} = \sqrt{T (T+1)/2}.
\end{align*}
Hence:
\begin{align*}
    \lambda_{\min}(S_T) \geq \frac{2}{T(T+1)} \geq \frac{1}{T^2}.
\end{align*}

Now we may proceed with the remaining eigenvalues.
We can write $S_T$ as the following block matrix, with $e_{T-1} \in \R^{T-1}$
denoting the $(T-1)$-th standard basis vector:
\begin{align*}
    S_T = \bmattwo{\Tri{2}{-1}{T-1}}{-e_{T-1}}{-e_{T-1}^\T}{1}.
\end{align*}
This matrix is of the form studied in \cite[Theorem~4.1]{kulkarni99tridiagonal};
for what follows we will borrow their notation.
Let $U_T(x)$ denote the $T$-th degree Chebyshev
polynomial of the 2nd kind.
We know that the eigenvalues of $S_T$ are 
given by $\lambda = 2(1-x)$,
where $x$ are the roots of the polynomial $p_T(x)$ defined as:
\begin{align}
    p_T(x) := (1+2x) U_{T-1}(x) - U_{T-2}(x). \label{eq:char_poly_SST}
\end{align}
Therefore, letting $\psi_1 \leq \dots \leq \psi_T$
denote the roots of \eqref{eq:char_poly_SST}
listed in increasing order, we have:
\begin{align*}
    \lambda_{i}(S_T) = 2(1-\psi_i), \quad i=1, \dots, T.
\end{align*}

Let $\eta_1 < \dots < \eta_{T-2}$ denote the $T-2$ roots of $U_{T-2}(x)$
listed in increasing order. Put $\eta_0 := -\infty$ and $\eta_{T-1} := +\infty$.
Because the roots of
$U_{T-2}(x)$ are given by $x = \cos(\frac{k\pi}{T-1})$, $k=1, \dots, T-2$, we have that:
\begin{align*}
    \eta_{i} = \cos\left( \frac{(T-1-i)\pi}{T-1} \right), \quad i = 1, \dots, T-2.
\end{align*}
\cite[Theorem~4.1]{kulkarni99tridiagonal} states that 
there is exactly one root of $p_T(x)$
in each of the intervals
$(\eta_{j}, \eta_{j+1})$ for $j \in \{0, \dots, T-2\} \setminus \{i_\star\}$,
with $i_\star$ satisfying:
\begin{align*}
    i_\star \in \begin{cases}
    \{ \floor{\frac{2(T-1)}{3}} \} &\text{if } 2(T-1) \mod 3 \neq 0, \\
    \{ \frac{2(T-1)}{3}, \frac{2(T-1)}{3} + 1 \} &\text{otherwise},
    \end{cases}
\end{align*}
and furthermore $(\eta_{i_\star}, \eta_{i_\star+1})$ contains exactly two roots of $p_T(x)$.
Therefore, for $i \in \{i_\star+3, \dots, T-1\}$:
\begin{align*}
    \psi_i \leq \eta_{i-1} \Longrightarrow \lambda_i(S_T) \geq 2(1-\eta_{i-1}) = 2\left(1 - \cos\left( \frac{( T - i)\pi}{T-1} \right)\right).
\end{align*}
For $i \in \{i_\star+3, \dots, T-1\}$, we have:
\begin{align*}
    \frac{T-i}{T-1} \leq \frac{T-i_\star-3}{T-1} = \frac{ T - (\frac{2(T-1)}{3}-1) - 3}{T-1} = \frac{1}{3} - \frac{1}{T-1} \leq \frac{1}{3}. 
\end{align*}
It is elementary to check that:
\begin{align*}
    2 (1-\cos(x)) \geq \frac{x^2}{2} \quad \forall x \in [0, \pi/3].
\end{align*}
Therefore for $i \in \{i_\star+3, \dots, T-1\}$,
\begin{align*}
    \lambda_i(S_T) \geq \frac{\pi^2}{2} \left( \frac{T-i}{T-1} \right)^2. 
\end{align*}
Furthermore,
for $i \in \{1, \dots, i_\star + 2\}$,
\begin{align*}
    \psi_i \leq \eta_{i_\star+1} \Longrightarrow \lambda_i(S_T) \geq 2(1 - \eta_{i_\star+1}) = 2 \left( 1 - \cos\left( \frac{(T-1-i_\star-1)\pi }{T-1 }\right) \right) \geq 2(1-\cos(\pi/21)).
\end{align*}
The last inequality holds by:
\begin{align*}
    \cos\left( \frac{(T-1-i_\star-1)\pi }{T-1 }\right) &\leq \cos\left( \frac{(T-1)-(2(T-1)/3+2)}{T-1} \pi \right) &&\text{since } i_\star \leq \frac{2(T-1)}{3} + 1 \\
    &= \cos\left( \left(\frac{1}{3} - \frac{2}{T-1}\right)\pi\right) \\
    &\leq \cos(\pi/21) &&\text{since } T \geq 8.
\end{align*}
Summarizing, we have shown that:
\begin{align*}
    \lambda_{T-k+1}(S_T) \geq \begin{cases}
      \frac{1}{T^2} &\text{if } k = 1, \\
      \frac{\pi^2}{2} \left( \frac{k-1}{T-1} \right)^2 &\text{if } k \in \{2, \dots, T-i_\star-2\}, \\
      2(1-\cos(\pi/21)) &\text{if } k \in \{T-i_\star-1, \dots, T\}.
    \end{cases}
\end{align*}
Since $\frac{k-1}{T-1} \geq \frac{k}{2T}$ when $k \geq 2$,
and since $2(1-\cos(\pi/21)) \geq 2(1-\cos(\pi/21)) \frac{k^2}{T^2}$ trivially, we have shown the desired conclusion:
\begin{align*}
    \lambda_{T-k+1}(S_T) \geq \min\left\{ 1, \frac{\pi^2}{8}, 2(1-\cos(\pi/21)) \right\} \frac{k^2}{T^2} \geq 0.02 \frac{k^2}{T^2}, \quad k = 1, \dots, T.
\end{align*}
\end{proof}

\subsection{A risk lower bound in the few trajectories regime}
\label{sec:appendix:lower_bounds:r_equals_one}

\begin{mylemma}
\label{thm:lower_bound_r_equals_1}
There exist universal positive constants $c_0$, $c_1$, $c_2$, and $c_3$ such that the following is true.
Suppose $\calA \subseteq \R^{n \times n}$ is any set containing 
$I_{n}$.
Let $T \geq c_0$, $n \geq c_1$, $mT \geq n$, and $m \leq c_2 n$.
We have that:
\begin{align*}
    \sfR(m, T, \Tnew; \{ \PxA{A} \mid A \in \calA\}) \geq c_3 \sigma_\xi^2 p \cdot \frac{n^2}{m^2 T} \cdot \frac{\Tnew}{T}.
\end{align*}
\end{mylemma}
\begin{proof}

Let $\{g_j\}_{j=1}^{m}$ be independent $N(0, I_T)$ random vectors,
and let $h \sim N(0, I_{n-1})$ be independent from $\{g_j\}$.
Let $\{\lambda_t\}_{t=1}^{T}$ denote the eigenvalues of $\Theta_{1,T,T}^{-1}$ listed in decreasing order.
Define the random variables $\{Z_i\}_{i=1}^{T}$ as:
\begin{align}
    Z_i := \min_{\beta \geq 0} \max_{\tau \geq 0} \left[ -\frac{\beta \norm{h}_2}{2\tau} +  \beta^2\sum_{j=1}^{m} \sum_{t=1}^{T} \frac{g_{j,t}^2}{\lambda_t + \beta \norm{h}_2 \tau}  + (\Theta_{1,T,T}^{-1} + \beta \norm{h}_2 \tau I_T)^{-1}_{ii} \right].
\end{align}
We now lower bound the minimax risk as follows:
\begin{align}
    &~~~~\sfR(m,T,\Tnew;\{\PxA{I_n}\}) \nonumber \\
    &\geq \sigma_\xi^2 p \cdot \E_{\otimes_{i=1}^{m} \PxA{I_n}}\left[ \Tr\left( \Gamma_{\Tnew}(I_n)^{1/2} (X_{m,T}^\T X_{m,T})^{-1} \Gamma_{\Tnew}(I_n)^{1/2} \right) \right] && \text{by \Cref{thm:trace_inv_lower_bounds_minimax_risk}} \nonumber \\
    &\geq \sigma_\xi^2 p \cdot \E \Tr((W^\T \mathsf{BDiag}(\Theta_{1,T,\Tnew}, m) W)^{-1}) && \text{by \Cref{lemma:decoupling_blocks}} \nonumber \\
    &= \sigma_\xi^2 p \cdot \frac{\Tnew + 1}{T+1} \cdot \E\Tr((W^\T \mathsf{BDiag}(\Theta_{1,T,T}, m) W)^{-1}) &&\text{using \eqref{eq:Tnew_equals_T_wlog_r_equals_one_lower_bound}} \nonumber \\
    &\geq \sigma_\xi^2 p \cdot \frac{\Tnew}{2T} \cdot \E\Tr((W^\T \mathsf{BDiag}(\Theta_{1,T,T}, m) W)^{-1}) \nonumber \\
    &\geq \sigma_\xi^2 p \cdot \frac{\Tnew}{2T} \cdot \frac{n}{2m} \cdot \left[ \sum_{i=1}^{T} \E[Z_i] \right]^{-1} &&\text{by \Cref{lemma:trace_inv_lower_bound}}. \label{eq:lower_bound_with_sum_Zis_inverse}
\end{align}

Next, define:
\begin{align*}
    n_1 := \frac{n}{64}, \quad
    p(y) := \sum_{j=1}^{m} \sum_{t=1}^{T} \frac{y}{\lambda_{t} + y} g_{j,t}^2 - \frac{n_1}{2}.
\end{align*}
Assuming that $c_1 \geq 6$ so that $n \geq 6$ and $mT \geq n$, we
can invoke \Cref{stmt:Z_i_helper} to conclude there exists
an event $\calE_1$ (over the probability of $\{g_j\}$ and $h$)
such that:
\begin{enumerate}[label=(\alph*)]
    \item on $\calE_1$, 
    there exists a unique root $y^* \in (0, \infty)$
    such that $p(y^*) = 0$,
    \item the following inequalities holds:
    \begin{align}
        Z_i \leq (\Theta_{1,T,T})_{ii}, \quad \ind\{\calE_1\} Z_i \leq\ind\{\calE_1\} \frac{1}{\lambda_i + y^*}, \label{eq:bound_on_Z_i_good_v2}
    \end{align}
    \item the following estimate holds:
    \begin{align*}
    \Pr(\calE_1^c) \leq e^{-n/128} + e^{-mT/16}.
    \end{align*}
\end{enumerate}

The remainder of the proof is to estimate a lower bound on
$y^*$.
Towards this goal, we define an auxiliary function:
\begin{align*}
    \tilde{p}(y) := \E[p_1(y)] = m\sum_{t=1}^{T} \frac{y}{\lambda_t + y} - \frac{n_1}{2}.
\end{align*}
Let $\bar{y}^*$ be the unique solution
to $\tilde{p}(y) = 0$.
A unique root exists because $\tilde{p}(0) < 0$, 
$\lim_{y \rightarrow \infty}\tilde{p}(y) = mT - n_1/2 \geq n - n/64 > 0$,
and $\tilde{p}$ is continuous and strictly increasing.
We derive a lower bound on $y^*$ through a 
lower bound on $\bar{y}^*$.
For any fixed $\alpha > 0$, the function
$x \mapsto \frac{x}{\alpha + x}$ is monotonically
increasing and concave on $\R_{> 0}$.
Therefore, the function $p(y)$ is monotonically
increasing and concave on $\R_{> 0}$.
By \Cref{prop:linear_approx_root_lower_bound}, the root of the linear approximation to 
$p(y)$ at $\bar{y}^*$ is a lower bound 
to $y^*$:
\begin{align}
    \ind\{ \calE_1 \} y^* \geq \ind\{ \calE_1 \} \left[\bar{y}^* - \frac{p(\bar{y}^*)}{p'(\bar{y}^*)}\right]. \label{eq:y_i_lower_bound}
\end{align}
Equation~\eqref{eq:y_i_lower_bound} is a crucial step for the proof, because 
it turns analyzing $y^*$, which is the root of a random function, into analyzing the pointwise evaluation
of a random function on a deterministic quantity.
To lower bound the RHS, 
we need a upper bound on $p(\bar{y}^*)$
and lower bounds on both $\bar{y}^*$ and $p'(\bar{y}^*)$.

\paragraph{Upper and lower bounds on $\bar{y}^*$.}
We first derive a crude upper bound by Jensen's inequality.
Observe that $\tilde{p}(\bar{y}^*) = 0$
implies that:
\begin{align*}
    mT - \frac{n_1}{2} = m \sum_{t=1}^{T} \frac{\lambda_t}{\lambda_t + \bar{y}^*}.
\end{align*}
The function $x \mapsto x / (x + \bar{y}^*)$
is concave on $\R_{> 0}$.
Let $\bar{\lambda} := \frac{1}{T} \sum_{t=1}^{T} \lambda_t$.
Jensen's inequality states that
$T \frac{\bar{\lambda}}{\bar{\lambda} + \bar{y}^*} \geq \sum_{t=1}^{T} \frac{\lambda_t}{\lambda_t + \bar{y}^*}$.
Therefore:
\begin{align*}
    1 - \frac{n_1}{2mT} \leq \frac{\bar{\lambda}}{\bar{\lambda} + \bar{y}^*} \Longrightarrow \bar{y}^* \leq \bar{\lambda} \frac{n_1}{2mT} \frac{1}{1-n_1/(2mT)}.
\end{align*}
Recalling the definition of $S_T$ from \eqref{eq:S_T_defn},
we can immediately bound
\begin{align*}
    \bar{\lambda} = \frac{1}{T}\sum_{t=1}^{T} \lambda_t = \frac{1}{T} \Tr(\Theta_{1,T,T}^{-1}) = \frac{1}{T} \Tr\left( \frac{T+1}{2} S_T\right) \leq \Tr(S_T) \leq 2T. 
\end{align*}
Therefore, since $mT \geq n$,
\begin{align*}
    \bar{y}^* \leq  \frac{n_1}{m} \frac{1}{1-n_1/(2mT)} \leq \frac{2 n_1}{m}.
\end{align*}
Now for the lower bound on $\bar{y}^*$.
Noting that $\lambda_{T-k+1} = \lambda_{T-k+1}(\Theta_{1,T,T}^{-1}) = \frac{T+1}{2} \lambda_{T-k+1}(S_T)$,
Corollary~\ref{cor:SST_lower_upper_bounds}
implies (assuming that $c_0 \geq 8$ so $T \geq 8$) that
\begin{align}
    0.01 \frac{k^2}{T} \leq \lambda_{T-k+1} \leq \pi^2 \frac{k^2}{T}, \quad k=1, \dots, T. \label{eq:lambda_upper_lower_bounds}
\end{align}
Therefore, $\tilde{p}(\bar{y}^*) = 0$ implies that:
\begin{align*}
    \frac{1}{\bar{y}^*} &= \frac{2m}{n_1} \sum_{t=1}^{T} \frac{1}{\lambda_t + \bar{y}^*} \leq \frac{2m}{n_1} \sum_{t=1}^{T} \frac{1}{0.01 t^2/T + \bar{y}^*} \\
    &\leq \frac{2m}{n_1} \int_0^{T} \frac{1}{0.01 x^2/T + \bar{y}^*} \,\rmd x 
    = \frac{20 m\sqrt{T}}{n_1 \sqrt{\bar{y}^*}} \tan^{-1}\left(\frac{\sqrt{T}}{10\sqrt{\bar{y}^*}}\right) \leq \frac{10 \pi m\sqrt{T}}{n_1 \sqrt{\bar{y}^*}}.
\end{align*}
Solving for $\bar{y}^*$ yields:
\begin{align*}
    \bar{y}^* \geq \frac{1}{100\pi^2} \frac{n_1^2}{m^2T}.
\end{align*}
Next, we use this lower bound on $\bar{y}^*$ to 
bootstrap our upper bound $\bar{y}^* \leq 2 n_1/m$ into something stronger.
Using the upper bounds on $\lambda_t$ from \eqref{eq:lambda_upper_lower_bounds},
\begin{align*}
    \frac{1}{\bar{y}^*} &= \frac{2m}{n_1} \sum_{t=1}^{T} \frac{1}{\lambda_t + \bar{y}^*} \geq \frac{2m}{n_1} \sum_{t=1}^{T} \frac{1}{ \pi^2 t^2/T + \bar{y}^*} \geq \frac{2m}{n_1} \int_1^{T+1} \frac{1}{\pi^2 x^2/T + \bar{y}^*} \,\rmd x \\
    &= \frac{2 m \sqrt{T}}{\pi n_1\sqrt{\bar{y}^*}} \left[ \tan^{-1}\left( \frac{(T+1)\pi}{\sqrt{T \bar{y}^*}}\right) - \tan^{-1}\left(\frac{\pi}{\sqrt{T\bar{y}^*}}\right) \right].
\end{align*}
The function $\tan^{-1}(x)$ is increasing.
Using the $\bar{y}^* \leq 2 n_1/m$ upper bound
and the assumption that $mT \geq n$,
\begin{align*}
    \frac{(T+1)\pi}{\sqrt{T\bar{y}^*}} \geq \pi \sqrt{\frac{mT}{2n_1}} \geq \pi \sqrt{32} \Longrightarrow \tan^{-1}\left(\frac{(T+1)\pi}{\sqrt{T\bar{y}^*}}\right) \geq \tan^{-1}( \pi\sqrt{32} ). 
\end{align*}
On the other handing, using the bound $\bar{y}^* \geq \frac{1}{100\pi^2} \frac{n_1^2}{m^2T}$
and the assumption that $m \leq \sqrt{2}n/320$,
\begin{align*}
    \frac{\pi}{\sqrt{T\bar{y}^*}} \leq 10\pi \frac{m}{n_1} \leq \pi\sqrt{32}/2 \Longrightarrow \tan^{-1}\left(\frac{\pi}{\sqrt{T\bar{y}^*}}\right) \leq \tan^{-1}( \pi\sqrt{32}/2 ).
\end{align*}
Combining these inequalities:
\begin{align*}
    \frac{1}{\bar{y}^*} \geq \frac{2m \sqrt{T}}{\pi n_1\sqrt{\bar{y}^*}}\left[ \tan^{-1}(\pi\sqrt{32}) - \tan^{-1}(\pi\sqrt{32}/2) \right] \geq \frac{2 \cdot 0.05}{\pi} \frac{m\sqrt{T}}{n_1\sqrt{\bar{y}^*}} \Longrightarrow \bar{y}^* \leq 791\pi^2 \frac{n_1^2}{m^2 T}.
\end{align*}
Therefore we have the following upper and lower bounds on $\bar{y}^*$:
\begin{align}
    \frac{1}{100\pi^2} \frac{n_1^2}{m^2 T} \leq \bar{y}^* \leq \min\left\{ 791\pi^2 \frac{n_1^2}{m^2 T}, 2 \frac{n_1}{m} \right\}.
    \label{eq:ybar_lower_upper_bounds}
\end{align}
For the remainder of the proof, 
in order to avoid precisely tracking constants,
we let $c_0, c_1, c_2, c_3$
be any positive universal constants
such that:
\begin{align}
    c_0 \frac{k^2}{T} &\leq \lambda_{T-k+1} \leq c_1 \frac{k^2}{T}, \quad k=1, \dots, T, \label{eq:lambda_bounds_c_0_c_1} \\
    c_2 \frac{n_1^2}{m^2 T} &\leq \bar{y}^* \leq c_3 \frac{n_1^2}{m^2 T}. \label{eq:bar_y_star_bounds_c_2_c3}
\end{align}
Equations \eqref{eq:lambda_upper_lower_bounds} and \eqref{eq:ybar_lower_upper_bounds} give one valid setting
of these constants.

\paragraph{Upper bound on $p(\bar{y}^*)$.}

To upper bound $p(\bar{y}^*)$, 
we note that:
\begin{align*}
    p(\bar{y}^*) &= \sum_{j=1}^{m}\sum_{t=1}^{T} \frac{\bar{y}^*}{\lambda_t + \bar{y}^*} g_{j,t}^2 - \frac{n_1}{2} \\
    &= \sum_{j=1}^{m}\sum_{t=1}^{T} \frac{\bar{y}^*}{\lambda_t + \bar{y}^*} (g_{i,j}^2-1) + \sum_{j=1}^{m}\sum_{t=1}^{T} \frac{\bar{y}^*}{\lambda_t + \bar{y}^*} - \frac{n_1}{2} \\
    &= \sum_{j=1}^{m}\sum_{t=1}^{T} \frac{\bar{y}^*}{\lambda_t + \bar{y}^*} (g_{i,j}^2-1) &&\text{since } \tilde{p}(\bar{y}^*) = 0.
\end{align*}
Therefore, by \Cref{lemma:chi_squared_tail_bounds},
\begin{align}
    \Pr\left( p(\bar{y}^*) > 2 \sqrt{t} \sqrt{m\sum_{t=1}^{T} \left( \frac{\bar{y}^*}{\lambda_t + \bar{y}^*} \right)^2} + 2 t \max_{t=1, \dots, T} \frac{\bar{y}^*}{\lambda_t + \bar{y}^*}  \right) \leq e^{-t} \quad \forall t > 0. \label{eq:upper_tail}
\end{align}
We upper bound:
\begin{align}
    m\sum_{t=1}^{T} \left(\frac{\bar{y}^*}{\lambda_t + \bar{y}^*}\right)^2
    &\leq m \sum_{t=1}^{T} \left(\frac{\bar{y}^*}{c_0 t^2/T + \bar{y}^*}\right)^2 &&\text{using } \eqref{eq:lambda_bounds_c_0_c_1} \nonumber \\
    &\leq m \int_0^T \left(\frac{\bar{y}^*}{c_0 x^2/T + \bar{y}^*}\right)^2 \,\rmd x \nonumber \\
    &= \frac{m (\bar{y}^*)^2 T}{2 c_0 T \bar{y}^* + 2 (\bar{y}^*)^2} + \frac{\sqrt{T \bar{y}^*}}{2\sqrt{c_0}} \tan^{-1}\left(\sqrt{\frac{c_0 T}{\bar{y}^*}}\right) \nonumber \\
    &\leq \frac{m \bar{y}^*}{2c_0} + \frac{\pi \sqrt{T\bar{y}^*}}{4\sqrt{c_0}} \nonumber \\
    &\leq \frac{c_3}{2c_0} \frac{n_1^2}{mT} + \frac{\pi}{4}\sqrt{\frac{c_3}{c_0}} \frac{n_1}{m} &&\text{using } \eqref{eq:bar_y_star_bounds_c_2_c3} \nonumber \\
    &= \left[ \frac{c_3}{128 c_0} + \frac{\pi}{4}\sqrt{\frac{c_3}{c_0}} \right] n_1  &&\text{since } mT \geq n \text{ and } m \geq 1 \nonumber \\
    &=: c_4 n_1. \label{eq:first_bound}
\end{align}
Next, we immediately have:
\begin{align}
    \max_{t=1, \dots, T} \frac{\bar{y}^*}{\lambda_t + \bar{y}^*} \leq 1. \label{eq:second_bound}
\end{align}
Thus, combining \eqref{eq:upper_tail},
\eqref{eq:first_bound}, and \eqref{eq:second_bound}, we have:
\begin{align}
    \Pr\left( p(\bar{y}^*) > 2 \sqrt{t_u} \sqrt{c_4 n_1} + 2 t_u \right) \leq e^{-t_u} \:\forall t_u > 0. \label{eq:p_i_upper_bound}
\end{align}

\paragraph{Lower bound on $p'(\bar{y}^*)$.}

Differentiating $p(y)$ yields:
\begin{align*}
    p'(y) = \sum_{j=1}^{m}\sum_{t=1}^{T} \frac{\lambda_t}{(\lambda_t + y)^2} g_{j,t}^2.
\end{align*}
Applying \Cref{lemma:chi_squared_tail_bounds} yields,
\begin{align}
    \Pr\left(  p'(\bar{y}^*) < m\sum_{t=1}^{T} \frac{\lambda_t}{(\lambda_t + \bar{y}^*)^2} - 2 \sqrt{t} \sqrt{ m\sum_{t=1}^{T} \frac{\lambda_t^2}{(\lambda_t + \bar{y}^*)^4} } \right) \leq e^{-t} \quad \forall t > 0. \label{eq:lower_tail}
\end{align}

Our first goal is to lower bound $m\sum_{t=1}^{T} \frac{\lambda_t}{(\lambda_t + \bar{y}^*)^2}$.
The function $x \mapsto x/(x+\bar{y}^*)^2$ is 
increasing when $x \in [0, \bar{y}^*]$
and decreasing when $x \in (\bar{y}^*, \infty)$.
Let $t^* \in \{0, \dots, T\}$ be such that $c_1 t^2/T \leq \bar{y}^*$ for $t \in \{1, \dots, t^*\}$
and $c_1 t^2/T > \bar{y}^*$ for $t \in \{t^* + 1, \dots, T\}$
($t^* = 0$ if $c_1/T > \bar{y}^*$).
We write:
\begin{align*}
    m \sum_{t=1}^{T} \frac{\lambda_t}{(\lambda_t + \bar{y}^*)^2} 
    &\geq \frac{c_0}{c_1}m \sum_{t=1}^{T} \frac{c_1 t^2/T}{(c_1t^2/T + \bar{y}^*)^2} &&\text{using } \eqref{eq:lambda_bounds_c_0_c_1} \\
    &= \frac{c_0}{c_1} m \left[ \sum_{t=1}^{t^*} \frac{c_1t^2/T}{(c_1t^2/T + \bar{y}^*)^2} + \sum_{t=t^*+1}^{T} \frac{c_1t^2/T}{(c_1t^2/T + \bar{y}^*)^2}   \right] \\
    &\geq \frac{c_0}{c_1} m \left[ \int_0^{t^*} \frac{c_1 x^2/T}{(c_1x^2/T + \bar{y}^*)^2} \,\rmd x + \int_{t^*+1}^{T+1} \frac{c_1 x^2/T}{(c_1x^2/T + \bar{y}^*)^2} \,\rmd x \right] \\
    &= \frac{c_0}{c_1} m \left[ \int_0^{T+1} \frac{c_1 x^2/T}{(c_1x^2/T + \bar{y}^*)^2} \,\rmd x - \int_{t^*}^{t^*+1}\frac{c_1 x^2/T}{(c_1x^2/T + \bar{y}^*)^2} \,\rmd x    \right].
\end{align*}
The function $z \mapsto \frac{z}{(z+\bar{y}^*)^2}$ is upper bounded by $\frac{1}{4\bar{y}^*}$.
Therefore,
\begin{align*}
    \int_{t^*}^{t^*+1}\frac{c_1 x^2/T}{(c_1x^2/T + \bar{y}^*)^2} \,\rmd x \leq \frac{1}{4 \bar{y}^*} \leq \frac{1}{4c_2} \frac{m^2 T}{n_1^2}.
\end{align*}
Next,
\begin{align*}
    \int_0^{T+1} \frac{c_1 x^2/T}{(c_1x^2/T + \bar{y}^*)^2} \,\rmd x &= c_1 T \left[ \frac{1}{2c_1^{3/2} \sqrt{T\bar{y}^*}} \tan^{-1}\left( \frac{(T+1) \sqrt{c_1}}{ \sqrt{T\bar{y}^*}} \right) - \frac{T+1}{2 c_1^2 (T+1)^2 + 2 c_1 T \bar{y}^*} \right] \\
    &\geq c_1 T \left[ \frac{m}{2c_1^{3/2} \sqrt{c_3} n_1} \tan^{-1}\left( \frac{(T+1) \sqrt{c_1}}{ \sqrt{T\bar{y}^*}} \right) - \frac{1}{2c_1^2T} \right] \\
    &\geq  c_1 T \left[ \frac{m}{2c_1^{3/2} \sqrt{c_3} n_1} \tan^{-1}\left(64\sqrt{\frac{c_1}{c_3}} \right) - \frac{1}{2c_1^2T} \right].
\end{align*}
The last inequality holds because:
\begin{align*}
    \frac{(T+1)\sqrt{c_1}}{ \sqrt{T\bar{y}^*}} \geq (T+1) \sqrt{\frac{c_1}{c_3}} \frac{m}{n_1} \geq \sqrt{\frac{c_1}{c_3}} \frac{mT}{n_1} \geq 64 \sqrt{\frac{c_1}{c_3}}.
\end{align*}
Above, the first inequality holds using \eqref{eq:bar_y_star_bounds_c_2_c3} and the last inequality holds
since $mT \geq n$.
Therefore, 
assuming that $mT \geq 2 \sqrt{\frac{c_3}{c_1}} \frac{1}{\tan^{-1}(64\sqrt{c_1/c_3})} n_1$,
\begin{align*}
    \int_0^{T+1} \frac{c_1 x^2/T}{(c_1x^2/T + \bar{y}^*)^2} \,\rmd x &\geq 
    \frac{\tan^{-1}(64\sqrt{c_1/c_3})}{4\sqrt{c_1 c_3}} \frac{m T}{n_1}.
\end{align*}
Combining these inequalities,
assuming that $m \leq \frac{c_2}{2\sqrt{c_1c_3}} \tan^{-1}(64\sqrt{c_1/c_3}) n_1$, we have:
\begin{align}
    m\sum_{t=1}^{T} \frac{\lambda_t}{(\lambda_t + \bar{y}^*)^2} \geq \frac{c_0}{c_1} m \left[ \frac{\tan^{-1}(64\sqrt{c_1/c_3})}{4\sqrt{c_1 c_3}} \frac{m T}{n_1} - \frac{m^2 T}{4c_2 n_1^2} \right] \geq \frac{c_0\tan^{-1}(64\sqrt{c_1/c_3})}{8 c_1^{3/2} \sqrt{c_3} } =: c_5 \frac{m^2 T}{n_1}. \label{eq:third_bound}
\end{align}

Next, we turn to upper bounding
$m\sum_{t=1}^{T} \frac{\lambda_t^2}{(\lambda_t + \bar{y}^*)^4}$.
Again the function $x \mapsto x^2/(x+\bar{y}^*)^4$
is increasing when $x \in [0, \bar{y}^*]$
and decreasing when $x \in (\bar{y}^*, \infty)$,
and therefore $x^2/(x + \bar{y}^*)^4 \leq \frac{1}{16 (\bar{y}^*)^2}$
for all $x \geq 0$.
Let $t^* \in \{0, \dots, T\}$ be such that
$c_0 t^2/T \leq \bar{y}^*$ for $t \in \{1, \dots, t^*\}$
and $c_0 t^2/T > \bar{y}^*$ for $t \in \{t^* + 1, \dots, T\}$.
In the case when $c_0/T > \bar{y}^*$, we set $t^* = 0$.
We have:
\begin{align*}
    &~~~~m \sum_{t=1}^{T} \frac{\lambda_t^2}{(\lambda_t + \bar{y}^*)^4} \\
    &\leq \frac{c_1^2}{c_0^2} m \sum_{t=1}^{T} \frac{ (c_0 t^2/T)^2 }{ ( c_0t^2/T + \bar{y}^* )^4 } \qquad\qquad\qquad\qquad\qquad~~~\text{using } \eqref{eq:lambda_bounds_c_0_c_1} \\
    &= \frac{c_1^2}{c_0^2} m \left[ \sum_{t=1}^{t^*-1} \frac{ (c_0 t^2/T)^2 }{ ( c_0t^2/T + \bar{y}^* )^4 } + \sum_{t=t^*+2}^{T} \frac{ (c_0 t^2/T)^2 }{ ( c_0t^2/T + \bar{y}^* )^4 } + \frac{ (c_0 (t^*)^2/T)^2 }{ ( c_0(t^*)^2/T + \bar{y}^* )^4 } + \frac{ (c_0 (t^*+1)^2/T)^2 }{ ( c_0(t^*+1)^2/T + \bar{y}^* )^4 }   \right] \\
    &\leq \frac{c_1^2}{c_0^2} m \left[ \int_{1}^{t^*} \frac{ (c_0 x^2/T)^2 }{ ( c_0x^2/T + \bar{y}^* )^4 } \,\rmd x + \int_{t^*+1}^{T} \frac{ (c_0 x^2/T)^2 }{ ( c_0x^2/T + \bar{y}^* )^4 } \,\rmd x + \frac{ (c_0 (t^*)^2/T)^2 }{ ( c_0(t^*)^2/T + \bar{y}^* )^4 } + \frac{ (c_0 (t^*+1)^2/T)^2 }{ ( c_0(t^*+1)^2/T + \bar{y}^* )^4 }   \right] \\
    &\leq \frac{c_1^2}{c_0^2} m \left[ \int_{0}^{T}\frac{ (c_0 x^2/T)^2 }{ ( c_0x^2/T + \bar{y}^* )^4 } \,\rmd x + \frac{ (c_0 (t^*)^2/T)^2 }{ ( c_0(t^*)^2/T + \bar{y}^* )^4 } + \frac{ (c_0 (t^*+1)^2/T)^2 }{ ( c_0(t^*+1)^2/T + \bar{y}^* )^4 }   \right] \\
    &\leq \frac{c_1^2}{c_0^2} m \left[ \int_{0}^{T}\frac{ (c_0 x^2/T)^2 }{ ( c_0x^2/T + \bar{y}^* )^4 } \,\rmd x + \frac{1}{8 (\bar{y}^*)^2}  \right] \qquad\qquad  \text{since } \max_{x > 0} \frac{x^2}{(x + \bar{y}^*)^4} \leq \frac{1}{16(\bar{y}^*)^2}  \\
    &\leq \frac{c_1^2}{c_0^2} m \left[ \int_{0}^{T}\frac{ (c_0 x^2/T)^2 }{ ( c_0x^2/T + \bar{y}^* )^4 } \,\rmd x + \frac{1}{8c_2^2} \frac{m^4 T^2}{n_1^4} \right] .
\end{align*}
We now bound:
\begin{align*}
    \int_{0}^{T}\frac{ (c_0 x^2/T)^2 }{ ( c_0x^2/T + \bar{y}^* )^4 } \,\rmd x &= c_0^2 T^2 \left[ \frac{(3c_0T + \bar{y}^*)(c_0T - 3\bar{y}^*)}{48 c_0^2 T \bar{y}_* (c_0 T + \bar{y}^*)^3} + \frac{\tan^{-1}\left(\sqrt{\frac{c_0 T}{\bar{y}^*}}\right)}{16 c_0^{5/2} T^{3/2} (\bar{y}^*)^{3/2}} \right] \\
    &\leq c_0^2 T^2 \left[ \frac{1}{16c_0 \bar{y}^*(c_0 T + \bar{y}^*)^2} + \frac{\pi}{32 c_0^{5/2} T^{3/2} (\bar{y}^*)^{3/2}} \right] \\
    &\leq c_0^2 T^2 \left[ \frac{1}{16c_0^3 \bar{y}^* T^2} + \frac{\pi}{32 c_0^{5/2} T^{3/2} (\bar{y}^*)^{3/2}} \right] \\
    &\leq c_0^2 T^2 \left[ \frac{1}{16c_0^3 c_2} \frac{m^2}{n_1^2 T} + \frac{\pi}{32 c_0^{5/2} c_2^{3/2}} \frac{m^3}{n_1^3} \right] &&\text{using } \eqref{eq:bar_y_star_bounds_c_2_c3} \\
    &\leq \left[ \frac{1}{1024 c_0 c_2} + \frac{\pi}{32 c_0^{1/2} c_2^{3/2}} \right] \frac{m^3 T^2}{n_1^3} &&\text{since } mT \geq n.
\end{align*}
Combining these inequalities, assuming that $m \leq n_1$:
\begin{align}
    m\sum_{t=1}^{T} \frac{\lambda_t^2}{(\lambda_t + \bar{y}^*)^4} &\leq \frac{c_1^2}{c_0^2} \left[ \left[ \frac{1}{1024 c_0 c_2} + \frac{\pi}{32 c_0^{1/2} c_2^{3/2}} \right] \frac{m^4 T^2}{n_1^3} + \frac{1}{8c_2^2} \frac{m^5 T^2}{n_1^4}  \right] \nonumber \\
    &\leq \frac{c_1^2}{c_0^2} \left[\frac{1}{1024 c_0 c_2} + \frac{\pi}{32 c_0^{1/2} c_2^{3/2}} +  \frac{1}{8c_2^2} \right] \frac{m^4 T^2}{n_1^3}  &&\text{since } m \leq n_1 \nonumber \\
    &=: c_6 \frac{m^4 T^2}{n_1^3}. \label{eq:fourth_bound}
\end{align}
Combining \eqref{eq:lower_tail},
\eqref{eq:third_bound}, and \eqref{eq:fourth_bound}
yields
\begin{align}
    \Pr\left( p'(\bar{y}^*) < c_5 \frac{m^2 T}{n_1} - 2 \sqrt{t_\ell} \sqrt{c_6} \frac{m^2 T}{n_1^{3/2}} \right) \leq e^{-t_\ell} \:\forall t_\ell > 0. \label{eq:p_i_prime_lower_bound}
\end{align}

\paragraph{Lower bounds on $y^*$.}

We now combine \eqref{eq:p_i_upper_bound}
with \eqref{eq:p_i_prime_lower_bound} to established a lower
bound on $y^*$. 
Equations \eqref{eq:y_i_lower_bound} and \eqref{eq:bar_y_star_bounds_c_2_c3} imply that:
\begin{align*}
    y^* \geq \bar{y}^* - \frac{p(\bar{y}^*)}{p'(\bar{y}^*)} \geq \frac{c_2 n_1^2}{m^2 T} - \frac{p(\bar{y}^*)}{p'(\bar{y}^*)}.
\end{align*}
We first set $t_\ell = \frac{c_5^2}{16c_6} n_1$,
so that by \eqref{eq:p_i_prime_lower_bound},
\begin{align*}
    \Pr\left( p'(\bar{y}^*) < \frac{c_5}{2} \frac{m^2 T}{n_1} \right) \leq e^{-\frac{c_5^2}{16c_6} n_1}.
\end{align*}
We next set $t_u = \beta n_1$ for a $\beta > 0$ to be specified.
By \eqref{eq:p_i_upper_bound},
\begin{align*}
    \Pr\left( p(\bar{y}^*) > 2 (\sqrt{c_4 \beta} + \beta) n_1 \right) \leq e^{-\beta n_1}.
\end{align*}
Let $\calE_2$ denote the event:
\begin{align*}
    \calE_2 := \left\{ p'(\bar{y}^*) \geq \frac{c_5}{2} \frac{m^2 T}{n_1} \right\} \cap \left\{ p(\bar{y}^*) \leq 2(\sqrt{c_4 \beta} + \beta)n_1 \right\}.
\end{align*}
By a union bound, $\Pr(\calE_2^c) \leq e^{-\frac{c_5^2}{16c_6}n_1} + e^{-\beta n_1}$.
Furthermore,
\begin{align*}
    \ind\{ \calE_2 \} \left[ \frac{c_2 n_1^2}{m^2 T} - \frac{p(\bar{y}^*)}{p'(\bar{y}^*)} \right] \geq \ind\{\calE_2\}\left[ c_2 - \frac{4(\sqrt{c_4 \beta} + \beta) }{c_5} \right] \frac{n_1^2}{m^2 T}.
\end{align*}
Setting $\beta = c_7 := \min\left\{ \frac{c_2c_5}{16}, \frac{c_2^2 c_5^2}{16^2 c_4} \right\}$, we have that 
$c_2 - \frac{4(\sqrt{c_4 \beta} + \beta) }{c_5} \geq c_2/2$,
and therefore from \eqref{eq:y_i_lower_bound},
\begin{align}
    \ind\{ \calE_{1} \} y^* \geq \ind\{ \calE_{1} \cap \calE_2 \} \left[ \frac{c_2 n_1^2}{m^2 T} - \frac{p(\bar{y}^*)}{p'(\bar{y}^*)} \right] \geq \ind\{ \calE_{1} \cap \calE_2\} \frac{c_2}{2} \frac{n_1^2}{m^2 T}. \label{eq:final_bar_y_star_lower_bound}
\end{align}

\paragraph{Finishing the proof.}
Define $\calE := \calE_1 \cap \calE_2$ and
define $\underline{y}^* := \frac{c_2}{2} \frac{n_1^2}{m^2 T}$.
By a union bound,
\begin{align}
    \Pr(\calE) &\leq e^{-n/128} + e^{-mT/16} + e^{-\frac{c_5^2}{16c_6}n_1} + e^{-c_7 n_1} \nonumber \\
    &\leq e^{-n/128} + e^{-n/16} + e^{-\frac{c_5^2}{1024 c_6} n} + e^{-\frac{c_7}{64} n} &&\text{since } mT \geq n \nonumber \\
    &\leq 4 \exp\left(- \min\left\{ \frac{1}{128}, \frac{1}{16}, \frac{c_5^2}{1024 c_6} , \frac{c_7}{64}\right\} n \right) =: 4 e^{-c_8 n}. \label{eq:final_union_bound}
\end{align}
From \eqref{eq:bound_on_Z_i_good_v2},
since $y^* \geq \underline{y}^*$ on $\calE$ by \eqref{eq:final_bar_y_star_lower_bound},
\begin{align}
    \ind\{ \calE \} Z_i \leq  \ind\{ \calE \} \frac{1}{\lambda_i + y^*}
    \leq \frac{1}{\lambda_i + \underline{y}^*}. \label{eq:final_Z_i_E_i_bound}
\end{align}
Next, by \Cref{prop:invert_log_t_over_t},
if $n \geq 2 \max\{1, c_8^{-1}\} \log(4 \max\{1, c_8^{-1}\})$,
then we have
\begin{align*}
    n \geq c_8^{-1} \log{n} \Longleftrightarrow n e^{-c_8 n} \leq 1.
\end{align*}
We now bound,
\begin{align*}
    \sum_{i=1}^{T} \E[Z_i] &= \sum_{i=1}^{T} \left[\E[\ind\{\calE\} Z_i] + \E[\ind\{\calE^c\} Z_i]\right] \\
    &\leq \sum_{t=1}^{T} \left[ \frac{1}{\lambda_t + \underline{y}^*} + \Pr(\calE^c)  (\Theta_{1,T,T})_{tt} \right] &&\text{using } \eqref{eq:final_Z_i_E_i_bound} \text{ and } Z_i \leq (\Theta_{1,T,T})_{ii} \\
    &= \sum_{t=1}^{T} \frac{1}{\lambda_t + \underline{y}^*} + \Pr(\calE^c)  T &&\text{since } \Tr(\Theta_{1,T,T}) = T \\
    &\leq \sum_{t=1}^{T} \frac{1}{c_0 t^2/T + \underline{y}^*} + 4 T e^{-c_8 n} &&\text{using } \eqref{eq:lambda_bounds_c_0_c_1} \text{ and } \eqref{eq:final_union_bound} \\
    &\leq \int_0^T \frac{1}{c_0 x^2/T + \underline{y}^*} \,\rmd x + 4 T e^{-c_8 n} \\
    &\leq \frac{\pi}{2}\sqrt{\frac{T}{c_0 \underline{y}^*}} + 4 T e^{-c_8 n} 
    = \frac{\sqrt{2}\pi}{2\sqrt{c_0c_2}} \frac{m T}{n_1} + 4 T e^{-c_8 n} \\
    &\leq  \left[\frac{\sqrt{2}\pi}{2\sqrt{c_0c_2}}  +  \frac{1 }{16} \right] \frac{m T}{n_1}  =: c_8 \frac{mT}{n_1} &&\text{since } n e^{-c_8 n} \leq 1 \text{ and } m \geq 1.
\end{align*}
Plugging this upper bound into \eqref{eq:lower_bound_with_sum_Zis_inverse}:
\begin{align*}
    \sfR(m,T,\Tnew; \{\PxA{I_n}\}) \geq \sigma_\xi^2 p \cdot \frac{\Tnew}{2T} \cdot \frac{n}{2m} \cdot \frac{1}{c_8} \frac{n_1}{mT} = \frac{1}{256c_8} \sigma_\xi^2 \cdot \frac{pn^2}{m^2 T} \cdot \frac{\Tnew}{T}.
\end{align*}
The claim now follows.
\end{proof}

\subsection{Proof of \Cref{stmt:lower_bound_main}}

\lowerboundmain*
\begin{proof}
Let $\calP_x := \{\PxA{0_{n \times n}}, \PxA{I_n}\}$.
We let $c'_0$, $c'_1$, $c'_2$, and $c'_3$ denote the universal positive
constants in the statement of \Cref{thm:lower_bound_r_equals_1}.
We first invoke \Cref{stmt:minimax_jensen_rate} to conclude that:
\begin{align}
    \mathsf{R}(m,T,\Tnew;\calP_x) \geq \frac{\sigma_\xi^2}{2}
        \cdot \frac{pn}{mT}
        \cdot \max \left\{ \frac{\Tnew}{T}, 1 \right\}. \label{eq:baseline_risk}
\end{align}

The proof now proceeds in three cases:

\paragraph{Case $n\Tnew/(mT) \leq 1$.}
In this case, we trivially have
$\max\left\{ \frac{\Tnew}{T}, 1 \right\} = \max\left\{ \frac{n\Tnew}{mT}, \frac{\Tnew}{T}, 1 \right\}$.
Therefore, \eqref{eq:baseline_risk} yields:
\begin{align*}
    \mathsf{R}(m,T,\Tnew;\calP_x) \geq \frac{\sigma_\xi^2}{2}
        \cdot \frac{pn}{mT}
        \cdot \max\left\{ \frac{n\Tnew}{mT}, \frac{\Tnew}{T}, 1 \right\}.
\end{align*}

\paragraph{Case $n\Tnew/(mT) > 1$ and $m \leq c_2' n$.}
In this case, we can invoke \Cref{thm:lower_bound_r_equals_1} to conclude that:
\begin{align}
    \mathsf{R}(m,T,\Tnew;\calP_x) \geq c_3' \sigma_\xi^2 \cdot \frac{pn}{mT} \cdot \frac{n\Tnew}{mT} = c_3' \sigma_\xi^2 \cdot \frac{pn}{mT} \cdot \max\left\{ \frac{n\Tnew}{mT}, 1 \right\}. \label{eq:second_case}
\end{align}
Since $n/m \geq 1/c_2'$, we have that
$n\Tnew/(mT) \geq \Tnew/(c_2' T)$.
Therefore:
\begin{align*}
    \max\left\{ \frac{n\Tnew}{mT}, 1 \right\} = \max\left\{ \frac{n\Tnew}{mT}, \frac{\Tnew}{c_2' T}, 1 \right\} \geq \min\{1,1/c_2'\} \max\left\{ \frac{n\Tnew}{mT}, \frac{\Tnew}{T}, 1 \right\}.
\end{align*}
Hence, from \eqref{eq:second_case},
\begin{align*}
     \mathsf{R}(m,T,\Tnew;\calP_x) \geq \min\{c_3',c_3'/c_2'\} \sigma_\xi^2 \cdot \frac{pn}{mT} \cdot \max\left\{ \frac{n\Tnew}{mT}, \frac{\Tnew}{T}, 1 \right\}.
\end{align*}

\paragraph{Case $n\Tnew/(mT) > 1$ and $m > c_2' n$.}
In this case, we have
$\Tnew/T > c_2' n \Tnew/(mT)$.
Therefore, we have:
\begin{align*}
    \max\left\{ \frac{\Tnew}{T}, 1 \right\} = \max\left\{ c_2' \frac{n\Tnew}{mT}, \frac{\Tnew}{T}, 1 \right\} \geq \min\{1, c_2'\} \max\left\{ \frac{n\Tnew}{mT}, \frac{\Tnew}{T}, 1 \right\}.
\end{align*}
Hence, from \eqref{eq:baseline_risk},
\begin{align*}
    \mathsf{R}(m,T,\Tnew;\calP_x) \geq \min\{1/2, c_2'/2\} \sigma_\xi^2 
        \cdot \frac{pn}{mT}
        \cdot \max\left\{ \frac{n\Tnew}{mT}, \frac{\Tnew}{T}, 1 \right\}.
\end{align*}
The claim now follows taking
$c_0 = c_0'$, $c_1 = c_1'$, and
$c_2 = \min\{1/2, c_3', c_3'/c_2', c_2'/2\}$.
\end{proof}

\end{document}